\PassOptionsToPackage{table}{xcolor}
\documentclass{article}

\usepackage{iclr2027_conference,times}
\usepackage[T1]{fontenc}
\usepackage[utf8]{inputenc}
\usepackage{microtype}
\usepackage{amsmath,amssymb,amsthm,mathtools,bm}
\usepackage{aliascnt}
\usepackage{booktabs,multirow,array}
\usepackage{graphicx}
\usepackage{subcaption}
\usepackage{enumitem}
\usepackage{url}
\usepackage{hyperref}
\usepackage[nameinlink,capitalise,noabbrev]{cleveref}
\usepackage{placeins}
\usepackage{titletoc}
\usepackage{tabularx}
\usepackage{nicefrac}
\newcolumntype{C}{>{\centering\arraybackslash}X}

\definecolor{SignFill}{HTML}{D9ECFA}
\definecolor{SignText}{HTML}{155C8A}

\definecolor{GDFill}{HTML}{E9E0F4}
\definecolor{GDText}{HTML}{5C3B86}

\definecolor{BiasFill}{HTML}{FDE7D2}
\definecolor{BiasText}{HTML}{A74600}

\definecolor{insightbg}{RGB}{245,248,252}
\definecolor{insightborder}{RGB}{165,185,210}

\newcolumntype{Y}{>{\raggedright\arraybackslash}X}

\newcolumntype{L}{>{\raggedright\arraybackslash}X}

\newcommand{\R}{\mathbb{R}}

\newcommand{\1}{\mathbf{1}}

\newcommand{\ip}[2]{\left\langle #1,#2\right\rangle}
\newcommand{\norm}[1]{\lVert #1 \rVert}

\newcommand{\cM}{\mathcal{M}}

\newcommand{\softmax}{\operatorname{softmax}}
\newcommand{\sign}{\operatorname{sign}}
\newcommand{\diag}{\operatorname{diag}}

\newcommand{\argmin}{\operatorname*{arg\,min}}

\newcommand{\eps}{\epsilon}

\newcommand{\eprod}{\mathbin{\odot}}

\newcommand{\one}{\mathbf{1}}
\newcommand{\nor}[2][]{\left\lVert #2\right\rVert_{#1}}

\newcommand{\dist}{\operatorname{dist}}

\newtheorem{theorem}{Theorem}

\newaliascnt{lemma}{theorem}
\newtheorem{lemma}[lemma]{Lemma}
\aliascntresetthe{lemma}
\newaliascnt{proposition}{theorem}
\newtheorem{proposition}[proposition]{Proposition}
\aliascntresetthe{proposition}
\newtheorem{localproposition}{Proposition}[section]
\newaliascnt{corollary}{theorem}
\newtheorem{corollary}[corollary]{Corollary}
\aliascntresetthe{corollary}
\newaliascnt{assumption}{theorem}

\aliascntresetthe{assumption}
\theoremstyle{definition}
\newaliascnt{definition}{theorem}

\aliascntresetthe{definition}
\newaliascnt{remark}{theorem}

\aliascntresetthe{remark}

\crefname{theorem}{Theorem}{Theorems}
\Crefname{theorem}{Theorem}{Theorems}
\crefname{lemma}{Lemma}{Lemmas}
\Crefname{lemma}{Lemma}{Lemmas}
\crefname{proposition}{Proposition}{Propositions}
\Crefname{proposition}{Proposition}{Propositions}
\crefname{localproposition}{Proposition}{Propositions}
\Crefname{localproposition}{Proposition}{Propositions}
\crefname{corollary}{Corollary}{Corollaries}
\Crefname{corollary}{Corollary}{Corollaries}
\crefname{assumption}{Assumption}{Assumptions}
\Crefname{assumption}{Assumption}{Assumptions}
\crefname{definition}{Definition}{Definitions}
\Crefname{definition}{Definition}{Definitions}
\crefname{remark}{Remark}{Remarks}
\Crefname{remark}{Remark}{Remarks}

\hypersetup{
  colorlinks=true,
  linkcolor=blue!55!black,
  citecolor=blue!55!black,
  urlcolor=blue!55!black,
  pdftitle={Adam’s Two-Timescale Implicit Bias in Separable Linear Classification},
  pdfauthor={Anonymous authors}
}

\newif\ifarxiv
\arxivtrue       

\ifarxiv
  \iclrfinalcopy
\fi

\title{On the Two Faces of Adam in Separable Linear Classification}

\author{%
  {\large\bfseries
    Chen Fan
    \qquad
    Csaba Szepesv\'{a}ri
  }\\[0.45em]
  {\normalsize
    University of Alberta
  }\\[0.15em]
  {\small\ttfamily
    cfan8@ualberta.ca
    \qquad
    szepesva@ualberta.ca
  }
}

\makeatletter
\renewcommand{\@maketitle}{%
  \vbox{%
    \hsize\textwidth
    \centering

    {\LARGE\bfseries
      \@title\par
    }

    \vspace{0.26in}

    {\@author\par}

    \vskip 0.30in minus 0.10in
  }%
}
\makeatother

\begin{document}
\maketitle
\pagestyle{fancy}
\fancyhf{}                         
\fancyfoot[C]{\thepage}            
\renewcommand{\headrulewidth}{0pt} 
\renewcommand{\footrulewidth}{0pt}
\thispagestyle{fancy}  

\begin{abstract}
We consider the behavior of 
deterministic,
full-batch, bias-corrected Adam in separable linear classification
with softmax parametrization under log-loss.
In this setting, under a wide range of conditions
Adam is known to approach max-norm-margin optimality when its stability constant \(\epsilon\) is zero,
while with a positive \(\epsilon\), it is known to approach Euclidean-margin optimality.
Our main contribution is the quantitative description of Adam's behavior for small fixed positive $\epsilon$.
We give sufficient conditions under which an Adam-trained classifier nearly maximizes the max-norm margin
before the updates become gradient-like. 
We also show that the classifier reaches a fixed target Euclidean margin only much later.
Specifically, we show that for polynomially decreasing stepsizes with exponent \(a\), where \(1/3<a<1\), 
the updates become approximately proportional to the negative gradient after 
$\Theta(\log(1/\epsilon)^{1/(1-a)})$ iterations. 
At that time, the 
classifier still nearly maximizes the max-norm margin. Reaching a fixed target 
Euclidean margin above that of every max-norm-optimal classifier, but below the optimum, 
 is shown to require $\epsilon^{-\Theta(1)/(1-a)}$ iterations. Under inverse-linear stepsize decay (\(a=1\)), the update 
 transition takes polynomially many iterations, whereas reaching the target margin takes 
 exponentially many. Experiments support these predictions. The later change in the classifier 
 can improve or worsen generalization after training error reaches zero, connecting the
  analysis to grokking and its reverse.

\end{abstract}

\section{Introduction}

Adam's implicit bias in separable linear classification depends on its
stability parameter \(\epsilon\). Existing analyses show that removing this
parameter leads to max-norm-margin maximization \citep{zhang2024adam}, whereas
keeping it positive ultimately leads to Euclidean-margin maximization
\citep{wang2022does}. These results describe
different limiting behaviors of the same algorithm. We study how a trajectory
with a small positive \(\epsilon\) evolves between these behaviors, and how the
learning-rate schedule affects this evolution.

This qualitative behavior,
which has already been observed
empirically \citep{fan2025spectral,tsilivis2025flavorsmarginimplicitbias},
is not unexpected:
If \(\epsilon\) is sufficiently small, the gradients initially dominate it, and
Adam's updates are approximately \emph{sign-like}: They 
are nearly aligned with the sign function applied componentwise to the gradient.
When gradients are small relative to $\epsilon$,
the updates are \emph{gradient-like}: They are aligned, up to scaling, with the actual gradient.
Now, it is also true that the updates cannot be gradient-like unless
the gradients are small. The latter can only happen when
the parameter vector has a large magnitude.
But because initially the updates are sign-like,
at this time, the parameter vector needs to be nearly $\ell_\infty$-margin optimal.
After the updates become gradient-like,
it then takes some extra time for the parameter vector to get aligned
with the $\ell_2$-max-margin solution direction.

Our contribution is to quantify this intuitive behavior for deterministic,
full-batch, bias-corrected Adam in separable linear classification, using
logistic loss for binary classification, with an extension to multiclass
softmax cross-entropy. Our analysis applies to polynomially decreasing
stepsizes
\[
\eta_t=\eta_0(t+1)^{-a},\qquad a\in(1/3,1],
\]
a family studied in previous implicit-bias analyses
\citep{nacson2019convergence,zhang2024adam,fan2025spectral}. Each trajectory uses a
fixed positive \(\epsilon\); our results characterize the family of
trajectories as \(\epsilon\) decreases. We bound the max-norm margin gap at
the time marking permanent entry into a
gradient-like regime, which we call the \emph{update transition time}.
We also characterize when this update transition occurs and how much
additional training is needed for the classifier to reach a prescribed
Euclidean-margin accuracy.

At a high level, our main results are as follows. For \(1/3<a<1\), the update
transition time satisfies
\[
T_{\mathrm{update}}
=\Theta\!\left((\log(1/\epsilon))^{1/(1-a)}\right).
\]
At this time, the classifier's max-norm margin gap is small, with a bound that
vanishes as \(\epsilon\) decreases. We also bound its subsequent
deterioration: the classifier remains nearly max-norm-margin optimal while
the accumulated movement after \(T_{\mathrm{update}}\) is small relative to
the weight norm at that time. Let \(\Delta>0\) be the smallest
Euclidean-margin gap among max-norm-optimal classifiers. For any fixed target
gap \(\zeta\in(0,\Delta)\), the first subsequent time at which the
Euclidean-margin gap falls below \(\zeta\) satisfies
\[
T_{\mathrm{Euclidean}}
=\epsilon^{-\Theta(1)/(1-a)}.
\]
Thus, the update transition occurs on a polylogarithmic scale in
\(1/\epsilon\), while reaching this Euclidean-margin accuracy requires a
polynomial scale. For \(a=1\), these scales become polynomial and
exponential, respectively.

There is an intuitive explanation for these bounds. The initial phase
produces a weight vector with norm of order \(\log(1/\epsilon)\). 
Changing the normalized parameter vector by a fixed amount requires 
displacement comparable to this norm: 
a much smaller displacement barely changes its direction. 
After the transition, however, the accumulated distance 
traveled grows only logarithmically in the cumulative stepsize,
as in gradient descent on separable logistic
regression \citep{soudry2018implicit}.
Consequently, substantial classifier change requires the 
cumulative stepsize to be at least polynomial in \(1/\epsilon\).
Cumulative stepsizes grow polynomially in the iteration count when \(a<1\), 
but only logarithmically when \(a=1\). 
This explains the particularly long transition times under 
harmonic stepsize decay.

These results clarify the relationship between the two existing descriptions of Adam’s implicit bias. They predict, for example, that inverse-square-root (\(a=1/2\)) and inverse-linear (\(a=1\)) stepsize decay lead to the same eventual margin geometry, but reaching a fixed Euclidean-margin accuracy takes polynomially many iterations in \(1/\epsilon\) under the former and exponentially many under the latter. Depending on the data distribution, the change in the classifier can improve or worsen generalization after the training examples are already classified correctly. This connects the trajectory description to grokking and its reverse: delayed improvement or deterioration in predictive performance associated with changing implicit bias \citep{lyu2024dichotomy}.

The guarantees hold for finite positive \(\epsilon\) under data- and
algorithm-dependent sufficient conditions, while the scaling statements
describe their behavior as \(\epsilon\) decreases. 
We leave for future work determining the 
polynomial exponent in the bound for \(T_{\mathrm{Euclidean}}\) for
the case when $a<1$. Similarly, the sufficient conditions are conservative;
sharpening them to characterize, for a given dataset and optimizer
parameters, when the described phase behavior occurs is an important
direction for future work. The theory and accompanying experiments provide
a more detailed understanding of how Adam's stability parameter and
learning-rate schedule determine its implicit-bias trajectory.

\section{Problem setup and Preliminaries}

Let $\{(x_i,y_i)\}_{i=1}^n$ be a dataset with $x_i\in\mathbb{R}^d$ and
$y_i\in\{-1,+1\}$. Denote $R_1  :=\max_{i \in [n]} \norm{x_i}_1$ and $R_2  :=\max_{i \in [n]} \norm{x_i}_2$. Define $z_i:=x_i y_i$ and
$[n]:=\{1,\ldots,n\}$. We study the empirical logistic loss

\begin{equation}
L(w)=\frac1n\sum_{i=1}^n\ell(z_i^\top w),
\qquad
\text{where} \qquad
\ell(u)=\log(1+e^{-u}).
\label{eq:loss}
\end{equation}
Write $g(w):=\nabla L(w)$, $L_t:=L(w_t)$, and $g_t:=g(w_t)$.  Define the unnormalized margin
$\rho(w):=\min_{i\in[n]}z_i^\top w$ and the normalized $p$-margin
$\widehat\gamma_p(w):=\rho(w)/\norm{w}_p$ for $p\in\{2,\infty\}$.
The optimal max-norm and Euclidean data margins are
\begin{equation}
\gamma_\infty:=\max_{\norm{u}_\infty\le1}\min_{i \in [n]} z_i^\top u,
\qquad
\gamma_2:=\max_{\norm{u}_2\le1}\min_{i \in [n]} z_i^\top u.
\label{eq:max-margins}
\end{equation}

We analyze deterministic, bias-corrected Adam with a fixed stability constant
$\eps>0$ \citep{kingma2015adam}. Initialize $\bar m_{-1}:=0$ and
$\bar v_{-1}:=0$, and choose $\beta_1,\beta_2 \in [0,1)$. The iterates are
\begin{align}
&\bar m_t=\beta_1\bar m_{t-1}+(1-\beta_1)g_t, \qquad \bar v_t=\beta_2\bar v_{t-1}+(1-\beta_2)g_t^2, \notag \\
&m_t=\bar m_t/(1-\beta_1^{t+1}), \qquad v_t=\bar v_t/(1-\beta_2^{t+1}), \notag \\
&w_{t+1}=w_t-\eta_tD_t, \qquad D_t:=m_t /(\sqrt{v_t}+\eps\1), \label{eq:adam_update}
\end{align}  
where all operations are entrywise. The default PyTorch choice for $\eps$ is
$10^{-8}$ \citep{paszke2019pytorch}. 

We assume that the data is separable: there exists $w \in \mathbb R^d$ such that $\min_{i \in [n]} z_i^\top w > 0$, which is a standard assumption in the analyses of optimization implicit bias
\citep{soudry2018implicit,ji2019risk,nacson2019convergence,
wu2023implicit,zhang2024adam}. Thus, both $\gamma_\infty > 0$ and $\gamma_2 > 0$ hold.

We use the decreasing stepsize schedule
$\eta_t=\eta_0(t+1)^{-a}$, where $\eta_0>0$ and $a\in(1/3,1]$, as in prior
analyses of Adam and steepest descent
\citep{nacson2019convergence,zhang2024adam,fan2025spectral,
baek2026implicit,li2026implicit}. In addition to satisfying $\sum_{t=0}^{\infty} \eta_t = \infty$, this schedule also satisfies, for every $q\in(0,1)$ and $c>0$, $\sum_{r=0}^t q^r
[
\exp(c\sum_{s=t-r}^{t-1}\eta_s)-1]
\le C_{\rm geo} \eta_t, \, \, t \geq t_0$, where $t_{0} = t_0(q,c,a,\eta_0)$ and $C_{\rm geo}=C_{\rm geo}(q,c,a,\eta_0)$ \citep{zhang2024adam};
see \cref{prop:ema-sum} in App. \ref{app:prelim}.

For the analysis, we use the gradient-scale proxy introduced
by \citet{zhang2024adam}: $G(w):=-\frac{1}{n}\sum_{i=1}^n \ell'(z_i^\top w)$. 
It controls the gradient norms through the data margins:
\begin{align*}
    \gamma_\infty G(w)\le\norm{g(w)}_1\le R_1G(w),
\qquad
\gamma_2G(w)\le\norm{g(w)}_2\le R_2G(w).
\end{align*}
Write $G_t:=G(w_t)$. We can control the bias-corrected moment
errors in terms of $G_t$: $\max \{ \norm{m_t-g_t}_1, \norm{\sqrt{v_t}-|g_t|}_1  \} \lesssim \eta_t G_t$ after some $t_{\rm ema}$ burn-in period (as shown in \cref{lem:first-track,lem:second-track} in App. \ref{app:prelim}).  
Thus, the bias-corrected moments track $g_t$ and $|g_t|$ at scale
$O(\eta_tG_t)$. In the idealized case where
$m_t[j]\approx g_t[j]$ and $\sqrt{v_t[j]}\approx |g_t[j]|$, the update is
approximated by the soft-sign map
$\sigma_\eps(g)[j]:=\frac{g[j]}{|g[j]|+\eps}, \, j \in [d]$.
This map satisfies (see \cref{lem:softsign} in App. \ref{app:prelim})
\begin{align}
 \ip{g}{\sigma_\eps(g)}
=\sum_j\frac{g[j]^2}{|g[j]|+\eps}
\ge \norm{g}_1-d\eps,
\qquad
\ip{g}{\sigma_\eps(g)}
\ge\frac{\norm{g}_1^2}{\norm{g}_1+d\eps}.
\label{eq:main_soft_sign}
\end{align}
 The first inequality is informative when $\eps\ll\norm{g}_1$, whereas the
second applies throughout the trajectory. Both are useful for studying Adam's trajectories. However, we must still control the discrepancy between $D_t$ and $\sigma_\eps(g_t)$ for any $\eps>0$, as shown in \cref{lem:alignment} in App. \ref{app:prelim}: 
\begin{align*}
\max \bigl \{ \norm{g_t\eprod(D_t-\sigma_\eps(g_t))}_1, \eps\norm{D_t-\sigma_\eps(g_t)}_2, \left|\ip{g_t}{D_t-\sigma_\eps(g_t)}\right| \bigr \} \le C_{\rm ema}\eta_tG_t, \,\, t \geq t_{\rm ema}. 
\end{align*}
Combining this with a Taylor expansion yields the basic loss recursion (see \cref{prop:master-descent} in App. \ref{app:prelim})
\begin{align*}
L_{t+1}
\le L_t-
\eta_t\ip{g_t}{\sigma_\eps(g_t)}
+O(\eta_t^2G_t), \qquad t \geq t_{\rm ema}. 
\end{align*}
Margin analysis additionally requires a bound on the accumulated
displacement. The unit coefficient on the accumulated step size in the next
result is essential: applying the triangle inequality using only the pointwise
bound $\norm{D_t}_\infty\leq C_D := \sqrt{(1-\beta_1)/(1-\beta_2)}$ from \cref{lem:uniform-D} in App. \ref{app:prelim} would not recover the
$1/\gamma_\infty$ coefficient in \cref{thm:first-clock} or the max-norm optimality in
\cref{thm:margin_main} (see below). 
\begin{lemma}
\label{lem:radius} Suppose that $0 \leq \beta_1 \leq \beta_2 < 1$. 
If there exists a constant $B_g$ such that $B_g \ge \norm{g(w)}_{\infty}$ for every $w \in \mathbb{R}^d$, then the following holds for every integer $r$ with $0\le r<T$ and $\eps>0$:
\begin{equation}
\left\|\sum_{t=r}^{T-1}\eta_tD_t\right\|_\infty
\le
\sum_{t=r}^{T-1}\eta_t
+C_{\rm rad}\eta_r
\left[
1+\log\frac{B_g^2+\eps^2}{(1-\beta_2)\eps^2}
\right],
\label{eq:restart-radius-main}
\end{equation}
where the constant $C_{\rm rad}$ depends only on $\beta_1$ and $\beta_2$.
\end{lemma}
 For $\eps=0$, \citet{zhang2024adam} obtained a similar bound under the
additional condition $|g_0[j]|\geq\rho$ for every $j\in[d]$ and some
$\rho>0$, together with a $\rho$-dependent burn-in time. Because $\eps>0$
keeps the update well defined, \cref{lem:radius} permits arbitrary parameter initialization $w_0$,
applies to every interval $[r,T)$, and 
accounts for bias correction. The subtle point is controlling momentum carried into \([r,T)\) while keeping the leading coefficient \(1\); see App. \ref{app:radius} for details. 
Moreover, \cref{lem:radius,lem:uniform-D} rely on the condition $\beta_1\leq\beta_2$, which is also assumed in prior analyses of Adam's implicit bias \citep{zhang2024adam,baek2026implicit}. We assume that $0 \leq \beta_1 \leq \beta_2 < 1$ throughout the following sections. It remains open whether our results extend beyond this condition. 

Next, we introduce the gradient and sign residuals: 
 \begin{align*}
\lambda_t:=\frac{G_t}{\eps},
\qquad
 r_t^{\rm sgn}
:=\frac{\norm{g_t\eprod(D_t-\sign(g_t))}_1}{\norm{g_t}_1},
\qquad
 r_t^{\rm grad}
:=\frac{\norm{\eps D_t-g_t}_2}{\norm{g_t}_2}. 
 \end{align*}
The order parameter $\lambda_t$ separates the two regimes. Under the heuristic
$D_t\approx\sigma_\eps(g_t)$, the bound $\norm{g_t}_1\leq R_1G_t$ gives
$|g_t[j]|/\eps\leq R_1\lambda_t$ for every $j\in[d]$. Thus, when
$\lambda_t\ll1$, every coordinate is small relative to $\eps$ and
$D_t\approx g_t/\eps$. When $\lambda_t\gg1$, some coordinates may remain
small, so coordinatewise closeness to $\operatorname{sign}(g_t)$ need not
hold. We instead measure sign-likeness through relative first-order loss
decrease. Using \eqref{eq:main_soft_sign} and
$\norm{g_t}_1\geq\gamma_\infty G_t=\gamma_\infty\eps\lambda_t$, we have 
$
    1-\frac{\langle g_t,\sigma_\eps(g_t)\rangle}{\norm{g_t}_1}
    \leq \frac{d}{\gamma_\infty\lambda_t}$.
Hence, as $\lambda_t\to\infty$, the soft-sign direction achieves the same
first-order loss decrease as sign descent. This motivates $r_t^{\mathrm{sgn}}$,
since
$ \left|\frac{\langle g_t,D_t\rangle}{\norm{g_t}_1}-1\right|
    \leq r_t^{\mathrm{sgn}}$,
whereas $r_t^{\mathrm{grad}}$ directly measures the discrepancy between $D_t$
and $g_t/\eps$. See \cref{prop:local-interpolation_app} in App. \ref{app:res} for a general characterization of these residuals.

Finally, we establish an $\eps$-uniform reference
time $T_{*}$ after which every trajectory lies in the low-loss regime.
\cref{prop:uniform-clock} shows that $G_t$ and $L_t$ are uniformly comparable for
$t\geq T_{*}$, with $G_{T_{*}}\geq g_{*}>0$. Since the crossover
thresholds defined below vanish as $\eps\downarrow0$ and $G_t\to0$, each threshold lies below
$G_{T_{*}}$ for all sufficiently small $\eps$ and is reached at a finite time after $T_{*}$.

\begin{localproposition}
\label{prop:uniform-clock} 
In the same setting as \cref{lem:radius},
the loss and the proxy converge to zero. Moreover, there exist a time $T_* > t_{\rm ema}$ and a constant $g_*>0$ such that for every $\eps\in(0,1]$, we have  $G_{T_*}\ge g_*$, and $L_t\le\nicefrac{\log2}{n}$ and $\nicefrac12L_t\le G_t\le L_t$ for every $t \ge T_{*}$.
\end{localproposition}
\section{Main Results}
\label{sec:local}

We introduce three stopping times that mark the geometric transition of the update. Recall $\ell_\eps:=\log(1/\eps)$. Fix $b,\delta\in(0,1)$, and let $T_*$ be the time from \cref{prop:uniform-clock}. For $\eps \in (0,1)$, we define
\begin{align*}
    T_1^-(\epsilon;b)
    :=
    \inf\left\{
        t\ge T_*:
        G_t
        \le
        \epsilon\ell_\epsilon^b
    \right\}, \quad
    T_1^{\delta}(\epsilon)
    :=
    \inf\left\{
        t\ge T_*:
        G_t
        \le
        \delta\epsilon
    \right\},
\end{align*}
and
\[
    T_1^+(\epsilon;b)
    :=
    \inf\left\{
        t\ge T_*:
        G_t
        \le
        \epsilon\ell_\epsilon^{-b}
    \right\}.
\]
Here, $\ell_\eps^b\to\infty$ and $\ell_\eps^{-b}\to0$ as
$\eps\downarrow0$. For simplicity, we write $T_1^{-} := T_1^-(\epsilon;b)$, $T_1 := T_1^{\delta}(\eps)$, and $T_1^{+} := T_1^+(\epsilon;b)$. These stopping times correspond to
$\lambda_t=G_t/\eps$ first falling below the levels
$\ell_\eps^b$, $\delta$, and $\ell_\eps^{-b}$, respectively.
 We have the ordering $T_1^{-} < T_1 < T_1^{+}$ for sufficiently small $\eps$, as shown by \cref{prop:time_bracket_app} in App. \ref{app:res}. Thus, $T_1^-$ and $T_1^+$ bracket the update crossover, while
$T_1=T_1^\delta(\eps)$ is the reference crossover time. \cref{thm:residual_nonasymp} shows that for small $\eps$, Adam's local update is sign-like at $T_1^{-}$ and gradient-like at $T_1^{+}$, as indicated by the corresponding small sign and gradient residuals. 
Moreover, at $T_1$, the gradient residual is bounded by $O(\delta)$
plus a term that vanishes with $\eps$. \cref{cor:residual_bound_app} in App. \ref{app:transition} strengthens this
conclusion: for suitably chosen $\delta$ and sufficiently small $\eps$,
every $t \geq T_1$ satisfies
$
r_t^{\mathrm{grad}} \leq p$ and 
$r_t^{\mathrm{sgn}} \geq q$,
where $p,q \in (0,1)$ and $p+q=1$. Thus, $T_1$ marks permanent entry
into the gradient-like regime. 

\begin{theorem}
\label{thm:residual_nonasymp} For any $b,\delta \in (0,1)$, there exists a constant $\Tilde{\eps}_{\rm res}(b,\delta)$ such that for every $0 < \eps \leq \Tilde{\eps}_{\rm res}(b,\delta)$, we have 
    \[
\max\left\{
r_{T_1^-(\epsilon;b)}^{\mathrm{sgn}},
r_{T_1^+(\epsilon;b)}^{\mathrm{grad}}
\right\}
=
\begin{cases}
O\!\left(
\ell_\epsilon^{-b}
+
\ell_\epsilon^{-\frac{a}{1-a}}
\right),
& a\in(1/3,1),\\[2mm]
O\!\left(\ell_\epsilon^{-b}\right),
& a=1.
\end{cases}
\]
Denote $K_0 = R_1 C_D$. We also have 
\[
r_{T_1^\delta(\epsilon)}^{\mathrm{grad}}
\leq
\frac{R_1\delta}{1+R_1\delta}
+
\begin{cases}
O\!\left(
\ell_\epsilon^{-\frac{a}{1-a}}
\right),
& a\in(1/3,1),\\[2mm]
O\!\left(
\epsilon^{\frac{1}{2K_0\eta_0}}
\right),
& a=1.
\end{cases}
\]
\end{theorem} 

We next turn to the classifier geometry, beginning with a few key definitions: 
\begin{align*}
    \cM_{\infty} := \{u: \norm{u}_{\infty} \leq 1, \rho(u) = \gamma_{\infty} \}, \quad \gamma_{2|\infty} := \max_{u \in \cM_{\infty}} \frac{\rho(u)}{\norm{u}_2}, \qquad \Delta_{\rm geom} := \gamma_2 - \gamma_{2|\infty}.
\end{align*}
Here, $\gamma_{2\mid\infty}$ is the largest Euclidean-normalized margin
attained by a max-norm-optimal classifier. Hence,
$\Delta_{\mathrm{geom}}\geq0$ measures the separation between the two optimal
geometries. 
The condition $\Delta_{\mathrm{geom}}>0$ rules out any direction that is optimal under both norm geometries. The following theorem shows that, for small $\eps$, the $\ell_\infty$-margin gap is small at $T_1^{-}$, $T_1$, and $T_1^{+}$ despite the change in update geometry, whereas the $\ell_2$-margin gap is bounded below by $\Delta_{\mathrm{geom}}$ up to a vanishing error.
\begin{theorem} \label{thm:margin_main} Let
$\mathcal{T}_\epsilon(b,\delta)
:=
\left\{
T_1^-(\epsilon;b),\,
T_1^\delta(\epsilon),\,
T_1^+(\epsilon;b)
\right\}$. 
For any $\delta,b \in (0,1)$, 
there exists a constant $\Tilde{\eps}_2(\delta,b)$ such that for every $0< \eps \leq \Tilde{\eps}_2(\delta,b)$ and every
\(T\in\mathcal{T}_\epsilon(b,\delta)\), we have 
\[
\gamma_\infty-\widehat{\gamma}_\infty(w_T)
\leq
O(r_\epsilon), \qquad 
\gamma_2-\widehat{\gamma}_2(w_T)
\geq
\Delta_{\mathrm{geom}}
- O(r_\epsilon), \qquad r_\eps :=
\begin{cases}
\ell_\eps^{-a}, & a\in(1/3,1),\\[2pt]
\nicefrac{\log \ell_\eps}{\ell_\eps}, & a=1.
\end{cases}
\]
\end{theorem}
    We discuss four ingredients in the proof of \cref{thm:margin_main}: \textbf{1. How to obtain the constant $\gamma_\infty$; 2. How to derive the sharp weight-norm bound; 3. How to establish the lower bound on the $\ell_2$-margin gap; 4. How to transfer the results at $T_1$ to $T_1^-$ and $T_1^{+}$}; see App. \ref{app:margin} for details. 
    Since the $\ell_\infty$-margin is already nearly optimal at the post-crossover time $T_1^{+}$, we then ask \textbf{how much additional training is required to attain a prescribed Euclidean-margin accuracy}. Importantly, addressing these questions clarifies how the various results of this work fit together.
\paragraph{1. The $\gamma_{\infty}$ constant} 
We first discuss why the direct application of the bounds in \eqref{eq:main_soft_sign} fails. The first bound in \eqref{eq:main_soft_sign} gives
$\left\langle
g_t,\sigma_{\epsilon}(g_t)
\right\rangle
\geq
\|g_t\|_1-d\epsilon
\geq
\left(
\gamma_\infty-\nicefrac{d\epsilon}{G_t}
\right)G_t$. If we worked only with $T_1$, then we have $G_t>\delta\epsilon$ for
$t< T_1$, which leads to 
$\nicefrac{d\epsilon}{G_t}<\nicefrac{d}{\delta}$.
Consequently, we only obtain
$
\left\langle g_t,\sigma_\epsilon(g_t)\right\rangle
\ge
\left(\gamma_\infty-\nicefrac{d}{\delta}\right)G_t$.
The coefficient may be nonpositive and, even when positive, does
not converge to $\gamma_\infty$ as $\epsilon\downarrow0$. Thus,
this estimate cannot recover the sharp bound
$S_{T_1}-S_{T_*}
\le
\nicefrac{\ell_\epsilon}{\gamma_\infty}
+\text{lower-order terms}$.
We could use the other result in
\eqref{eq:main_soft_sign}:
$
\left\langle g_t,\sigma_\epsilon(g_t)\right\rangle
\ge
\kappa_\delta G_t$ where 
$\kappa_\delta:=
\nicefrac{\gamma_\infty^2}{(R_1+d/\delta)}$.
However, $\kappa_\delta$ can be much smaller than
$\gamma_\infty$, yielding the leading term
$\ell_\epsilon/\kappa_\delta$ instead of the sharp term
$\ell_\epsilon/\gamma_\infty$.

To obtain a refined analysis of the trajectory up to $T_1$, we define another stopping time  
\[\bar{T}_{\epsilon}
:=
\inf\left\{
t\geq T_*:G_t\leq\epsilon\Lambda_{\epsilon}
\right\},\]
where $\Lambda_{\epsilon}$ (to be specified next) satisfies $\Lambda_{\epsilon}\to\infty$ and 
$\epsilon\Lambda_{\epsilon}\to 0$. Importantly, $\bar{T}_{\epsilon}$ is an auxiliary proof device rather than a new transition time.
For sufficiently small $\epsilon$, we have
$\Lambda_{\epsilon}>\delta$ and consequently
$\bar{T}_{\epsilon}\leq T_1$. Therefore, the use of $\bar{T}_{\epsilon}$ splits the
trajectory into
\[
\underbrace{
T_*\leq t<\bar{T}_{\epsilon}
}_{\text{Dominant part}}
\qquad+\qquad
\underbrace{
\bar{T}_{\epsilon}\leq t<T_1
}_{\text{Short transition period}}.
\]
Its purpose is
to isolate the part of the trajectory where the first bound in \eqref{eq:main_soft_sign} is sufficiently accurate to recover the exact leading coefficient $\nicefrac{1}{\gamma_\infty}$.  

Concretely, if we consider $t<\bar{T}_{\epsilon}$, we have $G_t>\epsilon\Lambda_{\epsilon}$ and therefore $\nicefrac{d\epsilon}{G_t}
<
\nicefrac{d}{\Lambda_{\epsilon}}
=
o(1)$.  Consequently, it holds that
$
\left\langle
g_t,\sigma_{\epsilon}(g_t)
\right\rangle
\geq
\left(
\gamma_\infty-\nicefrac{d}{\Lambda_{\epsilon}}
\right)G_t$.
Therefore, the effective descent coefficient is
$\gamma_\infty-o(1)$, which preserves the sharp leading constant.
Summing the loss recursion gives
\[
S_{\bar{T}_{\epsilon}}-S_{T_*}
\leq
\frac{\ell_{\epsilon}}{\gamma_\infty}
+
O\left(
\frac{\ell_{\epsilon}}{\Lambda_{\epsilon}}
+
\sum_{t<\bar{T}_{\epsilon}}\eta_t^2
+
1
\right).
\]
Define
\[
\omega_a(\ell)
:=
\begin{cases}
\ell^{1-a},
& a\in(1/3,1),\\[1mm]
\log\ell,
& a=1.
\end{cases}, \qquad \Lambda_{\epsilon}
:=
\begin{cases}
\ell_{\epsilon}^{a},
& a\in(1/3,1),\\[2mm]
\ell_{\epsilon} / \log\ell_{\epsilon},
& a=1. 
\end{cases}
\]
We have $\nicefrac{\ell_{\epsilon}}{\Lambda_{\epsilon}} = \omega_a(\ell_{\epsilon})$. Using $\sum_{t<\bar{T}_{\epsilon}}\eta_t^2 =
O\bigl(\omega_a(\ell_{\epsilon})\bigr)$, we obtain
$
    S_{\bar{T}_{\epsilon}}-S_{T_*}
\leq
\nicefrac{\ell_{\epsilon}}{\gamma_\infty}
+
O\bigl(\omega_a(\ell_{\epsilon})\bigr)$.

Finally, between $\bar{T}_{\epsilon}$ and $T_1$, the gradient
proxy only needs to decrease from
$G_t\asymp\epsilon\Lambda_{\epsilon}$ to 
$G_t\asymp\delta\epsilon$. This corresponds to a multiplicative factor of order
$\Lambda_{\epsilon}/\delta$. The weaker but uniformly positive
descent coefficient $\kappa_\delta$ is sufficient over this
shorter interval and gives
$S_{T_1} - S_{\bar{T}_{\epsilon}}
=
O_\delta\bigl(\log\Lambda_{\epsilon}\bigr)$.
The choice of $\Lambda_{\epsilon}$ ensures that
$\log\Lambda_{\epsilon}=
O\bigl(\omega_a(\ell_{\epsilon})\bigr)$.
Therefore, the transition period only contributes a lower-order term and we obtain the desired 
\begin{align}
    S_{T_1}-S_{T_*}
\leq
\frac{\ell_{\epsilon}}{\gamma_\infty}
+
O\bigl(\omega_a(\ell_{\epsilon})\bigr). \label{eq:s_tt_main}
\end{align}
A similar approach can be used to obtain the scaling behavior of $S_{T_1}$ stated in the following theorem. Its proof is simpler than that of \cref{thm:margin_main} and no additional stopping time is required. It shows that $S_{T_1^-} \sim S_{T_1} \sim S_{T_1^+} \sim (1/\gamma_{\infty}) \ell_{\eps}$ as $\eps \downarrow 0$. 
\begin{theorem}
\label{thm:first-clock} For every $b , \delta \in (0,1)$,  
we have $S_{T_1^{\delta}(\eps)}-S_{T_1^-(\eps;b)} = O(\log\ell_\eps)$, $S_{T_1^{+}(\eps;b)}-S_{T_1^{\delta}(\eps)} = \Theta(\ell_\eps^b)$, and $S_{T_1^{\delta}(\eps)}-S_{T_*} =\frac{1}{\gamma_\infty} \ell_{\eps} +o\!\left( \ell_{\eps} \right)$. 

\end{theorem}

\paragraph{2. The weight-norm bound} Our next goal is to use \eqref{eq:s_tt_main} to obtain a sharp bound on the norm of $w_{T_1}$. 
Based on \cref{lem:radius}, for any $0\leq r<T_1$, we have 
\begin{align}
    \left\|
\sum_{t=r}^{T_1-1}\eta_tD_t
\right\|_\infty
\leq
S_{T_1}-S_r
+
C_{\mathrm{rad}}\eta_rH_\epsilon, \label{eq:dd_temp_main}
\end{align}
where $H_\epsilon
:=
1+
\log\left(
\frac{R_1^2+\epsilon^2}
     {(1-\beta_2)\epsilon^2}
\right)
=
O(\ell_\epsilon)$. If we apply \cref{lem:radius} directly with $r=T_*$, then we would only obtain $\eta_{T_*}H_\epsilon =
O(\ell_\epsilon)$, which is of the same order as the leading term
$\ell_\epsilon/\gamma_\infty$. Therefore, this approach does not give the
desired lower-order remainder
$O\bigl(\omega_a(\ell_\epsilon)\bigr) =
o(\ell_\epsilon)$. 

Instead, we split the trajectory at $k_\epsilon
:= \left\lceil\ell_\epsilon\right\rceil$ and write 
$w_{T_1}
= w_{k_\epsilon} -
\sum_{t=k_\epsilon}^{T_1-1}\eta_tD_t$. We can show 
$ S_{k_\epsilon} =
O\bigl(\omega_a(\ell_\epsilon)\bigr)
= o(\ell_\epsilon)$ and $S_{T_1} = \Omega(\ell_\epsilon)$. 
Hence, for sufficiently small $\epsilon$, we have $k_{\eps} \geq T_*$ and 
$S_{k_\epsilon}<S_{T_1}$, which implies 
$k_\epsilon < T_1$. This is needed so that \cref{lem:radius}
can be applied to the interval $[k_\epsilon,T_1)$.
The choice of $k_{\eps}$ balances the two terms $\norm{w_{k_\epsilon}}_{\infty}$ and $\eta_{k_\epsilon}H_\epsilon$. Concretely, the contribution of the trajectory up to $k_{\eps}$ is small: $\|w_{k_\epsilon}\|_\infty
\leq
\|w_0\|_\infty+C_DS_{k_\epsilon}
=
O\bigl(\omega_a(\ell_\epsilon)\bigr)$, while the $\eta_{k_\epsilon}H_\epsilon$ term in \cref{lem:radius} is also small:
$
\eta_{k_\epsilon}H_\epsilon =
O\bigl(\omega_a(\ell_\epsilon)\bigr)$. 
Consequently, using \eqref{eq:dd_temp_main} with $r = k_{\eps}$, we obtain 
\[
\begin{aligned}
\|w_{T_1}\|_\infty
&\leq
\|w_{k_\epsilon}\|_\infty
+
\left\|
\sum_{t=k_\epsilon}^{T_1-1}
\eta_tD_t
\right\|_\infty \leq
S_{T_1}-S_{k_\epsilon}
+
O\bigl(\omega_a(\ell_\epsilon)\bigr)
\end{aligned}. 
\]
Finally, using \eqref{eq:s_tt_main} and $k_{\eps} \geq T_*$ (hence, $S_{T_1}-S_{k_\epsilon} \leq S_{T_1}-S_{T_*}$), we obtain 
\begin{align}
    \|w_{T_1}\|_\infty
\leq
\frac{\ell_\epsilon}{\gamma_\infty}
+
O(\omega_a(\ell_\epsilon)). \label{eq:radius_tau_main}
\end{align}

\paragraph{3. From the $\ell_{\infty}$-margin to the $\ell_2$-margin} Define
$x_\epsilon
:=
w_{T_1} / \|w_{T_1}\|_\infty$. We first complete the argument for the $\ell_{\infty}$-margin gap at $T_1$ in \cref{thm:margin_main}. 
At $T_1$, we have
$G_{T_1}\leq\delta\epsilon$, and 
$L_{T_1}
\leq
2G_{T_1}
\leq
2\delta\epsilon$ by \cref{prop:uniform-clock}. The loss-margin relationship in 
\cref{lem:loss-proxy} (App. \ref{app:prelim}) also gives
$L(w_{T_1})
\geq
\frac{\log 2}{n}e^{-\rho(w_{T_1})}$.
Therefore, it holds that $
\rho(w_{T_1})
\geq
\ell_\epsilon-O(1)$. Dividing this lower bound by the radius upper bound in~\eqref{eq:radius_tau_main} and using the positive homogeneity of $\rho$, we obtain 
\begin{align}
    \widehat{\gamma}_\infty(w_{T_1})
=
\frac{\rho(w_{T_1})}
     {\|w_{T_1}\|_\infty} = \rho(x_{\eps})
\geq
\frac{\ell_\epsilon-O(1)}
     {\ell_\epsilon/\gamma_\infty
      +O \!\left(\omega_a(\ell_\epsilon)\right)}
=
\gamma_\infty-O(r_\epsilon), \label{eq:main_gamma}
\end{align}
where $
r_\epsilon =
\nicefrac{\omega_a(\ell_\epsilon)}{\ell_\epsilon}$. Next, we show how to derive the second inequality of \cref{thm:margin_main} based on this.
     
Set $A := [I_d;-I_d;-Z] \in \mathbb{R}^{(2d + n) \times d}$ where the $i$-th row of $Z$ is $z_i^T$, and $b := [\mathbf{1}_d;\mathbf{1}_d;-\gamma_{\infty} \mathbf{1}_n] \in \mathbb{R}^{2d + n}$ where $\mathbf{1}$ is a vector of all ones. 
It then holds that $\mathcal{M}_\infty
=
\left\{
u\in\mathbb{R}^d:
Au\leq b
\right\}$. For this set, applying Hoffman's error bound in convex analysis \citep{hoffman1952approximate} (see \cref{lem:hoff} in App. \ref{app:margin}), we obtain for any $x \in \mathbb{R}^d$ that 
\begin{align}
\operatorname{dist}_2(x,\mathcal{M}_\infty)
\leq
H_{2,\infty}(A)
\left\|
(Ax-b)_+
\right\|_\infty, \label{eq:hoff_main}
\end{align}
where $\operatorname{dist}_2\!\left(x,\mathcal{M}_\infty\right) := \min_{u \in \cM_{\infty}} \norm{x - u}_2$, $[y]_+ := (\max \{y_j, 0 \})_j$, and the constant $H_{2,\infty}(A)$ depends only on the constraint matrix $A$ (which depends on the data). We can further show $
\left\|
(Ax_\epsilon-b)_+
\right\|_\infty
=
\gamma_\infty-\rho(x_\epsilon)$. By using $\gamma_\infty-\rho(x_\epsilon)
\leq
O(r_\epsilon)$ (obtained from \eqref{eq:main_gamma}) and \eqref{eq:hoff_main}, we obtain $\operatorname{dist}_2
\left(x_\epsilon,\mathcal M_\infty\right)
\leq
O(r_\epsilon)$. Equivalently, there exists $u_\epsilon\in\mathcal M_\infty$ such that $
\|x_\epsilon-u_\epsilon\|_2 \leq
O(r_\epsilon)$.
This is the central idea of this part of the proof. A small max-margin deficit only shows that
$\rho(x_\epsilon) \approx
\gamma_\infty$. 
The Hoffman bound strengthens this scalar statement into the geometric
consequence:
$
x_\epsilon
\approx
u_\epsilon$ for some
$u_\epsilon\in\mathcal M_\infty$. This allows us to compare the Euclidean-normalized margins. Define
$F(v)
:=
\nicefrac{\rho(v)}{\|v\|_2},
\, \, 
v\neq 0$.
Since $u_\epsilon\in\mathcal M_\infty$, it follows from the definition
of $\gamma_{2\mid\infty}$ that
$F(u_\epsilon)
\leq
\gamma_{2\mid\infty}$.
The Lipschitz property of the function $\rho$ in \cref{{lem:rho_lip}} (App. \ref{app:margin}) then gives
\[
F(x_\epsilon)
\leq
F(u_\epsilon)
+
(R_2+\gamma_\infty)
\|x_\epsilon-u_\epsilon\|_2.
\]
Using the bounds on $F(u_{\eps})$ and $\norm{x_{\eps} - u_{\eps}}_2$, we obtain
$F(x_\epsilon)
\leq
\gamma_{2\mid\infty}
+
O_\delta(r_\epsilon)$. Finally, by the positive homogeneity of $\rho$, it holds that $F(x_\epsilon)
=
\widehat{\gamma}_2(w_{T_1})$, which gives the desired inequality.

\paragraph{4. Transfer to $T_1^{-}$ and $T_1^{+}$} The core idea in transferring the margin results at $T_1$ to $T_1^-$ is to compare $\norm{w_{T_1} - w_{T_1^{-}}}_{\infty}$ and $\norm{w_{T_1^{-}}}_{\infty}$. If the former is small relative to the latter, then the change in the normalized margin is also small. Concretely, \cref{lem:margin_helper} (App. \ref{app:margin}) gives $|\hat{\gamma}_p(x) - \hat{\gamma}_{p}(y)|\lesssim \norm{x-y}_p / \norm{y}_p, \, \, p \in \{\infty,2\}$. After establishing $\norm{w_{T_1^-}}_{p} \gtrsim \ell_{\eps}$ and $\norm{w_{T_1} - w_{T_1^-}}_{p} \lesssim \log \ell_{\eps}$ for $p \in \{\infty, 2\}$,  we obtain
$
|
\hat{\gamma}_p(w_{T_1})
-
\hat{\gamma}_p(w_{T_1^-})
| = 
O
(\log\ell_\eps / \ell_\eps)
$. The desired results then follow from the corresponding margin bounds at
\(T_1\) and the triangle inequality. 

However, the same argument does not generally recover the desired \(O(r_\eps)\) rate when transferring the margin bounds from $T_1$ to \(T_1^+\).
Since $S_{T_1^{+}} - S_{T_1} = O(\ell_{\eps}^b)$ (given in \cref{thm:first-clock}), we obtain $\norm{w_{T_1^+} - w_{T_1}}_{p} \lesssim \ell_{\eps}^b$ for $p \in \{\infty, 2\}$ using the triangle inequality and 
the uniform update bound in \cref{lem:uniform-D} (App. \ref{app:prelim}).
Since $\norm{w_{T_1}}_p \gtrsim \ell_{\eps}, \, \, p \in \{\infty, 2\}$, \cref{lem:margin_helper} (App. \ref{app:margin}) gives  $
|
\hat{\gamma}_p(w_{T_1^{+}})
-
\hat{\gamma}_p(w_{T_1})
| = 
O(\ell_{\eps}^{b-1}), \, \, p \in \{ \infty, 2\}$. 
Consequently, the resulting rate is
$O(r_\eps+\ell_\eps^{b-1})$, rather than $O(r_\eps)$. For $a \in (1/3,1)$, $\ell_{\eps}^{b-1}$ is worse than $r_{\eps}$ if $b > 1-a$; and for $a = 1$, it is always worse. Therefore, our next goal is to give a tighter analysis of the trajectory after $T_1$. 

Since Adam has entered the gradient-like regime by $T_1$, a natural idea is to decompose its update into a gradient direction and a residual direction for $t \geq T_1$. Specifically, we ensure that the gradient direction is of unit norm and the norm of the residual direction is exactly the gradient residual. Let $u_t:=g_t/\norm{g_t}_2$. We write Adam's iterate as 
 \begin{align}
    w_{t+1}=w_t-\theta_t(u_t+e_t), \quad \text{where} \quad e_t := \frac{\eps D_t - g_t}{\norm{g_t}_2}, \quad \theta_t:=\frac{\eta_t\norm{g_t}_2}{\eps}.  \label{eq:local_ngd}
\end{align}
We regard $\theta_t$ as the effective step size and define the accumulated original and effective step
sizes by
\[
    V_t^\delta(\eps)
    :=\sum_{s=T_1^\delta(\eps)}^{t-1}\eta_s,
    \qquad
    H_t^\delta(\eps)
    :=\sum_{s=T_1^\delta(\eps)}^{t-1}\theta_s.
\]
For simplicity, we let $V_t := V_t^\delta(\eps)$ and $H_t := H_t^\delta(\eps)$. 
 \cref{prop:bounded_euc} relates these quantities and shows that the
post-$T_1$ trajectory is a normalized-gradient trajectory with a
summable perturbation.

\begin{localproposition}\label{prop:bounded_euc} For every $\delta \in (0,1)$, there exists a constant $\Tilde{\eps}_{1}(\delta)$ such that, for every $0<\eps\le \Tilde{\eps}_{1}(\delta)$ and every $t\ge T_1^{\delta}(\eps)$, we have $L_t \asymp G_t\asymp\frac{\eps}{1+V_t^{\delta}(\eps)}$, $H_t^{\delta}(\eps) \asymp \log(1+V_t^{\delta}(\eps))$, and 
$\sum_{s=T_1^{\delta}(\eps)}^{\infty}\theta_s\norm{e_s}_2\le E_0$, where $E_0$ is a constant that depends on the data and the step-size schedule.
\end{localproposition} 
Using \cref{prop:bounded_euc} instead, we obtain $\norm{w_{T_1^+} - w_{T_1}}_{p} \leq H_{T_1^+} + E_0, \, \, p \in \{ 2, \infty\}$. Since $V_{T_1^+} = O(\ell_{\eps}^b)$, it holds that $H_{T_1^+} = O(\log \ell_{\eps})$ by $H_{T_1^+} \lesssim \log (1 + V_{T_1^+})$. Consequently, we have $\norm{w_{T_1^+} - w_{T_1}}_{p} = O(\log \ell_{\eps}), \, \, p \in \{ 2, \infty \}$, which leads to $|
\widehat{\gamma}_p(w_{T_1^+})
-
\widehat{\gamma}_p(w_{T_1})|
=
O\left(
\nicefrac{\log \ell_\eps}{\ell_\eps}
\right)
=
O(r_\eps)$. The margin gap results at $T_1^+$ in \cref{thm:margin_main} then follow. We generalize this result to any $t \geq T_1$ in the following proposition.  Essentially, the margin gap is controlled by two terms: $r_\eps$ and $\nicefrac{H_t}{\ell_{\eps}}$. The former controls the gap that is already present at $T_1$, while the latter measures the post-crossover movement relative to the classifier’s existing size.

\begin{localproposition} \label{prop:post_euc} For every $\delta \in (0,1)$, there exists a constant $\Tilde{\eps}_{3}(\delta)$ such that for every $0< \eps \leq  \Tilde{\eps}_3(\delta)$ and $t \geq T_1^{\delta}(\eps)$, 
it holds that
\begin{align*}
    \gamma_{\infty} - \hat{\gamma}_{\infty}(w_t) \leq O \left( r_{\eps} + \frac{H_t^{\delta}(\eps)}{\ell_{\eps}} \right), \qquad \gamma_2-\widehat{\gamma}_2(w_t)
\geq
\Delta_{\mathrm{geom}}
-
O
\left(
r_\epsilon
+
\frac{H_t^\delta(\epsilon)}{\ell_\epsilon}
\right).
\end{align*}
\end{localproposition}

\paragraph{Euclidean Optimality} We quantify the movement required to approach the $\ell_2$-optimal
geometry. Assume $\Delta_{\mathrm{geom}}>0$ and fix
$\zeta\in(0,\Delta_{\mathrm{geom}})$. By \cref{thm:margin_main}, we have for all sufficiently
small $\eps$ that 
\begin{align}
    \gamma_2-\widehat\gamma_2\bigl(w_{T_1}\bigr)
    \geq \Delta_{\mathrm{geom}}-C_\delta r_\eps
    \geq \nicefrac{(\Delta_{\mathrm{geom}}+\zeta)}{2}
    >\zeta, \label{eq:gamma2_main}
\end{align}
where $C_{\delta} > 0$ is some constant.
Therefore, the Euclidean-margin gap has not yet fallen below \(\zeta\) at \(T_1\). Consequently, the following hitting time records a genuine post-crossover change:
\[
    \tau_\zeta^\delta(\eps)
    :=\inf\left\{
        t\geq T_1^\delta(\eps):
        \gamma_2-\widehat\gamma_2(w_t)\leq\zeta
    \right\}.
\]
The second inequality of \cref{prop:post_euc} suggests that reducing the
Euclidean-margin gap below $\zeta$ requires effective movement of order
$\ell_\eps$. \cref{thm:margin_euc} sharpens this observation by establishing a
two-sided comparison with the classifier norm inherited at the crossover.

\begin{theorem} \label{thm:margin_euc} Suppose that $\Delta_{\rm geom} > 0$. For any $\delta \in (0,1)$ and $\zeta \in (0, \Delta_{\rm geom})$, there exists a constant $\tilde{\eps}_{4}(\delta, \zeta)$ such that the following statements hold for every $0 < \eps \leq \tilde{\eps}_{4}(\delta,\zeta)$:

\emph{(i) For every $t \geq T_1^{\delta}(\eps)$, 
we have $\gamma_2 - \hat{\gamma}_2 (w_t) \leq O \bigl( \frac{ 1+ \norm{w_{T_1^{\delta}(\eps)}}_2}{1 + H_t^{\delta} (\eps)} \bigr)$.
}

\emph{(ii) There exist constants $0< c_{\zeta} \leq C_{\zeta} < \infty$ such that $c_{\zeta} \norm{w_{T_1^{\delta}(\eps)}}_2 \leq H_{\tau_{\zeta}^{\delta}(\eps)}^{\delta}(\eps) \leq C_{\zeta} (1 + \norm{w_{T_1^{\delta}(\eps)}}_2)$. Consequently, it holds that
    $H_{\tau_{\zeta}^{\delta}(\eps)}^{\delta}(\eps) = \Theta(\log \frac{1}{\eps})$ and $V_{\tau_{\zeta}^{\delta}(\eps)} = \eps^{- \Theta(1)}$. 
}
\end{theorem}

We sketch the proof of Part~(ii) (see App. \ref{app:euclidean} for details). Let $\tau_{\zeta} := \tau_{\zeta}^{\delta}(\eps)$ and $R_{T_1} := \norm{w_{T_1}}_2$. 
Part~(i) shows that after \(T_1\), the Euclidean-margin gap is controlled by the ratio between the classifier norm inherited at the crossover and the subsequently accumulated effective step size \(H_t\).
Since $H_t\to\infty$ by \cref{prop:bounded_euc},
the gap is at most $\zeta$ once $H_t$ is a sufficiently large
$\zeta$-dependent multiple of $R_{T_1}$. Accounting for the bounded one-step
overshoot gives $H_{\tau_{\zeta}}\lesssim1+R_{T_1}$.

For the lower bound, we have $\gamma_2-\widehat{\gamma}_2(w_{T_1})
\geq
\nicefrac{(\Delta_{\rm geom}+\zeta)}{2}$ for all
sufficiently small \(\eps\) (see \eqref{eq:gamma2_main}). 
At \(\tau_\zeta\), the Euclidean-margin gap is at most \(\zeta\).
Together, they imply $\widehat{\gamma}_2(w_{\tau_\zeta})
-
\widehat{\gamma}_2(w_{T_1})
\geq
\nicefrac{(\Delta_{\rm geom}-\zeta)}{2}$, and \cref{lem:margin_helper} (App. \ref{app:margin}) therefore gives
$\nicefrac{(\Delta_{\rm geom}-\zeta)}{2}
\leq
\nicefrac{
2R_2\norm{w_{\tau_\zeta}-w_{T_1}}_2
}{
R_{T_1}
}$. Using the representation in
\eqref{eq:local_ngd} and the summability bound in
\cref{prop:bounded_euc}, we obtain $
\norm{w_{\tau_\zeta}-w_{T_1}}_2
\leq
H_{\tau_\zeta}+E_0$.
Combining these inequalities yields
$H_{\tau_\zeta}
\geq
\frac{\Delta_{\rm geom}-\zeta}{4R_2}R_{T_1}-E_0$.
Since \(R_{T_1}\to\infty\), the fixed constant \(E_0\) can be absorbed
for sufficiently small \(\eps\), giving
$H_{\tau_\zeta}
\gtrsim
R_{T_1}$.
The lower and upper bounds together imply 
$
H_{\tau_\zeta}
\asymp
\norm{w_{T_1}}_2
=
\Theta(\ell_\eps)$ (see App. \ref{app:margin} for details).
Combining the $S$-time characterization in \cref{thm:first-clock} with \cref{thm:margin_euc} gives the corresponding iteration-count scales.

\begin{corollary}
\label{cor:iter_complex}
Suppose that $\Delta_{\rm geom} > 0$. Fix any 
$\zeta\in(0,\Delta_{\mathrm{geom}})$ and $\delta \in (0,1)$. 
Under the schedule $\eta_t=\eta_0(t+1)^{-a}$, the following holds as $\eps\downarrow0$:
\[
\begin{cases}
T_1^\delta(\eps)
=\Theta\!\left((\log(1/\eps))^{1/(1-a)}\right),\\[2pt]
\tau_\zeta^\delta(\eps)
=\eps^{-\Theta(1)/(1-a)},
\end{cases} a\in(1/3,1); \qquad
\begin{cases}
T_1^\delta(\eps)
=\eps^{-1/(\eta_0\gamma_\infty)+o(1)},\\[2pt]
\tau_\zeta^\delta(\eps)
=\exp\!\left(\eps^{-\Theta(1)}\right),
\end{cases} a=1.
\]
\end{corollary}

Using the gradient proxy for the cross-entropy loss introduced by \citet{fan2025spectral}, all the results in this section can be extended to the separable multiclass classification setting (details in App. \ref{app:multiclass}). 


\paragraph{Experiments}
\label{sec:experiments}

\begin{figure*}[t]
    \centering
    \captionsetup[subfigure]{
        font=small,
        labelfont=bf,
        skip=2pt
    }

    \begin{subfigure}[t]{0.242\textwidth}
        \centering
        \includegraphics[width=\linewidth]
        {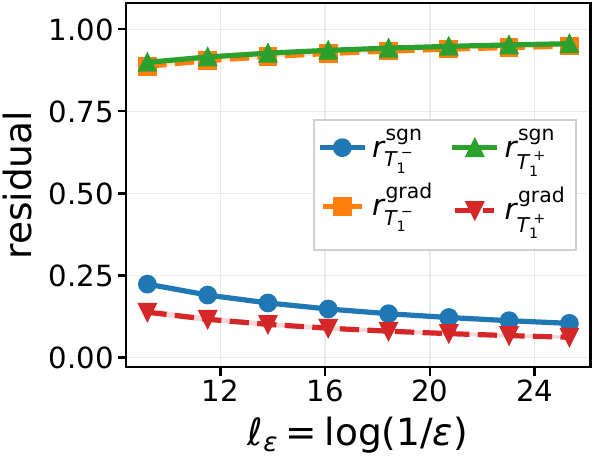}
        \caption{Update residuals}
        \label{fig:update-residuals}
    \end{subfigure}\hfill%
    \begin{subfigure}[t]{0.242\textwidth}
        \centering
        \includegraphics[width=\linewidth]
        {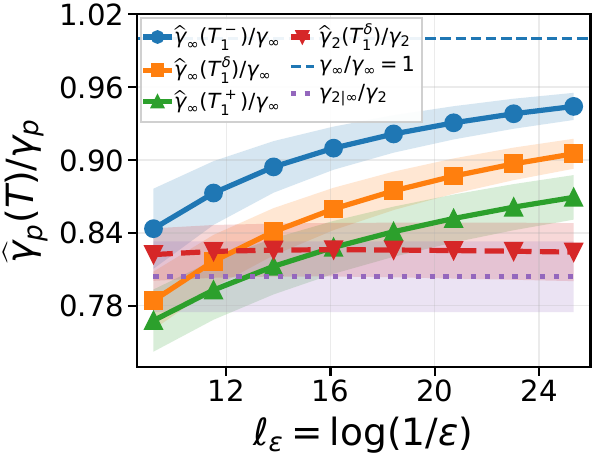}
        \caption{Classifier margins}
        \label{fig:classifier-margins}
    \end{subfigure}\hfill%
    \begin{subfigure}[t]{0.242\textwidth}
        \centering
        \includegraphics[width=\linewidth]
        {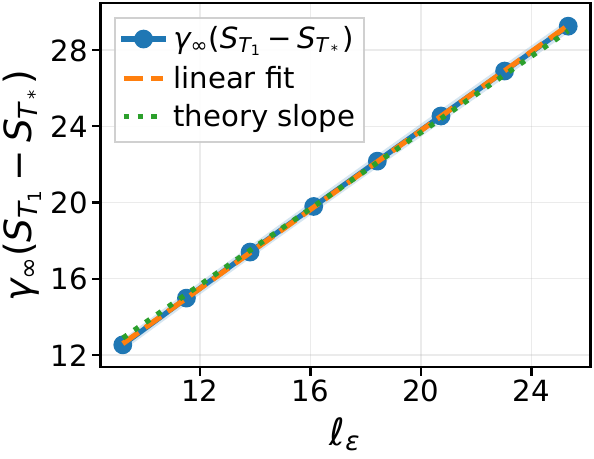}
        \caption{Local-update crossover}
        \label{fig:update-crossover}
    \end{subfigure}\hfill%
    \begin{subfigure}[t]{0.242\textwidth}
        \centering
        \includegraphics[width=\linewidth]
        {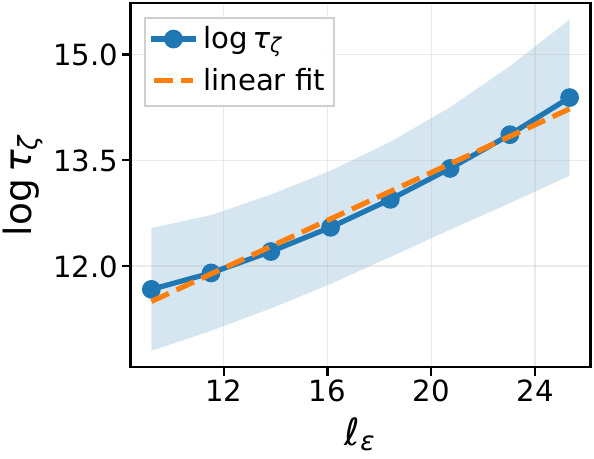}
        \caption{Classifier transition}
        \label{fig:classifier-transition}
    \end{subfigure}

    \caption{\textbf{Update--classifier separation and transition
    scales.}
    \textbf{(a)} Sign and gradient residuals at $T_1^{-}$ and
    $T_1^{+}$ versus $\ell_{\eps}=\log(1/\eps)$.
    \textbf{(b)} Max-norm margin ratios at $T_1^{-}$,
    $T_1$, and $T_1^{+}$, together with the
    Euclidean-margin ratio at $T_1$. The horizontal lines
    indicate $1$ and the across-seed mean of
    $\gamma_{2\mid\infty}/\gamma_2$.
    \textbf{(c)}
$\gamma_{\infty}\bigl(S_{T_1}-S_{T_\star}\bigr)$
    versus $\ell_{\eps}$. The fitted slope is $1.036$, close to the
    predicted unit slope. Since $\eta_t=\eta_0(t+1)^{-1/2}$ implies
    $S_t=2\eta_0\sqrt{t}+O(1)$, this also supports
    $T_1=\Theta(\ell_{\eps}^2)$.
    \textbf{(d)} $\log\tau_{\zeta}$ versus
    $\ell_{\eps}$, consistent with
    $\tau_{\zeta}=\eps^{-\Theta(1)}$.  See App. \ref{app:additional_binary}/App. \ref{app:nonlinear} for more binary/nonlinear experiments.}
    \label{fig:update-classifier-transition}
\end{figure*}
We test our predictions on eight balanced Gaussian-mixture datasets with $n=64$ and $d=128$. For seed $s$, we draw $\widetilde{\nu}_s\sim\mathcal{N}(0,I_d)$, set $\nu_{\star,s}=\widetilde{\nu}_s/\|\widetilde{\nu}_s\|_2$, and sample $X\mid Y=y\sim\mathcal{N}\!\left(8y\nu_{\star,s},D^2\right)$ where $D_{jj}=3^{(j-1)/(d-1)}$. Thus, the coordinatewise standard deviations range from $1$ to $3$. For each $\eps\in\{10^{-4},\ldots,10^{-11}\}$, we run Adam from $w_0=0$ with $(\beta_1,\beta_2)=(0.9,0.99)$ and $\eta_t=10^{-3}(t+1)^{-1/2}$. Define $T_\star:=\max_{\eps}\inf\{t:L_t\leq \log(2)/n\}$. We take $b=0.9$, $\delta=0.2$, and $\zeta=0.75\Delta_{\mathrm{geom}}$. All curves report the mean and one standard deviation across the seeds.

\cref{fig:update-residuals,fig:classifier-margins} confirm the separation between the update and classifier transitions. As $\eps\downarrow0$, the ordered residual pair $(r^{\mathrm{sgn}},r^{\mathrm{grad}})$ approaches $(0,1)$ at $T_1^-$ and $(1,0)$ at $T_1^+$. Meanwhile, the max-norm margin ratios at $T_1^-$, $T_1$, and $T_1^+$ all approach one, whereas the Euclidean-margin ratio at $T_1$ remains close to $\gamma_{2\mid\infty}/\gamma_2<1$. At the smallest tested value of $\eps$, we have $r_{T_1^+}^{\mathrm{grad}}=0.062\pm0.005$ while $ \nicefrac{( \gamma_2-\widehat{\gamma}_2(w_{T_1^+}))}{\Delta_{\mathrm{geom}}} =0.831\pm0.030$. Thus, more than $83\%$ of the max-norm-to-Euclidean gap remains even after the update has become strongly gradient-like. 
\cref{fig:update-crossover,fig:classifier-transition} support the predicted separation in time scales. The fitted slope of $\gamma_\infty(S_{T_1}-S_{T_\star})$ against $\ell_\eps$ is $1.036$, close to the predicted unit slope. Since $a=1/2$, the approximately linear dependence of $\sqrt{T_1}$ and $\log\tau_\zeta$ on $\ell_\eps$ agrees with $T_1=\Theta(\ell_\eps^2)$ and $\tau_\zeta=\eps^{-\Theta(1)}$. The experiments therefore reproduce the predicted polylogarithmic-versus-polynomial separation between the local-update and classifier transitions.
\section{Related Work}
\label{sec:relatedworks}

For separable linear classification, gradient descent converges in direction to the Euclidean max-margin classifier
\citep{soudry2018implicit,ji2019risk,nacson2019convergence}.
Its margin gap decays at a rate \(O(1/\log t)\), while normalized gradient descent and accelerated methods improve this rate to \(O(1/t)\) and \(O(\nicefrac{\log t}{t^2})\), respectively
\citep{ji2021primaldual,ji2021fast}.
Related results cover arbitrary constant step sizes
\citep{wu2023implicit},
multiclass classification
\citep{soudry2018implicit,ravi2024implicit},
homogeneous neural networks
\citep{lyu2019gradient},
matrix factorization
\citep{gunasekar2017implicit,li2020towards},
and self-attention
\citep{tarzanagh2023transformers};
see \citet{vardi2023implicit} for a survey.
Implicit bias under more general steepest-descent geometries has also been studied
\citep{fan2025spectral,tsilivis2025flavorsmarginimplicitbias,li2026implicit},
with sign and spectral descent as prominent examples
\citep{bernstein2018signsgd,carlson2015preconditioned,bernstein2024old}.
The implicit bias of mirror descent has also been investigated
\citep{sun2022mirror,JMLR:v24:23-0836}.

For Adam on separable linear data, a fixed positive stability constant leads to the Euclidean max-margin direction, whereas removing it yields convergence of the normalized margin to the optimal max-norm margin
\citep{wang2022does,zhang2024adam}.
In the per-sample setting without a stability constant, \citet{baek2026implicit} show that incremental sampling can alter this bias and, in an extreme case, recover the Euclidean max-margin direction.
Related analyses connect AdamW to $\ell_\infty$-constrained optimization
\citep{xie2024implicitbiasadamwellinfty}
and study the implicit bias of adaptive methods in homogeneous neural networks
\citep{wang2021implicit,gronich2026implicit}. In another related work, \citet{wang2026stability} study memoryless smoothed-sign descent---the \(\beta_1=\beta_2=0\) special case of Adam---under an annealed stability constant.
They show that the annealing rate selects a continuum of limiting classifiers along a Burg-type barrier path.
Their analysis compares endpoint geometries across annealing schedules; we instead study the within-trajectory evolution of the update and classifier geometries for bias-corrected Adam with a fixed stability constant.

\section{Conclusion}
\label{sec:conclusion}


We studied deterministic, full-batch Adam with a fixed $\eps>0$ on linearly separable classification problems. At the update crossover, Adam is already gradient-like, yet its classifier remains nearly max-norm-margin optimal and, when $\Delta_{\mathrm{geom}}>0$, bounded away from Euclidean optimality. Closing this gap requires $\Theta(\log(1/\eps))$ additional effective movement, yielding polylogarithmic-versus-polynomial transition times for $a\in(1/3,1)$ and polynomial-versus-exponential times for $a=1$. Important directions for future work include extending the current analysis to stochastic Adam, Adam with constant step sizes, and neural networks.

\bibliography{references}
\bibliographystyle{iclr2027_conference}

\clearpage
\clearpage
\appendix

\startcontents[appendices]

\printcontents[appendices]{}{1}{%
  \setcounter{tocdepth}{2}%
}

\vspace{1em}

\section{Experiments}
\label{app:exp}

\subsection{Binary Experiments} \label{app:additional_binary}
We present additional results from the binary experiments beyond those in \cref{fig:update-classifier-transition}. Recall that $T_1^{-} := T_1^{-}(\eps;b)$, $T_1 := T_1^{\delta}(\eps)$, and $T_1^{+} := T_1^+(\eps;b)$. All curves report the mean and one standard deviation across the seeds.  
For $\eps = 10^{-11}$, 
\cref{fig:app1} shows that the sign and gradient residuals converge to $(0,1)$ and $(1,0)$ when $\lambda_t = G_t/\eps$ become large and small, respectively. This demonstrates that $\lambda_t$ serves as an effective order parameter and supports the claims in \cref{prop:local-interpolation_app} in App. \ref{app:res}. \cref{fig:app2} shows that $\nicefrac{(S_{T_1} - S_{T_1^{-}})}{\log \ell_{\eps}}$ is bounded above (in fact, decreases as $\ell_{\eps}$ increases), supporting the claim $S_{T_1} - S_{T_1^-} = O(\log \ell_{\eps})$ in \cref{thm:first-clock}. Moreover, as $\ell_{\eps}$ increases, the curve of $\nicefrac{(S_{T_1^+} - S_{T_1})}{\ell_{\eps}^b}$ levels off after an initial increase, supporting the claim $S_{T_1^+} - S_{T_1} = \Theta(\ell_{\eps}^b)$ in \cref{thm:first-clock}. \cref{fig:app3} shows that the accumulated gradient residual \(E_t\) levels off, while the accumulated effective step size $H_t \asymp \log (1+V_t)$ diverges, supporting the claim of \cref{prop:bounded_euc}.  \cref{fig:app4} shows a linear relationship between $\sqrt{T_1}$ and $\ell_{\eps}$, further supporting the scaling behavior claimed in \cref{cor:iter_complex}. 

For each dataset, we compute $\gamma_\infty$ by linear programming using SciPy's HiGHS solver. We then minimize $\frac12\|u\|_2^2$ subject to $\|u\|_\infty\leq 1$ and $Zu\geq\gamma_\infty\mathbf{1}$. Here, $Z\in\mathbb{R}^{n\times d}$ has $i$-th row
$z_i^\top=y_i x_i^\top$. Denoting the resulting solution by $u_\star$, we compute $\gamma_{2|\infty}=\gamma_\infty/\|u_\star\|_2$. 

\begin{figure*}[h!]
    \centering
    \captionsetup[subfigure]{
        font=small,
        labelfont=bf,
        skip=2pt
    }

    \begin{subfigure}[t]{0.242\textwidth}
        \centering
        \includegraphics[width=\linewidth]
        {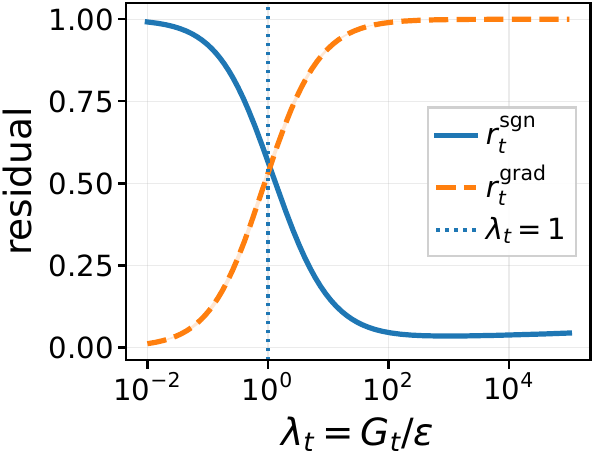}
        \caption{Residual vs. $\lambda_t$}
        \label{fig:app1}
    \end{subfigure}\hfill%
    \begin{subfigure}[t]{0.242\textwidth}
        \centering
        \includegraphics[width=\linewidth]
        {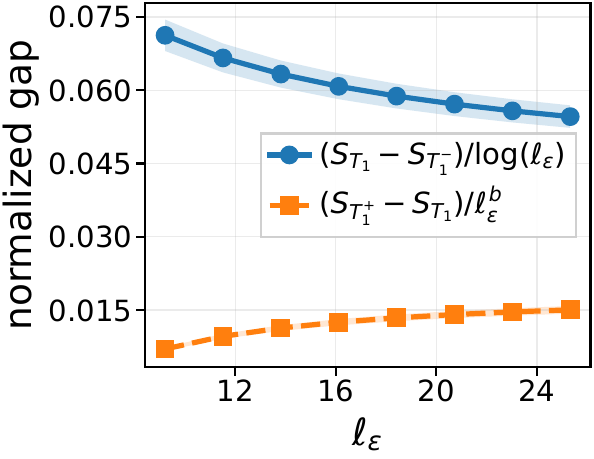}
        \caption{$S$-time difference}
        \label{fig:app2}
    \end{subfigure}\hfill%
    \begin{subfigure}[t]{0.242\textwidth}
        \centering
        \includegraphics[width=\linewidth]
        {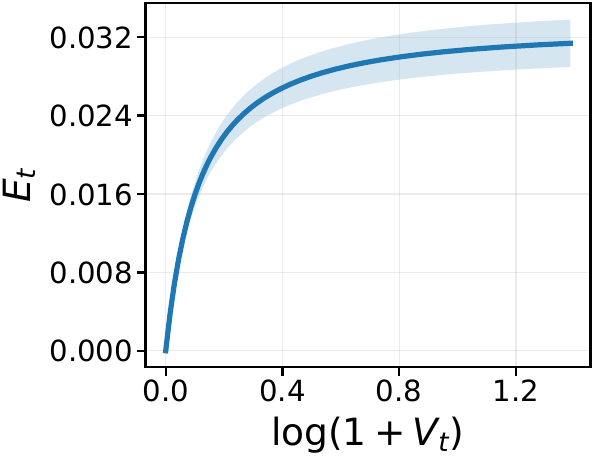}
        \caption{Accumulated Residual}
        \label{fig:app3}
    \end{subfigure}\hfill%
    \begin{subfigure}[t]{0.242\textwidth}
        \centering
        \includegraphics[width=\linewidth]
        {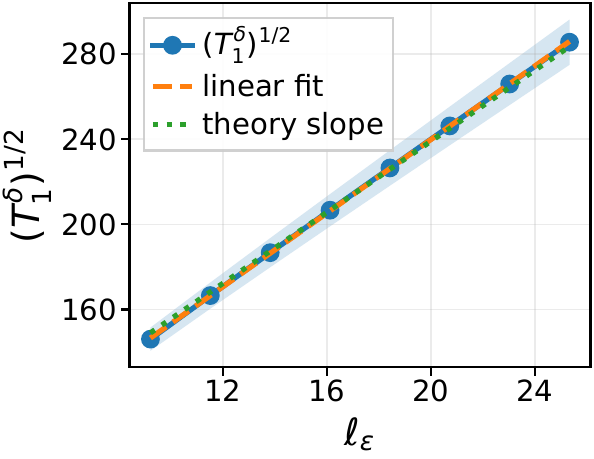}
        \caption{$\sqrt{T_1}$ vs. $\ell_{\eps}$}
        \label{fig:app4}
    \end{subfigure}

    \caption{Additional Binary Experiments. \textbf{(a)} Fix $\eps = 10^{-11}$. Sign ($r_t^{\rm sgn}$) and gradient residuals ($r_t^{\rm grad}$) against the order parameter $\lambda_t = G_t / \eps$. \textbf{(b)} $\nicefrac{(S_{T_1} - S_{T_1^{-}})}{\log \ell_{\eps}}$ and $\nicefrac{(S_{T_1^+} - S_{T_1})}{\ell_{\eps}^b}$ against $\ell_{\eps}$. (c) $E_t:= \sum_{s=T_1}^{t-1} \theta_s \norm{e_s}_2$ against $\log (1 + V_t)$. (d) $\sqrt{T_1}$ against $\ell_{\eps}$.}
    \label{fig:binary_app}
\end{figure*}

\subsection{Nonlinear Experiments} \label{app:nonlinear}
These experiments provide a qualitative extension of the linear results. Because the network is trained with a constant stepsize, we do not use them to test the iteration-time exponents derived for linear models. Moreover, exact max-margin computation is challenging, so we use empirical reference margins as approximations.

\paragraph{Experimental setting.} We adopt the experimental setup of \citet{tsilivis2025flavorsmarginimplicitbias}. For each of the eight seeds, we sample without replacement 50 images of the digit 3 and 50 images of the digit 6 from the MNIST training set. We assign labels $-1$ and $+1$, respectively, and shuffle the resulting $n=100$ examples. Each image is flattened into $x_i\in\mathbb{R}^{784}$ and divided by $255$ without additional preprocessing. Testing uses all 1,968 images labeled 3 or 6 in the official MNIST test set. The test data are used only to report loss and accuracy. 

We train a two-layer ReLU network of width $m = 128$, given by
\begin{equation} f_{\theta}(x) = \sum_{r=1}^{m} a_r \operatorname{ReLU}(w_r^\top x), \qquad \theta=(W,a), \label{eq:nonlinear-model} \end{equation} where $W=(w_1^\top,\ldots,w_m^\top)\in\mathbb{R}^{m\times d}$ and both layers are trained. The network has no bias terms and no $1/\sqrt{m}$ output normalization. We use the inactive ReLU subgradient at zero. With $\alpha=0.05$, the initialization is \begin{equation} W_{rj,0} \stackrel{\mathrm{i.i.d.}}{\sim} \mathcal{N}\!\left(0,\frac{\alpha^2}{d}\right), \qquad a_{r,0} \stackrel{\mathrm{i.i.d.}}{\sim} \mathcal{N}(0,\alpha^2). \label{eq:nonlinear-initialization} \end{equation} The signed logits and the logistic loss are 
\begin{align} 
q_i(\theta):=y_i f_\theta(x_i), \quad L(\theta):=\frac{1}{n}\sum_{i=1}^{n} \log\!\left(1+e^{-q_i(\theta)}\right). \label{eq:nonlinear-loss}
\end{align}
We further define
\begin{align}
     G(\theta):=\frac{1}{n}\sum_{i=1}^{n} \frac{1}{1+e^{q_i(\theta)}}, \qquad s_i(\theta):= \frac{1}{1 + \exp(q_i(\theta))}, \, i \in [n]. 
\end{align}
Denote $G_t := G(\theta_t)$ and $g_t := \nabla L(\theta_t)$.  The quantity $s_i(\theta)$ is the magnitude of the derivative of the logistic loss with respect to the signed logit $q_i(\theta)$. Therefore, $G_t$ is the average of this loss-derivative magnitude across samples. However, it does not track the actual parameter gradient scale. The actual parameter gradient is $g_t
=
-\frac{1}{n}
\sum_{i=1}^{n}
y_i s_i(\theta_t)
\nabla_{\theta} f_{\theta_t}(x_i)$, which involves the changing Jacobians $\nabla_{\theta} f_{\theta_t}(x_i)$. Consequently, unlike the linear setting, $G_t / \epsilon$ is only an approximate order parameter in the neural-network experiments. For Adam's hyperparameters, we set $(\beta_1,\beta_2)=(0.9,0.99)$ and use a constant stepsize $\eta=3\times10^{-3}$. 
We consider $\eps\in\{10^{-4},10^{-8},10^{-12}\}$ and run each trajectory for 100,000 updates, recording 480 logarithmically spaced checkpoints. 


For geometric references, we run paired GD and sign descent (SignGD) trajectories for at most 20,000 iterations using the same data, initialization, and stepsize. Because GD's update magnitude shrinks with the gradient, its late trajectory is stable. We define its reference margin as the median over the final $10\%$ of saved checkpoints. In contrast, because SignGD's update magnitude remains nearly constant, it can continue to move even after the loss becomes very small. We stop it at the first checkpoint at which the training loss falls below $10^{-7}$ and use its margin at that iterate. Because exact max-margin computation is intractable, we use the GD reference margin as the empirical proxy for $\ell_2$-max-margin.
We denote the reference margins of SignGD and GD by $\bar{\gamma}_{2,s}^{\mathrm{Sign}}$ and  $\bar{\gamma}_{2,s}^{\mathrm{GD}}$ for each seed $s$. The average values across seeds and standard deviations are: $\bar{\gamma}_{2}^{\mathrm{Sign}} = 0.026472 \pm 0.014431$ and $\bar{\gamma}_{2}^{\mathrm{GD}} = 0.507090 \pm 0.038898$. For every seed $s$, the value of  $\bar{\gamma}_{2,s}^{\mathrm{GD}} - \bar{\gamma}_{2,s}^{\mathrm{Sign}}$ is positive. 

\paragraph{Update and margin geometry.} After flattening $W$ and concatenating it with $a$, we measure the sign and gradient residuals \begin{align} r_t^{\mathrm{sgn}} := \frac{ \left\|g_t\odot \bigl(D_t-\operatorname{sign}(g_t)\bigr)\right\|_1 }{ \|g_t\|_1 }, \qquad
 r_t^{\mathrm{grad}} := \frac{\|\eps D_t-g_t\|_2}{\|g_t\|_2}. \label{eq:nonlinear-gradient-residual} 
\end{align}

A small $r_t^{\mathrm{sgn}}$ indicates a sign-like direction, whereas a small $r_t^{\mathrm{grad}}$ indicates $D_t\approx g_t/\eps$ and hence a gradient-like direction. We also compute the gradient residual separately on the $W$ and $a$ blocks, denoted as $r^{\rm grad}_{t,W}$ and $r^{\rm grad}_{t,a}$, respectively.  
The raw training margin is $\rho(\theta):=\min_{i\in[n]}y_i f_\theta(x_i)$. Because the network is two-homogeneous under the joint scaling $\theta\mapsto c\theta$, we use the scale-invariant margin: 
\[ \widehat{\gamma}_{2}(\theta) := \frac{\rho(\theta)}{\|\theta\|_2^2}, \qquad \|\theta\|_2^2 := \|W\|_F^2+\|a\|_2^2. \] This normalized margin is invariant to joint scaling but not to the function-preserving layer rebalancing $(W,a)\mapsto(cW,a/c)$. 
We define the following quantity to measure the classifier’s geometric transition:
\begin{equation} 
R_{2,s}(t) := \frac{ \bar{\gamma}_{2,s}^{\mathrm{GD}} - \widehat{\gamma}_{2,s}^{\mathrm{Adam}}(t) }{ \bar{\gamma}_{2,s}^{\mathrm{GD}} - \bar{\gamma}_{2,s}^{\mathrm{Sign}} }.
\label{eq:nonlinear-defs} 
\end{equation} 
It is equal to one when Adam's $\ell_2$-margin is replaced by the SignGD reference. For $R_2(t)$, the denominator is the total empirical Euclidean-margin gap between GD and SignGD, while the numerator is the part of this gap that Adam has not yet closed. Therefore, $R_2(t)$ is the fraction of the empirical SignGD-to-GD Euclidean-margin gap that remains, while $1-R_2(t)$ measures the progress towards the GD reference. 
Moreover, we can rewrite \eqref{eq:nonlinear-defs} as
\begin{align*}
\widehat{\gamma}_{2,s}^{\mathrm{Adam}}(t)
=
R_{2,s}(t)\overline{\gamma}_{2,s}^{\mathrm{Sign}}
+
\bigl(1-R_{2,s}(t)\bigr)
\overline{\gamma}_{2,s}^{\mathrm{GD}}.
\end{align*}
Therefore, when \(R_{2,s}(t)\in[0,1]\), it is an affine coefficient between the two corresponding scalar reference margins. 

\paragraph{Empirical transition times.}
We define a few time points that mark the geometric transitions of the update and the classifier. They are 
\begin{align*}
T_*
&:=
\text{first of five consecutive checkpoints with $100\%$ training accuracy},\\
T_G^{0.1}
&:=
\text{first of five consecutive checkpoints at or after $T_*$ with }
\frac{G_t}{\eps}\leq 0.1,\\
T_{\mathrm{upd}}^{0.30}
&:=
\text{first of five consecutive checkpoints at or after $T_*$ with }
\max\left\{
r_{t,W}^{\mathrm{grad}},
r_{t,a}^{\mathrm{grad}}
\right\}\leq 0.30,\\
\tau_{0.5}^{\mathrm{emp}}
&:=
\text{first checkpoint at or after $T_{\mathrm{upd}}^{0.30}$ from which }
R_2(t)\leq 0.5
\text{ holds at every remaining checkpoint},
\end{align*}
where the last definition requires at least five remaining checkpoints. The threshold \(0.30\) in $T_{\rm upd}^{0.30}$ means that in each layer, the mismatch between \( \epsilon D_t\) and \(g_t\) is at most \(30\%\) of the gradient norm. It signals that the update has entered a gradient-like regime where $\epsilon D_t \approx g_t$. The half-gap condition $R_2(t) \leq 0.5$ in the definition of $\tau_{0.5}^{\rm emp}$ is equivalent to $\hat{\gamma}_2^{\rm Adam} (t) \geq \frac{\bar{\gamma}_2^{\rm GD} + \bar{\gamma}_2^{\rm Sign}}{2}$. Thus, at \(\tau_{0.5}^{\mathrm{emp}}\), Adam has traversed at least half of the empirical Euclidean-margin interval from SignGD to GD. Moreover, this time is restricted to $t \geq T_{\rm upd}^{0.30}$. The experiment asks how long the classifier takes to change after the local update has already become gradient-like. It can happen that $R_2(t) > 0.5$ throughout the trajectory and $\tau_{0.5}^{\mathrm{emp}}$ is not reached by the end of training.   


\paragraph{Results.} 
\cref{fig:app_nonlinear1,fig:app_nonlinear2,fig:app_nonlinear3} show the results for $\eps = 10^{-4}$, $\eps = 10^{-8}$, and $\eps = 10^{-12}$, respectively. Each curve shows the mean over the eight seeds, and shaded bands show one sample standard deviation. The gray dashed, blue dotted, and green dashed vertical lines in each figure mark the medians of $T_*$, $T_G^{0.1}$, and $T_{\mathrm{upd}}^{0.30}$, respectively. Their values together with the first and third quartiles are reported in the following table for the three choices of $\eps$. We first note that the training accuracy saturates very early: the median values of $T_*$ are 8, 5, and 5, while the logistic loss continues to decrease by many orders
of magnitude. Thus, predictive accuracy alone does not reveal either geometric transition. 

\begin{table*}[h!]
    \centering
    \caption{ Event times $T_*$, $T_G^{0.1}$, $T_{\mathrm{upd}}^{0.30}$, and $\tau_{0.5}^{\mathrm{emp}}$ for $\eps \in \{10^{-4}, 10^{-8}, 10^{-12}\}$, reported as median \([Q_1,Q_3]\) over eight seeds and rounded to the
    nearest iteration, where \(Q_1\) and \(Q_3\) are the first and third quartiles.
    }
    \label{tab:nonlinear-event-times}
    \resizebox{\textwidth}{!}{
    \begin{tabular}{cccccc}
        \toprule
        $\eps$
        & $T_*$
        & $T_G^{0.1}$
        & $T_{\mathrm{upd}}^{0.30}$
        & $\tau_{0.5}^{\mathrm{emp}}$ \\
        \midrule
        $10^{-4}$
        & $8\,[4,8]$
        & $592\,[588,610]$
        & $1{,}106\,[1{,}100,1{,}106]$
        & $64{,}821\,[56{,}686,73{,}648]\;(3/8\ \mathrm{confirmed})$ \\

        $10^{-8}$
        & $5\,[2,8]$
        & $2{,}120\,[2{,}070,2{,}172]$
        & $2{,}970\,[2{,}952,2{,}970]$
        & Not observed \\

        $10^{-12}$
        & $5\,[2,8]$
        & $3{,}779\,[3{,}756,3{,}871]$
        & $4{,}808\,[4{,}808,4{,}808]$
        & Not observed \\
        \bottomrule
    \end{tabular}
    }
\end{table*}

For $\eps \in \{10^{-4}, 10^{-8}, 10^{-12} \}$, the corresponding medians of $T_{\rm upd}^{0.3}$ are 1106, 2970, and 4808. Therefore, the geometric transition is delayed as $\eps$ decreases. While the median sign residual is approximately one, the median gradient residual at $T_{\mathrm{upd}}^{0.30}$ decreases from \(0.280\) to \(0.241\) and then to \(0.212\). Thus, the updates in all three cases have entered the gradient-like regime and are no longer sign-like.   

The classifier transition is substantially slower. Every trajectory satisfies
$R_2(T_{\mathrm{upd}}^{0.30})>0.5$, with median values $0.677$, $0.765$, $0.818$  for $\eps = 10^{-4},10^{-8}$ and $10^{-12}$, respectively. At the last iteration, the corresponding median values remain $0.500$, $0.700$, and $0.776$. The defining event of 
$\tau_{0.5}^{\mathrm{emp}}$ is confirmed in only three of eight seeds for
$\eps=10^{-4}$ and in none of the seeds for $\eps=10^{-8}$ or $10^{-12}$.
Therefore, the transition toward Euclidean optimality occurs much later than the update's geometric transition. Moreover, decreasing $\eps$ leaves a larger empirical Euclidean-margin gap in the gradient-like regime.

 \begin{figure*}[h!]
    \centering

    \begin{subfigure}[t]{0.333\textwidth}
      \centering
      \includegraphics[width=\linewidth]{
        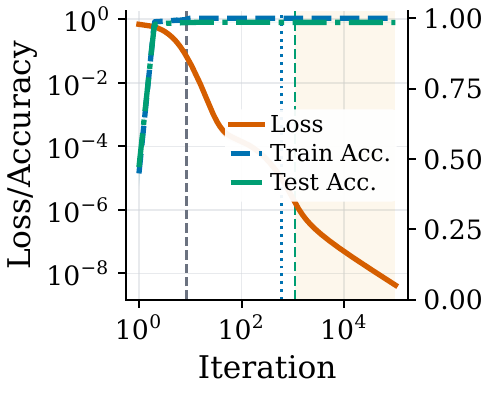
      }
      \caption{Loss and accuracy.}
      \label{}
    \end{subfigure}\hfill%
    \begin{subfigure}[t]{0.333\textwidth}
      \centering
      \includegraphics[width=\linewidth]{               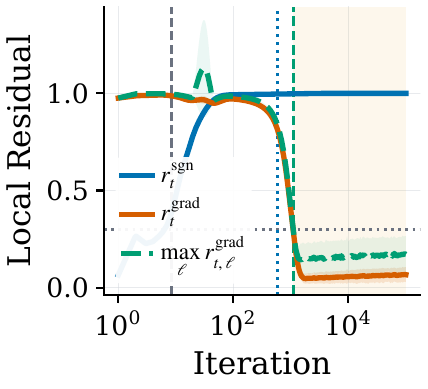
      }
      \caption{Update residuals.}
      \label{}
    \end{subfigure}\hfill%
    \begin{subfigure}[t]{0.333\textwidth}
      \centering
      \includegraphics[width=\linewidth]{
        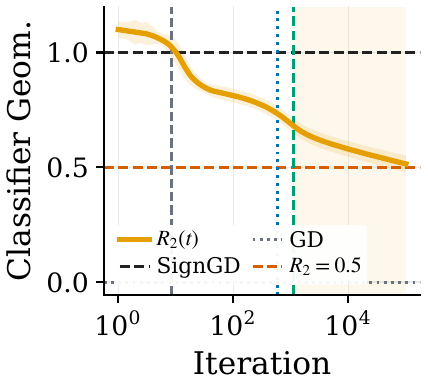
      }
      \caption{Margin geometry.}
      \label{}
    \end{subfigure}\hfill%

    \caption{Nonlinear experiment with $\eps=10^{-4}$. Curves show the mean over eight paired seeds, and shaded bands show one sample standard deviation. (a) Logistic loss and training and test accuracy against iterations. (b) The sign and gradient residuals and the maximum layerwise gradient residual against iterations. (c) $R_2(t)$ against iterations; the SignGD and GD references are one and zero, respectively, and $R_2=0.5$ marks half of the empirical Euclidean-margin gap.     
    }
    \label{fig:app_nonlinear1}
\end{figure*}

 \begin{figure*}[h!]
    \centering

    \begin{subfigure}[t]{0.333\textwidth}
      \centering
      \includegraphics[width=\linewidth]{
        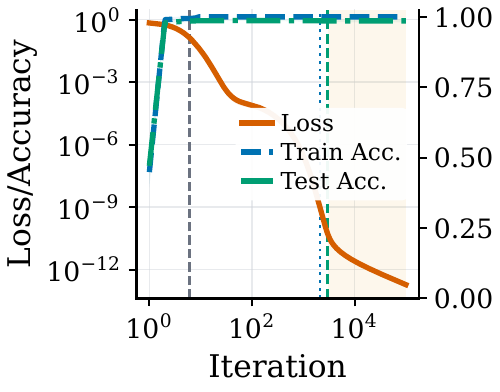
      }
      \caption{Loss and accuracy.}
      \label{}
    \end{subfigure}\hfill%
    \begin{subfigure}[t]{0.333\textwidth}
      \centering
      \includegraphics[width=\linewidth]{               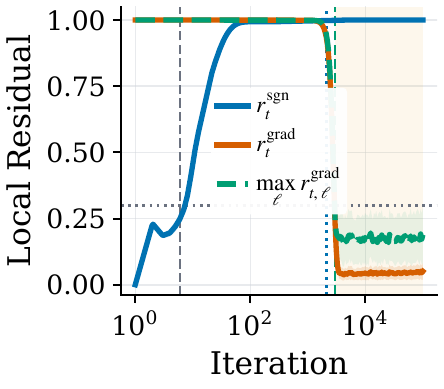
      }
      \caption{Update residuals.}
      \label{}
    \end{subfigure}\hfill%
    \begin{subfigure}[t]{0.333\textwidth}
      \centering
      \includegraphics[width=\linewidth]{
        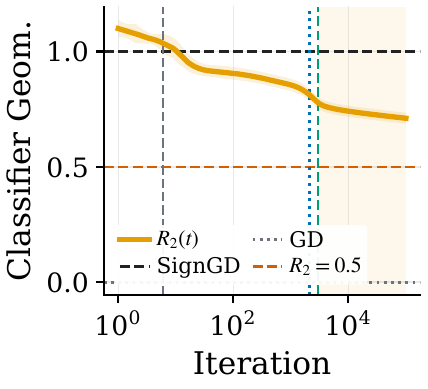
      }
      \caption{Margin geometry.}
      \label{}
    \end{subfigure}\hfill%
    \caption{
    Nonlinear experiment with $\eps=10^{-8}$. The plottings are the same as in Figure~\ref{fig:app_nonlinear1}.     
    }
    \label{fig:app_nonlinear2}
\end{figure*}

 \begin{figure*}[h!]
    \centering

    \begin{subfigure}[t]{0.333\textwidth}
      \centering
      \includegraphics[width=\linewidth]{
        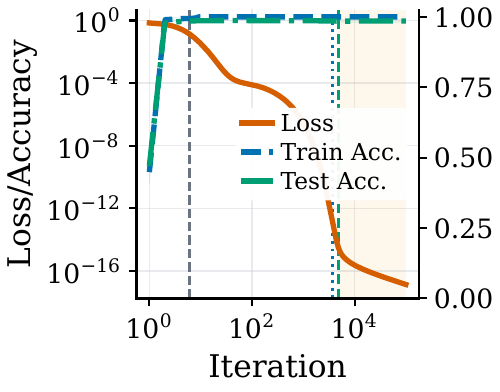
      }
      \caption{Loss and accuracy.}
      \label{}
    \end{subfigure}\hfill%
    \begin{subfigure}[t]{0.333\textwidth}
      \centering
      \includegraphics[width=\linewidth]{               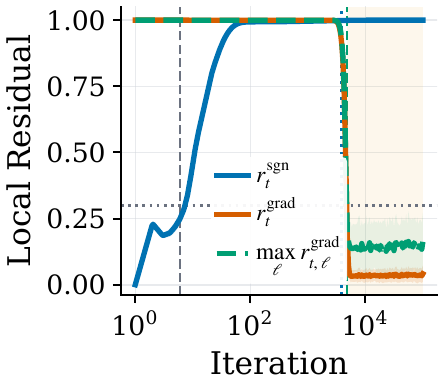
      }
      \caption{Update residuals.}
      \label{}
    \end{subfigure}\hfill%
    \begin{subfigure}[t]{0.333\textwidth}
      \centering
      \includegraphics[width=\linewidth]{
        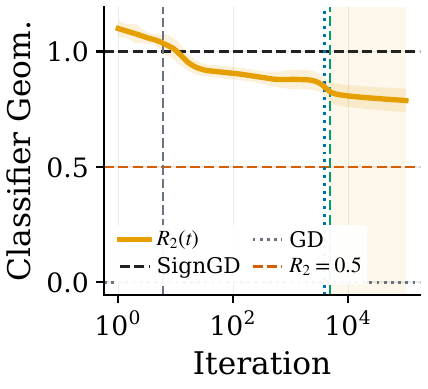
      }
      \caption{Margin geometry.}
      \label{}
    \end{subfigure}\hfill%
    \caption{
    Nonlinear experiment with $\eps=10^{-12}$. The plottings are the same as in Figure~\ref{fig:app_nonlinear1}.
    }
    \label{fig:app_nonlinear3}
\end{figure*}
\section{Preliminaries}
\label{app:prelim}

Our objective is 
$L(w) = \frac{1}{n} \sum_{i=1}^n \ell(z_i^T w)$ where $\ell(u) := \log(1+e^{-u})$.
Recall that 
\begin{align*}
    a_i(w):= -l'(z_i^T w) = \frac{1}{1 + e^{z_i^T w}}, \quad G(w):= \frac{1}{n} \sum_{i=1}^n a_i(w), \quad g(w) := \nabla L(w) = -\frac{1}{n} \sum_{i=1}^n a_i(w) z_i.
\end{align*}
For simplicity, we let $G_t := G(w_t)$, $L_t := L(w_t)$, and $g_t := g(w_t)$. We use the decreasing step size schedule $\eta_t = \eta_0 (1+t)^{-a}$ where $\eta_0 > 0$ and $a \in (1/3,1]$. We assume that the data is separable and consequently $\gamma_{\infty}, \gamma_2 > 0$. We further assume that $\beta_1$, $\beta_2$, and $\eps$ are chosen such that $0 \leq \beta_1 \leq \beta_2 < 1$ and $\eps \in (0,1]$. Recall that $R_1  :=\max_{i \in [n]} \norm{x_i}_1$ and $R_2  :=\max_{i \in [n]} \norm{x_i}_2$. 

\subsection{The loss and the proxy} 
\label{app:sub_prelim}
The lower bound on the $\ell_1$-norm of $g(w)$ was proved in \citet[Lemma~A.1]{zhang2024adam}. We give an extension to the $\ell_2$-norm case following the same approach. 
\begin{lemma}\label{lem:proxy}
It holds for every $w \in \R^d$ that 
\begin{align}
    \gamma_\infty G(w)\le\norm{g(w)}_1\le R_1G(w),
\qquad
\gamma_2G(w)\le\norm{g(w)}_2\le R_2G(w). \label{eq:proxy-gradient}
\end{align}
In particular, we have $G(w)>0$ and $g(w)\neq0$ at every finite $w$.
\end{lemma}

\begin{proof} Recall that $z_i := x_i y_i$. We let $\phi(u) := \min_{1 \leq i \leq n} z_i^T u$. This function is continuous because it is the pointwise minimum of finitely many continuous functions. Hence, the maxima are obtained when maximizing  $\phi(u)$ over the compact feasible sets $\{ u: \norm{u}_{\infty} \leq 1\}$ and $\{u: \norm{u}_2 \leq 1\}$. We denote the maximizers by $u_{\infty}$ and $u_{2}$ respectively and they satisfy: 
\begin{align*}
    \norm{u_{\infty}}_{\infty} \leq 1, \min_{i}z_i^T u_{\infty} = \gamma_{\infty} \quad \text{and} \quad \norm{u_{2}}_{2} \leq 1, \min_{i}z_i^T u_{2} = \gamma_{2}.
\end{align*}
Consequently, we have $z_{i}^T u_{\infty} \geq \gamma_{\infty}, \forall i$ and $z_i^T u_2 \geq \gamma_2, \forall i$ (note that $\gamma_{\infty}$ and $\gamma_2$ are positive given the data is separable). This implies the following
\begin{align*}
    -\langle g(w), u_{\infty} \rangle = \frac{1}{n} \sum_{i=1}^n a_i(w) z_i^T u_{\infty} \geq  \frac{1}{n} \sum_{i=1}^n a_i(w) \gamma_{\infty} = \gamma_{\infty} G(w). 
\end{align*}
Consequently, we have $\gamma_{\infty} G(w) \leq -\langle g(w), u_{\infty} \rangle \leq |\langle g(w), u_{\infty} \rangle| \leq \norm{g(w)}_1 \norm{u_{\infty}}_{\infty} \leq \norm{g(w)}_1$. Similarly for the Euclidean norm, it holds that
\begin{align*}
        -\langle g(w), u_{2} \rangle = \frac{1}{n} \sum_{i=1}^n a_i(w) z_i^T u_{2} \geq  \frac{1}{n} \sum_{i=1}^n a_i(w) \gamma_{2} = \gamma_{2} G(w). 
\end{align*}
Consequently, we have $\gamma_{2} G(w) \leq -\langle g(w), u_{2} \rangle \leq |\langle g(w), u_{2} \rangle| \leq \norm{g(w)}_2 \norm{u_{2}}_{2} \leq \norm{g(w)}_2$. For the upper bounds, we have via triangle inequality that
\begin{align*}
    \norm{g(w)}_2 = \norm{-\frac{1}{n} \sum_{i=1}^n a_i(w) z_i}_2 \leq \frac{1}{n} \sum_{i=1}^n a_i(w) \norm{z_i}_2 \leq R_2 G(w), \\
    \norm{g(w)}_{1} = \norm{-\frac{1}{n} \sum_{i=1}^n a_i(w) z_i}_{1} \leq \frac{1}{n} \sum_{i=1}^n a_i(w) \norm{z_i}_{1} \leq R_1 G(w).
 \end{align*}
Finally, since $e^{z_i^T w} \in (0, \infty)$, it holds that $a_i(w) = \frac{1}{1 + e^{z_i^T w}} \in (0,1)$ and consequently $G(w) = \frac{1}{n} \sum_{i=1}^n a_{i}(w) > 0$, and the conclusion $g(w) \neq 0$ for every finite $w$ follows as $\norm{g(w)}_1 > 0$ ($\norm{g(w)}_2 > 0$). 
\end{proof}

The lower bound on $\frac{G(w)}{L(w)}$ and the conclusion $\rho(w) \geq 0$ are proved in \citet[Lemma~C.4,C.7]{zhang2024adam}. We provide the proof here for completeness.

\begin{lemma}
\label{lem:loss-proxy}
For every $w \in \R^d$, we have
\begin{equation}
G(w)\le L(w),
\qquad
\frac{G(w)}{L(w)}\ge1-\frac{nL(w)}2.
\label{eq:G-L-sandwich}
\end{equation}
If either $G(w)\le1/(2n)$ or $L(w)\le\log2/n$ holds, then  we have $L(w)\le2G(w)$. Define $\rho(w):= \min_{i \in [n]} z_i^T w$. The inequality $L(w)\le\log2/n$ implies that
\begin{equation}
\rho(w)\ge0 \quad \text{and} \quad 
L(w)\ge\frac{\log2}{n}e^{-\rho(w)}.
\label{eq:loss-margin}
\end{equation}
\end{lemma}

\begin{proof}
Let $\ell_i :=\ell(z_i^\top w)$. Since
$e^{-\ell_i}=1/(1+e^{-z_i^\top w})$, it holds that
\[
a_i(w)=\frac1{1+e^{z_i^\top w}}=1-e^{-\ell_i}.
\]
For $x\ge0$, we have
$1-e^{-x}\le x$ and $1-e^{-x}\ge x-x^2/2$. Averaging these inequalities gives
\[
G(w)\le L(w),
\qquad
G(w)\ge L(w)-\frac1{2n}\sum_i\ell_i^2 \stackrel{(a)}{\geq} L(w)-\frac{n L(w)^2}{2},
\]
where (a) follows from 
$\sum_i\ell_i^2\le(\sum_i\ell_i)^2=n^2L(w)^2$. Rearranging leads to the desired. 

If $G(w)\le1/(2n)$, we have $a_i(w)\le nG(w)\le1/2, \forall i$. Since
$\ell_i=-\log(1-a_i)$ and $-\log(1-x)\le2x$ on $[0,1/2]$, we obtain
\[
L(w)=\frac1n\sum_i\ell_i = \frac1n\sum_i -\log(1 - a_i(w)) \le\frac2n\sum_i a_i(w)=2G(w).
\]
If instead $L(w)\le\log2/n$, then each individual loss satisfies
$\ell_i\le nL(w)\le\log2$, hence $a_i=1-e^{-\ell_i}\le1/2$, and the same argument applies.

Given $L(w) \leq \frac{\log 2}{n}$, we have $\ell_i \leq \log 2, \forall i$. The function $\ell(u) = \log(1 + e^{-u})$ is strictly decreasing as $\ell'(u) = -\frac{1}{1+e^u} < 0$. Thus, given $z_i^T w < 0$, we have that $\ell_i = \ell(z_i^T w) > l(0) = \log 2$, which is a contradiction. Since $z_i^T w \geq 0, \forall i$, we conclude that $\rho(w) \geq 0$ by taking the minimum over $i$. 

Let $i^*$ be the index such that $\rho(w) = z^T_{i^*}$. We then have
\begin{align}
    L(w) = \frac{1}{n} \sum_{i=1}^n \log (1 + e^{-z_i^T w}) \geq \frac{1}{n} \log (1 + e^{-z_{i^*}^T w}) = \frac{1}{n} \log (1 + e^{-\rho(w)}). \label{eq:temp}
\end{align}
Define $x:= e^{-\rho(w)}$. We have $0 < x \leq 1$ given $\rho(w) \geq 0$. The function $f(x) = \log(1+x)$ is concave on $[0,1]$ because $f''(x) = - \frac{1}{(1+x)^2} < 0$. By concavity, we have
\begin{align*}
    \log(1+x) = f(x) = f((1-x)0 + 1x) \geq (1-x)f(0) + xf(1) = x \log2.
\end{align*}
Substituting $x = e^{-\rho_(w)}$, we obtain $\log (1 + e^{-\rho(w)}) \geq (\log 2) e^{-\rho(w)}$. Combining this with \eqref{eq:temp} finishes the proof. 
\end{proof}

The result in the following proposition was shown to hold for the step size schedule $\eta_t = (t+2)^{-a}$ in \citet[Lemma~C.1.]{zhang2024adam}. We provide a proof here for completeness and show it holds for the schedule $\eta_t = \eta_0 (t+1)^{-a}$. 
\begin{proposition}
\label{prop:ema-sum}
Let $\eta_t=\eta_0(t+1)^{-a}$ with $a\in(0,1]$ and $\eta_0 > 0$. For every $q\in[0,1)$ and $c>0$, there exist a time $t_{0} = t_0(q,c,a,\eta_0)$ and a constant $C_{\rm geo}=C_{\rm geo}(q,c,a,\eta_0)$ such that
\begin{equation}
\sum_{r=0}^t q^r
\left[
\exp\!\left(c\sum_{s=t-r}^{t-1}\eta_s\right)-1
\right]
\le C_{\rm geo} \eta_t,
\qquad t\ge t_{0}.
\label{eq:ema-sum}
\end{equation} 
\end{proposition}

\begin{proof} If \(q=0\), then the sum is zero. In this case, we can take \(t_0=0\) and \(C_{\mathrm{geo}}=0\). Henceforth, we assume \(q\in(0,1)\). We first prove 
\begin{align}
    \sum_{r=0}^{\infty} q^r r^k \leq \frac{k ! }{(1-q)^{k+1}}, \label{eq:use_q}
\end{align}
where $k$ is an integer with $k \geq 1$. For every integer $r \geq 0$, it holds that $r^k \leq (r+1)(r+2)\cdots (r+k) = k!\binom{r+k}{k} $. Therefore, we have $\sum_{r=0}^{\infty} q^r r^k
\le
k!\sum_{r=0}^{\infty} \binom{r+k}{k} q^r$. To evaluate the last series, we can differentiate $\sum_{n=0}^{\infty} q^n = \frac{1}{1-q}$ exactly $k$ times. Since $0 < q < 1$, the power series can be differentiated term by term and we obtain $\sum_{n=k}^{\infty}\frac{n!}{(n-k)!}\,q^{n-k}
= \frac{k!}{(1-q)^{k+1}}$. Setting $r = n-k$ and dividing by $k!$, it holds that $\sum_{r=0}^{\infty} \binom{r+k}{k} q^r = \frac{1}{(1-q)^{k+1}}$. Substituting this into the previous inequality finishes the proof of \eqref{eq:use_q}. 

Define $H_{t,r} := \sum_{s = t - r}^{t-1} \eta_s, 0 \leq r \leq t$ and $\mathcal E_t = \sum_{r=0}^t q^r (e^{c H_{t,r}} - 1)$. We derive a bound on $H_{t,r}$. Changing variables to $j = s+1$, we can write $H_{t,r}
= \eta_0\sum_{j=t-r+1}^{t}j^{-a}$. For every integer $j \geq 1$, it holds that
\begin{align*}
    j^{-a}=
\left(\frac{j+1}{j}\right)^a (j+1)^{-a} \le
2^a (j+1)^{-a}.
\end{align*}
Moreover, because $x^{-a}$ is decreasing, we have $(j+1)^{-a}
\le \int_j^{j+1} x^{-a}\,dx$. Combining these two inequalities leads to $H_{t,r} \le
2^a\eta_0\int_{t-r+1}^{t+1}x^{-a}\,dx,
\, 0\le r\le t$. Summing over $j=t-r+1,...,t$, we have
\begin{align}
    H_{t,r}
\le
2^a\eta_0\int_{t-r+1}^{t+1}x^{-a}\,dx,
\qquad 0\le r\le t. \label{eq:H_bound}
\end{align}

We first consider the case where $0< a <1$. Evaluating the integral in \eqref{eq:H_bound}, we obtain
\begin{align}
    H_{t,r}
\le
\frac{2^a\eta_0}{1-a}
\left[
(t+1)^{1-a}-(t-r+1)^{1-a}
\right]. \label{eq:H_bound2}
\end{align}
Define $u:= \frac{t - r+1}{t+1} \in (0,1]$. Since $0 < 1 - a < 1$, it holds that $u^{1 -a} \geq u$. Consequently, we have
\begin{align*}
    (t+1)^{1-a}-(t-r+1)^{1-a}
&= (t+1)^{1-a}\left(1-u^{1-a}\right) \le (t+1)^{1-a}(1-u)\\
&= (t+1)^{1-a}\frac{r}{t+1} = \frac{r}{(t+1)^a}.
\end{align*}
Substituting this into \eqref{eq:H_bound2} gives $H_{t,r} \le \frac{2^a\eta_0}{1-a}\frac{r}{(t+1)^a}
= \frac{2^a}{1-a}\,r\eta_t$. Define $d_t := \frac{2^a c}{1 - a} \eta_t$. We have $c H_{t,r} \leq d_t r$ and consequently $\mathcal E_t \leq \sum_{r = 0}^t q^r (e^{d_t r} - 1)$. 

Using $e^{d_t r} - 1 = \sum_{k=1}^{\infty} \frac{d_t^k r^k}{k !}$, we obtain
\begin{align*}
    \mathcal E_t
&\le
\sum_{r=0}^{t} q^r
\sum_{k=1}^{\infty} \frac{d_t^k r^k}{k!} =
\sum_{k=1}^{\infty} \frac{d_t^k}{k!}
\sum_{r=0}^{t} q^r r^k \le
\sum_{k=1}^{\infty} \frac{d_t^k}{k!}
\sum_{r=0}^{\infty} q^r r^k.
\end{align*}
The exchange of the sums is justified because all terms are nonnegative. Applying \eqref{eq:use_q}, we further obtain
\begin{align*}
    \mathcal E_t \le
\sum_{k=1}^{\infty}
\frac{d_t^k}{k!}
\frac{k!}{(1-q)^{k+1}} =
\frac{1}{1-q}
\sum_{k=1}^{\infty}
\left(\frac{d_t}{1-q}\right)^k.
\end{align*}
Next, we choose $t_0$ to be 
\begin{align}
    t_0
:=
\left\lceil
\left(
\frac{2^{a+1}c\eta_0}{(1-a)(1-q)}
\right)^{1/a}
\right\rceil. \label{eq:t_01}
\end{align}
Then, for every $t \geq t_0$, it holds that $(t+1)^a
\ge
\frac{2^{a+1}c\eta_0}{(1-a)(1-q)}$, or equivalently, $\frac{d_t}{1-q} \leq \frac{1}{2}$.
For any $x \in [0,1/2]$, we have $\sum_{k=1}^{\infty} x^k = \frac{x}{1 - x} \leq 2x$. We apply this to $\frac{d_t}{1-q}$ to obtain $\mathcal E_t
\le
\frac{1}{1-q}\frac{2d_t}{1-q}
=
\frac{2d_t}{(1-q)^2}.$
Substituting the definition of $d_t$, we obtain
\begin{align*}
    \mathcal E_t  \leq \frac{2^{a+1} c}{(1-a)(1-q)^2} \eta_t.
\end{align*}
Consequently, for $0 < a < 1$, we can take $C_{\rm geo} = \frac{2^{a+1} c}{(1-a)(1-q)^2}$ and $t_0$ to be the one in \eqref{eq:t_01}. 

Next, we consider the case where $a = 1$. In this case, we have $\eta_t = \frac{\eta_0}{t + 1}$. From \eqref{eq:H_bound}, it holds that 
\begin{align*}
    H_{t,r} \le
2\eta_0\int_{t-r+1}^{t+1}\frac{dx}{x} =
2\eta_0\log\left(\frac{t+1}{t-r+1}\right) =
2\eta_0\log\left(1+\frac{r}{t-r+1}\right).
\end{align*}
Set $m :=\lceil 2 c \eta_0 \rceil$. We then obtain $e^{cH_{t,r}} \le
\left(1+\frac{r}{t-r+1}\right)^{2c\eta_0} \le
\left(1+\frac{r}{t-r+1}\right)^m$. From this, applying the binomial theorem leads to
\begin{align*}
    e^{cH_{t,r}}-1
\le
\sum_{k=1}^{m}
\binom{m}{k}
\left(\frac{r}{t-r+1}\right)^k.
\end{align*}
Multiplying by $q^r$ and summing gives 
\begin{align}
    \mathcal{E}_{t} \leq \sum_{k=1}^{m}
\binom{m}{k} U_{t,k}, \qquad U_{t,k}:=
\sum_{r=0}^{t}
q^r\left(\frac{r}{t-r+1}\right)^k. \label{eq:uk}
\end{align} 
We show each $U_{t,k}$ where $1 \leq k \leq m$ is bounded by a constant multiple of $1/t$. For $t \geq 1$, we decompose $U_{t,k}$ into $U_{t,k} = U^{\rm s}_{t,k} + U^{\rm l}_{t,k}$ where 
\begin{align*}
    U_{t,k}^{\mathrm{s}}
:=
\sum_{r=0}^{\lfloor t/2\rfloor}
q^r\left(\frac{r}{t-r+1}\right)^k, \qquad U_{t,k}^{\mathrm{l}}
:=
\sum_{r=\lfloor t/2\rfloor+1}^{t}
q^r\left(\frac{r}{t-r+1}\right)^k. 
\end{align*}
If $r \leq \lfloor t/2 \rfloor$, then it holds that $t  - r + 1 \geq \frac{t}{2}$ and consequently $\frac{r}{t-r+1} \leq \frac{2r}{t}$. Given this, we obtain for $U_{t,k}^{\mathrm{s}}$ that
\begin{align*}
    U_{t,k}^{\mathrm{s}} \le
\left(\frac{2}{t}\right)^k
\sum_{r=0}^{\lfloor t/2\rfloor} q^r r^k \le
\left(\frac{2}{t}\right)^k
\sum_{r=0}^{\infty} q^r r^k \le
\frac{2^k k!}{(1-q)^{k+1}}\,t^{-k}. 
\end{align*}
For $t \geq 1$ and $k \geq 1$, it holds that $t^{-k} \leq t^{-1}$. Therefore, we conclude that
\begin{align}
    U_{t,k}^{\mathrm{s}}
\le
\frac{2^k k!}{(1-q)^{k+1}}\,\frac{1}{t}, \qquad \text{when} \, t \geq 1, \, 1 \leq k \leq m. \label{eq:us}
\end{align}
For $r > \lfloor t/2 \rfloor$, we have $r > t /2$ and consequently $q^r \leq q^{t/2}$. Moreover, given $t - r + 1 \geq 1$, it holds that $(\frac{r}{t -r + 1})^k \leq r^k \leq t^k$. Since there are at most $t$ terms in the sum $U_{t,k}^{\mathrm{l}}$, we have $U_{t,k}^{\mathrm{l}} \leq q^{t/2} t^{k+1}$. To make this at most $1/t$ for every $1 \leq k \leq m$, it is sufficient to require $q^{t/2} t^{m+2} \leq 1$. Let $\lambda := - \log q > 0$. This condition is equivalent to $(m+2) \log t \leq \frac{\lambda t}{2}$, which is guaranteed whenever $(m+2) \sqrt{t} \leq \frac{\lambda t }{2}$ given $\log t \leq \sqrt{t}$ for $t \geq 1$. Therefore, we choose $t_0$ in this case to be
\begin{align}
    t_0
:=
\left\lceil
\left(
\frac{2(m+2)}{-\log q}
\right)^2
\right\rceil, \qquad m = \lceil 2 c \eta_0 \rceil. \label{eq:t01}
\end{align}
For every $t \geq t_0$, it holds that $q^{t/2} t^{m+2} \leq 1$. Therefore, we obtain for $1 \leq k \leq m$ that
\begin{align}
    U_{t,k}^{\mathrm{l}} \le q^{t/2}t^{k+1} = \bigl(q^{t/2}t^{m+2}\bigr)t^{k-m-1} \le t^{k-m-1} \le \frac{1}{t}. \label{eq:ul}
\end{align}
Combining \eqref{eq:us} and \eqref{eq:ul}, we obtain 
\begin{align*}
    U_{t,k} \leq [\frac{2^k k!}{(1-q)^{k+1}}+1] \frac{1}{t}, \qquad t \geq t_0.
\end{align*}
Substituting this into \eqref{eq:uk}, we obtain
\begin{align*}
    \mathcal{E}_t \leq \frac{1}{t} \sum_{k=1}^m \binom{m}{k} [\frac{2^k k!}{(1 - q)^{k+1}} + 1], \qquad t \geq t_0.  
\end{align*}
Define $A_m(q) := \sum_{k=1}^m \binom{m}{k} [\frac{2^k k!}{(1 - q)^{k+1}} + 1]$. Since $t \geq t_0 \geq 1$, we have $\frac{1}{t} \leq \frac{2}{t+1} = \frac{2}{\eta_0} \eta_t$. Therefore, we conclude that $\mathcal{E}_t\leq \frac{2 A_m(q)}{\eta_0} \eta_t$. For $a = 1$, we can take $C_{\rm geo}$ to be 
\begin{align*}
    C_{\rm geo}
=
\frac{2m}{\eta_0}
\max_{1\le k\le m}\binom{m}{k}
\left[
\frac{2^m m!}{(1-q)^{m+1}}+1
\right],
\end{align*}
and $t_0$ to be the one in \eqref{eq:t01}. 
\end{proof}

\begin{lemma}
\label{lem:derivative-ratio}
For any $u,v\in\R$, we have
\begin{equation}
e^{-|u-v|}
\le\frac{-\ell'(u)}{-\ell'(v)}
\le e^{|u-v|}.
\label{eq:derivative-ratio}
\end{equation}
Consequently, it holds that
$\left|\frac{-\ell'(u)}{-\ell'(v)}-1\right|
\le e^{|u-v|}-1$.

\end{lemma}

\begin{proof}
Define $\phi(x):= \log(-\ell'(x))$. Since $-\ell'(x) = \frac{1}{1+e^x}$, we have $\phi(x) = -\log(1 + e^x)$ and differentiating it leads to $\phi'(x) = - \frac{e^x}{1 + e^x}$. Because $0 < \frac{e^x}{1+e^x}<1$, we have $|\phi'(x)| \leq 1$. Therefore, the function $\phi$ is globally 1-Lipschitz. This implies that $|\phi(u) - \phi(v)| \leq |u - v|$. 
Substituting $\phi(x) = \log(-\ell'(x))$, we obtain
\begin{align*}
    |\log(-\ell'(u)) - \log(-\ell'(v))| \leq |u - v|. 
\end{align*}
This is equivalent to $-|u - v| \leq \log \frac{-\ell'(u)}{-\ell'(v)} \leq |u - v|$. Exponentiating gives $e^{-|u-v|}
\le\frac{-\ell'(u)}{-\ell'(v)}
\le e^{|u-v|}$.

Let $h := |u - v| \geq 0$ and $R:= \frac{-\ell'(u)}{-\ell'(v)} \geq 0$. We have already shown that $e^{-h} \leq R \leq e^h$. Next, we apply this to show $|R - 1| \leq e^h - 1$. To do this, we consider two cases where $R \geq 1$ and $R < 1$. In the former case, we have
\begin{align*}
    |R -1| = R - 1 \leq e^h - 1.
\end{align*}
In the latter case, we have
\begin{align*}
    |R - 1| = 1 - R \leq 1 - e^{-h} = \frac{e^h - 1}{e^h} \leq e^h - 1,
\end{align*}
where the last inequality is by $e^h \geq 1$. Combining both cases finishes the proof. 
\end{proof}

\subsection{Adam's Update and EMA} \label{app:adam_ema_binary}

In this section, we first state and prove a few results on Adam's update. They ultimately lead to the proof of the loss recursion \eqref{eq:master-descent-app} in \cref{prop:master-descent}. In the end, we prove \cref{prop:uniform-clock}. 
\begin{lemma}
\label{lem:uniform-D}For every $t \geq 0$ and $\eps > 0$, Adam's update in \eqref{eq:adam_update} satisfies
\begin{equation}
\norm{D_t}_\infty\le C_D 
\quad \text{with} \quad
C_D:=\sqrt{\frac{1-\beta_1}{1-\beta_2}}.
\label{eq:uniform-D}
\end{equation}
\end{lemma}

\begin{proof} In the case where $\beta_1 = \beta_2 = 0$, we have $|D_{t}[j]| = \frac{|g_t[j]|}{|g_t[j] + \eps|} \leq 1 = C_D$. From now on, we assume that $\beta_2 > 0$. Fix a coordinate $j$ and let $x_s:= g_s[j]$ for $s=0,...,t$. Given $\bar{m}_{-1} = \bar{v}_{-1} = 0$, we have that $\bar{m}_t[j] = (1 - \beta_1) \sum_{s=0}^t \beta_1^{t - s} x_s$ and $\bar{v}_t[j] = (1 - \beta_2) \sum_{s=0}^t \beta_2^{t - s} x_s^2$. After bias correction, we have that
\begin{align*}
    m_t[j] = \frac{1 - \beta_1}{1 - \beta_1^{t+1}} \sum_{s=0}^t \beta_1^{t - s} x_s \quad \text{and} \quad v_t[j] = \frac{1 - \beta_2}{1 - \beta_2^{t+1}} \sum_{s=0}^t \beta_2^{t - s} x_s^2
\end{align*}
Define $a_{t,s} := \frac{(1 - \beta_1) \beta_1^{t-s}}{1 - \beta_1^{t+1}}$ and $b_{t,s} := \frac{(1 - \beta_2) \beta_2^{t-s}}{1 - \beta_2^{t+1}}$, then we have $m_t[j] = \sum_{s=0}^t a_{t,s} x_s$ and $v_t[j] = \sum_{s=0}^t b_{t,s} x_s^2$. Note that $a_{t,s} \geq 0, b_{t,s} \geq 0, \sum_{s=0}^t a_{t,s} = 1$ and $\sum_{s=0}^t b_{t,s} = 1$. When $\beta_1 = 0$, the interpretation is that $a_{t,t} = 1$ and $a_{t,s} = 0$ for $s < t$ (the factor $0^0$ is treated as $1$). By Jensen's inequality, we have
\begin{align*}
    m_t[j]^2 = (\sum_{s=0}^t a_{t,s}x_s)^2 \leq \sum_{s=0}^t a_{t,s} x_s^2. 
\end{align*}
For every $s \in \{0,...,t\}$ , the ratio $\frac{a_{t,s}}{b_{t,s}}$ satisfies
\begin{align*}
    \frac{a_{t,s}}{b_{t,s}} = \frac{1 - \beta_1}{1 - \beta_2} \frac{1 - \beta_2^{t+1}}{1 - \beta_1^{t+1}} (\frac{\beta_1}{\beta_2})^{t-s} \leq \frac{1 - \beta_1}{1 - \beta_2} = C_D^2,
\end{align*}
where we have used $\frac{1 - \beta_2^{t+1}}{1 - \beta_{1}^{t+1}} \leq 1$. Consequently, it holds that $a_{t,s} \leq C_D^2 b_{t,s}$. Using this, we have 
\begin{align*}
    m_t[j]^2 \leq \sum_{s=0}^t a_{t,s} x_s^2 \leq C_D^2 \sum_{s=0}^t b_{t,s} x_s^2 = C_D^2 v_t[j].
\end{align*}
Taking the square root leads to $|m_t[j]| \leq C_D \sqrt{v_t[j]}$. 
Consequently, we have  
\begin{align*}
    |D_t[j]| = \frac{|m_t[j]|}{\sqrt{v_t[j]} + \eps} \leq C_D \frac{\sqrt{v_t[j]}}{\sqrt{v_t[j]} + \eps} \leq C_D, \forall j. 
\end{align*}
The desired follows by taking the maximum over all the coordinates. 
\end{proof}

\begin{lemma}
\label{lem:proxy-path-ratio}
For all $u,v\in\R^d$, we have 
\begin{equation} 
 e^{-R_1\norm{u-v}_\infty}G(v)
 \le G(u)
 \le e^{R_1\norm{u-v}_\infty}G(v).
 \label{eq:G-path-ratio}
\end{equation}
Further assume that $0 \leq \beta_1 \leq \beta_2 < 1$. It holds for every $0 \leq r \leq t$ that 
\[
 e^{-R_1C_D\sum_{s=r}^{t-1}\eta_s}G_r
 \le G_t
 \le e^{R_1C_D\sum_{s=r}^{t-1}\eta_s}G_r.
\]
\end{lemma}

\begin{proof} Fix any $i \in [n]$. Define $p_i := z_i^T u$ and $q_i := z_i^T v$. By definition, we have $a_i(u) = -\ell'(p_i)$ and $a_i(v) = -\ell'(q_i)$. Applying Lemma \ref{lem:derivative-ratio} with $a = p_i$ and $b = q_i$ gives
\begin{align*}
    e^{-|z_i^T(u - v)|} \leq \frac{a_i(u)}{a_i(v)} \leq e^{|z_i^T(u - v)|}.
\end{align*}
Multiplying both sides by $a_i(v) = \frac{1}{1+e^{z_i^T v}} > 0$ leads to 
\begin{align*}
        e^{-|z_i^T(u - v)|} a_i(v)\leq a_i(u) \leq e^{|z_i^T(u - v)|} a_i(v).
\end{align*}
By $|z_i^T(u - v)| \leq \norm{z_i}_1 \norm{u - v}_{\infty} \leq R_1 \norm{u - v}_{\infty}$, we further obtain
\begin{align*}
    e^{-R_1 \norm{u - v}_{\infty}} a_i(v) \leq a_i(u) \leq e^{R_1 \norm{u - v}_{\infty}} a_i(v).
\end{align*}
Summing this inequality over $i$, dividing by $n$, and using the definition of the function $G$ leads to \eqref{eq:G-path-ratio}. Applying \eqref{eq:G-path-ratio} with $u = w_t$ and $v = w_r$ gives
\begin{align*}
     e^{-R_1\norm{w_t-w_r}_\infty}G_r
 \le G_t
 \le e^{R_1\norm{w_t-w_r}_\infty}G_r.
\end{align*}
By Lemma \ref{lem:uniform-D}, it holds that for any $t \geq r$
\begin{align}
    \norm{ w_t - w_r }_{\infty} = \norm{-\sum_{s=r}^{t-1} \eta_s D_s}_{\infty} \leq C_D \sum_{s=r}^{t-1} \eta_s. \label{eq:g_path_temp}
\end{align}
Consequently, we have $R_1 \norm{w_t - w_r}_{\infty} \leq R_1 C_D \sum_{s=r}^{t-1} \eta_s$. Combining the preceding bound with \eqref{eq:g_path_temp} finishes the proof. 
\end{proof}



\begin{lemma}
\label{lem:taylor}
 For every $t\ge0$ and every $\eps > 0$, we have 
\begin{equation}
L(w_{t+1})
\le L(w_t)-\eta_t\ip{g_t}{D_t}
+C_H\eta_t^2G_t,
\label{eq:taylor-adam}
\end{equation}
where $C_H = \frac{R_1^2}{2} \frac{1-\beta_1}{1 - \beta_2} \exp(R_1 \eta_0 \sqrt{\frac{1 - \beta_1}{1 - \beta_2}})$. 
\end{lemma}

\begin{proof}
The gradient of the loss is $\nabla L(w) = \frac{1}{n} \sum_{i=1}^n \ell'(z_i^T w) z_i = -\frac{1}{n} \sum_{i=1}^n a_i(w) z_i$, differentiating again leads to the Hessian
\[
\nabla^2L(w)=\frac1n\sum_i\ell''(z_i^\top w)z_iz_i^\top
\quad \text{where} \quad
0\le\ell''(u)=\frac{e^u}{(1+e^u)^2}\le\frac1{1+e^u}=-\ell'(u).
\]
Let $\widetilde w_{\xi}:=w_t-\xi\eta_tD_t$ for some $\xi\in[0,1]$. By \cref{lem:uniform-D}, we have
$|z_i^\top(\widetilde w_{\xi}-w_t)|\le R_1C_D\eta_t\le R_1C_D\eta_0$. Therefore, \cref{lem:derivative-ratio} gives
\[
-\ell'(z_i^\top\widetilde w_{\xi})
\le e^{|z_i^T \widetilde{w}_{\xi} - z_i^T w_t|}[-\ell'(z_i^T w_t)]
\le e^{R_1C_D\eta_0}a_i(w_t),
\]
where we have used $-\ell'(z_i^T w_t) = a_i(w_t)$.
Moreover, it holds that $|z_i^\top D_t|\le \norm{z_i}_1 \norm{D_t}_{\infty} \le R_1C_D$. Consequently, we have
\begin{align}
D_t^\top\nabla^2L(\widetilde w_{\xi})D_t
&\le\frac1n\sum_i
\bigl[-\ell'(z_i^\top\widetilde w_{\xi})\bigr]
(z_i^\top D_t)^2 \le e^{R_1C_D\eta_0}R_1^2C_D^2G_t = M_H G_t, \label{eq:md_1}
\end{align}
where we have let $M_H:= e^{R_1C_D\eta_0}R_1^2C_D^2$.
Define $\phi(\omega):= L(w_t - \omega \eta_t D_t)$ with $\omega \in [0,1]$. By Taylor's theorem, there exists $\xi_t \in (0,1)$ such that 
\begin{align*}
    \phi(1) = \phi(0) + \phi'(0) + \frac{1}{2} \phi''(\xi_t).  
\end{align*}
Substituting $\phi(0) = L(w_t)$, $\phi(1) = L(w_{t+1})$, $\phi'(0) = -\eta_t \langle g_t, D_t \rangle$, and $\phi''(\xi_t) = \eta_t^2 D_t^T \nabla^2 L(\widetilde{w}_{\xi_t}) D_t$, we obtain
\begin{align*}
    L(w_{t+1})
&= L(w_t)-\eta_t\ip{g_t}{D_t} + \frac{\eta_t^2}{2} D_t^T \nabla^2 L(\widetilde{w}_{\xi_t}) D_t\\ 
&\leq L(w_t)-\eta_t\ip{g_t}{D_t} + \frac{M_H}{2} \eta_t^2 G_t.  
\end{align*}
Letting $C_H = \frac{M_H}{2}$ and substituting the definition of $C_D$ finish the proof. 
\end{proof}

\begin{lemma}
\label{lem:gradient-drift}
For every $0\le r\le t$, we have
\begin{equation}
\norm{g_{t-r}-g_t}_1
\le
R_1G_t
\left[
\exp\!\left(R_1C_D\sum_{s=t-r}^{t-1}\eta_s\right)-1
\right].
\label{eq:gradient-drift}
\end{equation}
\end{lemma}

\begin{proof}
By \cref{lem:uniform-D}, it holds that
\[
\norm{w_t-w_{t-r}}_\infty
\le\sum_{s=t-r}^{t-1}\eta_s\norm{D_s}_\infty
\le C_D\sum_{s=t-r}^{t-1}\eta_s.
\]
Hence, for every sample $i \in [n]$, we have 
\[
|z_i^\top(w_t-w_{t-r})|\le \norm{z_i}_1 \norm{w_t-w_{t-r}}_{\infty}
\le R_1C_D\sum_{s=t-r}^{t-1}\eta_s=:\Delta_{t,r}.
\]
Applying \cref{lem:derivative-ratio} with $a = z_i^T w_{t-r}$ and $b = z_i^T w_t$, we obtain
\begin{align*}
    |\frac{-\ell'(z_i^T w_{t-r})}{-\ell'(z_i^T w_t)} - 1| = |\frac{a_i(w_{t-r})}{a_i(w_t)} - 1|
    \le e^{|z_i^\top(w_t-w_{t-r})|} - 1 \leq e^{\Delta_{t,r}} - 1,
\end{align*}
where we have used $a_i(w_{t-r}) = -\ell'(z_i^T w_{t-r})$ and $a_i(w_{t}) = -\ell'(z_i^T w_{t})$. Consequently, it holds that $|a_i(w_{t-r}) - a_i(w_t)| \leq a_i(w_t) (e^{\Delta_{t,r}} - 1)$. By the definition of $g$, with $g_t := g(w_t), \forall t$, we have
\begin{align*}
    \norm{g_{t-r} - g_t}_1 &= \norm{- \frac{1}{n} \sum_{i=1}^n [a_i(w_{t-r}) - a_i(w_t)] z_i}_1 \leq \frac{1}{n} \sum_{i=1}^n  |a_i(w_{t-r}) - a_i(w_t)| \norm{z_i}_1 \\
    &\leq \frac{R_1}{n}\sum_{i=1}^n  |a_i(w_{t-r}) - a_i(w_t)| \leq R_1 (e^{\Delta_{t,r}} - 1) \frac{1}{n} \sum_{i=1}^n a_i(w_t) = R_1 G_t (e^{\Delta_{t,r}} - 1). 
\end{align*}
Substituting $\Delta_{t,r}:=R_1C_D\sum_{s=t-r}^{t-1}\eta_s$ into this inequality finishes the proof. 
\end{proof}
\begin{lemma}
\label{lem:first-track} Let $C_{\rm geo}(q, c, a, \eta_0)$ be the constant in \cref{prop:ema-sum}. There exists a time $t_m = t_m(q = \beta_1, c=R_1 C_D, a, \eta_0)$ such that for all $t\ge t_{m}$, we have
\[
\norm{m_t-g_t}_1\le C_m\eta_tG_t,
\]
where $C_m = R_1 C_{\rm geo}(q = \beta_1, c=R_1 C_D, a, \eta_0)$. 
\end{lemma}

\begin{proof} If $\beta_1=0$, then $m_t=g_t$ and the bound clearly holds. Suppose that $\beta_1\in(0,1)$.
Similar to Lemma \ref{lem:uniform-D}, we can write 
$
m_t=\sum_{r=0}^t a_{t,r}g_{t-r}$ with $
a_{t,r}=\frac{(1-\beta_1)\beta_1^r}{1-\beta_1^{t+1}} \geq 0$ and $\sum_{r=0}^t a_{t,r} = 1$. Since
$1-\beta_1^{t+1}\ge1-\beta_1$, we have $a_{t,r}\le\beta_1^r$. These lead to 
\begin{align*}
\norm{m_t-g_t}_1
&=\norm{\sum_{r=0}^ta_{t,r}(g_{t-r}-g_t)}_1 \le\sum_{r=0}^ta_{t,r}\norm{g_{t-r}-g_t}_1\\
&\le\sum_{r=0}^t\beta_1^r\norm{g_{t-r}-g_t}_1.
\end{align*}
Applying \cref{lem:gradient-drift}, we obtain
\begin{align*}
    \norm{m_t - g_t}_1 \leq R_1 G_t \sum_{r=0}^t \beta_1^r [\exp(R_1 C_D \sum_{s=t-r}^{t-1} \eta_s) - 1] \leq C_m \eta_t G_t, \quad \forall t \geq t_{m}.
\end{align*} 
where $C_{m} = R_1 C_{\rm geo}(q,c,a,\eta_0)$ with $q=\beta_1$ and $c=R_1C_D$.
\end{proof}

\begin{lemma}
\label{lem:second-track} Let $C_{\rm geo}(q, c, a, \eta_0)$ be the constant in \cref{prop:ema-sum}. There exists a time $t_v = t_v(q = \sqrt{\beta_2}, c=R_1 C_D, a, \eta_0)$ such that for all $t\ge t_{v}$, we have 
\[
\norm{\sqrt{v_t}-|g_t|}_1\le C_v\eta_tG_t.
\]
where $C_v = R_1 C_{\rm geo}(q = \sqrt{\beta_2}, c=R_1 C_D, a, \eta_0)$.
\end{lemma}

\begin{proof}
If $\beta_2=0$, then $v_t=g_t^{\eprod2}$ and the left side is zero and the bound trivially holds. Suppose that $\beta_2\in(0,1)$. We define the normalized EMA weights
\[
b_{t,r}:=\frac{(1-\beta_2)\beta_2^r}{1-\beta_2^{t+1}},
\qquad r=0,\ldots,t,
\]
and note that $\sum_{r=0}^t b_{t,r} = 1$. Moreover, it holds that $\bar{v}_t = \beta_2 \bar{v}_{t-1} + (1 - \beta_2) g_t^{\eprod2} = (1-\beta_2) \sum_{r=0}^t \beta_2^r g_{t-r}^{\eprod2}$. After bias-correction, this becomes $v_t = \sum_{r=0}^t b_{t,r} g_{t-r}^{\eprod2}$. For a fixed feature coordinate $j$, we introduce two vectors in $\R^{t+1}$,
\[
X^{(j)}=(\sqrt{b_{t,r}}g_{t-r}[j])_{r=0}^t,
\qquad
Y^{(j)}=(\sqrt{b_{t,r}}g_t[j])_{r=0}^t.
\]
Then, we have 
\begin{align*}
    \norm{X^{(j)}}_2^2 = \sum_{r=0}^t b_{t,r} g_{t-r}[j]^2 = v_t[j], \quad \norm{Y^{(j)}}_2^2 = \sum_{r=0}^t b_{t,r} g_{t}[j]^2 = g_{t}[j]^2 \sum_{r=0}^t b_{t,r} = g_{t}[j]^2.
\end{align*}
Consequently, it holds that
$\norm{X^{(j)}}_2=\sqrt{v_t[j]}$ and
$\norm{Y^{(j)}}_2=|g_t[j]|$. Reverse triangle inequality gives
\begin{align*}
\left|\sqrt{v_t[j]}-|g_t[j]|\right|
&\le\norm{X^{(j)}-Y^{(j)}}_2\\
&=\left(\sum_{r=0}^tb_{t,r}|g_{t-r}[j]-g_t[j]|^2\right)^{1/2}\\
&\le\sum_{r=0}^t\sqrt{b_{t,r}}|g_{t-r}[j]-g_t[j]|,
\end{align*}
where the last inequality is by applying the inequality $(\sum_{r=0}^t x_r^2)^{1/2} \leq \sum_{r=0}^t x_r$ (for any nonnegative numbers $x_0,...,x_t$) with $x_r = \sqrt{b_{t,r}}|g_{t-r}[j] - g_t[j]|$.  
Because $1-\beta_2^{t+1}\ge1-\beta_2$, we have 
$\sqrt{b_{t,r}} = \sqrt{\frac{1-\beta_2}{1 - \beta_2^{t+1}}}\beta_2^{r/2} \le\beta_2^{r/2}$. Summing the preceding inequality over $j = 1,...,d$ yields
\begin{align*}
    \norm{\sqrt{v_t}-|g_t|}_1 &= \sum_{j=1}^d \big | \sqrt{v_t[j]} - |g_t[j]| \big | \leq \sum_{j=1}^d \sum_{r=0}^t \beta_2^{r/2} |g_{t-r}[j] - g_t[j]|\\ 
&= \sum_{r=0}^t \beta_2^{r/2} \sum_{j=1}^d |g_{t-r}[j] - g_t[j]|= \sum_{r=0}^t\beta_2^{r/2}\norm{g_{t-r}-g_t}_1.
\end{align*}
Now applying \cref{lem:gradient-drift,prop:ema-sum} (with $q=\sqrt{\beta_2}$ and $c = R_1 C_D$), we obtain 
\begin{align*}
    \norm{\sqrt{v_t}-|g_t|}_1 &\leq R_1 G_t \sum_{r=0}^t (\sqrt{\beta_2})^r \bigl[ \exp(R_1 C_D \sum_{s=t-r}^{t-1} \eta_s) - 1 \bigr] \leq C_v \eta_t G_t, \quad \forall t \geq t_{v},
\end{align*}
where $C_v = R_1 C_{\rm geo}(q = \sqrt{\beta_2}, c=R_1 C_D, a, \eta_0)$. 
\end{proof}

\begin{lemma}
\label{lem:alignment} Let $t_{\rm ema} := \max \{t_m,t_v\}$ where $t_m$ and $t_v$ are the times in \cref{lem:first-track} and \cref{lem:second-track} respectively. For all $t\ge t_{\rm ema}$ and all $\eps > 0$, we have 
\begin{align}
\norm{g_t\eprod(D_t-\sigma_\eps(g_t))}_1
&\le C_{\rm ema}\eta_tG_t,
\label{eq:weighted-alignment-app}\\
\eps\norm{D_t-\sigma_\eps(g_t)}_2
&\le C_{\rm ema}\eta_tG_t,
\label{eq:floor-alignment-app}\\
\left|\ip{g_t}{D_t-\sigma_\eps(g_t)}\right|
&\le C_{\rm ema}\eta_tG_t,
\label{eq:alignment-app}
\end{align}
where $C_{\rm ema} = \max \{C_m + C_D C_v, C_m + C_v\}$. Recall that $C_D$,$C_m$,$C_v$ are the constants in \cref{lem:first-track,lem:second-track,lem:uniform-D} respectively. 
\end{lemma}

\begin{proof}
Fix a coordinate $j$, we let 
$a:=g_t[j]$, $b:=m_t[j]$, and $c:=\sqrt{v_t[j]}$. 
We have $\sigma_{\eps}(g_t)[j] = \frac{a}{|a| + \eps}$, $D_t[j] = \frac{b}{c+\eps}$, and we further define $\triangle_j := D_t[j] - \sigma_{\eps}(g_t)[j ]= \frac{b}{c+\eps} - \frac{a}{|a| + \eps}$.  
We first have the following equality
\begin{align*}
    \frac{a}{|a|+\eps}-\frac{b}{c+\eps}
&= \frac{a(c + \eps) - b(|a| + \eps)}{(|a|+\eps)(c+\eps)}\\
&=
\frac{(a-b)(c+\eps)+b(c-|a|)}{(|a|+\eps)(c+\eps)}.
\end{align*}
Multiplying by $|a|$, applying triangle inequality, and bounding the two terms separately gives
\begin{align*}
|a|\left|
\frac{a}{|a|+\eps}-\frac{b}{c+\eps}
\right|
&\le\frac{|a|}{|a|+\eps}|a-b|
+\frac{|a|}{|a|+\eps}\frac{|b|}{c+\eps}|c-|a||\\
&\le |a-b|+C_D|c-|a||,
\end{align*}
where the second inequality is by $0 \leq \frac{|a|}{|a| + \eps} \leq 1$ and $\frac{|b|}{c + \eps} = |D_t[j]| \le C_D$ by \cref{lem:uniform-D}. Consequently, we have $|a||\triangle_{j}| \leq |a -b| + C_D \big| c - |a| \big |$. Applying this inequality, we obtain
\begin{align*}
    \norm{g_t\eprod(D_t-\sigma_\eps(g_t))}_1 &= \sum_{j=1}^d |g_t[j]| |D_t[j] - \sigma_{\eps}(g_t)[j]| \\
    &\leq \sum_{j=1}^d |g_t[j] - m_t[j]| + C_D \sum_{j=1}^d \big | \sqrt{v_t[j]} - |g_{t}[j]|\big | \\
    &= \norm{m_t - g_t}_1 + C_D \norm{\sqrt{v_t} - |g_t|}_1 \\
    &\leq C_m \eta_t G_t + C_D C_v \eta_t G_t \\
    &=(C_m + C_D C_v) \eta_t G_t, \quad \forall t \geq t_{\rm ema},
\end{align*}
where we have applied \cref{lem:first-track,lem:second-track} and let $t_{\rm ema}:= \max \{ t_v, t_m\}$. 
The inner-product inequality \eqref{eq:alignment-app} follows immediately from
\[
\left|\ip{g_t}{D_t-\sigma_\eps(g_t)}\right|
\le \norm{g_t\eprod(D_t-\sigma_\eps(g_t))}_1.
\]
Next, we prove a different coordinatewise inequality for \eqref{eq:floor-alignment-app}
\begin{align*}
\frac{b}{c+\eps}-\frac{a}{|a|+\eps}
= \frac{b-a}{c+\eps} + [\frac{a}{c+\eps} - \frac{a}{|a|+\eps}]=
\frac{b-a}{c+\eps}
+\frac{a(|a|-c)}{(c+\eps)(|a|+\eps)}.    
\end{align*}
Multiplying by $\eps$, taking absolute values, and using the triangle inequality gives
\begin{align*}
\eps\left|
\frac{b}{c+\eps}-\frac{a}{|a|+\eps}
\right|
&\le
\frac{\eps}{c+\eps}|b-a|
+\frac{\eps}{c+\eps}\frac{|a|}{|a|+\eps}|c-|a||\\
&\le |b-a|+|c-|a||,
\end{align*}
where in the last inequality we used $0 \leq \frac{\eps}{c + \eps} \leq 1$ and $0 \leq \frac{|a|}{|a| + \eps} \leq 1$. By substituting the definition of $\triangle_j$, we have $\eps |\triangle_j| \leq |b-a| + |c - |a||$. 
Based on this, we obtain
\begin{align*}
    \eps \norm{D_t - \sigma_{\eps}(g_t)}_2 &\leq \eps \norm{D_t - \sigma_{\eps}(g_t)}_1 = \sum_{j=1}^d \eps |D_t[j] - \sigma_{\eps}(g_t)[j]|\\
    &\leq \sum_{j=1}^d |m_t[j] - g_t[j]| + \sum_{j=1}^d \big| \sqrt{v_t[j]} - |g_t[j]| \big| \\
    &= \norm{m_t - g_t}_1 + \norm{\sqrt{v_t} - |g_t|}_1 \\
    &\leq C_m \eta_t G_t + C_v \eta_t G_t = (C_m + C_v) \eta_t G_t, \quad \forall t \geq t_{\rm ema}. 
\end{align*}
Letting $C_{\rm ema} = \max \{C_m + C_D C_v, C_m + C_v\}$ finishes the proof. 
\end{proof}

In \cref{lem:softsign}, \eqref{eq:softsign-linear} is used in Part~1 of the proofs of
\cref{thm:first-clock} and \cref{thm:margin_app_T1}, where $G_t/\eps\to\infty$. Its purpose is
to obtain the sharp coefficient $\gamma_\infty$, whereas \eqref{eq:softsign-titu}
is used in Part~2 of both proofs. 
\begin{lemma}
\label{lem:softsign}
For every $g\in\R^d$ and $\eps > 0$, we have 
\begin{align}
\ip{g}{\sigma_\eps(g)}
&\ge\norm{g}_1-d\eps,
\label{eq:softsign-linear}\\
\ip{g}{\sigma_\eps(g)}
&\ge\frac{\norm{g}_1^2}{\norm{g}_1+d\eps}. 
\label{eq:softsign-titu}
\end{align}
\end{lemma}

\begin{proof}
When $\eps=0$, we have $\sigma_0(g)=\sign(g)$ gives
$\ip{g}{\sigma_0(g)}=\norm{g}_1$. Hence, both inequalities hold. Suppose now that $\eps>0$. For each coordinate, we have
\[
\frac{g_j^2}{|g_j|+\eps} = \frac{|g_j|(|g_j| + \eps) - \eps |g_j|}{|g_j|+\eps}  
=|g_j|-\eps\frac{|g_j|}{|g_j|+\eps}.
\]
Summing over $j$ gives 
\begin{align*}
    \langle g, \sigma_{\eps}(g)\rangle &= \sum_{j=1}^d g_j \sigma_{\eps}(g)_j  = \sum_{j=1}^d \frac{g_j^2}{|g_j| + \eps}\\
    &= \sum_{j=1}^d |g_j| - \eps \sum_{j=1}^d \frac{|g_j|}{|g_j|+\eps} = \norm{g}_1 - \eps \sum_{j=1}^d \frac{|g_j|}{|g_j|+\eps}. 
\end{align*}
For every $j$, it holds that $0 \leq \frac{|g_j|}{|g_j| + \eps} \leq 1$; therefore, we have $\langle g, \sigma_{\eps}(g)\rangle \geq \norm{g}_1 - d \eps$. 
Let $x_j := |g_j|$ and $y_j := |g_j| + \eps$ (note that $y_j > 0,\forall j \in [d]$). From $\sum_{j=1}^d x_j = \sum_{j=1}^d \frac{x_j}{\sqrt{y_j}} \sqrt{y_j}$, we have $(\sum_{j=1}^d x_j)^2 \leq (\sum_{j=1}^d \frac{x_j^2}{y_j})(\sum_{j=1}^d y_j)$ by Cauchy-Schwarz. Rearranging this leads to $\sum_{j=1}^d \frac{x_j^2}{y_j} \geq \frac{(\sum_{j} x_j)^2}{\sum_{j} y_j}$. Substituting the definitions of $x_j$ and $y_j$, we obtain 
\[
\sum_j\frac{|g_j|^2}{|g_j|+\eps}
\ge
\frac{(\sum_j|g_j|)^2}{\sum_j(|g_j|+\eps)}
=\frac{\norm{g}_1^2}{\norm{g}_1+d\eps}.
\]
\end{proof}

\begin{proposition}
\label{prop:master-descent}
The following holds for every $\eps > 0$,
\begin{equation}
L_{t+1}
\le
L_t-
\eta_t\ip{g_t}{\sigma_\eps(g_t)}
+C_{\rm md}\eta_t^2G_t, \quad t \ge t_{\rm ema},
\label{eq:master-descent-app}
\end{equation}
where $t_{\rm ema}$ is defined in \cref{lem:alignment} and $C_{\rm md} := C_H + C_{\rm ema}$. Recall that $C_H$ and $C_{\rm ema}$ are the constants in \cref{lem:taylor,lem:alignment} respectively. 
\end{proposition} 

\begin{proof}
Let $\delta_t := \langle g_t, D_t - \sigma_{\eps}(g_t) \rangle$.
By \cref{lem:alignment}, we have $|\delta_t| \leq C_{\rm ema} \eta_t G_t,\forall t\ge t_{\rm ema}$. Consequently, it holds that $\delta_t \geq -C_{\rm ema} \eta_t G_t$ which implies 
\begin{align*}
    \langle g_t, D_t \rangle &= \langle g_t, \sigma_{\eps}(g_t) + D_t - \sigma_{\eps}(g_t) \rangle = \langle g_t, \sigma_{\eps}(g_t) \rangle + \delta_t \\
    &\geq \langle g_t, \sigma_{\eps}(g_t) \rangle - C_{\rm ema} \eta_t G_t, \quad \forall t \ge t_{\rm ema},  
\end{align*}
From this, \cref{lem:taylor} gives
\begin{align*}
L_{t+1}
&\le L_t-\eta_t\ip{g_t}{D_t}+C_H\eta_t^2G_t\\
&\le L_t-\eta_t\ip{g_t}{\sigma_\eps(g_t)} + C_{\rm ema} \eta_t^2 G_t + C_H\eta_t^2G_t \\
&= L_t-\eta_t\ip{g_t}{\sigma_\eps(g_t)}
+(C_H+C_{\rm ema})\eta_t^2G_t, \quad \forall t \ge t_{\rm ema}.
\end{align*}
Letting $C_{\rm md}:=C_H+C_{\rm ema}$ finishes the proof.
\end{proof}
\begin{proof}[Proof of \cref{prop:uniform-clock}] Fix any $\eps \in (0,1]$. Combining \eqref{eq:proxy-gradient} and \eqref{eq:softsign-titu}, we obtain 
\begin{align*}
    \langle g_t, \sigma_{\eps}(g_t) \rangle \geq \frac{\norm{g_t}_1^2}{\norm{g_t}_1 + d \eps} \geq \frac{\gamma_{\infty}^2}{R_1 + d} G_t^2,
\end{align*}
where the last step uses $G_t \leq 1$ and $\eps \leq 1$. Define $\kappa := \frac{\gamma_{\infty}^2}{R_1 + d}$. Substituting this into \eqref{eq:master-descent-app}, we have
\begin{align*}
    L_{t+1} \leq L_t - \kappa \eta_t G_t^2 + C_{\rm md} \eta_t^2 G_t.
\end{align*}
By Young's inequality, we have $C_{\rm md} \eta_t^2 G_t = (\sqrt{\kappa \eta_t} G_t) (\frac{C_{\rm md}}{\sqrt{\kappa}} \eta_t^{3/2}) \leq \frac{\kappa}{2} \eta_t G_t^2 + \frac{C_{\rm md}^2}{2 \kappa} \eta_t^3$. 
Consequently, it holds that
\begin{align}
    L_{t+1} \leq L_t - c \eta_t G_t^2 + K \eta_t^3, \qquad \forall t \geq t_{\rm ema}, \label{eq:lowloss_1}
\end{align}
where $c:= \frac{\kappa}{2} = \frac{\gamma_{\infty}^2}{2(R_1 + d)}$ and $K:= \frac{C_{\rm md}^2}{2 \kappa} = \frac{(R_1 + d) C_{\rm md}^2}{2 \gamma_{\infty}^2}$.
Since $\eta_t = \frac{\eta_0}{(t+1)^a}$ with $a \in  (\frac{1}{3},1]$, we have $\sum_{t=0}^{\infty} \eta_t^3 = \eta_0^3 \sum_{t=0}^{\infty} (t+1)^{-3a} < \infty$. Define $R_t := \sum_{s =t}^{\infty} \eta_s^3$. We then have $R_{t+1} = R_t - \eta_t^3$ and $R_t \downarrow 0$. This holds because
\begin{align*}
    R_t = \eta_0^3 \sum_{s =t}^{\infty} \frac{1}{(s+1)^{3a}} \leq \eta_0^3 \int_{t}^{\infty} x^{-3a} dx = \frac{\eta_0^3}{3a - 1} t^{1 - 3a}, 
\end{align*}
and $t^{1 - 3a} \rightarrow 0$ as $1 - 3a < 0$. Further define $\Phi_{t} := L_t + K R_t, \forall t$. Using \eqref{eq:lowloss_1}, we obtain 
\begin{align}
    \Phi_{t+1} &= L_{t+1} + K R_{t+1} \nonumber \\
    &\leq L_t - c \eta_t G_t^2 + K \eta_t^3 + K(R_t - \eta_t^3) \nonumber\\
    &= L_t + K R_t - c \eta_t G_t^2 \nonumber\\ 
    &= \Phi_t - c \eta_t G_t^2, \qquad \forall t \geq t_{\rm ema}. \label{eq:lowloss_descent} 
\end{align}
Note that $\Phi_{t}$ is non-increasing for $t \geq t_{\rm ema}$. Rearranging and summing from $t=t_{\rm ema}$ to $T - 1$, we obtain
\begin{align}
    c \sum_{t=t_{\rm ema}}^{T-1} \eta_t G_t^2 \leq \Phi_{t_{\rm ema}} - \Phi_T \leq \Phi_{t_{\rm ema}}. 
\end{align}
Taking $T \rightarrow \infty$ leads to $\sum_{t=t_{\rm ema}}^{\infty} \eta_t G_t^2 \leq \frac{\Phi_{\rm ema}}{c} < \infty$. The prefix from $t = 0$ to $t_{\rm ema} - 1$ contains only finitely many terms. Therefore, we conclude that $\sum_{t=0}^{\infty} \eta_t G_t^2 < \infty$. This holds for every $\eps \in (0,1]$. 

Next, we show that $G_{t} \rightarrow 0$. As $a \leq 1$, we have $\sum_{t=0}^{\infty} \eta_t = \infty$. Suppose that $\liminf_{t \rightarrow \infty} G_t > 0$. Then there exists $\delta > 0$ and $t_{\delta}$ such that $G_t \geq \delta$ for $t \geq t_{\delta}$. Consequently we have $\sum_{t = t_{\delta}}^{\infty} \eta_t G_t^2 \geq \delta^2 \sum_{t = t_{\delta}}^{\infty} \eta_t = \infty$, which contradicts $\sum_{t=0}^{\infty} \eta_t G_t^2 < \infty$. Hence, we must have 
$\liminf_{t \rightarrow \infty} G_t = 0$. Since $\Phi_t$ is non--increasing and bounded below by zero, there exists $\Phi_{\infty} \geq 0$ such that $\Phi_t \rightarrow \Phi_{\infty}$. Since $R_t \rightarrow 0$, we further have $L_t = \Phi_t - K R_t \rightarrow \Phi_{\infty}$. Therefore, $L_t$ has a limit, which we denote by $L_{\infty}$ (note that $L_{\infty} = \Phi_{\infty}$). Given $\liminf_{t \rightarrow \infty} G_t = 0$, there exists a subsequence $t_k \rightarrow \infty$ such that $G_{t_k} \rightarrow 0$. Hence, we have $G_{t_k} \leq \frac{1}{2n}$ for all sufficiently large $k$. By Lemma \ref{lem:loss-proxy}, we have $L_{t_k} \leq 2G_{t_k}$. Therefore, $L_{t_k} \rightarrow 0$ given $G_{t_k} \rightarrow 0$. Given the entire sequence $L_t$ has the limit $L_{\infty}$, every subsequence must have the same limit. This implies that $L_{\infty} = 0$. Therefore, we conclude that $L_t \rightarrow 0$. Moreover, given $0 < G_t \leq L_t$ by Lemma \ref{lem:loss-proxy}, we further obtain $G_t \rightarrow 0$. 

Next, we construct a time $T_*$ that works for every $\eps \in (0,1]$. Let $t_0 := t_{\rm ema}$. By Lemma \ref{lem:uniform-D}, 
we have $\norm{w_{t_0}}_{\infty} = \norm{w_0 - \sum_{s=0}^{t_0 - 1} \eta_s D_s}_{\infty} \leq \norm{w_0}_{\infty} + C_D \sum_{s=0}^{t_0-1} \eta_s$. Define $M_0 := \norm{w_0}_{\infty} + C_D S_{t_0}$. Then, we have $\norm{w_{t_0}}_{\infty} \leq M_0$ uniformly in $\eps$. Note that $S_{t_0} \leq \eta_0 t_0$ and $M_0 \leq \norm{w_0}_{\infty} + C_D \eta_0 t_0$. For each sample, we have $|z_i^T w_{t_0}| \leq \norm{z_i}_1 \norm{w_{t_0}}_{\infty} \leq R_1 M_0$. This leads to $l(z_i^T w_{t_0}) \leq \log (1 + e^{R_1 M_0})$ and by averaging, we further obtain $L_{t_{0}} \leq B_L$ where $B_L := \log(1 + e^{R_1 M_0})$. Therefore, we obtain
\begin{align}
    \Phi_{t_0} = L_{t_0} + K R_{t_0} \leq B, \qquad B:= B_L + K R_{t_0} < \infty. \label{eq:phi0}
\end{align}
Importantly, $B$ is independent of $\eps$. Let $\ell_i := \ell(z_i^T w)$ and $a_i(w) := 1 - e^{-\ell_i}$. Recall that $L(w) = \frac{1}{n} \sum_{i=1}^n \ell_i$ and $G(w) = \frac{1}{n} \sum_{i=1}^n a_i(w)$.  
Given $L(w)$ is the average of the $\ell_i$'s, there is at least one index $i_{*}$ such that $\ell_{i_*} \geq L(w)$. Given the function $x \rightarrow 1 - e^{-x}$ is increasing, we have 
\begin{align*}
    a_{i_*}(w) = 1 - e^{-\ell_{i_*}} \geq 1 - e^{-L(w)}. 
\end{align*}
Therefore, we have 
\begin{align}
    G(w) = \frac{1}{n} \sum_{i=1}^n a_i(w) \geq  \frac{1}{n} a_{i_*}(w) \geq \frac{1}{n}(1 - e^{-L(w)}). \label{eq:lowloss_2}
\end{align}

Define $h := \frac{\log 2}{n}$. We can choose a deterministic time $t_1 \geq t_0$ such that $K R_{t_1} \leq \frac{h}{2}$. Concretely,  we have via integral approximation that 
\begin{align*}
    R_t = \eta_0^3 \sum_{s =t}^{\infty} \frac{1}{(s+1)^{3a}} = \eta_0^3  \sum_{k=t+1}^{\infty} k^{-3a} \leq \eta_0^3 \int_{t}^{\infty} x^{-3a} dx = \frac{\eta_0^3}{3a - 1} t^{1 - 3a}, \qquad t \geq 1. 
\end{align*}
Therefore, $K R_t \leq \frac{h}{2}$ is guaranteed whenever $t^{3a-1} \geq \frac{2K \eta_0^3}{(3a-1)h}$. We define $t_1 := \max \{t_0,1,\lceil (\frac{2 K \eta_0^3}{(3a-1)h})^{\frac{1}{3a-1}}\rceil\}$. Then, it holds that $K R_t \leq \frac{h}{2}, \forall t \geq t_1$.  
Note that $t_1$ does not depend on $\eps$. 

Suppose that $t \geq t_1$ and $\Phi_{t} > h$. Then, we have 
\begin{align*}
    L_t = \Phi_t - K R_t > h - K R_{t_1} \geq \frac{h}{2},
\end{align*}
where we used $R_t \leq R_{t_1}$. By \eqref{eq:lowloss_2}, we have 
\begin{align*}
    G_t \geq \frac{1}{n}(1 - e^{-h/2}) = \frac{1}{n} (1 - \exp(-\frac{\log 2}{2n})) =: \bar{g} > 0.
\end{align*}
Therefore, $\Phi_t > h, t \geq t_1$ implies that $G_t \geq \bar{g}$. By \eqref{eq:lowloss_descent}, we have the following whenever $t \geq t_1$ and $\Phi_t > h$, 
\begin{align}
    \Phi_{t+1} \leq \Phi_t - c \bar{g}^2 \eta_t. \label{eq:lowloss_3}
\end{align}
Since $\sum_t \eta_t = \infty$, we can choose a deterministic time $T_* > t_1$ such that $c \bar{g}^2 \sum_{t = t_1}^{T_* - 1} \eta_t > B$. Concretely, in the case where $1/3 < a < 1$, we have 
\begin{align*}
    \sum_{t = t_1}^{T-1} \eta_t = \eta_0 \sum_{k = t_1 + 1}^T k^{-a} \geq \eta_0 \int_{t_1 + 1}^{T+1} x^{-a} dx = \frac{\eta_0}{1-a}[(T+1)^{1-a} - (t_1 + 1)^{1-a}].
\end{align*}
Therefore, it is sufficient to require $\frac{c \bar{g}^2 \eta_0}{1 - a} [(T_*+1)^{1-a} - (t_1 + 1)^{1-a}]> B$. Solving for $T_*$ gives
\begin{align*}
    T_* = \bigl \lceil [(t_1 + 1)^{1-a} + \frac{(1 - a)B}{c \bar{g}^2 \eta_0}]^{\frac{1}{1-a}} \bigr \rceil, \qquad \frac{1}{3} < a < 1.
\end{align*}
For the case of $a = 1$, we have
\begin{align*}
    \sum_{t = t_1}^{T-1} \eta_t = \eta_0 \sum_{k = t_1 + 1}^T \frac{1}{k} \geq \eta_0 \int_{t_1 + 1}^{T + 1} \frac{dx}{x} = \eta_0 \log (\frac{T + 1}{t_1 + 1}).
\end{align*}
Therefore, it is sufficient to require $c \bar{g}^2 \eta_0 \log (\frac{T^* + 1}{t_1 + 1}) > B$. Solving for $T_*$ leads to
\begin{align*}
    T_* = \bigl \lceil (t_1 + 1) \exp(\frac{B}{c \bar{g}^2 \eta_0}) \bigr \rceil, \qquad a = 1.
\end{align*}
We now show that $\Phi_t$ must fall below $h$ by time $T_*$. Suppose that $\Phi_t > h$ for $t = t_1,...,T_*$. Then iterating \eqref{eq:lowloss_3}, we obtain
\begin{align*}
    \Phi_{T_*} \leq \Phi_{t_1} - c \bar{g}^2 \sum_{t = t_1}^{T_* - 1} \eta_t \leq B - c \bar{g}^2 \sum_{t = t_1}^{T_* - 1} \eta_t < 0,
\end{align*}
which is impossible because $\Phi_{T_*} \geq 0$. Therefore, there exists some $\tau \leq T_*$ such that $\Phi_{\tau} \leq h$. Since $\Phi_t$ is non-increasing, we have $\Phi_t \leq h = \frac{\log 2}{n}, t \geq T_{*}$. Because $L_t \leq \Phi_{t}$, we obtain
\begin{align}
    L_t \leq \frac{\log 2}{n}, \qquad t \geq T_*. \label{eq:lowloss_4}
\end{align}
Importantly, $T_*$ is independent of $\eps$ and it works for every $\eps \in (0,1]$. By \cref{lem:loss-proxy}, we always have $G_t \leq L_t$. By \eqref{eq:lowloss_4} and \cref{lem:loss-proxy}, we have $L_t \leq 2G_t$. Consequently, we have $\frac{1}{2} L_t \leq G_t \leq L_t$ for every $t \geq T_*$. 

Finally, \cref{lem:proxy-path-ratio} gives $G_t \geq G_r \exp(-R_1 C_D \sum_{s=r}^{t-1} \eta_s)$ for $0 \leq r \leq t$. Applying this with $r = 0$ and $t = T_*$, we obtain 
\begin{align*}
    G_{T_*} \geq G_0 \exp(-R_1  C_D S_{T_*}) =: g_* > 0. 
\end{align*}
The last inequality holds because $G_0 > 0$. Note that $g_*$ is independent of $\eps$. 
\end{proof}

\section{Sign and Gradient Residual Characterization} \label{app:res}
Throughout this section, we assume data separability, $0 \leq \beta_1 \leq \beta_2 < 1$, and $\eps \in (0,1]$. Recall that $R_1  :=\max_{i \in [n]} \norm{x_i}_1$ and $R_2  :=\max_{i \in [n]} \norm{x_i}_2$. The step size schedule is $\eta_t = \eta_0 (t+1)^{-a}$ with $\eta_0 > 0$ and $a \in (1/3,1]$. 
\begin{localproposition}
\label{prop:local-interpolation_app}
Denote $C_{\rm loc} := C_{\rm ema} \max \{ \frac{1}{\gamma_{\infty}}, \frac{1}{\gamma_2} \}$. The following holds for every $\eps>0$ and $t\ge t_{\rm ema}$ 
\begin{align}
\left[
\frac{1}{1+R_1\lambda_t}-C_{\rm loc}\eta_t
\right]_+
&\le \mathfrak r_t^{\rm sgn}
\le
\min\left\{1,\frac{d}{\gamma_\infty\lambda_t}\right\}
+C_{\rm loc}\eta_t,
\label{eq:sign-residual-bounds}\\
\left[
1-\frac{\sqrt d}{\gamma_2\lambda_t}-C_{\rm loc}\eta_t
\right]_+
&\le \mathfrak r_t^{\rm grad}
\le
\frac{R_1\lambda_t}{1+R_1\lambda_t}
+C_{\rm loc}\eta_t,
\label{eq:gradient-residual-bounds}
\end{align}
where $[x]_+:=\max\{x,0\}$. Consequently, for every family $t=t(\eps)$ with $t(\eps)\to\infty$ as $\eps \downarrow 0$, we have $\lambda_{t(\eps)}\to\infty$ implying $\bigl(\mathfrak r_{t(\eps)}^{\rm sgn},
\mathfrak r_{t(\eps)}^{\rm grad}\bigr)\to(0,1)$, and $\lambda_{t(\eps)}\to0$ implying $\bigl(\mathfrak r_{t(\eps)}^{\rm sgn},
\mathfrak r_{t(\eps)}^{\rm grad}\bigr)\to(1,0)$. 
\end{localproposition}

\begin{proof}[Proof of \cref{prop:local-interpolation_app}]
Fix any $\eps > 0$ and $t\ge t_{\rm ema}$, and abbreviate $g := g_t$, $D := D_t$, $G := G_t$, and $\lambda := \lambda_t = G_t/\eps$. We introduce the corresponding soft-sign residuals
\[
\bar{\mathfrak r}_t^{\rm sgn}
:=\frac{\norm{g\eprod(\sigma_\eps(g)-\sign(g))}_1}{\norm{g}_1},
\qquad
\bar{\mathfrak r}_t^{\rm grad}
:=\frac{\norm{\eps\sigma_\eps(g)-g}_2}{\norm{g}_2}.
\]
Let $A:= g \odot (D - \sign(g))$ and $B:= g \odot (\sigma_{\eps}(g) - \sign(g))$. Then we have $A - B = g \odot (D - \sigma_{\eps}(g))$, $\frac{\norm{A}_1}{\norm{g}_1} = \mathfrak r_t^{\rm sgn}$, and $\frac{\norm{B}_1}{\norm{g}_1} = \bar{\mathfrak r}_t^{\rm sgn}$. The reverse triangle inequality $\big| \norm{A}_1 - \norm{B}_1 \big| \leq \norm{A - B}_1$ gives
\begin{align*}
    | \mathfrak r_t^{\rm sgn} - \bar{\mathfrak r}_t^{\rm sgn}| = \frac{|\norm{A}_1 - \norm{B}_1|}{\norm{g}_1}\leq \frac{\norm{g \odot (D - \sigma_{\eps}(g))}_1}{\norm{g}_1} \leq \frac{C_{\rm ema}}{\gamma_{\infty}} \eta_t,
\end{align*}
where we used $\norm{g}_1 \geq \gamma_{\infty} G_t$ and $\norm{g_t\eprod(D_t-\sigma_\eps(g_t))}_1
\le C_{\rm ema}\eta_tG_t$ by \cref{lem:alignment}. Similarly, we let $\tilde{A} :=  \eps D - g$ and $\tilde{B}:= \eps \sigma_{\eps}(g) - g$.
Then, we have $\tilde{A} - \tilde{B} = \eps(D - \sigma_{\eps}(g))$, $\frac{\norm{\tilde{A}}_2}{\norm{g}_2} = \mathfrak r_t^{\rm grad}$, and $\frac{\norm{\tilde{B}}_2}{\norm{g}_2} = \bar{\mathfrak r}_t^{\rm grad}$. 
The reverse triangle inequality $\big| \norm{\tilde{A}}_2 - \norm{\tilde{B}}_2 \big| \leq \norm{\tilde{A} - \tilde{B}}_2$ gives
\begin{align*}
    | \mathfrak r_t^{\rm grad} - \bar{\mathfrak r}_t^{\rm grad}| = \frac{|\norm{\tilde{A}}_2 - \norm{\tilde{B}}_2|}{\norm{g}_2}\leq \frac{\eps \norm{D - \sigma_{\eps}(g)}_2}{\norm{g}_2} \leq \frac{C_{\rm ema}}{\gamma_2} \eta_t,
\end{align*}
where we used $\norm{g}_2 \geq \gamma_{2} G_t$ and $\eps \norm{D_t-\sigma_\eps(g_t)}_2
\le C_{\rm ema}\eta_t G_t$ by \cref{lem:alignment}. Combining the previous bounds, we obtain
\begin{align}
    \max \{ | \mathfrak r_t^{\rm sgn} - \bar{\mathfrak r}_t^{\rm sgn}|, | \mathfrak r_t^{\rm grad} - \bar{\mathfrak r}_t^{\rm grad}|\} \leq C_0 \eta_t \quad \text{where} \quad C_0 := C_{\rm ema} \max \{ \frac{1}{\gamma_{\infty}}, \frac{1}{\gamma_2} \} \label{eq:prop4_1}.
\end{align} 

Define $C_{\rm loc} := C_0 = C_{\rm ema} \max \{ \frac{1}{\gamma_{\infty}}, \frac{1}{\gamma_2} \}$. Next, we upper and lower bound $\bar{\mathfrak r}_t^{\rm sgn}$ (and consequently $\mathfrak r_t^{\rm sgn}$ by \eqref{eq:prop4_1}). For a coordinate with $g_j \neq 0$, we have $\sigma_{\eps}(g)_j = \sign(g_j) \frac{|g_j|}{|g_j| + \eps}$. This gives $\sigma_{\eps}(g)_j - \sign(g_j) = - \sign(g_j) \frac{\eps}{|g_j| + \eps}$ and consequently $|g_j| |\sigma_{\eps}(g)_j - \sign(g_j)| = \frac{\eps |g_j|}{|g_j| + \eps}$. The last equality also holds when $g_j = 0$. Summing over $j$ gives $\bar{\mathfrak r}_t^{\rm sgn} = \frac{\eps}{\norm{g}_1} \sum_{j=1}^d \frac{|g_j|}{|g_j| + \eps}$. Let $p_j := \frac{|g_j|}{\norm{g}_1}$. Note that $p_j \geq 0, \forall j$ and $\sum_{j=1}^d p_j = 1$ as $g \neq 0$. From this, we obtain
\begin{align*}
    \bar{\mathfrak r}_t^{\rm sgn} = \sum_{j=1}^d p_j \frac{\eps}{|g_j| + \eps} \geq \sum_{j=1}^d p_j \frac{\eps}{\norm{g}_{\infty} + \eps} = \frac{\eps}{\norm{g}_{\infty} + \eps}. 
\end{align*}
From $\norm{g}_{\infty} \leq \norm{g}_1 \leq R_1 G$, we further obtain $\bar{\mathfrak r}_t^{\rm sgn} \geq \frac{\eps}{R_1 G + \eps} = \frac{1}{1 + R_1 \lambda}$. Next, from $0 \leq \frac{\eps}{|g_j| + \eps} \leq 1$, it holds that $\bar{\mathfrak r}_t^{\rm sgn} \leq 1$. Moreover, we have $\bar{\mathfrak r}_t^{\rm sgn} = \frac{\eps}{\norm{g}_1} \sum_{j=1}^d \frac{|g_j|}{|g_j| + \eps} \leq \frac{d \eps}{\norm{g}_1} \leq \frac{d \eps}{\gamma_{\infty} G} = \frac{d}{\gamma_{\infty} \lambda}$ where the last inequality is by $\norm{g}_1 \geq \gamma_{\infty} G$. Putting things together, we have
\begin{align*}
    \frac{1}{1 + R_1 \lambda} \leq \bar{\mathfrak r}_t^{\rm sgn} \leq \min \{1, \frac{d}{\gamma_{\infty} \lambda} \}.
\end{align*}
From \eqref{eq:prop4_1}, we have $\bar{\mathfrak r}_t^{\rm sgn} - C_0 \eta_t \leq \mathfrak r_t^{\rm sgn}\leq \bar{\mathfrak r}_t^{\rm sgn}  + C_0 \eta_t$. This leads to 
\begin{align}
    \frac{1}{1 + R_1 \lambda} - C_0 \eta_t \leq \mathfrak r_t^{\rm sgn} \leq \min \{1, \frac{d}{\gamma_{\infty} \lambda}\} + C_0 \eta_t. \label{eq:tau_sign}
\end{align}
Given $\mathfrak r_t^{\rm sgn} \geq 0$, the lower bound becomes $\mathfrak r_t^{\rm sgn} \geq \big[ \frac{1}{1 + R_1 \lambda} - C_0 \eta_t\big]_+$ where $[x]_+:=\max\{x,0\}$.

Next, we focus on bounding $\bar{\mathfrak r}_t^{\rm grad}$ (and consequently $\mathfrak r_t^{\rm grad}$ by \eqref{eq:prop4_1}). From the coordinatewise equality $\eps \sigma_{\eps}(g)_j - g_j = -g_j \frac{|g_j|}{|g_j| + \eps}$, we have $\eps \sigma_{\eps} (g) - g = -g \odot \frac{|g|}{|g|+\eps}$. Therefore, the gradient residual $\bar{\mathfrak r}_t^{\rm grad}$ is $\bar{\mathfrak r}_t^{\rm grad} = \frac{\norm{g \odot \frac{|g|}{|g|+\eps}}_2}{\norm{g}_2}$. Given the scalar function $x \rightarrow \frac{x}{x + \eps}$ is increasing for $x \geq 0$, it holds that $\frac{|g_j|}{|g_j|+\eps} \leq \frac{\norm{g}_{\infty}}{\norm{g}_{\infty} + \eps}$ and consequently $\norm{g \odot \frac{|g|}{|g|+\eps}}_2 \leq \frac{\norm{g}_{\infty}}{\norm{g}_{\infty} + \eps} \norm{g}_2$. From this, we conclude that 
\begin{align*}
    \bar{\mathfrak r}_t^{\rm grad} = \frac{\norm{g \odot \frac{|g|}{|g|+\eps}}_2}{\norm{g}_2} \leq \frac{\frac{\norm{g}_{\infty}}{\norm{g}_{\infty} + \eps} \norm{g}_2}{\norm{g}_2} = \frac{\norm{g}_{\infty}}{\norm{g}_{\infty} + \eps}. 
\end{align*}
Using $\norm{g}_{\infty} \leq \norm{g}_1 \leq R_1 G$ and applying the monotonicity of $\frac{x}{x+\eps}$, we obtain
\begin{align*}
    \bar{\mathfrak r}_t^{\rm grad} \leq \frac{R_1 G}{R_1 G + \eps} = \frac{R_1 \lambda}{1 + R_1 \lambda}. 
\end{align*}
For the lower bound, using the reverse triangle inequality $\norm{\eps \sigma_{\eps}(g) - g}_2 \geq \norm{g}_2 - \eps \norm{\sigma_{\eps}(g)}_2$ and dividing by $\norm{g}_2$, we obtain $\bar{\mathfrak r}_t^{\rm grad} \geq 1 - \frac{\eps \norm{\sigma_{\eps}(g)}_2}{\norm{g}_2}$. From $|\sigma_{\eps}(g)_j| = \frac{|g_j|}{|g_j| + \eps} \leq 1$, it holds $\norm{\sigma_{\eps}(g)}_2 \leq \sqrt{d}$. Moreover, we also have $\norm{g}_{2} \geq \gamma_2 G$. Putting them together, we conclude
\begin{align*}
    \bar{\mathfrak r}_t^{\rm grad} \geq 1 - \frac{\eps \norm{\sigma_{\eps}(g)}_2}{\norm{g}_2} \geq 1 - \frac{\eps \sqrt{d}}{\gamma_2 G} = 1 - \frac{\sqrt{d}}{\gamma_2 \lambda}. 
\end{align*}
From \eqref{eq:prop4_1}, we have $\bar{\mathfrak r}_t^{\rm grad} - C_0 \eta_t \leq \mathfrak r_t^{\rm grad}\leq \bar{\mathfrak r}_t^{\rm grad}  + C_0 \eta_t$. This leads to 
\begin{align}
         1 - \frac{\sqrt{d}}{\gamma_2 \lambda} - C_0 \eta_t \leq \mathfrak r_t^{\rm grad} \leq \frac{R_1 G}{R_1 G + \eps} + C_0 \eta_t = \frac{R_1 \lambda}{1 + R_1 \lambda} + C_0 \eta_t. \label{eq:tau_grad}
\end{align}
Given $\mathfrak r_t^{\rm grad}$ is non-negative, the lower bound becomes $\mathfrak r_t^{\rm grad} \geq \bigl[ 1 - \frac{\sqrt{d}}{\gamma_2 \lambda} - C_0 \eta_t \bigr]_{+}$.

Let $t(\eps)$ be any family such that $t(\eps) \rightarrow \infty$ and $\lambda_{t(\eps)} \rightarrow \infty$ as $\eps \downarrow 0$. Since the learning rate is $\eta_t = \eta_0 (t+1)^{-a}$ with $a>0$, we have $\eta_{t(\eps)} \rightarrow 0$. By \eqref{eq:tau_sign}, we have $0 \leq \mathfrak r_{t(\eps)}^{\rm sgn} \leq \min \{ 1, \frac{d}{\gamma_{\infty} \lambda_{t(\eps)} } \} + C_{0} \eta_{t(\eps)}$. Hence, we conclude that $\mathfrak r_{t(\eps)}^{\rm sgn} \rightarrow 0$. For $\mathfrak r_{t(\eps)}^{\rm grad}$, given $1 - \frac{\sqrt{d}}{\gamma_2 \lambda_{t(\eps)}} - C_0 \eta_{t(\eps)} \rightarrow 1$ and $\frac{R_1 \lambda_{t(\eps)}}{1 + R_1 \lambda_{t(\eps
)}} + C_0 \eta_{t(\eps)} \rightarrow 1$, we conclude that $\mathfrak r_{t(\eps)}^{\rm grad} \rightarrow 1$ by \eqref{eq:tau_grad}. 

Let $t(\eps)$ be any family such that $t(\eps) \rightarrow \infty$ and $\lambda_{t(\eps)} \rightarrow 0$ as $\eps \downarrow 0$. We again have $\eta_{t(\eps)} \rightarrow 0$. From \eqref{eq:tau_grad}, we have that $0 \leq \mathfrak r_{t(\eps)}^{\rm grad} \leq \frac{R_1 \lambda_{t(\eps)}}{1 + R_1 \lambda_{t(\eps)}} + C_0 \eta_{t(\eps)}$. Consequently, it holds that $\mathfrak r_{t(\eps)}^{\rm grad} \rightarrow 0$. For $\mathfrak r_{t(\eps)}^{\rm sgn}$, given $\frac{1}{1 + R_1 \lambda_{t(\eps)}} - C_0 \eta_{t(\eps)} \rightarrow 1$ and $\frac{d}{\gamma_{\infty} \lambda_{t(\eps)}} \rightarrow \infty$, we conclude from \eqref{eq:tau_sign} that $\mathfrak r_{t(\eps)}^{\rm sgn} \rightarrow 1$.
\end{proof}

\begin{localproposition} \label{prop:time_bracket_app} For every $b, \delta \in (0,1)$, there exists a constant $\eps_0(b,\delta)$ such that 
\begin{align*}
    T_1^-(\epsilon;b) < T_1^{\delta}(\epsilon) < T_1^+(\epsilon;b), 
\end{align*} 
for every $0 < \eps \leq \eps_0(b,\delta)$. Moreover, it holds that $T^{-}_{1}(\eps;b) \rightarrow \infty$, $\lambda_{T^{-}_{1}(\eps;b)} \rightarrow \infty$, $T^{+}_{1}(\eps;b) \rightarrow \infty$, and $\lambda_{T^{+}_{1}(\eps;b)} \rightarrow 0$ as $\eps \downarrow 0$. Consequently, \cref{prop:local-interpolation_app} gives 
\begin{align*}
        \left(
        \mathfrak r_{T_1^-(\epsilon;b)}^{\rm sgn},
        \mathfrak r_{T_1^-(\epsilon;b)}^{\rm grad}
    \right)
    \rightarrow
    (0,1), \quad \text{and} \quad     \left(
        \mathfrak r_{T_1^+(\epsilon;b)}^{\rm sgn},
        \mathfrak r_{T_1^+(\epsilon;b)}^{\rm grad}
    \right)
    \rightarrow
    (1,0) \qquad \text{as} \, \, \eps \downarrow 0.
\end{align*}
Moreover, it holds for $T_1^{\delta}(\epsilon)$ that 
$\limsup_{\epsilon\downarrow0}
\max\left\{
\left|
r_{T_1^{\delta}(\epsilon)}^{\mathrm{sgn}}-1
\right|,
r_{T_1^{\delta}(\epsilon)}^{\mathrm{grad}}
\right\}
\le
\frac{R_1\delta}{1+R_1\delta}$. 
\end{localproposition}

\begin{proof}[Proof of \cref{prop:time_bracket_app}] Fix any $b,\delta \in (0,1)$. For simplicity, we let $T^{-}_{\eps}:= T^{-}_1(\eps;b)$, $T^{\delta}_{\eps
}:= T^{\delta}_{1}(\eps)$, and $T^{+}_{\eps} := T^{+}_{1} (\eps;b)$. Recall that $\ell_{\eps} = \log \frac{1}{\eps}$. Define $h^{-}_{\eps} := \eps \ell_{\eps}^b$, $h^{\delta}_{\eps} := \delta \eps$, and $h^{+}_{\eps} := \eps \ell_{\eps}^{-b}$. As $\eps \downarrow 0$, it holds that $h^{-}_{\eps} \rightarrow 0$, $h^{\delta}_{\eps} \rightarrow 0$, and $h^{+}_{\eps} \rightarrow 0$. For each fixed $\eps$, given $G_t \rightarrow 0$, we have $T^{-}_{\eps}, T^{\delta}_{\eps},T^{+}_{\eps} < \infty$. As $\ell_{\eps}^b \rightarrow \infty$ and $\ell_{\eps}^{-b} \rightarrow 0$, we have $\ell_{\eps}^b > \delta > \ell_{\eps}^{-b}$ (or equivalently, $h_{\eps}^{-} > h_{\eps}^{\delta} > h_{\eps}^{+}$) for sufficiently small $\eps$. Concretely, we can let $\bar{\eps}_{0}(b,\delta):= \exp(-2 \delta^{-1/b})$ and choose any $\eps$ such that $0< \eps \leq \bar{\eps}_{0}(b,\delta)$. This implies that 
\begin{align*}
    \{ t\geq T_*: G_t \leq h_{\eps}^{+}\} \subseteq \{ t\geq T_*: G_t \leq h_{\eps}^{\delta}\} \subseteq
    \{ t\geq T_*: G_t \leq h_{\eps}^{-}\}.
\end{align*}
Consequently, we have $T_{\eps}^{-} \leq T_{\eps}^{\delta} \leq T_{\eps}^{+}$. 

Let $S_t := \sum_{s=0}^{t-1} \eta_s$ and $K:= R_1 C_D$. By \cref{lem:proxy-path-ratio}, we have $G_t \geq G_0 e^{-K S_t}$. Consequently, it holds that $G_0 e^{-K S_{T_{\eps}^{-}}} \leq G_{T_{\eps}^-} \leq \eps \ell_{\eps}^b$, where the second inequality is by the definition of $T_{\eps}^{-}$. From $G_0 e^{-K S_{T_{\eps}^-}} \leq \eps \ell_{\eps}^b$, we have $S_{T_{\eps}^-} \geq \frac{1}{K} \log  \frac{G_0}{\eps \ell_{\eps}^b}$ by taking logarithms and rearranging. Since $\log \frac{G_0}{\eps \ell_{\eps}^b} = \ell_{\eps} - b \log \ell_{\eps} + \log G_0 \rightarrow \infty$ as $\eps \downarrow 0$, we conclude that $S_{T^-_{\eps}} \rightarrow \infty$ and consequently $T_{\eps}^{-} \rightarrow \infty$. 

Since $T_{\eps}^{-} \rightarrow \infty$, we have $T_{\eps}^{-} > T_* $ for sufficiently small $\eps$. Concretely, we set $\eps_{\rm burn}(b,\delta) := \min \{e^{-2}, \frac{g_*^2}{4} \}$ (recall that $g_*$ is the constant in \cref{prop:uniform-clock}) and choose $0 < \eps \leq \eps_{\rm burn}$. For $\eps \leq e^{-2}$, we have $\ell_{\eps} \geq 2$. Since $0 < b < 1$, we have $\ell_{\eps}^b \leq \ell_{\eps}$. For $\ell_{\eps} \geq 2$, we also have $\ell_{\eps} \leq e^{\ell_{\eps} / 2}$. Consequently, it holds that
\begin{align*}
    \eps \ell_{\eps}^b &= e^{-\ell_{\eps}} \ell_{\eps}^b \leq e^{-\ell_{\eps}}  e^{\ell_{\eps}/2} = e^{-l_{\eps} / 2} = \sqrt{\eps} \leq \frac{g_*}{2} < g_*.
\end{align*}
Therefore, we must have $T_{\eps}^{-} > T_* $ given $G_{T_*} \geq g_* > \eps \ell_{\eps}^b$.  

By minimality of $T^{-}_{\eps}$, we have $G_{T_{\eps}^{-} - 1} > \eps \ell_{\eps}^b$ and $G_{T_{\eps}^{-}} \leq \eps \ell_{\eps}^b$. These inequalities and \cref{lem:proxy-path-ratio} lead to
\begin{align}
    \eps \ell_{\eps}^b \geq G_{T_{\eps}^{-}} \geq e^{-K \eta_{T_{\eps}^{-} - 1}} G_{T_{\eps}^{-} - 1} > e^{-K \eta_{T_{\eps}^{-} - 1}} \eps \ell_{\eps}^b. \label{eq:pp_1}
\end{align}
Therefore, we obtain after rearranging that 
\begin{align*}
    e^{-K \eta_{T_{\eps}^- - 1}} < \frac{G_{T_{\eps}^-}}{\eps \ell_{\eps}^b} \leq 1.
\end{align*}
Given $T_{\eps}^{-} \rightarrow \infty$, we have $\eta_{T_{\eps}^- - 1} \rightarrow 0$ and consequently $e^{-K \eta_{T_{\eps}^- - 1}} \rightarrow 1$. This gives $\frac{G_{T_{\eps}^-}}{\eps \ell_{\eps}^{b}} \rightarrow 1$. It follows that
\begin{align*}
    \lambda_{T_{\eps}^-} = \frac{G_{T_{\eps}^-}}{\eps} = \ell_{\eps}^b \frac{G_{T_{\eps}^-}}{\eps \ell_{\eps}^b} \rightarrow \infty. 
\end{align*}
Set $N := \max \{T_*, \lceil (\frac{K \eta_0}{\log 2})^{1/a} \rceil \}$. Then we have $K \eta_N < \log 2$. Consequently, it holds that $e^{-K \eta_N} > e^{- \log 2} = \frac{1}{2}$. Define $c_N := G_0 e^{-K S_N} > 0$. \cref{lem:proxy-path-ratio} gives $G_t \geq G_0 e^{-K S_t} \geq c_N$ when $0 \leq t \leq N$. Set $x_0(b,\delta) := \max \{2, (2\delta)^{1/b},2 \log \frac{2}{c_N}\}$ and choose $0 < \eps \leq e^{-x_0}$. Letting $x = \log (\frac{1}{\eps})$, we have $x \geq x_0 \geq 2$ and $x^b \leq e^{x/2}$ since $b \in (0,1)$. Consequently, we obtain
\begin{align*}
    \eps \ell_{\eps}^b = e^{-x} x^b \leq e^{-x/2} \leq e^{-x_0 / 2} \leq \frac{c_N}{2} < c_N. 
\end{align*}
Combining this with $G_t \geq c_N, 0\leq t\leq N$, we conclude that $T_{\eps}^{-} > N$ (as $G_t > \eps \ell_{\eps}^b,0 \leq t \leq N$). 
Consequently, it holds that $\eta_{T_{\eps}^{-} - 1} \leq \eta_N$. 
Combining this with \eqref{eq:pp_1}, we obtain
\begin{align}
    \frac{G_{T_{\eps}^{-}}}{\eps} > e^{-K \eta_{T_{\eps}^{-} - 1}} \ell_{\eps}^b \geq e^{-K \eta_N} \ell_{\eps}^b > \frac{1}{2}l_{\eps}^b = \frac{1}{2} x^b \geq \delta, \label{eq:res_1}
\end{align}
where in the last inequality we used $x^b \geq x_0^b \geq 2 \delta$. We have already shown $T_{\eps}^- \leq T_{\eps}^{\delta}$. Therefore, $T_{\eps}^- < T_{\eps}^{\delta}$; otherwise, we would have $G_{T_{\eps}^-} = G_{T_{\eps}^{\delta}} \leq \delta \eps$, contradicting $G_{T_{\eps}^-} > \delta \eps$. 
At $T^{\delta}_{\eps}$, it holds that $G_{T_{\eps}^\delta} \leq \delta \eps$. By \cref{lem:proxy-path-ratio}, we have $G_0 e^{-K S_{T^{\delta}_{\eps}}} \leq \delta \eps$ and therefore 
\begin{align*}
    S_{T^{\delta}_{\eps}} \geq \frac{1}{K} \log \frac{G_0}{\delta \eps} = \frac{1}{K} (\ell_{\eps} + \log \frac{G_0}{\delta}) \rightarrow \infty, \quad \text{as} \, \eps \downarrow 0.
\end{align*}
Consequently, it holds that $T_{\eps}^\delta \rightarrow \infty$. This implies that $T_{\eps}^{\delta} > T_*$ when $\eps$ is sufficiently small. Concretely, we set $\eps_{\rm base}(b,\delta) := \min \{1, \frac{g_*}{2\delta}\}$ and choose $0 < \eps \leq \eps_{\rm base}$. 
Then, we have that $\delta \eps \leq \frac{g_*}{2} < g_* \leq G_{T_*}$ and consequently $T_{\eps}^\delta > T_*$. 
By minimality of $T_{\eps}^{\delta}$ and \cref{lem:proxy-path-ratio}, we have $G_{T_{\eps}^\delta - 1} > \delta \eps$ and $G_{T_{\eps}^\delta} \geq G_{T_{\eps}^\delta - 1} e^{-K \eta_{T_{\eps}^\delta - 1}} > \delta \eps e^{-K \eta_{T_{\eps}^\delta - 1}}$. Together with the defining upper bound, we have $\delta \eps e^{-K \eta_{T_{\eps}^\delta - 1}} < G_{T_{\eps}^\delta} \leq \delta \eps$. After dividing by $\delta \eps$, we obtain
\begin{align}
    e^{-K \eta_{T_{\eps}^\delta - 1}} < \frac{G_{T_{\eps}^\delta}}{\delta \eps} \leq 1. \label{eq:tau_s} 
\end{align}
Consequently, we conclude that $\frac{G_{T_{\eps}^\delta}}{\delta \eps} \rightarrow 1$ given $\eta_{T_{\eps}^\delta - 1} \rightarrow 0$ as $T^{\delta}_{\eps} \rightarrow \infty$. In particular, it holds that $G_{T_{\eps}^{\delta}} > \frac{\delta}{2} \eps$ for sufficiently small $\eps$. Concretely, given 
$\delta \eps e^{-K \eta_{T_{\eps}^\delta - 1}} < G_{T_{\eps}^\delta}$, to obtain $G_{T_{\eps}^{\delta}} > \frac{\delta}{2} \eps$, it is sufficient to guarantee that $e^{-K \eta_{T_{\eps}^\delta - 1}} > \frac{1}{2}$ (or equivalently, $K \eta_{T_{\eps}^\delta - 1} < \log 2$). For the choice of $N$ above, it holds that $e^{-K \eta_{N}} > \frac{1}{2}$. Recall that we have shown $G_t \geq c_N, 0\leq t\leq N$ where $c_N := G_0 e^{-K S_N} > 0$. Set $\bar{\eps}_{3}(b,\delta):= \min \{1, \frac{c_N}{2 \delta}\}$ and choose $0 < \eps \leq \bar{\eps}_{3}(b,\delta)$. Then we have
\begin{align*}
    \delta \eps \leq \frac{c_N}{2} < c_N \leq G_t, \quad \forall \, 0 \leq t\leq N. 
\end{align*}
Hence, the condition $G_t \leq \delta \eps$ cannot occur by time $N$, and therefore $T_{\eps}^{\delta} > N$. Since $T_{\eps}^{\delta} > N$, it also holds that $T_{\eps}^{\delta} - 1 \geq N$ and consequently $\eta_{T_{\eps}^{\delta} - 1} \leq \eta_N$. From this, we conclude that
\begin{align*}
    G_{T_{\eps}^{\delta}} > \delta \eps e^{-K \eta_{T_{\eps}^{\delta} - 1}} \geq \delta \eps e^{-K \eta_{N}} > \frac{\delta}{2} \eps. 
\end{align*}

Moreover, as $\ell_{\eps}^{-b} \rightarrow 0$, we have $\ell_{\eps}^{-b} < \frac{\delta}{2}$ for sufficiently small $\eps$. Concretely, we set  $\eps_{0}^+(b,\delta) = \exp[-2(\frac{2}{\delta})^{1/b}]$ and choose $0 < \eps \leq \eps_{0}^+(b,\delta)$. Combining them together, we have $G_{T_{\eps}^{\delta}} > \frac{\delta}{2} \eps > \eps \ell_{\eps}^{-b}$. Therefore, the threshold for $T_{\eps}^+$ has not yet been reached at $T_{\eps}^{\delta}$. Since $T_{\eps}^{\delta} \leq T_{\eps}^{+}$, we conclude that $T_{\eps}^{\delta} < T_{\eps}^+$. Otherwise, we would have $G_{T_{\eps}^{\delta}} = G_{T_{\eps}^+} \leq \eps \ell_{\eps}^{-b}$, contradicting $G_{T_{\eps}^{\delta}} > \eps l_{\eps}^{-b}$. This finishes the proof of $T_{\eps}^- < T_{\eps}^{\delta} < T_{\eps}^+$. 

We already proved $T^{\delta}_{\eps} \rightarrow \infty$. Since $T_{\eps}^{+} > T_{\eps}^{\delta}$, it follows that $T_{\eps}^+ \rightarrow \infty$. By the definition of $T^{+}_{\eps}$, we have $G_{T^{+}_{\eps}} \leq \eps \ell_{\eps}^{-b}$. Therefore, it holds that $0 < \lambda_{T^+_{\eps}} = \frac{G_{T_{\eps}^+}}{\eps} \leq \ell_{\eps}^{-b}$. Since $\ell_{\eps}^{-b} \rightarrow 0$, we conclude that $\lambda_{T^+_{\eps}} \rightarrow 0$. 
So far, we have shown that $T^{-}_{1}(\eps;b) \rightarrow \infty$, $\lambda_{T^{-}_{1}(\eps;b)} \rightarrow \infty$, $T^{+}_{1}(\eps;b) \rightarrow \infty$, and $\lambda_{T^{+}_{1}(\eps;b)} \rightarrow 0$ as $\eps \downarrow 0$. By 
 \cref{prop:local-interpolation_app}, we have
\begin{align*}
        \left(
        \mathfrak r_{T_1^-(\epsilon;b)}^{\rm sgn},
        \mathfrak r_{T_1^-(\epsilon;b)}^{\rm grad}
    \right)
    \rightarrow
    (0,1), \quad \text{and} \quad     \left(
        \mathfrak r_{T_1^+(\epsilon;b)}^{\rm sgn},
        \mathfrak r_{T_1^+(\epsilon;b)}^{\rm grad}
    \right)
    \rightarrow
    (1,0) \qquad \text{as} \, \, \eps \downarrow 0.
\end{align*}
Since $T_{\eps}^{\delta} \to \infty$ as $\eps \downarrow 0$, we have
$\eta_{T_{\eps}^{\delta}-1}\to 0$. Therefore, we obtain by \eqref{eq:tau_s} that 
\[
\frac{G_{T_{\eps}^{\delta}}}{\delta\epsilon}\to 1,
\qquad
\lambda_{T_{\eps}^{\delta}}
=
\frac{G_{T_{\eps}^{\delta}}}{\epsilon}
\to \delta.
\]
Consequently, applying 
 \cref{prop:local-interpolation_app} gives 
\[
\frac{1}{1+R_1\delta}
\le
\liminf_{\epsilon\downarrow 0}
r_{T_{\eps}^{\delta}}^{\mathrm{sgn}}
\le
\limsup_{\epsilon\downarrow 0}
r_{T_{\eps}^{\delta}}^{\mathrm{sgn}}
\le
\min\left\{
1,\frac{d}{\gamma_\infty\delta}
\right\},
\]
and
\[
\left[
1-\frac{\sqrt{d}}{\gamma_2\delta}
\right]_+
\le
\liminf_{\epsilon\downarrow 0}
r_{T_{\eps}^{\delta}}^{\mathrm{grad}}
\le
\limsup_{\epsilon\downarrow 0}
r_{T_{\eps}^{\delta}}^{\mathrm{grad}}
\le
\frac{R_1\delta}{1+R_1\delta}.
\]
Define $A_\delta:=\frac{R_1\delta}{1+R_1\delta}$.
Since $\frac{1}{1+R_1\delta}=1-A_\delta$,
we have
\begin{align*}
    1-A_\delta-o_\epsilon(1)
\le
r_{T_{\eps}^{\delta}}^{\mathrm{sgn}}
\le
1+o_\epsilon(1), \qquad 0
\le
r_{T_{\eps}^{\delta}}^{\mathrm{grad}}
\le
A_\delta+o_\epsilon(1).
\end{align*}
Equivalently, it holds that 
\[
\limsup_{\epsilon\downarrow0}
\max\left\{
\left|
r_{T_{\eps}^{\delta}}^{\mathrm{sgn}}-1
\right|,
r_{T_{\eps}^{\delta}}^{\mathrm{grad}}
\right\}
\le
\frac{R_1\delta}{1+R_1\delta}.
\]
Because $A_\delta=O(\delta)$, the residual pair at
$T_1^\delta(\epsilon)$ lies within $O(\delta)$ of the gradient-like residual pair $(1,0)$.

Finally, we summarize the requirements on $\eps$. We set 
\begin{align*}
    \eps_{\rm all} (b,\delta) := \min \{ \bar{\eps}_{0}(b,\delta), \eps_{\rm burn}(b,\delta), e^{-x_0(b,\delta)}, \eps_{\rm base}(b,\delta), \bar{\eps}_{3}(b,\delta) , \eps_0^{+}(b,\delta) \},
\end{align*} 
where $\bar{\eps}_{0}(b,\delta) = \exp(-2 \delta^{-1/b})$, $\eps_{\rm burn}(b,\delta) = \min \{e^{-2}, \frac{g_*^2}{4} \}$, $x_0(b,\delta) = \max \{2, (2\delta)^{1/b},2 \log \frac{2}{c_N}\}$, $\eps_{\rm base}(b,\delta) = \min \{1, \frac{g_*}{2\delta} \}$, $\bar{\eps}_{3}(b,\delta) = \min \{1, \frac{c_N}{2 \delta}\}$ ($c_N = G_0 e^{-R_1 C_D S_N}$), and $\eps_{0}^+(b,\delta) = \exp[-2(\frac{2}{\delta})^{1/b}]$. 
\end{proof}

\begin{proof}[Proof of \cref{thm:residual_nonasymp}] For brevity, we write
$T^-:=T_1^-(\eps;b)$, $T^\delta:=T_1^\delta(\eps)$, and
$T^+:=T_1^+(\eps;b)$.
Recall that
$C_D:=\sqrt{\frac{1-\beta_1}{1-\beta_2}}$,
$K_0:=R_1C_D$, and 
$C_{\mathrm{loc}}
:=
C_{\mathrm{ema}}
\max\left\{
\frac{1}{\gamma_\infty},
\frac{1}{\gamma_2}
\right\}$.
Let $\eps_{\mathrm{all}}(b,\delta)$ be the threshold
defined in the proof of \cref{prop:time_bracket_app}. Define
$L_{\mathrm{clk}}
:=
\max\left\{
16,\,
4\left[\log\frac{1}{G_0}\right]_+
\right\}$ and set
$
\Tilde{\eps}_{\mathrm{res}}(b,\delta)
:=
\min\left\{
\eps_{\mathrm{all}}(b,\delta),
e^{-L_{\mathrm{clk}}}
\right\}$.
We prove the result holds for every
$0<\eps\le \Tilde{\eps}_{\mathrm{res}}(b,\delta)$.

By the construction in the proof of \cref{prop:uniform-clock}, we have 
$T_*>t_{\mathrm{ema}}$. Consequently, we also have
$T^-,T^\delta,T^+\ge T_*>t_{\mathrm{ema}}$.
Therefore, \cref{prop:local-interpolation_app} applies at all three stopping times. 
In particular, for every $t\ge t_{\mathrm{ema}}$, we have 
\begin{align}
r_t^{\mathrm{sgn}}
&\le
\min\left\{
1,\frac{d}{\gamma_\infty\lambda_t}
\right\}
+C_{\mathrm{loc}}\eta_t, \qquad 
r_t^{\mathrm{grad}} \le
\frac{R_1\lambda_t}{1+R_1\lambda_t}
+C_{\mathrm{loc}}\eta_t.
\label{eq:res-proof}
\end{align}

We first control $\lambda_t$ at the three stopping times. From \eqref{eq:res_1}, we have 
\begin{equation}
\lambda_{T^-}
=
\frac{G_{T^-}}{\eps}
>
\frac{1}{2}\ell_\eps^b.
\label{eq:lambda-minus-proof}
\end{equation}
At the other two stopping times, their definitions directly give
\begin{equation}
\lambda_{T^\delta}
=
\frac{G_{T^\delta}}{\eps}
\le \delta,
\qquad
\lambda_{T^+}
=
\frac{G_{T^+}}{\eps}
\le \ell_\eps^{-b}.
\label{eq:lambda-delta-plus-proof}
\end{equation}

Substituting \eqref{eq:lambda-minus-proof} into
\eqref{eq:res-proof}, we obtain
\begin{equation}
r_{T^-}^{\mathrm{sgn}}
\le
\frac{2d}{\gamma_\infty}\ell_\eps^{-b}
+C_{\mathrm{loc}}\eta_{T^-}.
\label{eq:sign-minus-proof}
\end{equation}
Similarly, using \eqref{eq:lambda-delta-plus-proof} in
\eqref{eq:res-proof}, together with the monotonicity of
$x\mapsto R_1x/(1+R_1x)$ over $x \geq 0$, we obtain 
\begin{align}
r_{T^+}^{\mathrm{grad}}
&\le
\frac{R_1\ell_\eps^{-b}}
     {1+R_1\ell_\eps^{-b}}
+C_{\mathrm{loc}}\eta_{T^+}
\le
R_1\ell_\eps^{-b}
+C_{\mathrm{loc}}\eta_{T^+},
\label{eq:grad-plus-proof}\\
r_{T^\delta}^{\mathrm{grad}}
&\le
\frac{R_1\delta}{1+R_1\delta}
+C_{\mathrm{loc}}\eta_{T^\delta}.
\label{eq:grad-delta-proof}
\end{align}
Recall that $S_t:=\sum_{s=0}^{t-1}\eta_s$.
 The definition of $T^-$ together with \cref{lem:proxy-path-ratio} imply
$
G_0e^{-K_0S_{T^-}}
\le G_{T^-}
\le \eps\ell_\eps^b$.
Taking logarithms and rearranging gives 
\begin{equation}
S_{T^-}
\ge
\frac{1}{K_0}
\log\frac{G_0}{\eps\ell_\eps^b}
=
\frac{1}{K_0}
\left(
\ell_\eps-b\log\ell_\eps+\log G_0
\right).
\label{eq:S-minus-preliminary}
\end{equation}
Because $\eps\le e^{-L_{\mathrm{clk}}}$, we have
$\ell_\eps\ge L_{\mathrm{clk}}\ge16$. Therefore, it holds that 
$b\log\ell_\eps
\le
\log\ell_\eps
\le
\frac{1}{4}\ell_\eps$. 
Moreover, by the definition of $L_{\mathrm{clk}}$, we have
$
\log G_0
\ge
-\frac{1}{4}\ell_\eps$.
Consequently, we obtain
$
\ell_\eps-b\log\ell_\eps+\log G_0
\ge
\frac{1}{2}\ell_\eps$.
Substituting this into \eqref{eq:S-minus-preliminary} gives $S_{T^-}\ge \frac{\ell_\eps}{2K_0}$.

Suppose first that $a\in(1/3,1)$. Since the function $x^{-a}$ is positive and decreasing, it holds that $\sum_{k=2}^t k^{-a} \leq \int_{1}^{x} x^{-a} dx = \frac{t^{1-a} - 1}{1 - a}$. This implies that $\sum_{k=1}^t k^{-a} \leq 1 + \frac{t^{1-a} - 1}{1 - a} = \frac{t^{1-a} - a}{1 -a} \leq \frac{t^{1-a}}{1-a}$. Consequently, 
for every integer $t\ge1$, we have 
$S_t
=
\eta_0\sum_{k=1}^t k^{-a}
\le
\frac{\eta_0}{1-a}t^{1-a}$.
Combining this with $S_{T^-} \geq \frac{\ell_{\eps}}{2 K_0}$ gives
$T^-
\ge
\left(
\frac{1-a}{2K_0\eta_0}\ell_\eps
\right)^{\frac{1}{1-a}}$.
Consequently, we obtain 
\begin{equation}
\eta_{T^-}
=
\eta_0(T^-+1)^{-a}
\le
\eta_0
\left(
\frac{2K_0\eta_0}{1-a}
\right)^{\frac{a}{1-a}}
\ell_\eps^{-\frac{a}{1-a}}
=
O\left(
\ell_\eps^{-\frac{a}{1-a}}
\right).
\label{eq:eta-minus-poly-proof}
\end{equation}
Since the step sizes are non-increasing and
$T^-<T^\delta<T^+$, we have 
\[
\max\{\eta_{T^-},\eta_{T^\delta},\eta_{T^+}\}
\leq 
\eta_0
\left(
\frac{2K_0\eta_0}{1-a}
\right)^{\frac{a}{1-a}}
\ell_\eps^{-\frac{a}{1-a}}
=
O\left(
\ell_\eps^{-\frac{a}{1-a}}
\right). 
\]
Combining this estimate with
\eqref{eq:sign-minus-proof} and
\eqref{eq:grad-plus-proof} gives
$\max\left\{
r_{T^-}^{\mathrm{sgn}},
r_{T^+}^{\mathrm{grad}}
\right\}
=
O\left(
\ell_\eps^{-b}
+
\ell_\eps^{-\frac{a}{1-a}}
\right)$.
Similarly, \eqref{eq:grad-delta-proof} gives
$r_{T^\delta}^{\mathrm{grad}}
\le
\frac{R_1\delta}{1+R_1\delta}
+
O\left(
\ell_\eps^{-\frac{a}{1-a}}
\right)$.

In the case of $a = 1$, we have
$S_t
=
\eta_0\sum_{k=1}^t\frac{1}{k}
\le
\eta_0(1+\log t)$.
Together with $S_{T^-} \geq \frac{\ell_{\eps}}{2 K_0}$, we obtain 
$\eta_0(1+\log T^-)
\ge
\frac{\ell_\eps}{2K_0}$,
and hence 
\[
T^-
\ge
\exp\left(
\frac{\ell_\eps}{2K_0\eta_0}-1
\right)
=
e^{-1}\eps^{-\frac{1}{2K_0\eta_0}}.
\]
It follows that
\begin{align*}
\eta_{T^-}
=
\frac{\eta_0}{T^-+1}
\le
e\eta_0
\eps^{\frac{1}{2K_0\eta_0}}
=
O\left(
\eps^{\frac{1}{2K_0\eta_0}}
\right).
\end{align*}
Once again, the ordering of the stopping times and the
monotonicity of $\eta_t$ imply
\begin{align}
\max\{\eta_{T^-},\eta_{T^\delta},\eta_{T^+}\}
\le
e\eta_0
\eps^{\frac{1}{2K_0\eta_0}}
=
O\left(
\eps^{\frac{1}{2K_0\eta_0}}
\right).
\label{eq:eta-minus-harmonic-proof}
\end{align}
Therefore, \eqref{eq:grad-delta-proof} yields $r_{T^\delta}^{\mathrm{grad}}
\le
\frac{R_1\delta}{1+R_1\delta}
+
O\left(
\eps^{\frac{1}{2K_0\eta_0}}
\right)$.

Assume that $q>0$ and $b>0$. Define
$
f(x):=e^{-qx}x^b,\, \, x>0$ and 
$\phi(x):=\log f(x)=b\log x-qx$.
Then, we have 
\[
\phi'(x)=\frac{b}{x}-q,
\qquad
\phi''(x)=-\frac{b}{x^2}<0.
\]
Thus, $\phi$ is strictly concave, and its unique critical point
$x=\frac{b}{q}$
is the global maximizer. Given the logarithm preserves the maximizing point, we have 
\[
\begin{aligned}
\sup_{x>0}e^{-qx}x^b
&=e^{-q(b/q)}\left(\frac{b}{q}\right)^b =e^{-b}\left(\frac{b}{q}\right)^b =\left(\frac{b}{eq}\right)^b.
\end{aligned}
\]
In particular, taking $x=\ell_\eps$ gives
$ \eps^q\ell_\eps^b
=
e^{-q\ell_\eps}\ell_\eps^b
\le
\left(\frac{b}{eq}\right)^b$, or equivalently,
$
\eps^q
\le
\left(\frac{b}{eq}\right)^b\ell_\eps^{-b}$. Using this in \eqref{eq:eta-minus-harmonic-proof} with $q = \frac{1}{2 K_0 \eta_0}$, we obtain 
\begin{align*}
    \max\{\eta_{T^-},\eta_{T^+}\}\leq e \eta_0 (\frac{2bK_0 \eta_0}{e})^b \ell_{\eps}^{-b} = O(\ell_{\eps}^{-b}). 
\end{align*}
Substituting this into \eqref{eq:sign-minus-proof} and \eqref{eq:grad-plus-proof} finishes the proof. 
\end{proof}
\section{Proof of \cref{lem:radius}}
\label{app:radius}


\cref{lem:radius} is essential for the proofs of \cref{thm:first-clock} and \cref{thm:margin_app_T1}. Before presenting its full proof, we first provide a proof sketch to illustrate the core steps. 

\paragraph{Proof Sketch} The main difficulty is to retain the sharp coefficient \(1\) in front of
$S_{r,T}:=\sum_{t=r}^{T-1}\eta_t$;
a uniform pointwise bound on \(D_t\) (see \cref{lem:uniform-D}) would only yield a constant multiple of \(S_{r,T}\). The case \(\beta_2=0\) is immediate. In this case, we have \(\beta_1=0\) and the magnitude of every coordinate of \(D_t\) is at most one.
We therefore assume \(\beta_2>0\) and work coordinatewise.

Fixing any \(j\in[d]\), we write
\[
x_t:=g_t[j],\qquad
\mu_t:=\bar m_t[j],\qquad
\nu_t:=\bar v_t[j],\qquad
d_t:=D_t[j].
\]
Let \(c_t:=1-\beta_1^{t+1}\) and \(s_t:=1-\beta_2^{t+1}\). We separate
the bias correction by writing
\[
d_t=\chi_tR_t,\qquad
\chi_t:=\frac{\sqrt{s_t}}{c_t},\qquad
R_t:=\frac{\mu_t}{\sqrt{\nu_t}+\epsilon\sqrt{s_t}}.
\]
By decomposing $d_t$ as $d_t = R_t + (\chi_t - 1)R_t$, it holds that
\begin{align*}
    |\sum_{t=r}^{T-1} \eta_t d_t| \leq |\sum_{t=r}^{T-1} \eta_t R_t| + \sum_{t=r}^{T-1} |\chi_{t-1}| |R_t| 
\end{align*}
A weighted Cauchy--Schwarz inequality applied to the  EMA recursions of $\mu_t$ and $\nu_t$ gives $|\mu_t|^2\leq K_\beta^2\nu_t,
$. This implies that 
\(|R_t|\leq K_\beta\), from which we show the second term is bounded by $O(\eta_r)$. The difficulty is in controlling the first term.  
We split the first-moment at time \(r\):
\[
\mu_t=o_t+n_t,\qquad
o_t:=\beta_1^{t-r+1}\mu_{r-1},\qquad
n_t:=(1-\beta_1)\sum_{s=r}^t\beta_1^{t-s}x_s.
\]
Consequently, the first term is bounded as
\begin{align*}
    |\sum_{t=r}^{T-1} \eta_t R_t| \leq \underbrace{|\sum_{t=r}^{T-1} \eta_t \frac{n_t}{\sqrt{\nu_t} + \epsilon \sqrt{s_t}}|}_{\rm post-r  \, \, contribution} + \underbrace{\sum_{t=r}^{T-1} \eta_t \frac{|o_t|}{\sqrt{\nu_t} + \epsilon \sqrt{s_t}}}_{\rm pre-r \, \, contribution}.   
\end{align*}
The pre-\(r\) contribution is controlled by a geometric series:
its numerator scales as \(\beta_1^{t-r+1}\), whereas the corresponding
part of \(\sqrt{\nu_t}\) scales as
\(\beta_2^{(t-r+1)/2}\). Since
\(\beta_1/\sqrt{\beta_2}<1\), its cumulative contribution is
\(O(\eta_r)\). It remains to bound the post-r contribution. We let $N := T -r$ and define
\begin{align*}
    \alpha_{k} = \eta_{r+k}, \qquad M_k := n_{r+k}
= (1-\beta_1)\sum_{\ell=0}^{k}
\beta_1^\ell x_{r+k-\ell}, \qquad \Delta_k:=\sqrt{\nu_{r+k}}+\epsilon\sqrt{s_{r+k}}.
\end{align*}
Then, it holds that $\sum_{k=0}^{N-1} \alpha_k \frac{M_k}{\Delta_{k}} = \sum_{t=r}^{T-1} \eta_t \frac{n_t}{\sqrt{\nu_t} + \epsilon \sqrt{s_t}}$. By Cauchy-Schwarz, we have $|\sum_{k=0}^{N-1} \alpha_k \frac{M_k}{\triangle_k}|^2 \leq S_{r,T} (\sum_{k=0}^{N-1} \alpha_k \frac{M_k^2}{\Delta_k^2})$. It remains to control the ratio $\frac{M_k^2}{\Delta_k^2}$. A natural lower bound on $\Delta_k^2$ is $\nu_{r+k} + \eps^2 s_{r+k}$. However, this quantity contains all history before time \(r\), whereas \(M_k\) contains only the first-moment history after time \(r\). Therefore, they have mismatched initial conditions. To avoid this, we introduce an alternative lower bound $Q_k$ whose index starts from $r$. We show 
\begin{align*}
    \nu_{r+k}
\geq
(1-\beta_2)\sum_{\ell=0}^{k}
\beta_2^\ell x_{r+k-\ell}^2, \qquad s_{r+k} \geq (1 - \beta_2) \sum_{l=0}^k \beta_2^l
\end{align*}
Combining them together, we have 
\begin{align*}
    \nu_{r+k}+\epsilon^2 s_{r+k}
\geq
(1-\beta_2)\sum_{\ell=0}^{k}
\beta_2^\ell
\left(x_{r+k-\ell}^2+\epsilon^2\right). 
\end{align*}
The right side of this inequality is precisely the definition of $Q_k$, i.e. $Q_k
:=
(1-\beta_2)\sum_{\ell=0}^{k}
\beta_2^\ell
\left(x_{r+k-\ell}^2+\epsilon^2\right)$. We can write $M_k$ and $Q_k$ in one-step recursion form:
\begin{align*}
M_k
=
\beta_1 M_{k-1}
+
(1-\beta_1)x_{r+k}, \qquad
    Q_k =
\beta_2 Q_{k-1} +
(1-\beta_2)\left(x_{r+k}^2+\epsilon^2\right),
\end{align*}
with $M_{-1} = Q_{-1} = 0$. Therefore, $M_k$ and $Q_k$ share the same initial conditions. This makes the recurrence and logarithmic cancellation clean. Using the lower bound $Q_k$ (on $\Delta_k^2$), we next control the ratio $\frac{M_k^2}{Q_k}$ (note that $\frac{M_k^2}{\Delta_k^2} \leq \frac{M_k^2}{Q_k}$). 

Let $h_t^2 := x_t^2 + \epsilon^2$. Using the upper bound $M_k \leq (1 - \beta_1) \sum_{l=0}^k \beta_1^l h_{r+k - l}^2$, the relationship $h_{r+k - l}^2 = \frac{Q_{k-l} - \beta_2 Q_{k-l-1}}{1 - \beta_2}$ (obtained from the $Q_k$ recursion), and the inequality $1 - z \leq - \log z$ (applied with $z = \frac{Q_{k-l}}{Q_k} > 0$), we obtain 
\[
\frac{M_k^2}{Q_k}
\leq
\underbrace{1}_{\text{main term}}+\underbrace{c_\beta 
\beta_1^k}_{{\text{geometric transient}}} + \underbrace{c_\beta (1-\beta_1)\sum_{\ell=1}^k
\beta_1^{\ell-1}
\bigl(\log Q_k-\log Q_{k-\ell}\bigr)}_{\text{weighted logarithmic increments}},
\]
where \(c_\beta\) depends only on \((\beta_1,\beta_2)\). 
After multiplying by
\(\alpha_k=\eta_{r+k}\) and summing over \(k\), the main term contributes
exactly \(S_{r,T}\), the geometric transient contributes
\(O(\eta_r)\), and the logarithmic increments contribute
\(O(\eta_rH_\epsilon)\). This is what preserves the sharp leading
coefficient \(1\).
\begin{proof} [Proof of \cref{lem:radius}]
    Define $H_{\eps}:= \log \frac{B_g^2 + \eps^2}{(1 - \beta_2) \epsilon^2}$. Fix one coordinate $j$, and let $x_t := g_t[j]$, $\mu_t := \bar{m}_t[j]$, $\nu_t := \bar{v}_t[j]$, and $d_t := D_t[j]$. Then, we have 
    \begin{align}
        \mu_t = \beta_1 \mu_{t-1} + (1 - \beta_1) x_t, \qquad \nu_t = \beta_2 \nu_{t-1} + (1 - \beta_2) x_t^2, \label{eq:radius_1}
    \end{align}
    with $\mu_{-1} = \nu_{-1} = 0$. Moreover, it also holds that $d_t = \frac{\mu_t/(1 - \beta_1^{t+1})}{\sqrt{\nu_t/(1 - \beta_2^{t+1})} + \eps}$. We prove that
    \begin{align}
        \bigl | \sum_{t=r}^{T-1} \eta_t d_t \bigr | \leq \sum_{t=r}^{T-1} \eta_t + C_{\rm rad} \eta_r (1 + H_{\epsilon}), \label{eq:main_weight}
    \end{align}
    for some constant $C_{\rm rad}$ that works for all coordinates $j$. In the special case where $\beta_2 = 0$, it holds that $\beta_1 = 0$ as $0 \leq \beta_1 \leq \beta_2$. Thus, we have $\mu_t = x_t$, $\nu_t = x_t^2$, and $d_t = \frac{x_t}{|x_t| + \epsilon}$. Therefore, it holds that $|d_t| \leq 1$ and consequently $|\sum_{t=r}^{T-1} \eta_t d_t| \leq \sum_{t=r}^{T-1} \eta_t$. In this case, the result holds with $C_{\rm rad} = 0$. From now on, we assume that $\beta_2 > 0$. 

    \paragraph{Part 1} By defining the bias-correction factors: $c_t := 1 - \beta_1^{t+1}$ and $s_t := 1 - \beta_2^{t+1}$, we have $d_t = \frac{\mu_t / c_t}{\sqrt{\nu_t / s_t} + \epsilon} = \frac{\sqrt{v_t}}{c_t} \frac{\mu_t}{\sqrt{\nu_t} + \epsilon \sqrt{s_t}}$. We define 
    \begin{align*}
        \chi_t := \frac{\sqrt{s_t}}{c_t}, \qquad R_t := \frac{\mu_t}{\sqrt{\nu_t} + \epsilon \sqrt{s_t}}.
    \end{align*}
    Then, it holds that $d_t = \chi_t R_t$. The proof first bounds $R_t$, and then controls the difference $d_t - R_t = (\chi_t - 1)R_t$. Next, we show $|\mu_t|^2 \leq K_{\beta}^2 := \frac{\beta_2 (1 - \beta_1)^2}{(1 - \beta_2)(\beta_2 - \beta_1^2)}$. This constant is finite as $\beta_1^2 \leq \beta_2^2 \leq \beta_2$ (implied by $\beta_1 \leq \beta_2 < 1$).
    Unrolling \eqref{eq:radius_1} gives $\mu_t = (1 - \beta_1) \sum_{k=0}^t \beta_1^k x_{t - k}$ and $\nu_t = (1- \beta_2) \sum_{k=0}^t \beta_2^k x^2_{t - k}$. By Cauchy-Schwarz, we obtain
    \begin{align*}
        |\mu_t|^2 &= (1 - \beta_1)^2 \bigr| \sum_{k=0}^t (\frac{\beta_1}{\sqrt{\beta_2}})^k \beta_2^{k/2} x_{t-k}\bigl|^2 \\ 
        &\leq (1 - \beta_1)^2 \bigl ( \sum_{k=0}^t (\frac{\beta_1^2}{\beta_2})^k \bigr) \bigl ( \sum_{k=0}^t \beta_2^k x^2_{t-k} \bigr). 
    \end{align*}
    From $\sum_{k=0}^t (\frac{\beta_1^2}{\beta_2})^k \leq \frac{1}{1 - \beta_1^2/ \beta_2} = \frac{\beta_2}{\beta_2 - \beta_1^2}$ (as $\frac{\beta_1^2}{\beta_2} < 1$) and $\sum_{k=0}^t \beta_2^k x^2_{t-k} = \frac{\nu_t}{1 - \beta_2}$, we further obtain
    \begin{align*}
        |\mu_t|^2 \leq \frac{\beta_2 (1 - \beta_1)^2}{(1 - \beta_2)(\beta_2 - \beta_1^2)} \nu_t = K_{\beta}^2 \nu_t, \qquad \forall t.
    \end{align*}
    If $\nu_t = 0$, then all $x_0,...,x_t$ are zero and $\mu_t = 0$ holds as well. In this case, the inequality trivially holds. Consequently, it holds that 
    \begin{align}
        |R_t| = \frac{|\mu_t|}{\sqrt{\nu_t} + \eps \sqrt{s_t}} \leq \frac{|\mu_t|}{\sqrt{\nu_t}} \leq K_{\beta}. \label{eq:radius_R}
    \end{align}

    For $t \geq r$, we split the first moment into two parts
    \begin{align*}
        \mu_t = o_t + n_t, \qquad o_t:= \beta_1^{t-r+1} \mu_{r-1}, \qquad n_t:= (1-\beta_1) \sum_{s=r}^t \beta_1^{t-s} x_s. 
    \end{align*}
    For the second moment, we have $\nu_t = \beta_2^{t-r+1} \nu_{r-1} + (1 - \beta_2) \sum_{s=r}^t \beta_2^{t-s} x_s^2 \geq \beta_2^{t-r+1} \nu_{r-1}$. Using $|\mu_{r-1}| \leq K_{\beta} \sqrt{\nu_{r-1}}$, we obtain
    \begin{align*}
        \frac{|o_t|}{\sqrt{\nu_t} + \epsilon \sqrt{s_t}} \leq \frac{\beta_1^{t-r+1}|\mu_{r-1}|}{\sqrt{\nu_t}} \leq \frac{\beta_1^{t-r+1} K_{\beta} \sqrt{\nu_{r-1}}}{\beta_2^{(t-r+1)/2} \sqrt{\nu_{r-1}}} = K_{\beta}(\frac{\beta_1}{\sqrt{\beta_2}})^{t-r+1}. 
    \end{align*}
    Define $\rho_{\beta}:= \frac{\beta_1}{\sqrt{\beta_2}}$. We have $\beta_1^2 \leq \beta_2^2 < \beta_2$. Since $\eta_t \leq \eta_r, t \geq r$, we obtain
    \begin{align}
        \sum_{t=r}^{T-1} \eta_t \frac{|o_t|}{\sqrt{\nu_t} + \epsilon \sqrt{s_t}} \leq K_{\beta} \eta_r \sum_{k=1}^{\infty} \rho_{\beta}^k = K_{\beta} \eta_r \frac{\rho_{\beta}}{1 - \rho_{\beta}}. \label{eq:radius_ot}
    \end{align}

    \paragraph{Part 2} Set $N := T - r$. For $k = 0,...,N-1$, we define
    \begin{align}
        M_k := n_{r + k} = (1 - \beta_1) \sum_{l=0}^k \beta_1^l x_{r+k - l}, \qquad Q_k := (1 - \beta_2) \sum_{l=0}^k \beta_2^l (x^2_{r+k-l} + \epsilon^2) \label{eq:radius_2}
    \end{align} 
    From $\nu_{r+k} \geq (1 - \beta_2) \sum_{l=0}^k \beta_2^l x^2_{r+k - l}$  and $s_{r+k} = 1 - \beta_2^{r+k+1} \geq 1 - \beta_2^{k+1} = (1 - \beta_2) \sum_{l=0}^k \beta_2^l$, we obtain using the definition of $Q_k$ that
    \begin{align*}
        \nu_{r+k} + \epsilon^2 s_{r+k} \geq Q_k.
    \end{align*}
    Moreover, we have
    \begin{align*}
        (\sqrt{\nu_{r+k}} + \epsilon \sqrt{s_{r+k}})^2 = \nu_{r+k} + \epsilon^2 s_{r+k} + 2\epsilon \sqrt{\nu_{r+k} s_{r+k}} \geq \nu_{r+k} + \epsilon^2 s_{r+k}.
    \end{align*}
    Hence, we obtain $(\sqrt{\nu_{r+k}} + \epsilon \sqrt{s_{r+k}})^2 \geq Q_k$ or equivalently $\sqrt{\nu_{r+k}} + \epsilon \sqrt{s_{r+k}} \geq \sqrt{Q_k}$. We note that $Q_k \geq (1-\beta_2) \eps^2$ because the $l=0$ term contains $\epsilon^2$. Since $|x_t| \leq B_g$, it holds that $x_t^2 + \epsilon^2 \leq B_g^2 + \epsilon^2$. Moreover, we have $(1-\beta_2) \sum_{l=0}^k \beta_2^l = 1 - \beta_2^{k+1} \leq 1$. Therefore, we obtain
    \begin{align}
        (1 - \beta_2) \epsilon^2 \leq Q_k \leq (1-\beta_2) \sum_{l=0}^k \beta_2^l (B_g^2 + \epsilon^2) \leq B_g^2 + \epsilon^2. \label{eq:radius_Q}
    \end{align}
    Let $L_{\epsilon}:= (1 - \beta_2) \epsilon^2$ and $U_{\epsilon}:= B_g^2 + \epsilon^2$. Thus, we have shown that $\log Q_i \in [\log L_{\epsilon}, \log U_{\epsilon}]$ for any $i \in \{ 0,...,k \}$. Consequently, we obtain
    \begin{align*}
        \max_{u,v \in \{0,...,k\}} |\log Q_u - \log Q_v| \leq \log U_{\eps} - \log L_{\eps} = \log \frac{U_{\epsilon}}{L_{\epsilon}} =: H_{\epsilon}. 
    \end{align*}
    Define $h^2_t := x_t^2 + \epsilon^2, \forall t \in \{r,...,T-1\}$ and $w_l:= (1 - \beta_1) \beta_1^l$. By Cauchy-Schwarz, we have
    \begin{align*}
        M_k^2 &= [(1 - \beta_1) \sum_{l=0}^k \beta_1^l x_{r+k-l}]^2 = [(1 - \beta_1)]^2 [\sum_{l=0}^k \sqrt{\beta_1^l} \sqrt{\beta_1^l} x_{r+k-l}]^2\\
        &\leq [(1 - \beta_1) \sum_{l=0}^k \beta_1^l][(1 - \beta_1) \sum_{l=0}^k \beta_1^l x_{r+k-l}^2]\\ 
        &\leq (1 - \beta_1) \sum_{l=0}^k \beta_1^l x_{r+k-l}^2 \\
        &\leq (1 - \beta_1) \sum_{l=0}^k \beta_1^l h_{r+k-l}^2. 
    \end{align*}
    where in the second inequality we used $\sum_{l=0}^k w_l = 1 - \beta_1^{k+1} \leq 1$. 

    By the definition of $Q_k$, we have 
    \begin{align}
        Q_k = \beta_2 Q_{k-1} + (1 - \beta_2) h_{r+k}^2, \qquad Q_{-1} := 0. \label{eq:radius_4}
    \end{align}
    Thus, it holds that $h_{r+k-l}^2 = \frac{Q_{k-l} - \beta_2 Q_{k-l-1}}{1 - \beta_2}, \forall l \in \{0,...,k\}$. Consequently, we obtain
    \begin{align}
        M_k^2 &\leq \frac{1 - \beta_1}{1 - \beta_2} \sum_{l=0}^k \beta_1^l (Q_{k-l} - \beta_2 Q_{k-l-1}) \nonumber \\
        &= \frac{1 - \beta_1}{1 - \beta_2} [Q_k - (\beta_2 - \beta_1) \sum_{l=1}^k \beta_1^{l-1} Q_{k-l}].  \label{eq:radius_3}
    \end{align}
    Define $c_{\beta}:= \frac{\beta_2 - \beta_1}{1 - \beta_2} \geq 0$. Then we have $\frac{1 - \beta_1}{1 - \beta_2} = 1 + c_{\beta}$. Dividing \eqref{eq:radius_3} by $Q_k$ and using the definition of $c_{\beta}$, we obtain
    \begin{align*}
        \frac{M_k^2}{Q_k} \leq 1 + c_{\beta}[1 - (1 - \beta_1) \sum_{l=1}^k \beta_1^{l-1} \frac{Q_{k-l}}{Q_k}]. 
    \end{align*}
    From $(1 - \beta_1) \sum_{l=1}^k \beta_1^{l-1} = 1 - \beta_1^k$, we have 
    \begin{align*}
        1 - (1 - \beta_1) \sum_{l=1}^k \beta_1^{l-1} \frac{Q_{k-l}}{Q_k} = \beta_1^k + (1 - \beta_1) \sum_{l=1}^k \beta_1^{l-1}(1 - \frac{Q_{k-l}}{Q_k}). 
    \end{align*}
    For every $z > 0$, it holds that $1 - z \leq - \log z$. Applying this with $z = \frac{Q_{-l}}{Q_k}$, we have $1 - \frac{Q_{k-l}}{Q_k} \leq \log Q_k - \log Q_{k-l}$. Putting things together, we have
    \begin{align}
        \frac{M_k^2}{Q_k} \leq 1 + c_{\beta}[\beta_1^k + (1 - \beta_1)\sum_{l=1}^k \beta_1^{l-1} (\log Q_k - \log Q_{k-l})] \label{eq:radius_5}.
    \end{align}
    Next, we verify \eqref{eq:radius_5} still holds in the $\beta_1= 0$ case. In this case, we have $M_k = x_{r+k}$, $c_{\beta} = \frac{\beta_2}{1-\beta_2}$, and the recursion in \eqref{eq:radius_4} still holds. Note that all logarithms are well-defined as $Q_k \geq (1 - \beta_2) \epsilon^2 > 0$ as $\epsilon > 0$. When $k = 0$, we have $M_0 = x_r$, $Q_0 = (1 - \beta_2)(x_r^2 + \epsilon^2)$, and therefore $\frac{M_0^2}{Q_0} = \frac{x_r^2}{(1 - \beta_2) (x_r^2 + \epsilon^2)} \leq \frac{1}{1 - \beta_2}$. Thus, \eqref{eq:radius_5} is valid at $k = 0$. For $k \geq 1$, we have $M_k^2 = x_{r+k}^2 \leq x_{r+k}^2 + \epsilon^2 = \frac{Q_k - \beta_2 Q_{k-1}}{1 - \beta_2}$ where the last equality is by the recursion $Q_k - \beta_2 Q_{k-1} = (1 - \beta_2)(x_{r+k}^2 + \epsilon^2)$. Dividing by $Q_k$ gives
    \begin{align*}
        \frac{M_k^2}{Q_k} \leq \frac{Q_k - \beta_2 Q_{k-1}}{(1 - \beta_2) Q_k} &= \frac{1 - \beta_2(Q_{k-1} / Q_k)}{1 - \beta_2} = 1 + \frac{\beta_2}{1 - \beta_2}(1 - \frac{Q_{k-1}}{Q_k}) \\ 
        & \leq 1 + \frac{\beta_2}{1 - \beta_2}(\log Q_k - \log Q_{k-1}). 
    \end{align*}
    This is the inequality \eqref{eq:radius_5} given $\beta_1^k = 0$ and only the $l = 1$ term survives with the convention that $0^0 = 1$. 
    \paragraph{Part 3} Define $\alpha_k := \eta_{r+k}, k \in \{0,...,N-1 \}$. After multiplying \eqref{eq:radius_5} by $\alpha_k$ and summing, we obtain
    \begin{align}
        \sum_{k=0}^{N-1} \alpha_k \frac{M_k^2}{Q_k} \leq \sum_{k=0}^{N-1} \alpha_k + c_{\beta} \sum_{k=0}^{N-1} \alpha_k \beta_1^k + c_{\beta} \Omega, \label{eq:radius_main}
    \end{align}
    where 
    \begin{align*}
        \Omega := (1 - \beta_1) \sum_{k=0}^{N-1} \alpha_k \sum_{l=1}^k \beta_1^{l - 1} (\log Q_k - \log Q_{k-l}).
    \end{align*}
    We now bound the term $\Omega$. We reorganize the terms in $\Omega$ and write it as $\Omega = \sum_{u=0}^{N-1} A_u \log Q_u$. The positive appearances of $\log Q_u$ come from $k=u$ and their total coefficient is $\alpha_u(1 - \beta_1) \sum_{l=1}^{u} \beta_1^{l-1} = \alpha_u (1 - \beta_1^u)$. The negative appearances occur when $k - l = u$, or $k = u + l$. Thus, we have 
    \begin{align}
        A_u = \alpha_u (1 - \beta_1^u) - (1 - \beta_1) \sum_{l=1}^{N - 1 - u} \alpha_{u + l} \beta_1^{l-1}. \label{eq:radius_6}
    \end{align}
    Next, we show two important properties of $A_u$. After summing \eqref{eq:radius_6}, we have
    \begin{align*}
        \sum_{u=0}^{N-1} A_u &= \sum_{u=0}^{N-1} \alpha_u (1 - \beta_1^u) - (1 - \beta_1) \sum_{u=0}^{N-1} \sum_{l=1}^{N - 1 - u} \alpha_{u + l} \beta_1^{l - 1} \\ 
        &= \sum_{u=0}^{N-1} \alpha_u (1 - \beta_1^u) - (1 - \beta_1) \sum_{k=1}^{N-1} \alpha_k \sum_{l=1}^k \beta_1^{l -1} \\
        &= \sum_{u=0}^{N-1} \alpha_u (1 - \beta_1^u) - \sum_{k=1}^{N-1} \alpha_k (1 - \beta_1^k) = 0,
    \end{align*}
    where in the second equality we set $k = u+l$. Then  $k = 1,...,N-1$, and for each $k$, we have $l = 1,...,k$. For the last equality, we used $\alpha_0 (1 - \beta_1^0) = 0$. Consequently, we conclude that $\sum_{u=0}^{N-1} A_u = 0$. For the second property, as $\alpha_{u + l} \leq \alpha_u$, we have 
    \begin{align*}
        A_u \geq \alpha_u (1 - \beta_1^u) - (1 - \beta_1) \alpha_u \sum_{l=1}^{N - 1 - u} \beta_1^{l-1} = \alpha_u (1 - \beta_1^u) - \alpha_u (1 - \beta_1^{N -1 - u}) = \alpha_u (\beta_1^{N-1-u} - \beta_1^u). 
    \end{align*}
    In particular, we have $A_u \geq -\alpha_u \beta_1^u$. Consequently, whenever $A_u < 0$, it holds that $-A_u \leq \alpha_u \beta_1^u$. From \eqref{eq:radius_Q}, we have $\log L_{\epsilon} \leq \log Q_u \leq \log U_{\epsilon}$ for any $u \in \{0,...,N-1\}$. Define $l_u := \log Q_u - \log U_{\epsilon}$. From the definition of $H_{\epsilon}$, we have $-H_{\epsilon} \leq l_u \leq 0, \forall u \in \{0,...,N-1\}$. Since $\sum_{u=0}^{N-1} A_u = 0$, we have $\sum_{u=0}^{N-1} A_u \log Q_u = \sum_{u=0}^{N-1} A_u l_u$. If $A_u \geq 0$, then we have $A_u l_u \leq 0$ because $l_u \leq 0$. If $A_u < 0$, from $- A_u \leq \alpha_u \beta_1^u$ and $-l_u \leq H_{\epsilon}$, we have $A_u l_u = (-A_u) (- l_u) \leq \alpha_u \beta_1^u H_{\epsilon}$. Therefore, we obtain $\Omega \leq H_{\epsilon} \sum_{u=0}^{N-1} \alpha_u \beta_1^u$. Since $\alpha_u \leq \alpha_0 = \eta_r$, it holds that $\sum_{u=0}^{N-1} \alpha_u \beta_1^u \leq \eta_r \sum_{u=0}^{\infty} \beta_1^u = \frac{\eta_r}{1 - \beta_1}$. Consequently, we have $\Omega \leq \frac{\eta_r H_{\epsilon}}{1 - \beta_1}$. Substituting these bounds into \eqref{eq:radius_main}, we obtain
    \begin{align*}
        \sum_{k=0}^{N-1} \alpha_k \frac{M_k^2}{Q_k} \leq \sum_{k=0}^{N-1} \alpha_k + \frac{c_{\beta}}{1 - \beta_1} \eta_r (1 + H_{\epsilon}). 
    \end{align*}
    Define $S_{r,T}:= \sum_{t=r}^{T-1} \eta_t = \sum_{k=0}^{N-1} \alpha_k$ and $C_{E} := \frac{c_{\beta}}{1 - \beta_1} = \frac{\beta_2 - \beta_1}{(1 - \beta_2)(1 - \beta_1)}$.Then we can further write
    \begin{align}
        \sum_{k=0}^{N-1} \alpha_k \frac{M_k^2}{Q_k} \leq S_{r,T} + C_{E} \eta_r (1 + H_{\epsilon}). \label{eq:radius_m2}
    \end{align}
    \paragraph{Part 4} Define $\triangle_k := \sqrt{\nu_{r+k}} + \epsilon \sqrt{s_{r+k}}$. Previously, we have shown that $\triangle_k^2 \geq Q_k$. By Cauchy-Schwarz, we have 
    \begin{align*}
        |\sum_{k=0}^{N-1} \alpha_k \frac{M_k}{\triangle_k}|^2 &\leq (\sum_{k=0}^{N-1} \alpha_k) (\sum_{k=0}^{N-1} \alpha_k \frac{M_k^2}{\triangle_k^2}) = S_{r,T} (\sum_{k=0}^{N-1} \alpha_k \frac{M_k^2}{\triangle_k^2}) \leq S_{r,T} (\sum_{k=0}^{N-1} \alpha_k \frac{M_k^2}{Q_k})\\ 
        &\leq S_{r,T}[S_{r,T} + C_{E} \eta_r(1 + H_{\epsilon}) ],
    \end{align*}
    where we used \eqref{eq:radius_m2} in the last inequality. For any $A,B \geq 0$, we have $\sqrt{A(A+B)} \leq A + \frac{B}{2}$. Applying this fact, we further obtain that
    \begin{align}
        |\sum_{k=0}^{N-1} \alpha_k \frac{M_k}{\triangle_k}| \leq S_{r,T} + \frac{C_{E}}{2} \eta_r (1 + H_{\epsilon}). \label{eq:radius_m3}
    \end{align}
    Recall that $\alpha_k = \eta_{r + k}$, $M_k = n_{r+k}$, and $\triangle_k = \sqrt{\nu_{r+k}} + \epsilon \sqrt{s_{r+k}}$. Consequently, it holds that 
    \begin{align*}
        \sum_{k=0}^{N-1} \alpha_k \frac{M_k}{\triangle_k} = \sum_{k=0}^{N-1} \eta_{r+k} \frac{n_{r+k}}{\sqrt{\nu_{r+k}} + \epsilon \sqrt{s_{r+k}}} = \sum_{t=r}^{T-1} \eta_t \frac{n_t}{\sqrt{\nu_t} + \epsilon \sqrt{s_t}}. 
    \end{align*}
    From this and \eqref{eq:radius_m3}, we have
    \begin{align}
        |\sum_{t=r}^{T-1} \eta_t \frac{n_t}{\sqrt{\nu_t} + \epsilon \sqrt{s_t}}| \leq S_{r,T} + \frac{C_E}{2} \eta_r (1 + H_{\epsilon}). \label{eq:radius_nt}
    \end{align}
    \paragraph{Part 5} Recall that $R_t = \frac{o_t + n_t}{\sqrt{\nu_t} + \eps \sqrt{s_t}}$. Combining \eqref{eq:radius_ot} and \eqref{eq:radius_nt}, we obtain 
    \begin{align*}
        |\sum_{t=r}^{T-1} \eta_t R_t| &\leq |\sum_{t=r}^{T-1}\eta_t \frac{n_t}{\sqrt{\nu_t} + \epsilon \sqrt{s_t}}| + \sum_{t=r}^{T-1} \eta_t \frac{|o_t|}{\sqrt{\nu_t} + \epsilon \sqrt{s_t}} \\
        &\leq S_{r,T} + \frac{C_E}{2} \eta_r (1 + H_{\epsilon}) + C_o \eta_r,
    \end{align*}
    where $C_o := K_{\beta} \frac{\rho_{\beta}}{1 - \rho_{\beta}}$. Since $H_{\epsilon} \geq 0$, it holds that $C_{o} \eta_r \leq C_o \eta_r (1 + H_{\epsilon})$. Hence, we obtain
    \begin{align*} 
        |\sum_{t=r}^{T-1} \eta_t R_t| \leq S_{r,T} + (\frac{C_{E}}{2} + C_o) \eta_r (1 + H_{\epsilon}). 
    \end{align*}
    Recall that $\chi_t = \frac{\sqrt{s_t}}{c_t}$, $c_t = 1 - \beta_1^{t+1}$, and $s_t = 1 - \beta_2^{t+1}$. Since $c_t \geq 1 - \beta_1$, we have
    \begin{align*}
        |\chi_t - 1| = \frac{|\sqrt{s_t} - c_t|}{c_t} \leq \frac{|1- \sqrt{s_t}| + |1 - c_t|}{1 - \beta_1}.
    \end{align*}
    Note that $1 - c_t = \beta_1^{t+1}$. For any $z \in [0,1]$, $1 - \sqrt{1 - z} = \frac{z}{1 + \sqrt{1 - z}} \leq z$. Taking $z = \beta_2^{t+1}$, we have $1 - \sqrt{s_t} = 1 - \sqrt{1 - \beta_{2}^{t+1}} \leq \beta_2^{t+1}$. Hence, we have
    \begin{align}
        |\chi_t - 1| \leq \frac{\beta_1^{t+1} + \beta_2^{t+1}}{1 - \beta_1}. \label{eq:radius_chi} 
    \end{align}
    From \eqref{eq:radius_R} and \eqref{eq:radius_chi}, we have 
    \begin{align*}
        \sum_{t = r}^{T - 1} \eta_t |\chi_t - 1| |R_t| &\leq \frac{K_{\beta}}{1 - \beta_1} \sum_{t = r}^{T - 1} \eta_t (\beta_1^{t+1} + \beta_2^{t+1}) \\
        &\leq \frac{K_{\beta} \eta_r}{1 - \beta_1} [\sum_{t=r}^\infty \beta_1^{t+1} + \sum_{t=r}^{\infty} \beta_2^{t+1}]\\
        &\leq \frac{K_{\beta} \eta_r}{1 - \beta_1} [\frac{\beta_1}{1 - \beta_1} + \frac{\beta_2}{1 - \beta_2}].
    \end{align*}
    Define $C_{\rm bc}:= \frac{K_{\beta}}{1 - \beta_1}[\frac{\beta_1}{1 - \beta_1} + \frac{\beta_2}{1 - \beta_2}]$. Then we have $\sum_{t = r}^{T - 1} \eta_t |\chi_t - 1| |R_t| \leq C_{\rm bc} \eta_r$. Since $d_t = \chi_t R_t = R_t + (\chi_t - 1)R_t$, we have
    \begin{align*}
        |\sum_{t = r}^{T - 1} \eta_t d_t| &\leq |\sum_{t = r}^{T - 1} \eta_t R_t| + \sum_{t = r}^{T - 1} \eta_t |\chi_t - 1| |R_t| \\
        &\leq S_{r,T} + (\frac{C_E}{2} + C_{o} + C_{\rm bc}) \eta_r (1 + H_{\epsilon}).
    \end{align*}
    This proves \eqref{eq:main_weight} with $C_{\rm rad} = \frac{C_E}{2} + C_o + C_{\rm bc}$. 
    \paragraph{Part 6} Recall that $d_t = D_t[j]$. Substituting this into \eqref{eq:main_weight} and using the definition of $H_{\epsilon}$, we have
    \begin{align*}
        |\sum_{t = r}^{T - 1} \eta_t D_t[j]| &\leq \sum_{t = r}^{T - 1} \eta_t + C_{\rm rad} \eta_r [1 + \log \frac{B_g^2 + \epsilon^2}{(1 - \beta_2) \epsilon^2}]
    \end{align*}
    Taking the maximum over the coordinates $j$ completes the proof.
\end{proof}
\section{Adam's Local Update Transition}
\label{app:transition}
Throughout this section, we assume data separability, $0 \leq \beta_1 \leq \beta_2 < 1$, and $\eps \in (0,1]$. Recall that $R_1  :=\max_{i \in [n]} \norm{x_i}_1$ and $R_2  :=\max_{i \in [n]} \norm{x_i}_2$. The step size schedule is $\eta_t = \eta_0 (t+1)^{-a}$ with $\eta_0 > 0$ and $a \in (1/3,1]$. 

\subsection{Proof of \cref{thm:first-clock}}
Recall the three stopping times:
\begin{align*}
    T_1^-(\epsilon;b) &=
    \inf\left\{
        t\ge T_{*}:
        G_t
        \le
        \epsilon\ell_\epsilon^b
    \right\},\\
    T_1^{\delta}(\epsilon)
    &=
    \inf\left\{
        t\ge T_{*}:
        G_t
        \le
        \delta\epsilon
    \right\},\\
    T_1^+(\epsilon;b)
    &=
    \inf\left\{
        t\ge T_*:
        G_t
        \le
        \epsilon\ell_\epsilon^{-b}
    \right\},
\end{align*}
where $\ell_{\epsilon} = \log \frac{1}{\epsilon}$ and $b \in (0,1)$. For simplicity, we let $T^- := T_1^-(\epsilon;b)$, $T^{\delta}:=T_1^{\delta}(\epsilon)$, and $T^+:= T_1^+(\epsilon;b)$. 

We first show a lemma that follows directly from \cref{prop:uniform-clock}.
\begin{lemma} \label{lem:trans_aux}
For every $\zeta > 0$, there exists $t_{\zeta} < \infty$ (independent of $\epsilon$) such that for every $\epsilon \in (0,1]$ 
\begin{align}
    L_t \leq \zeta, \qquad \forall t \geq t_{\zeta}. \label{eq:lemma_23}
\end{align}
\begin{proof} Fix $\zeta>0$. We use the constants $c,K$ from \eqref{eq:lowloss_1} and the notations $t_0:=t_{\mathrm{ema}}$,
$R_t:=\sum_{s=t}^{\infty}\eta_s^3$, and $\Phi_t:=L_t+KR_t$
introduced in the proof of \cref{prop:uniform-clock}. By \eqref{eq:lowloss_descent}, we have
\[
0\le \Phi_{t+1}\le \Phi_t-c\eta_tG_t^2,
\qquad t\ge t_0.
\]
Moreover, the bound in \eqref{eq:phi0} gives
$\Phi_{t_0}\le B$ where $B<\infty$ is a constant independent of $\epsilon$. Therefore, we obtain 
\begin{align}
    0\le L_t\le\Phi_t\le B, \qquad t\ge t_0, \label{eq:mon}
\end{align}
and $\Phi_t$ is non-increasing on this range. Since $R_t\downarrow0$, we can choose an integer $r_\zeta\ge t_0$ 
(independent of $\epsilon$) such that $KR_{r_\zeta}\le\frac{\zeta}{2}$. Using the estimate established in the proof of \cref{prop:uniform-clock}, one possible choice is
\[
r_\zeta
:=
\max\left\{
t_0,\,
1,\,
\left\lceil
\left(
\frac{2K\eta_0^3}{(3a-1)\zeta}
\right)^{1/(3a-1)}
\right\rceil
\right\}.
\]
This is the choice of $t_1$ in that proof with $h$ replaced by $\zeta$. The argument leading to \eqref{eq:lowloss_3} uses only $h>0$, $KR_{t_1}\le h/2$, \eqref{eq:lowloss_descent}, and \eqref{eq:lowloss_2}. Therefore, replacing $h \text{ by } \zeta$, $t_1 \text{ by } r_\zeta$, and 
$\bar g \text{ by } \bar g_\zeta:=\frac{1}{n} (1-e^{-\zeta/2}) > 0$, we obtain the following 
\[
t\ge r_\zeta,\quad \Phi_t>\zeta
\quad\Longrightarrow\quad
\Phi_{t+1}\le\Phi_t-c\bar g_\zeta^{\,2}\eta_t.
\]
Because $\sum_{t=0}^{\infty}\eta_t=\infty$, we can choose an integer
$t_\zeta>r_\zeta$ such that $c\bar g_\zeta^{\,2}
\sum_{t=r_\zeta}^{t_\zeta-1}\eta_t>B$.
Concretely, we define
\[
t_\zeta
:=
\min\left\{
T\in\mathbb{N}:
T>r_\zeta,\quad
c\bar g_\zeta^{\,2}
\sum_{t=r_\zeta}^{T-1}\eta_t>B
\right\}.
\]
The set is nonempty and $t_\zeta$ is independent of $\epsilon$. For every $\epsilon\in(0,1]$, the arguments in \cref{prop:uniform-clock} give a time $\tau=\tau(\epsilon,\zeta)\in[r_\zeta,t_\zeta]$ such that $\Phi_\tau\le\zeta$ (note that $\tau$ may depend on $\epsilon$ but  $t_\zeta$ does not). By the monotonicity of $\Phi_t$ following from \eqref{eq:lowloss_descent}, we obtain
\[
L_t\le\Phi_t\le\Phi_{\tau}\le\zeta,
\qquad t\ge t_\zeta.
\]
This holds for every $\epsilon\in(0,1]$.
\end{proof}
\end{lemma}

\begin{lemma} \label{lem:ratio} Let $h_{\epsilon} > 0 $ satisfy $h_{\epsilon} \rightarrow 0$ as $\epsilon \downarrow 0$ and define $\tau_{\epsilon} := \inf \{ t \geq T_*: G_t \leq h_{\epsilon}\}$. The following holds: 
\begin{align*}
    \tau_{\epsilon} \rightarrow \infty, \qquad \frac{G_{\tau_{\epsilon}}}{h_{\epsilon}} \rightarrow 1, \qquad \frac{L_{\tau_{\epsilon}}}{h_{\epsilon}} \rightarrow 1, \qquad \text{as} \, \epsilon \downarrow 0. 
\end{align*}
\end{lemma}

\begin{proof} 
Denote $K_0 := R_1 C_D$ where $C_D = \sqrt{\frac{1 - \beta_1}{1 - \beta_2}}$. By \cref{lem:proxy-path-ratio}, it holds that $G_t \geq G_0 e^{-K_0 S_t}$. Fix any deterministic integer $N$. Since $S_t \leq S_{N}$ for $t \leq N$, we have 
\begin{align*}
    G_t \geq G_0 e^{-K_0 S_{N}} =: c_N > 0, \qquad 0 \leq t \leq N. 
\end{align*}
For sufficiently small $\epsilon$, we have $h_{\epsilon} < c_N$. Therefore, the event $G_t \leq h_{\epsilon}$ cannot occur by time $N$. Therefore, it holds that $\tau_{\epsilon} > N$. Since $N$ was arbitrary, we conclude that $\tau_{\epsilon} \rightarrow \infty$.  

By minimality of $\tau_{\epsilon}$, we have $G_{\tau_{\epsilon} - 1} > h_{\epsilon}$ and $G_{\tau_{\epsilon}} \leq h_{\epsilon}$. Applying \cref{lem:proxy-path-ratio} over the single step from $\tau_{\epsilon} - 1$ to $\tau_{\epsilon}$, we obtain
\begin{align*}
    G_{\tau_{\epsilon}} \geq e^{-K_0 \eta_{\tau_{\epsilon} - 1}} G_{\tau_{\epsilon} - 1} > e^{-K_0 \eta_{\tau_{\epsilon} - 1}} h_{\epsilon}. 
\end{align*}
Therefore, it holds that $e^{-K_0 \eta_{\tau_{\epsilon} - 1}} < \frac{G_{\tau_{\epsilon}}}{h_{\epsilon}} \leq 1$. Because $\tau_{\epsilon} \rightarrow \infty$, it holds that $e^{-K_0 \eta_{\tau_{\epsilon} - 1}} \rightarrow 1$. Therefore, we conclude that $\frac{G_{\tau_{\epsilon}}}{h_{\epsilon}} \rightarrow 1$. For $t \geq T_{*}$, \cref{prop:uniform-clock} gives $G_t \leq L_t \leq 2 G_t$. Therefore, it holds that $L_{\tau_{\epsilon}} \leq 2 G_{\tau_{\epsilon}} \leq 2 h_{\epsilon} \rightarrow 0$. By \cref{lem:loss-proxy}, $1 \geq \frac{G_t}{L_t} \geq 1 - \frac{n L_t}{2}$. Since  $L_{\tau_{\epsilon}} \rightarrow 0$, we have $\frac{G_{\tau_{\epsilon}}}{L_{\tau_{\epsilon}}} \rightarrow 1$ and consequently $\frac{L_{\tau_{\epsilon}}}{G_{\tau_{\epsilon}}} \rightarrow 1$. Combining this with $\frac{G_{\tau_{\epsilon}}}{h_{\epsilon}} \rightarrow 1$, we obtain $\frac{L_{\tau_{\epsilon}}}{h_{\epsilon}} = \frac{L_{\tau_{\epsilon}}}{G_{\tau_{\epsilon}}} \frac{G_{\tau_{\epsilon}}}{h_{\epsilon}} \rightarrow 1$.   
\end{proof}

Applying \cref{lem:ratio} to the three stopping times, we obtain 
\begin{align}
    & G_{T^-} = \epsilon \ell_{\epsilon}^b (1 + o(1)), \qquad L_{T^-} = \epsilon \ell_{\epsilon}^b (1 + o(1)), \notag \\
    & G_{T^{\delta}} = \delta \epsilon (1 + o(1)), \qquad L_{T^{\delta}} = \delta \epsilon (1 + o(1)), \notag \\
    & G_{T^+} = \epsilon \ell_{\epsilon}^{-b} (1 + o(1)), \qquad L_{T^+} = \epsilon \ell_{\epsilon}^{-b} (1 + o(1)). \label{eq:lem26}
\end{align}
Moreover, all three stopping times diverge as
\(\epsilon\downarrow0\), and hence the corresponding learning rates
converge to zero.
\paragraph{Proof Sketch of \cref{thm:first-clock}} We first provide a proof sketch of \cref{thm:first-clock}. The main idea is that the loss satisfies two qualitatively different
recursions:
\[
L_{t+1}-L_t\approx-\gamma_\infty\eta_tL_t
\quad\text{before the crossover},
\]
whereas
\[
L_{t+1}-L_t\asymp
-\frac{\eta_t}{\epsilon}L_t^2
\quad\text{after the crossover}.
\]
Therefore, the loss decays exponentially in the
\(S\)-clock before the crossover, whereas after the crossover, the reciprocal loss grows
linearly in \(S_t/\epsilon\). 

\emph{\textbf{Part 1: Characterization of $S_{T^-}$.}}
Fix \(0<\zeta<2/n\). By \cref{lem:trans_aux}, there exists an
\(\epsilon\)-independent time \(t_\zeta\) such that
$L_t\leq \zeta$ when $t\geq t_\zeta$. Choose an integer $r$ such that $r \geq t_{\zeta}$, $r \geq T_*$, $C_{\rm md} \eta_r < \gamma_{\infty}$, and $\gamma_{\infty} \eta_{r} \leq \frac{1}{2}$. We highlight the important roles played by $r$ in the proof:

\begin{enumerate}
    \item To make \(G_t\) arbitrarily close to \(L_t\). By choosing \(r\geq t_\zeta\), \cref{lem:loss-proxy}  gives
\[
\frac{G_t}{L_t}
\geq
1-\frac{nL_t}{2}
\geq
1-\frac{n\zeta}{2}
=:\alpha_\zeta,
\qquad
t\geq r.
\] 
Therefore, it holds that $G_t\geq\alpha_\zeta L_t$ where $\alpha_\zeta\longrightarrow1$ as $\zeta\downarrow0$. This estimate is needed to recover the sharp leading coefficient
\(1/\gamma_\infty\). If we started directly from \(T_*\), we would only
know that $G_t\geq\frac{1}{2}L_t$.
This would lose a factor of 2 in the resulting upper bound (see \eqref{eq:firstclock_m1}). 

\item To obtain a strict contraction. 
Before \(T^-\), the proof obtains
\begin{align}
    L_{t+1} \leq
(1-\alpha_{\zeta} q_{\epsilon,r}\eta_t) L_t, \qquad q_{\epsilon,r}
:=
\gamma_\infty
-C_{\mathrm{md}}\eta_r
-d\ell_\epsilon^{-b}. \label{eq:recur}
\end{align}
For fixed \(r\), we have $q_{\epsilon,r}
\rightarrow
\gamma_\infty-C_{\mathrm{md}}\eta_r>0$ as $\epsilon\downarrow0$. Therefore, it holds that $0 < q_{\epsilon,r} \leq \gamma_{\infty}$ when $\epsilon$
is sufficiently small. Consequently, the recursion in \eqref{eq:recur} is a strict contraction since $0 <
\alpha_\zeta q_{\epsilon,r}\eta_t
\leq
\gamma_\infty\eta_r
\leq
\frac{1}{2}$ and it holds for every $t \geq r$. 
\item To separate the finite burn-in from the \(\epsilon\)-dependent time scale. We know $T^-\rightarrow \infty$ as $\epsilon\downarrow0$. Therefore, it holds that $T^ >r+1$ when $\eps$ is sufficiently small, and the recursion in \eqref{eq:recur} holds throughout the interval $r \leq t < T^-$. Moreover, because \(r\) is fixed, we have $S_r-S_{T_*} = O(1) = o(\ell_\epsilon)$. Thus, removing the finite initial segment \([T_*,r]\) does not affect the leading-order asymptotics on the \(\ell_\epsilon\) scale. The limits are taken in the order $\epsilon\downarrow0$, $r\to\infty$, and $\zeta\downarrow0$. For fixed \(r\) and \(\zeta\), the proof first obtains
\[
\limsup_{\epsilon\downarrow0}
\frac{S_{T^-}-S_{T_*}}{\ell_\epsilon}
\leq
\frac{1}{
\alpha_\zeta
\left(
\gamma_\infty-C_{\mathrm{md}}\eta_r
\right)}.
\] 
Since $\eta_r \rightarrow 0$ (as \(r\to\infty\)) and $\alpha_{\zeta} \rightarrow 1$ (as \(\zeta\downarrow0\)), we obtain $\limsup_{\epsilon\downarrow0}
\frac{S_{T^-}-S_{T_*}}{\ell_\epsilon}
\leq
\frac{1}{\gamma_\infty}$. 
\end{enumerate}
For the reverse inequality, \cref{prop:uniform-clock} gives $L_{T^-} \leq 2 G_{T^{-}} \leq 2 \eps \ell_{\epsilon}^b$. The loss-margin inequality in \cref{lem:loss-proxy} therefore yields
\begin{align*}
    \rho(w_{T^-}) \geq
\log\frac{c_{\log}}{L_{T^-}} \geq
\ell_\epsilon-b\log\ell_\epsilon
+\log\frac{c_{\log}}{2},
\end{align*}
where \(c_{\log}:=(\log 2)/n\). This implies that 
$\liminf_{\epsilon\downarrow0}
\frac{\rho(w_{T^-})}{\ell_\epsilon}
\geq 1$. By the definition of \(\gamma_\infty\), we have $\rho(w_{T^-})\leq\gamma_\infty\|w_{T^-}\|_\infty$. On the other hand, writing $w_{T^-} =
w_r-\sum_{t=r}^{T^--1}\eta_tD_t$ and applying \cref{lem:radius} gives 
\[
\|w_{T^-}\|_\infty
\leq
M_r+(S_{T^-}-S_r)
+C_{\mathrm{rad}}\eta_r
\left[
1+
\log
\frac{R_1^2+\epsilon^2}
     {(1-\beta_2)\epsilon^2}
\right],
\]
where \(M_r\) is a bound on $\norm{w_r}_{\infty}$ independent of \(\epsilon\). Chaining these two inequalities, rearranging, dividing by $\ell_{\epsilon}$, sending $\epsilon \downarrow 0$, and then letting $r \rightarrow \infty$, we obtain $\liminf_{\epsilon\downarrow0}
\frac{S_{T^-}-S_{T_*}}{\ell_\epsilon}
\geq
\frac{1}{\gamma_\infty}$. Combining the upper and lower bounds gives
\[
S_{T^-}-S_{T_*}
=
\frac{1}{\gamma_\infty}\ell_\epsilon
+o(\ell_\epsilon).
\]
\emph{\textbf{Part 2: Characterization of $S_{T^{\delta}} - S_{T^-}$.}} For \(T^-\leq t<T^\delta\), the definition of \(T^\delta\) gives
$G_t>\delta\epsilon$. Using the soft-sign bound from \cref{lem:softsign} and the gradient bound in \cref{lem:proxy}, we obtain
\begin{align*}
    \left\langle g_t,\sigma_\epsilon(g_t)\right\rangle \geq
\frac{\|g_t\|_1^2}{\|g_t\|_1+d\epsilon} \geq
\frac{\gamma_\infty^2G_t^2}
     {R_1G_t+d\epsilon} \geq
\frac{\gamma_\infty^2}
     {R_1+d/\delta}G_t
=:\kappa_\delta G_t.
\end{align*}
Because \(T^-\to\infty\) as $\epsilon \downarrow 0$, the term
\(C_{\mathrm{md}}\eta_t^2G_t\) in the loss recursion can be absorbed
for sufficiently small \(\epsilon\). Using \(G_t\geq L_t/2\), we obtain $L_{t+1}
\leq
\left(
1-\frac{\kappa_\delta}{4}\eta_t
\right)L_t$ when $T^-\leq t<T^\delta$.
Since \(L_{T^-}\leq2\epsilon\ell_\epsilon^b\), it follows after unrolling that
\[
L_t
\leq
2\epsilon\ell_\epsilon^b
\exp\left(
-\frac{\kappa_\delta}{4}
(S_t-S_{T^-})
\right), \qquad T^{-} \leq t \leq T^{\delta}. 
\]
The right-hand side is at most \(\delta\epsilon\) once
\[
S_t-S_{T^-}
\geq
\frac{4}{\kappa_\delta}
\log\frac{2\ell_\epsilon^b}{\delta} =: A_{\epsilon} >  0,
\]
where the last inequality holds for sufficiently small $\epsilon$. Consequently, if we denote $\bar{T} = \inf \{ t \geq T^-: S_t - S_{T^-} \geq A_{\eps} \}$, it must hold that $\bar{T} \geq T^- + 1$ as $S_{\bar{T}} - S_{T^-} > 0$. Moreover, it also holds that $\bar{T} \geq T^{\delta}$. Otherwise, we would have $G_{\bar{T}} \leq L_{\bar{T}} \leq \delta \epsilon$, contradicting the definition of $T^{\delta}$. Based on these two inequalities, we have by the minimality of $\bar{T}$ that 
\begin{align*}
    S_{T^{\delta}} - S_{T^{-}} \leq S_{\bar{T}} - S_{T^-} = S_{\bar{T} - 1} - S_{T^-} + \eta_{\bar{T} -1} < A_{\epsilon} + \eta_{\bar{T} -1} \leq A_{\epsilon} + \eta_{T^-},
\end{align*}
where we used $\eta_{\bar{T} - 1} \leq \eta_{T^-}$ as $\bar{T} - 1 \geq T^{-}$. Based on this, it follows that $S_{T^{\delta}} - S_{T^-} = O(\log \ell_{\epsilon})$. Since \(\log\ell_\epsilon=o(\ell_\epsilon)\), using the results of Part 1, we obtain
\[
S_{T^\delta}-S_{T_*}
= S_{T^\delta} - S_{T^-} + S_{T^-} - S_{T_*} =
\frac{1}{\gamma_\infty}\ell_\epsilon
+o(\ell_\epsilon).
\]

\emph{\textbf{Part 3: Characterization of $S_{T^{+}} - S_{T^{\delta}}$.}} Set $\tau:=T^\delta$.
We first prove inductively that $G_t\leq L_t\leq2\delta\epsilon, \,
t\geq\tau$. 
At \(t=\tau\), this follows from
$G_\tau\leq\delta\epsilon$ and 
$L_\tau\leq2G_\tau$. 
Suppose that \(L_t\leq2\delta\epsilon\). Then
it holds that
$\frac{G_t}{\epsilon}
\leq
\frac{L_t}{\epsilon}
\leq 2\delta$.
\cref{lem:alignment} gives
$\epsilon
\left\|D_t-\sigma_\epsilon(g_t)\right\|_2
\leq
C_{\mathrm{ema}}\eta_tG_t$.
Moreover, we can show $
\epsilon\|\sigma_\epsilon(g_t)\|_2
\leq
\|g_t\|_2
\leq
R_2G_t$. Combining them together, we obtain
$\|D_t\|_2
\leq
M_D\frac{G_t}{\epsilon} $
for a constant \(M_D\) independent of \(t\) and \(\epsilon\).
Using $ |g_t[j]|
\leq
\|g_t\|_1
\leq
R_1G_t
\leq
2R_1\delta\epsilon$
and \cref{lem:proxy}, we obtain
\[
a_\delta\frac{G_t^2}{\epsilon}
\leq
\left\langle g_t,\sigma_\epsilon(g_t)\right\rangle
\leq
R_2^2\frac{G_t^2}{\epsilon},
\qquad
a_\delta
:=
\frac{\gamma_2^2}{1+2R_1\delta}.
\]
The bound in \cref{lem:alignment} further gives $
\left|
\left\langle
g_t,D_t-\sigma_\epsilon(g_t)
\right\rangle
\right|
\leq
C_{\mathrm{ema}}R_2\eta_t
\frac{G_t^2}{\epsilon}$.
Since \(\eta_\tau\to0\) as \(\tau\to\infty\), the difference caused by using $D_t$ instead of $\sigma_{\epsilon}(g_t)$ can be
absorbed for sufficiently small \(\epsilon\). Therefore, it holds that 
\[
\frac{a_\delta}{2}\frac{G_t^2}{\epsilon}
\leq
\langle g_t,D_t\rangle
\leq
A_\delta\frac{G_t^2}{\epsilon}
\]
for a constant \(A_\delta>0\). The Taylor expansion of the loss is $L_{t+1}
=
L_t-\eta_t\langle g_t,D_t\rangle+R_t$. Since $\frac{G_{t}}{\epsilon} \leq 2 \delta$
and $\norm{D_t}_2 \leq M_{D}\frac{G_t}{\epsilon}$, 
we can show the Taylor remainder $R_t$ satisfies
$0\leq R_t
\leq
B_\delta\eta_t^2
\frac{G_t^3}{\epsilon^2}$
for a constant \(B_\delta>0\). 
Because \(G_t/\epsilon\leq2\delta\) and
\(\eta_t\leq\eta_\tau\to0\), the Taylor remainder is lower order than 
the first-order descent term. Consequently, we have 
\[
\frac{a_\delta}{4}
\frac{\eta_t}{\epsilon}G_t^2
\leq
L_t-L_{t+1}
\leq
A_\delta
\frac{\eta_t}{\epsilon}G_t^2.
\]
In particular, it holds that \(L_{t+1}\leq L_t\leq2\delta\epsilon\), which closes the
induction.

Since \(t\geq\tau\geq T_*\), \cref{prop:uniform-clock} gives $
\frac{1}{2}L_t\leq G_t\leq L_t$.
It follows that 
\begin{align}
    c_\delta
\frac{\eta_t}{\epsilon}L_t^2
\leq
L_t-L_{t+1}
\leq
A_\delta
\frac{\eta_t}{\epsilon}L_t^2, \label{eq:s_l}
\end{align}
for some \(c_\delta>0\). We also derive that
\(L_{t+1}\geq L_t/2\) for sufficiently small \(\epsilon\). Since $\frac{1}{L_{t+1}} - \frac{1}{L_t} = \frac{L_t - L_{t+1}}{L_t L_{t+1}}$, $L_t L_{t+1} \leq L_t^2$, and $L_t L_{t+1} \geq \frac{1}{2} L_t^2$, we further obtain via \eqref{eq:s_l} that 
\[
c_\delta\frac{\eta_t}{\epsilon}
\leq
\frac{1}{L_{t+1}}-\frac{1}{L_t}
\leq
C_\delta\frac{\eta_t}{\epsilon}.
\]
Summing from \(t=T^\delta\) to \(T^+-1\) gives
\[
c_\delta
\frac{S_{T^+}-S_{T^\delta}}{\epsilon}
\leq
\frac{1}{L_{T^+}}-\frac{1}{L_{T^\delta}}
\leq
C_\delta
\frac{S_{T^+}-S_{T^\delta}}{\epsilon}.
\]
Finally, the stopping-time estimates from \cref{lem:ratio} imply
\[
\begin{aligned}
\epsilon
\left(
\frac{1}{L_{T^+}}
-
\frac{1}{L_{T^\delta}}
\right)
&=
\ell_\epsilon^b(1+o(1))
-\frac{1}{\delta}(1+o(1)) =
\Theta(\ell_\epsilon^b),
\end{aligned}
\]
from which we conclude that 
$S_{T^+}-S_{T^\delta} = \Theta(\ell_\epsilon^b)$.

\begin{proof}[Proof of \cref{thm:first-clock}]
    Fix $0 < \zeta < \frac{2}{n}$ and define $\alpha_{\zeta} := 1 - \frac{n \zeta}{2} > 0$. By \cref{lem:trans_aux}, there is a time $t_{\zeta}$ (independent of $\epsilon$) such that $L_t \leq \zeta, t\geq t_{\zeta}$. Choose an integer $r$ such that $r \geq t_{\zeta}$, $r \geq T_*$, $C_{\rm md} \eta_r < \gamma_{\infty}$, and $\gamma_{\infty} \eta_{r} \leq \frac{1}{2}$. Recall that $C_{\rm md}$ is the constant in \eqref{eq:master-descent-app}, i.e.
    \begin{align}
        L_{t+1} \leq L_t - \eta_t \langle g_t, \sigma_{\epsilon}(g_t)\rangle + C_{\rm md} \eta_t^2 G_t. \label{eq:firstclock_1}
    \end{align}
    Importantly, the integer $r$ is fixed before sending $\epsilon \downarrow 0$. Since $T^{-} \rightarrow \infty$ by \cref{lem:ratio}, we have $T^{-} > r + 1$ for all sufficiently small $\epsilon$. For $r \leq t < T^{-}$, the definition of $T^-$ gives $G_t > h_{\epsilon} = \epsilon \ell_{\epsilon}^b$. From $\langle g_t, \sigma_{\epsilon}(g_t) \rangle \geq \norm{g_t}_1 - d \epsilon$ (\cref{lem:softsign}) and $\norm{g_t}_1 \geq \gamma_{\infty} G_t$ (\cref{lem:proxy}), we have $\langle g_t,  \sigma_{\epsilon}(g_t)\rangle \geq \gamma_{\infty} G_t - d \epsilon$. Substituting this into \eqref{eq:firstclock_1}, we obtain
    \begin{align}
        L_{t+1} &\leq L_t - \eta_t (\gamma_{\infty} G_t - d \epsilon) + C_{\rm md} \eta_t^2 G_t \nonumber \\
        &= L_t - \eta_t (\gamma_{\infty} - C_{\rm md} \eta_t - \frac{d \epsilon}{G_t}) G_t. \label{eq:firstclock_2}
    \end{align}
    From  $G_t > \epsilon \ell_{\epsilon}^b$, we have $\frac{\epsilon}{G_t} < \frac{1}{\ell_{\epsilon}^b}$. This implies that 
    $\gamma_{\infty} - C_{\rm md} \eta_t - \frac{d \epsilon}{G_t} \geq \gamma_{\infty} - C_{\rm md} \eta_r - \frac{d}{l_{\epsilon}^b}$. Define $q_{\epsilon,r}:= \gamma_{\infty} - C_{\rm md} \eta_r - \frac{d}{\ell_{\epsilon}^b}$. For fixed $r$, since $\ell_{\epsilon}^b \rightarrow \infty$, we have $q_{\epsilon,r} \rightarrow \gamma_{\infty} - C_{\rm md} \eta_r > 0$. Therefore, we have $q_{\epsilon,r} > 0$ for all sufficiently small $\epsilon$ (for example, choosing $\epsilon$ such that $\frac{d}{\ell_{\epsilon}^b} \leq \frac{\gamma_{\infty} - C_{\rm md} \eta_r}{2}$) and 
    \begin{align}
        L_{t+1} \leq L_t - q_{\epsilon,r} \eta_t G_t. \label{eq:firstclock_3}
    \end{align}   

    \paragraph{Part 1} From \cref{lem:loss-proxy}, we know $\frac{G_t}{L_t} \geq 1 - \frac{n L_t}{2}$. Since $t \geq r \geq t_{\zeta}$, from \eqref{eq:lemma_23}, we have $G_t \geq (1 - \frac{n \zeta}{2}) L_t = \alpha_{\zeta} L_t$. Substituting this into \eqref{eq:firstclock_3}, we have $L_{t+1} \leq (1 - \alpha_{\zeta} q_{\epsilon,r} \eta_t) L_t$. Given $0 < q_{\epsilon,r} \leq \gamma_{\infty}$ for sufficiently small $\epsilon$, we have $0 < \alpha_{\zeta} q_{\epsilon,r} \eta_t \leq \gamma_{\infty} \eta_r \leq \frac{1}{2}$.  Iterating from $r$ to $t-1$ and using $1 - x \leq e^{-x}$, we obtain
    \begin{align}
        L_t \leq L_r \prod_{s=r}^{t-1} (1 - \alpha_{\zeta} q_{\epsilon,r} \eta_s) \leq L_r \exp(-\alpha_{\zeta} q_{\epsilon,r} \sum_{s=r}^{t-1} \eta_s) = L_r \exp(-\alpha_{\zeta} q_{\epsilon,r}(S_t - S_r)). \label{eq:firstclock_4}
    \end{align}
    This holds for every $r \leq t < T^{-}$. By minimality of $T^-$ and \cref{lem:loss-proxy}, we have $L_{T^- - 1} \geq G_{T^- - 1} > h_{\epsilon}$. Applying \eqref{eq:firstclock_4} with $t = T^- - 1$, we obtain 
    \begin{align*}
        h_{\epsilon} < L_{T^- - 1} \leq L_{r} \exp(- \alpha_{\zeta} q_{\epsilon,r} (S_{T^- - 1} - S_r)). 
    \end{align*}
    Rearranging this leads to $S_{T^- - 1} - S_r \leq \frac{1}{\alpha_{\zeta} q_{\epsilon,r}} \log \frac{L_r}{h_{\epsilon}}$. Since $S_{T^-} = S_{T^- - 1} + \eta_{T^- - 1}$ and $\eta_{T^- - 1} \leq \eta_r$, we further obtain
    \begin{align*}
        S_{T^-} - S_r \leq \frac{1}{\alpha_{\zeta} q_{\epsilon,r}} \log \frac{L_r}{h_{\epsilon}} + \eta_r.
    \end{align*}
    Because $L_r \leq \zeta$, we have $\log \frac{L_r}{h_{\epsilon}} \leq \log \frac{\zeta}{\epsilon \ell_{\epsilon}^b} = \ell_{\epsilon} - b \log \ell_{\epsilon} + \log \zeta$. 
    Given this, we obtain
    \begin{align*}
        S_{T^-} - S_r \leq \frac{\ell_{\epsilon} - b \log \ell_{\epsilon} + \log \zeta}{\alpha_{\zeta} (\gamma_{\infty} - C_{\rm md} \eta_r - d \ell_{\epsilon}^{-b})} + \eta_r = \ell_{\epsilon} \frac{1 - \frac{b \log \ell_{\epsilon}}{\ell_{\epsilon}} + \frac{\log \zeta}{\ell_{\epsilon}}}{\alpha_{\zeta} (\gamma_{\infty} - C_{\rm md} \eta_r - d \ell_{\epsilon}^{-b})} + \eta_r. 
    \end{align*}
    From $S_{T^-} - S_{T_*} = (S_{T^-} - S_r) + (S_r - S_{T_*})$, we have
    \begin{align*}
        \frac{S_{T^-} - S_{T_*}}{\ell_{\epsilon}} \leq \frac{1 - \frac{b \log \ell_{\epsilon}}{\ell_{\epsilon}} + \frac{\log \zeta}{\ell_{\epsilon}}}{\alpha_{\zeta} (\gamma_{\infty} - C_{\rm md} \eta_r - d \ell_{\epsilon}^{-b})} + \frac{S_r - S_{T_*} + \eta_r}{\ell_{\epsilon}},
    \end{align*}
    Note that $S_r - S_{T_*} + \eta_r$ is independent of $\epsilon$. Using $\frac{\log \ell_{\epsilon}}{\ell_{\epsilon}} \rightarrow 0$ and $\ell_{\epsilon}^{-b} \rightarrow 0$ as $\epsilon \downarrow 0$, we obtain
    \begin{align*}
        \limsup_{\epsilon \downarrow 0} \frac{S_{T^-} - S_{T_*}}{\ell_{\epsilon}} \leq \frac{1}{\alpha_{\zeta} (\gamma_{\infty} - C_{\rm md} \eta_r)}. 
    \end{align*}
    This holds for every sufficiently large fixed $r \geq t_{\zeta}$. Given $\eta_r \rightarrow 0$ as $r \rightarrow \infty$, we have $\limsup_{\epsilon \downarrow 0} \frac{S_{T^-} - S_{T_*}}{\ell_{\epsilon}} \leq \frac{1}{\alpha_{\zeta} \gamma_{\infty}}$. Finally, letting $\zeta \downarrow 0$ gives $\alpha_{\zeta} = 1 - \frac{n \zeta}{2} \rightarrow 1$ and we conclude that
    \begin{align}
        \limsup_{\epsilon \downarrow 0} \frac{S_{T^-} - S_{T_*}}{\ell_{\epsilon}} \leq \frac{1}{\gamma_{\infty}}. \label{eq:firstclock_m1}
    \end{align}

    We now prove $\liminf_{\epsilon \downarrow 0} \frac{S_{T^-} - S_{T_*}}{\ell_{\epsilon}} \geq \frac{1}{\gamma_{\infty}}$. By definition, we have $G_{T^-} \leq h_{\epsilon}$. Because $T^- \geq T_*$, \cref{prop:uniform-clock} gives $L_{T^-} \leq 2 G_{T^-} \leq 2 h_{\epsilon} = 2 \epsilon \ell_{\epsilon}^b$. By \cref{lem:loss-proxy}, we have $L_t \geq \frac{\log 2}{n} e^{- \rho(w_t)}$. From this, we then have $\rho(w_{T-}) \geq \log \frac{c_{\log}}{L_{T^-}}$ where we let $c_{\log}:=\frac{\log 2}{n}$. From $L_{T^-} \leq 2\epsilon \ell_{\epsilon}^b$, we obtain
    \begin{align*}
        \rho(w_{T^-}) \geq \log \frac{c_{\log}}{2 \epsilon \ell_{\epsilon}^b} = \ell_{\epsilon} - b \log \ell_{\epsilon} + \log \frac{c_{\log}}{2}.  
    \end{align*}
    This implies that 
    \begin{align}
        \liminf_{\epsilon \downarrow 0} \frac{\rho(w_{T^-})}{\ell_{\epsilon}} \geq 1. \label{eq:firstclock_8}
    \end{align} 
    By the definition of $\gamma_{\infty}$, we have $\rho(w) \leq \gamma_{\infty} \norm{w}_{\infty}$ for every $w \in \mathbb R^d$. Therefore, we have $\frac{\rho(w_{T^-})}{\gamma_{\infty}} \leq \norm{w_{T^-}}_{\infty}$. Given $T^- > r$ for sufficiently small $\epsilon$, we have $w_{T^-} = w_r - \sum_{t=r}^{T^- - 1} \eta_t D_t$ and consequently 
    \begin{align}
        \norm{w_{T^-}}_{\infty} \leq \norm{w_r}_{\infty} + \norm{\sum_{t=r}^{T^- - 1} \eta_t D_t}_{\infty}. \label{eq:firstclock_5}
    \end{align}
    The first term on the right-hand side can be bounded as: $\norm{w_r}_{\infty} \leq \norm{w_0}_{\infty} + C_D S_r =: M_r$. Note that $M_r$ is independent of $\epsilon$. From \cref{lem:proxy}, we know $|g_t[j]| \leq \norm{g_t}_{\infty}\leq R_1 G_t \leq R_1$. Therefore, by taking $B_g = R_1$, we obtain 
    \begin{align}
        \left\|\sum_{t=r}^{T^{-} - 1}\eta_tD_t\right\|_\infty \le S_{T^-} - S_r + C_{\rm rad} \eta_r \bigl[ 1 + \log \frac{R_1^2 + \epsilon^2}{(1 - \beta_2) \epsilon^2} \bigr]. \label{eq:firstclock_6}
    \end{align} 
    Substituting \eqref{eq:firstclock_6} into \eqref{eq:firstclock_5}, we obtain
    \begin{align*}
        \frac{\rho(w_{T^-})}{\gamma_{\infty}} \leq \norm{w_{T^-}}_{\infty} \leq M_r + S_{T^-} - S_r + C_{\rm rad} \eta_r [1 + \log \frac{R_1^2 + \epsilon^2}{(1 - \beta_2) \epsilon^2}]. 
    \end{align*}
    Rearranging and dividing by $\ell_{\epsilon}$ gives
    \begin{align*}
        \frac{S_{T^-} - S_r}{\ell_{\epsilon}} \geq \frac{1}{\gamma_{\infty}} \frac{\rho(w_{T^-})}{\ell_{\epsilon}} - \frac{M_r}{\ell_{\epsilon}} - C_{\rm rad} \eta_r \frac{1}{\ell_{\epsilon}} [1 + \log \frac{R_1^2 + \epsilon^2}{(1 - \beta_2) \epsilon^2}]. 
    \end{align*}
    By taking $\epsilon \downarrow 0$ and using \eqref{eq:firstclock_8}, we obtain
    \begin{align}
        \liminf_{\epsilon \downarrow 0} \frac{S_{T^-} - S_{T_*}}{\ell_{\epsilon}} \geq \frac{1}{\gamma_{\infty}} - 2 C_{\rm rad} \eta_r \label{eq:firstclock_9}.
    \end{align}
    where we used $\log \frac{R_1^2 + \epsilon^2}{(1 - \beta_2) \epsilon^2} = 2 \log \frac{1}{\epsilon} + \log \frac{R_1^2 + \epsilon^2}{1 - \beta_2} = 2 \ell_{\epsilon} + O(1)$ and consequently $\frac{1}{\ell_{\epsilon}}[1 + \log \frac{R_1^2 + \epsilon^2}{(1 - \beta_2) \epsilon^2}] \rightarrow 2$. \eqref{eq:firstclock_9} holds for every sufficiently large fixed $r$. We let $r \rightarrow \infty$ and use $\eta_r \rightarrow 0$ to conclude that
    \begin{align}
        \liminf_{\epsilon \downarrow 0} \frac{S_{T^-} - S_{T_*}}{\ell_{\epsilon}} \geq \frac{1}{\gamma_{\infty}}. \label{eq:firstclock_m2}
    \end{align}
    Combining \eqref{eq:firstclock_m1} and \eqref{eq:firstclock_m2}, we have
    \begin{align*}
        \frac{1}{\gamma_{\infty}} \leq \liminf_{\epsilon \downarrow 0} \frac{S_{T^-} - S_{T_*}}{\ell_{\epsilon}} \leq \limsup_{\epsilon \downarrow 0} \frac{S_{T^-} - S_{T_*}}{\ell_{\epsilon}} \leq \frac{1}{\gamma_{\infty}}. 
    \end{align*}
    Therefore, we obtain
    \begin{align*}
        \frac{S_{T^-} - S_{T_*}}{\ell_{\epsilon}} \longrightarrow \frac{1}{\gamma_{\infty}}, \qquad \text{as} \, \epsilon \downarrow 0,
    \end{align*}
    or equivalently, 
    \begin{align}
        S_{T^-} - S_{T_*} = \frac{1}{\gamma_{\infty}} \ell_{\epsilon} + o(\ell_{\epsilon}). \label{eq:mmm}
    \end{align}
    \paragraph{Part 2} Next, we prove $S_{T^{\delta}} - S_{T^-} = O(\log \ell_{\epsilon})$. Using $\norm{g_t}_1 \geq \gamma_{\infty}G_t$, $\norm{g_t}_1 \leq R_1 G_t$ and \eqref{eq:softsign-titu}, we get 
    \begin{align*}
        \langle g_t, \sigma_{\epsilon}(g_t) \rangle \geq \frac{\gamma_{\infty}^2 G_t^2}{R_1 G_t + d \epsilon} = \frac{\gamma^2 G_t}{R_1 + d \epsilon/G_t} \geq \frac{\gamma_{\infty}^2 G_t}{R_1 + d / \delta}. 
    \end{align*}
    Define $\kappa_{\delta} = \frac{\gamma_{\infty}^2}{R_1 + d/\delta} = \frac{\gamma_{\infty}^2 \delta}{R_1 \delta + d} > 0$. Then we have $\langle g_t , \sigma_{\epsilon}(g_t) \rangle \geq \kappa_{\delta} G_t$. Substituting this into \eqref{eq:firstclock_1}, we obtain
    \begin{align*}
        L_{t+1} \leq L_t - \eta_t (\kappa_{\delta} - C_{\rm md} \eta_t) G_t.     
    \end{align*}
    Since $T^{-} \rightarrow \infty$, we have $\eta_{T^-} \rightarrow 0$. Thus, for sufficiently small $\eps$, we have $C_{\rm md} \eta_t \leq C_{\rm md} \eta_{T^-} \leq \frac{\kappa_{\delta}}{2}, t \geq T^-$. This leads to $L_{t+1} \leq L_t - \frac{\kappa_{\delta}}{2} \eta_t G_t$. Using 
    $G_t \geq L_t /2, \forall t \geq T_{*}$, we further obtain 
    \begin{align}
        L_{t+1} \leq (1 - \frac{\kappa_{\delta}}{4} \eta_t) L_t. \label{eq:firstclock_m3}
    \end{align} 
    At the initial time, it holds that $L_{T^-} \leq 2 G_{T^-} \leq 2 \epsilon \ell_{\epsilon}^b$. Iterating \eqref{eq:firstclock_m3} leads to
    \begin{align*}
        L_t \leq 2 \epsilon \ell_{\epsilon}^b \exp(- \frac{\kappa_{\delta}}{4} (S_t - S_{T^-})) =: U_t, \qquad T^{-} \leq t \leq T^{\delta}. 
    \end{align*}
    Let $\bar{T}$ be the first time at which the right side is at most $\delta \epsilon$, i.e. $\bar{T} := \inf \{ t \geq T^{-}: U_t \leq \delta \epsilon \}$. The condition $U_t \leq \delta \epsilon$ is equivalent to
    \begin{align*}
        S_t - S_{T^-} \geq \frac{4}{\kappa_{\delta}} \log \frac{2 \ell_{\epsilon}^b}{\delta} =: A_{\epsilon}. 
    \end{align*}
    Thus, we can rewrite $\bar{T}$ as $\bar{T} = \inf \{ t \geq T^-: S_t - S_{T^-} \geq A_{\eps} \}$. For sufficiently small $\epsilon$, we have $\frac{2 l_{\epsilon}^b}{\delta} > 1$ and consequently $A_{\epsilon} > 0$. Therefore, we have $\bar{T} \geq T^- + 1$ when $\epsilon$ is sufficiently small.  
    Moreover, we must have $\bar{T} \geq T^{\delta}$, otherwise we would have $G_{\bar{T}} \leq L_{\bar{T}} \leq \delta \epsilon$ contradicting the definition of $T^{\delta}$. Since $S_t$ is increasing and $T^{\delta} \leq \bar{T}$, we have $S_{T^{\delta}} \leq S_{\bar{T}}$ and 
    \begin{align}
        S_{T^{\delta}} - S_{T^-} \leq S_{\bar{T}} - S_{T^-}. \label{eq:ttt}
    \end{align}
    By the minimality of $\bar{T}$, we have
    \begin{align*}
        S_{\bar{T}} - S_{T^-} = S_{\bar{T} - 1} - S_{T^-} + \eta_{\bar{T} -1} < A_{\epsilon} + \eta_{\bar{T} -1} \leq A_{\epsilon} + \eta_{T^-},
    \end{align*}
    where we used $\eta_{\bar{T} - 1} \leq \eta_{T^-}$ in the last inequality as $\bar{T} - 1 \geq T^{-}$. Combining this with \eqref{eq:ttt} and using the definition of $A_{\epsilon}$, we obtain 
    \begin{align*}
        S_{T^{\delta}} - S_{T^-} \leq \frac{4}{\kappa_{\delta}} \log \frac{2 \ell_{\eps}^b}{\delta} + \eta_{T^-} = \frac{4 b}{\kappa_{\delta}} \log \ell_{\epsilon} + \frac{4}{\kappa_{\delta}} \log \frac{2}{\delta} + \eta_{T^-}.  
    \end{align*}
    Since $\frac{4}{\kappa_{\delta}} \log \frac{2}{\delta}$ is a constant independent of $\epsilon$ and $\eta_{T^-} \rightarrow 0$ as $\epsilon \downarrow 0$, we conclude that $ S_{T^{\delta}} - S_{T^-} = O(\log \ell_{\epsilon})$. Combining this with \eqref{eq:mmm}, we obtain
    \begin{align*}
        S_{T^{\delta}} - S_{T^*} = (S_{T^-} - S_{T^*}) + (S_{T^{\delta}} - S_{T^-}) = \frac{1}{\gamma_{\infty}} \ell_{\epsilon} + o(\ell_{\epsilon}).
    \end{align*}
    \paragraph{Part 3} Finally, we prove $S_{T^+} - S_{T^{\delta}} = \Theta(\ell_{\epsilon}^b)$. 
    Set $\tau := T^{\delta}$. 
    We prove inductively that $L_t \leq 2 \delta \epsilon, t \geq \tau$. For the base case, we have $G_{\tau} \leq \delta \epsilon$ by definition. Because $\tau \geq T_{*}$, by \cref{prop:uniform-clock}, it holds that $L_{\tau} \leq 2 G_{\tau} \leq 2 \delta \epsilon$. Now, we assume that the statement holds at time t and try to show that it also holds at $t+1$. Since $G_t \leq L_t$, we have $\lambda_t = \frac{G_t}{\epsilon} \leq \frac{L_t}{\epsilon} \leq 2 \delta$. \cref{lem:alignment} gives $\epsilon \norm{D_t - \sigma_{\epsilon}(g_t)}_2 \leq C_{\rm ema} \eta_t G_t$. Also, coordinatewise, we have $\epsilon |\sigma_{\epsilon}(g_t) [j]| = \frac{\epsilon |g_t[j]|}{|g_t[j]| + \epsilon} \leq |g_t[j]|$. Hence, we have $\epsilon \norm{\sigma_{\epsilon}(g_t)}_2 \leq \norm{g_t}_2 \leq R_2 G_t$. Putting things together, we obtain
    \begin{align*}
        \epsilon \norm{D_t}_2  \leq (R_2 + C_{\rm ema} \eta_t) G_t \leq M_D G_t,
    \end{align*}
    where $M_D := R_2 + C_{\rm ema} \eta_0$. 
    Therefore, we have $\norm{D_t}_2 \leq M_D \frac{G_t}{\epsilon}$. 
    
    Because $|g_t[j]| \leq \norm{g_t}_1 \leq R_1 G_t$ and $\frac{G_t}{\epsilon} \leq 2 \delta$, we have $|g_{t}[j]| + \epsilon \leq \epsilon(1 + 2 R_1 \delta)$. Therefore, we obtain
    \begin{align*}
        \langle g_t, \sigma_{\epsilon}(g_t)\rangle = \sum_{j} \frac{g_t[j]^2}{|g_t[j]| + \epsilon} \geq \frac{\norm{g_t}_2^2}{\epsilon(1 + 2 R_1 \delta)} \geq \frac{\gamma_2^2}{1 + 2R_1 \delta} \frac{G_t^2}{\epsilon} = a_{\delta} \frac{G_t^2}{\epsilon},
    \end{align*}
    where we define $a_{\delta} := \frac{\gamma_2^2}{1 + 2R_1 \delta} > 0$. Moreover, we further obtain
    \begin{align*}
        \langle g_t, \sigma_{\epsilon}(g_t)\rangle \leq \frac{1}{\epsilon} \sum_{j=1}^d g_t[j]^2 = \frac{\norm{g_t}_2^2}{\epsilon} \leq R_2^2 \frac{G_t^2}{\epsilon}. 
    \end{align*}
    Denote $\sigma_t := \sigma_{\epsilon}(g_t)$. Next, we bound the term $|\langle g_t, D_t - \sigma_t \rangle|$ as
    \begin{align*}
        |\langle g_t, D_t - \sigma_t \rangle| \leq \norm{g_t}_2 \norm{D_t - \sigma_t}_2 \leq R_2 G_t (C_{\rm ema} \eta_t \frac{G_t}{\epsilon}) = C_{\rm ema} R_2 \eta_t \frac{G_t^2}{\epsilon}.
    \end{align*}
    We can write $\langle g_t, D_t \rangle = \langle g_t, \sigma_t\rangle + \langle g_t, D_t - \sigma_t \rangle$. For the lower bound, we have
    \begin{align*}
        \langle g_t, D_t \rangle \geq \langle g_t, \sigma_t \rangle - |\langle g_t, D_t - \sigma_t \rangle| \geq a_{\delta} \frac{G_t^2}{\epsilon} - C_{\rm ema} R_2 \eta_t \frac{G_t^2}{\epsilon} = (a_{\delta} - C_{\rm ema} R_2 \eta_t) \frac{G_t^2}{\epsilon}. 
    \end{align*}
    For the upper bound, we have
    \begin{align*}
        \langle g_t, D_t \rangle &\leq \langle g_t, \sigma_t \rangle + |\langle g_t, D_t - \sigma_t \rangle| \leq R_2^2 \frac{G_t^2}{\epsilon} + C_{\rm ema} R_2 \eta_t \frac{G_t^2}{\epsilon} = (R_2^2  + C_{\rm ema} R_2 \eta_t) \frac{G_t^2}{\epsilon}. 
    \end{align*} 
    Recall that $\tau = T^{\delta} = T_1^{\delta}(\epsilon)$. It holds that $\tau \rightarrow \infty$ and $\eta_{\tau} \rightarrow 0$ as $\epsilon \downarrow 0$. Consequently, we have for sufficiently small $\epsilon$ that 
    \begin{align*}
        C_{\rm ema} R_2 \eta_t \leq C_{\rm ema} R_2 \eta_{\tau} \leq \frac{a_{\delta}}{2}, \qquad t \geq \tau. 
    \end{align*}
    Substituting this into the lower bound, we obtain $a_{\delta} - C_{\rm ema} R_2 \eta_t \geq a_{\delta} - \frac{a_{\delta}}{2} = \frac{a_{\delta}}{2}$. For the upper bound, we have $R_2^2 + C_{\rm ema} R_2 \eta_t \leq R_2^2 + \frac{a_{\delta}}{2}$. Putting things together, we conclude that
    \begin{align}
        \frac{a_{\delta}}{2} \frac{G_t^2}{\epsilon} \leq \langle g_t, D_t \rangle \leq A_{\delta} \frac{G_t^2}{\epsilon}, \label{eq:t1_temp} 
    \end{align}
    where $A_{\delta}:= R_2^2 + \frac{a_{\delta}}{2} = R_2^2 + \frac{\gamma_2^2}{2(1 + 2R_1 \delta)}$. 
    Recall the Taylor remainder formula used in \cref{lem:taylor}:
    \begin{align*}
        L_{t+1} = L_t - \eta_t \langle g_t,D_t\rangle + R_t,
    \end{align*}
    where $R_t = \eta_t^2 \int_{0}^1 (1 - s) D_t^T \nabla^2 L(w_t - s \eta_t D_t) D_t ds \geq 0$. We now exploit the bound $\norm{D_t}_2 \leq M_D \frac{G_t}{\epsilon}$
    shown above. Let $\tilde{w}_{t,s} := w_t - s \eta_t D_t$. For every sample $i$, \cref{lem:derivative-ratio} gives $a_i(\tilde{w}_{t,s}) \leq e^{|z_i^T(\tilde{w}_{t,s} - w_t)|}a_i(w_t)$. Since $\frac{G_{t}}{\epsilon} \leq 2 \delta$
    and $\norm{D_t}_2 \leq M_{D}\frac{G_t}{\epsilon}$, defining $E_{\delta}:= \exp (2 \delta R_2 M_D \eta_0)$, we have $a_i(\tilde{w}_{t,s}) \leq E_{\delta} a_{i}(w_t)$. 
    Recall that $a_i(w) = -\ell'(z_i^T w)$. From \eqref{eq:md_1}, we have 
    \begin{align*}
        D_t^T \nabla^2 L(\tilde{w}_{t,s}) D_t \leq \frac{E_{\delta}}{n} \sum_{i} a_i(w_t) R_2^2 \norm{D_t}_2^2 = E_{\delta} R_2^2 G_t \norm{D_t}_2^2 \leq E_{\delta} R_2^2 M_D^2 \frac{G_t^3}{\epsilon^2}. 
    \end{align*}
    Consequently, we have 
    \begin{align}
    0 \leq R_t \leq B_{\delta} \eta_t^2 \frac{G_t^3}{\epsilon^2}, \label{eq:t1_temp2}
    \end{align}
     where $B_{\delta} = \frac{1}{2} E_{\delta} R_2^2 M_D^2$. Denote $\triangle_t := L_t - L_{t+1}$. From $\triangle_t = \eta_t \langle g_t, D_t \rangle - R_t$ and using \eqref{eq:t1_temp} and \eqref{eq:t1_temp2}, we have 
     \begin{align*}
         \triangle_{t} \geq \frac{a_{\delta}}{2} \eta_t \frac{G_t^2}{\epsilon} - B_{\delta} \eta_t^2 \frac{G_t^3}{\epsilon^2} = \eta_t \frac{G_t^2}{\epsilon} (\frac{a_{\delta}}{2} - B_{\delta} \eta_t \frac{G_t}{\epsilon}). 
     \end{align*}
     Since $\frac{G_t}{\epsilon} \leq 2 \delta$ and $\eta_t \leq \eta_{\tau}$ (as $t \geq \tau$), we have for sufficiently small $\epsilon$ that
     \begin{align*}
         B_{\delta} \eta_t \frac{G_t}{\epsilon} \leq 2 \delta B_{\delta} \eta_{\tau} \leq \frac{a_{\delta}}{4}. 
     \end{align*}
     Therefore, it holds that $\triangle_t \geq \frac{a_{\delta}}{4} \eta_t \frac{G_t^2}{\epsilon} \geq 0$. This implies that $L_{t+1} \leq L_t \leq 2 \delta \epsilon$. This closes the induction argument and shows that $L_{t} \leq 2 \delta \epsilon, t\geq \tau$.  
     
     Given $R_t \geq 0$, we have $\triangle_t = \eta_t \langle g_t, D_t\rangle - R_t \leq \eta_t \langle g_t, D_t \rangle$. Using the upper bound in \eqref{eq:t1_temp}, we conclude that 
     \begin{align*}
         \frac{a_{\delta}}{4} \frac{\eta_t}{\epsilon} G_t^2 \leq L_t - L_{t+1} \leq A_{\delta} \frac{\eta_t}{\epsilon} G_t^2. 
     \end{align*}
     Since $t \geq \tau \geq T_*$, \cref{prop:uniform-clock} gives $\frac{1}{2} L_t \leq G_t \leq L_t$ and consequently $\frac{1}{4} L_t^2 \leq G_t^2 \leq L_t^2$. Therefore, we further obtain 
     \begin{align}
         c_{\delta} \frac{\eta_t}{\epsilon} L_t^2 \leq L_t - L_{t+1} \leq A_{\delta} \frac{\eta_t}{\epsilon} L_t^2, \label{eq:ccc}
     \end{align}
     where $c_{\delta} := \frac{a_{\delta}}{16}$. Given this, we have for $t \geq \tau$ that
     \begin{align}
         \frac{L_t - L_{t+1}}{L_t} \leq A_{\delta} \eta_t \frac{L_t}{\epsilon} \leq 2 \delta A_{\delta} \eta_{\tau}. \label{eq:upper_lower}
     \end{align}   
     Since $\eta_{\tau} \rightarrow 0$, for sufficiently small $\epsilon$, we have 
     $L_{t+1} \geq \frac{1}{2} L_t, \, t \geq \tau$. Note that $\frac{1}{L_{t+1}} - \frac{1}{L_t} = \frac{L_t - L_{t+1}}{L_t L_{t+1}}$. Given $L_t L_{t+1} \leq L_t^2$ (as $L_{t+1} \leq L_t$), we obtain via the lower bound in \eqref{eq:ccc} that
     \begin{align*}
         \frac{1}{L_{t+1}} - \frac{1}{L_t} \geq c_{\delta} \frac{\eta_t}{\epsilon}. 
     \end{align*}
     Given $L_{t} L_{t+1} \geq \frac{1}{2} L_t^2$, we obtain via the upper bound in \eqref{eq:ccc} that
     \begin{align*}
         \frac{1}{L_{t+1}} - \frac{1}{L_t} \leq 2 A_{\delta} \frac{\eta_t}{\epsilon}.
     \end{align*}
     Combining them together, we conclude that
    $c_{\delta} \frac{\eta_t}{\epsilon} \leq \frac{1}{L_{t+1}} - \frac{1}{L_t} \leq 2 A_{\delta} \frac{\eta_t}{\epsilon}, \, t \geq \tau$. Recall that $\tau = T^{\delta}$. 
    Letting $C_{\delta} = 2 A_{\delta}$ and summing from $t = T^{\delta}$ to $T^+ - 1$ gives 
    \begin{align}
        c_{\delta} \frac{S_{T^+} - S_{T^{\delta}}}{\epsilon} \leq \frac{1}{L_{T^+}} - \frac{1}{L_{T^{\delta}}} \leq C_{\delta} \frac{S_{T^+} - S_{T^{\delta}}}{\epsilon}. \label{eq:cc1}
    \end{align}
    Given $\frac{L_{T^+}}{\epsilon \ell_{\epsilon}^{-b}} \rightarrow 1$ and $\frac{L_{T^{\delta}}}{\delta \epsilon} \rightarrow 1$, we can write $L_{T^+} = \epsilon \ell_{\epsilon}^{-b}(1 + r_{\epsilon})$ and $L_{T^{\delta}} = \delta \epsilon (1 + s_{\epsilon})$ where $r_{\epsilon} \rightarrow 0$ and $s_{\epsilon} \rightarrow 0$ as $\epsilon \downarrow 0$. Define $\tilde{r}_{\epsilon} := - \frac{r_{\epsilon}}{1 + r_{\epsilon}}$ and $\tilde{s}_{\epsilon}:= - \frac{s_{\epsilon}}{1+ s_{\epsilon}}$. Then we have
    \begin{align*}
        \frac{\epsilon}{L_{T^+}} = \ell_{\epsilon}^b (1 + \tilde{r}_{\epsilon}) = \ell_{\epsilon}^b (1 + o(1)), \qquad \frac{\epsilon}{L_{T^{\delta}}} = \frac{1}{\delta} (1 + \tilde{s}_{\epsilon}) = \frac{1}{\delta}(1+o(1)).  
    \end{align*}
    Combining them together, we have
    \begin{align}
        \epsilon (\frac{1}{L_{T^+}} - \frac{1}{L_{T^{\delta}}}) = \ell_{\epsilon}^b (1 + o(1)) - \frac{1}{\delta} (1 + o(1)) = \Theta (\ell_{\epsilon}^b).  \label{eq:cc2}
    \end{align}
    Multiplying \eqref{eq:cc1} by $\epsilon$ and using \eqref{eq:cc2} leads to 
    \begin{align*}
        S_{T^+} - S_{T^{\delta}} = \Theta (\ell_{\epsilon}^b). 
    \end{align*}
\end{proof}

\subsection{Implications of \cref{thm:first-clock}} \label{app:thm3_imp}

\begin{corollary} \label{cor:invariant} For any $\delta \in (0,1)$, there exists a constant $\epsilon_{\delta}^{\rm inv}$ such that for every $0 < \epsilon \leq \epsilon_{\delta}^{\rm inv}$, it holds that 
\begin{align*}
    G_t \leq L_t \leq 2 \delta \epsilon, \qquad \forall t \geq T_{1}^{\delta}(\epsilon)
\end{align*}
\end{corollary}

\begin{proof} This is already proved in Part 3 of the proof of \cref{thm:first-clock}. Here, we specify the constant $\epsilon_{\delta}^{\rm inv}$. The two requirements appearing in Part 3 are: 
\begin{align*}
    C_{\rm ema} R_2 \eta_{\tau} \leq \frac{a_{\delta}}{2}, \qquad 2 \delta B_{\delta} \eta_{\tau} \leq \frac{a_{\delta}}{4},
\end{align*}
where $\tau := T_1^{\delta}(\epsilon)$, $a_{\delta} := \frac{\gamma_2^2}{1 + 2R_1 \delta}$, $B_{\delta}:= \frac{1}{2} E_{\delta} R_2^2 M_D^2$, and $M_D := R_2 + C_{\rm ema} \eta_0$. Define $\bar{\eta}_{\delta}:= \min \{ \eta_0, \frac{a_{\delta}}{2 C_{\rm ema} R_2}, \frac{a_{\delta}}{8 \delta B_{\delta}} \}$. It is sufficient to require $\eta_{\tau} \leq \bar{\eta}_{\delta}$. Choosing $N_{\delta} := \max \{ T_*, t_{\rm ema}, \lceil (\frac{\eta_0}{\bar{\eta}_{\delta}})^{1/a}\rceil\}$, it holds that $\eta_{N_{\delta}} \leq \bar{\eta}_{\delta}$. 

\cref{lem:proxy-path-ratio} gives $G_t \geq G_0 e^{- K_0 S_t}$ where $K_0 = R_1 C_D$. Consequently, for every $0 \leq t \leq N_{\delta}$, we have 
\begin{align*}
    G_t \geq G_0 e^{-K_0 S_{N_{\delta}}} =: c_{N_{\delta}} > 0. 
\end{align*}
Set $\epsilon_{\delta}^{\rm inv} := \min \{1, \frac{c_{N_{\delta}}}{2 \delta}\}$. Next, we show this choice works. Recall that $\tau = T_{1}^{\delta}(\epsilon) = \inf \{ t \geq T_*: G_t \leq \delta \epsilon \}$. If $0 < \epsilon \leq \epsilon_{\delta}^{\rm inv} $, then it holds that
\begin{align*}
    \delta \epsilon \leq \frac{c_{N_{\delta}}}{2} < c_{N_{\delta}} \leq G_t, \qquad 0 \leq t \leq N_{\delta}. 
\end{align*}
Therefore, the event $G_t \leq \delta \epsilon$ cannot occur before or at $N_{\delta}$, which implies that $\tau > N_{\delta}$. Since the learning rate is decreasing, it holds that $\eta_{\tau} \leq \eta_{N_{\delta}} \leq \bar{\eta}_{\delta}$ as desired. 
\end{proof}

\begin{corollary} \label{cor:residual_bound_app} Suppose that $p, q \in (0,1)$ are two constants such that $p + q = 1$. For any $0 < \delta < \min \{1, \frac{p}{2 R_1 q} \}$, there exists a constant $\hat{\epsilon}_{1}(\delta)$ such that for every $0 < \epsilon \leq \hat{\epsilon}_{1}(\delta)$, we have $\mathfrak r_t^{\rm grad} \leq p$, $\mathfrak r_t^{\rm sgn} \geq q$ for every $t\geq T_1^{\delta}(\epsilon)$. 
\end{corollary}

\begin{proof} [Proof of \cref{cor:residual_bound_app}]
Define $A'_{\delta} := \frac{2 R_1 \delta}{1 + 2 R_1 \delta}$. 
We first choose $\delta$ such that $A'_{\delta} < p$ (equivalently, $2 R_1 (1-p) \delta < p$). Therefore, a feasible $\delta$ satisfies
\begin{align*}
    0 < \delta < \min \{ 1, \frac{p}{2 R_1 (1-p)}\} = \min \{ 1, \frac{p}{2 R_1 q}\}. 
\end{align*}
For any fixed $\delta$ in this range, we define a slack variable $\omega_{\delta,p} := p - \frac{2 R_1 \delta}{1 + 2 R_1 \delta} = \frac{p - 2 R_1 q \delta}{1 + 2 R_1 \delta} > 0$. 
Recall the constants used in \cref{cor:invariant}: $a_{\delta} := \frac{\gamma_2^2}{1 + 2 R_1 \delta}$, $M_D := R_2 + C_{\rm ema} \eta_0$, $E_{\delta} := \exp(2 \delta R_2 M_D \eta_0)$, and $B_{\delta}:= \frac{1}{2} E_{\delta} R_2^2 M_D^2$. We set
\begin{align*}
    \bar{\eta}_{\delta,p} := \min \{ \eta_0, \frac{a_{\delta}}{2 C_{\rm ema} R_2}, \frac{a_{\delta}}{8 \delta B_{\delta}}, \frac{\omega_{\delta,p}}{C_{\rm loc}} \},
\end{align*}
where $C_{\rm loc} = C_{\rm ema} \max \{ \frac{1}{\gamma_{\infty}}, \frac{1}{\gamma_2} \}$ is the constant in \cref{prop:local-interpolation_app}.
Note that the first three constants in the definition of $\bar{\eta}_{\delta,p}$ are used to ensure the invariant conclusion of \cref{cor:invariant} holds. Since $\eta_t = \eta_0 (t+1)^{-a}$, we define
\begin{align*}
    N_{\delta,p} := \max \{ T_*, t_{\rm ema}, \lceil (\frac{\eta_0}{\bar{\eta}_{\delta,p}})^{1/a}\rceil \},
\end{align*}
and then it holds that $\eta_{N_{\delta,p}} \leq \bar{\eta}_{\delta,p}$. We further define $c_{N_{\delta},p} := G_0 \exp (-K_0 S_{N_{\delta},p})$ where $K_0 = R_1 C_D$. Next, we show the choice $\hat{\epsilon}_{1}(\delta) := \min \{ 1, \frac{c_{N_{\delta},p}}{2 \delta} \}$ works. Let $\tau := T_1^{\delta}(\epsilon) = \inf \{ t \geq T_*: G_t \leq \delta \epsilon \}$. For $0 < \epsilon \leq \hat{\epsilon}_{1}(\delta)$, by \cref{lem:proxy-path-ratio}, we have for every $t \leq N_{\delta,p}$ that
\begin{align*}
    G_t \geq G_0 e^{- K_0 S_t} \geq G_0 e^{-K_0 S_{N_{\delta
    },p}} = c_{N_{\delta},p} > \frac{c_{N_{\delta},p}}{2} \geq \delta \epsilon.
\end{align*}
Therefore, the event $G_t\leq \delta \epsilon$ cannot occur at or before $N_{\delta,p}$. This implies that $\tau > N_{\delta,p}$. Since the learning rate is decreasing, we have 
\begin{align*}
    \eta_{\tau} \leq \eta_{N_{\delta},p} \leq \bar{\eta}_{\delta,p} \leq \frac{\omega_{\delta,p}}{C_{\rm loc}}.
\end{align*}
Consequently, it holds that $C_{\rm loc} \eta_{\tau}\leq \omega_{\delta,p}$. 
For every $t \geq \tau$, \cref{cor:invariant} gives $\lambda_{t} = \frac{G_t}{\epsilon} \leq \frac{L_t}{\epsilon} \leq 2{\delta}$.
Using \cref{prop:local-interpolation_app} and $\eta_t \leq \eta_{\tau}$, it holds that
\begin{align*}
    \mathfrak r_t^{\rm grad} \le \frac{2 R_1 \delta}{1+ 2 R_1 \delta} + C_{\rm loc}\eta_{\tau} \leq A'_{\delta} + \omega_{\delta,p} = p. 
\end{align*}
Similarly, it also holds that
\begin{align*}
    \mathfrak r_t^{\rm sgn} \geq \frac{1}{1 + 2 R_1 \delta} - C_{\rm loc} \eta_{\tau} = 1 - A'_{\delta} - C_{\rm loc} \eta_{\tau} \geq 1 - A'_{\delta} - \omega_{\delta,p} = q. 
\end{align*}
\end{proof}

\subsection{Proof of \cref{prop:bounded_euc}}

\begin{lemma} \label{lem:helper} 
For a nonnegative integer $T$, define $V_{T,t} := \sum_{s=T}^{t-1} \eta_s$ and $V_{T,T} = 0$. It holds that 
\begin{align*}
    \sum_{t  = T}^{\infty} \frac{\eta_t^2}{ 1 + V_{T,t}} \leq Q_{\eta}, \qquad Q_{\eta} := \eta_0^2 + \frac{1}{a} \max \{ \eta_0^2,\eta_0 (1-a)\}.  
\end{align*}
\end{lemma}

\begin{proof} For $a =1$ (i.e. $\eta_t = \frac{\eta_0}{t + 1}$), we have
\begin{align*}
    \sum_{t=T}^{\infty} \frac{\eta_t^2}{1 + V_{T,t}} \leq \sum_{t=T}^{\infty} \eta_t^2 = \eta_0^2 \sum_{t=T}^{\infty} \frac{1}{(t+1)^2} \leq \eta_0^2 \sum_{k=1}^{\infty} \frac{1}{k^2} = \frac{\pi^2}{6} \eta_0^2 \leq 2 \eta_0^2 = Q_{\eta}. 
\end{align*}
Next, we consider the case where $0 < a < 1$. Fix any $T \geq 0$ and define the functions $f(x):=\eta_0 (x+1)^{-a}, x \geq T$ and $F_T(x) := \int_{T}^x f(r) d r$. Note that $f(t) = \eta_t$. Because $f$ is positive and decreasing, it holds that $f(s) \geq \int_{s}^{s+1} f(x) dx$ for every $s \geq T$. Therefore, we have 
\begin{align*}
    V_{T,t} = \sum_{s = T}^{t - 1} f(s) \geq \sum_{s=T}^{t-1} \int_{s}^{s+1} f(x) dx = \int_{T}^t f(x) d x = F_T(t). 
\end{align*}
From this, it holds that $\frac{\eta_t^2}{1 + V_{T,t}} \leq \frac{f(t)^2}{1 + F_T(t)}$. 

We further define $h_T(x) := \frac{f(x)^2}{1 + F{T}(x)}$. Differentiating this function gives
\begin{align*}
    h'_T(x) = \frac{2 f(x) f'(x) (1+ F_{T}(x)) - f(x)^3}{(1 + F_T(x))^2},
\end{align*}
where we used $F'_T(x) = f(x)$. Given $f(x) > 0$ and  $f'(x) < 0$, it holds that $h'_T(x) < 0$. Therefore, the function $h_T(x)$ is decreasing and consequently we have 
\begin{align*}
    \sum_{t=T}^{\infty} h_T(t) \leq h_T(T) + \int_{T}^{\infty} h_T(x) dx. 
\end{align*}
Combining this with $\frac{\eta_t^2}{1 + V_{T,t}} \leq \frac{f(t)^2}{1 + F_T(t)}$, we have 
\begin{align}
    \sum_{t=T}^{\infty} \frac{\eta_t^2}{1 + V_{T,t}} \leq h_T(T) + \int_{T}^{\infty} \frac{f(x)^2}{1 + F_T(x)} dx. \label{eq:helper_t1}
\end{align}
Given $h_T(T) = f(T)^2 = \eta_T^2 \leq \eta_0^2$, it remains to bound the second term $I_T := \int_{T}^{\infty} \frac{f(x)^2}{1 + F_T(x)} dx$. Given $F_T(x) = \eta_0 \int_{T}^x (r+1)^{-a} dr = \frac{\eta_)}{1-a} [(x+1)^{1-a} - (T+1)^{1-a}]$, it holds that
\begin{align*}
    I_T = \int_{T}^{\infty} \frac{\eta_0^2 (x+1)^{-2a}}{1 + \frac{\eta_0}{1 - a} [(x+1)^{1-a} - (T+1)^{1-a}]} dx 
\end{align*}
Define $A_T := (T+1)^{1 - a}$ and $y := (x+1)^{1-a} - A_T$. Then, it holds that $y \geq 0$, $(x+1)^{1-a} = A_T + y$, and $dx = \frac{(x+1)^a}{1-a} dy$. Consequently, we have
\begin{align*}
    (x+1)^{-2a} dx = \frac{1}{1-a} (x+1)^{-a} dy = \frac{1}{1-a} (A_T + y)^{-a/(1-a)} dy. 
\end{align*}
Setting $p: = \frac{a}{1-a}$ and $c:= \frac{\eta_0}{1-a}$ and substituting these results into the expression of $I_T$, we obtain 
\begin{align*}
    I_T = \frac{\eta_0^2}{1 - a} \int_{0}^{\infty} \frac{(A_T + y)^{-p}}{1 + cy} dy. 
\end{align*}
Because $A_T = (T + 1)^{1-a} \geq 1$, we have $A_T + y \geq 1 + y$. Since $p > 0$, we further have $(A_T + y)^{-p} \leq (1 + y)^{-p}$. Therefore, it holds that
\begin{align*}
    I_{T} \leq \frac{\eta_0^2}{1 - a} \int_{0}^{\infty} \frac{(1 + y)^{-p}}{1 + cy} dy. 
\end{align*}

Define $m:= \min \{ 1, c\} = \min \{ 1, \frac{\eta_0}{1-a} \}$. It holds that $1 + cy \geq m (1 + y)$ as $y \geq 0$. Therefore, we have
\begin{align}
    I_{T} \leq \frac{\eta_0^2}{(1-a) m} \int_{0}^{\infty} (1 + y)^{-(p+1)} dy = \frac{\eta_0^2}{(1 - a)mp} = \frac{\eta_0^2}{a \min \{ 1, \frac{\eta_0}{1 - a}} \}. \label{eq:helper_t2}
\end{align}

Combining \eqref{eq:helper_t1} and \eqref{eq:helper_t2} with the bound $h_T(T) \leq \eta_0^2$, we obtain $\sum_{t=T}^{\infty} \frac{\eta_t^2}{1 + V_{T,t}} \leq \eta_0^2 + \frac{\eta_0^2}{a \min \{ 1, \frac{\eta_0}{1 - a}} \}$. Because this holds for every $T \geq 0$, we conclude that
\begin{align*}
    \sup_{T \geq 0} \sum_{t=T}^{\infty} \frac{\eta_t^2}{1 + \sum_{s=T}^{t-1} \eta_s} \leq \eta_0^2 + \frac{\eta_0^2}{a \min \{ 1, \frac{\eta_0}{1 - a}\} } = \eta_0^2 + \frac{1}{a} \max \{ \eta_0^2, \eta_0 (1-a)\} = Q_{\eta}. 
\end{align*}
Therefore, the constant $Q_{\eta}$ works for both cases.  
\end{proof}

Before we present the proof of  \cref{prop:bounded_euc}, we first present a short sketch highlighting the main ideas. Let $\tau := T_{1}^{\delta}(\epsilon)$, $V_t := V_t^{\delta} (\epsilon) = \sum_{s = \tau}^{t-1} \eta_s$, and $H_t := H_t^{\delta}(\epsilon) = \sum_{s=\tau}^{t-1}\theta_s$. 

\paragraph{Proof Sketch of \cref{prop:bounded_euc}} 
The proof of \cref{thm:first-clock} gives
\[
c_\delta\frac{\eta_s}{\epsilon}
\leq
\frac{1}{L_{s+1}}-\frac{1}{L_s}
\leq
2A_\delta\frac{\eta_s}{\epsilon}, \qquad s \geq \tau. 
\]
Summing from \(s=\tau\) to \(t-1\) yields
$ \frac{\epsilon}{L_t}
= \frac{\epsilon}{L_\tau} +
\Theta_\delta(V_t)$. Next, we show $\frac{\epsilon}{L_\tau}\asymp_\delta 1$. This can be seen from $L_\tau=\delta\epsilon(1+o(1))$. The proof below provides a more explicit bound.  
Consequently, we have $\frac{\epsilon}{L_t}
\asymp_\delta
1+V_t$, and $L_t
\asymp_\delta
\frac{\epsilon}{1+V_t}$.
Since \cref{prop:uniform-clock} gives
$\frac{1}{2}L_t\leq G_t\leq L_t$,
we also obtain $G_t
\asymp_\delta
\frac{\epsilon}{1+V_t}$.
 
Next, Lemma~12 gives $\gamma_2G_s
\leq
\lVert g_s\rVert_2
\leq
R_2G_s$.
Consequently, we have 
$\theta_s = \frac{\eta_s\lVert g_s\rVert_2}{\epsilon}
\asymp_\delta
\frac{\eta_s}{1+V_s}$.
Because \(V_{s+1}=V_s+\eta_s\), setting
$
x_s
:=
\frac{\eta_s}{1+V_s}
\leq
\eta_0
$
gives
\[
\log(1+V_{s+1})-\log(1+V_s)
=
\log(1+x_s)
\asymp
x_s.
\]
Summing this telescoping relation proves
$H_t
\asymp_\delta
\log(1+V_t)$.

Recall that $e_t := \frac{\epsilon D_t-g_t}{\|g_t\|_2}$ and $\theta_t := \frac{\eta_t\|g_t\|_2}{\epsilon}$.
Therefore, it holds that $\theta_t\|e_t\|_2
= \frac{\eta_t}{\epsilon}\|\epsilon D_t-g_t\|_2$.
We decompose
\[
\epsilon D_t-g_t
=
\epsilon\bigl(D_t-\sigma_\epsilon(g_t)\bigr)
+
\bigl(\epsilon\sigma_\epsilon(g_t)-g_t\bigr).
\]
Based on \cref{lem:alignment} and \cref{lem:proxy}, we can show 
\begin{align}
    \theta_t\|e_t\|_2
\leq
C_{\mathrm{ema}}\eta_t^2\frac{G_t}{\epsilon}
+
R_1R_2\eta_t
\left(\frac{G_t}{\epsilon}\right)^2. \label{eq:theta_et}
\end{align} 
The proof of \cref{thm:first-clock} establishes 
\begin{align*}
    c_{\delta} \frac{\eta_t}{\epsilon} L_t^2 \leq L_t - L_{t+1}, \qquad t \geq \tau. 
\end{align*}
Based on this, we prove
\begin{align*}
    \frac{G_t}{\epsilon}
\leq
\frac{C_L}{1+V_t},\quad \text{and} \quad \eta_t
\left(\frac{G_t}{\epsilon}\right)^2
\leq
\frac{L_t-L_{t+1}}{c_*\epsilon}.
\end{align*}
The first inequality is used in
\cref{lem:helper} to yield  
$\sum_{t=\tau}^{\infty}
\eta_t^2\frac{G_t}{\epsilon}
\leq C_L Q_\eta$. 
For the second inequality, we obtain $\sum_{t=\tau}^{\infty}
\eta_t
\left(\frac{G_t}{\epsilon}\right)^2
\leq
\frac{L_\tau}{c_*\epsilon}
\leq
\frac{2}{c_*}$ after telescoping and using \(L_\tau\leq 2\delta\epsilon\leq 2\epsilon\).
Substituting these two bounds into~\eqref{eq:theta_et}
gives
\[
\sum_{t=\tau}^{\infty}
\theta_t\|e_t\|_2
\leq
C_{\mathrm{ema}}C_LQ_\eta
+
\frac{2R_1R_2}{c_*}
=:E_0.
\]

\begin{proof} [Proof of \cref{prop:bounded_euc}] Let $\tau := T_{1}^{\delta}(\epsilon)$, $V_t := V_t^{\delta} (\epsilon) = \sum_{s = \tau}^{t-1} \eta_s$, and $H_t := H_t^{\delta}(\epsilon) = \sum_{s=\tau}^{t-1}\theta_s$. From the argument following \eqref{eq:ccc}, we have for $t \geq \tau$ that
\begin{align*}
    c_{\delta} \frac{\eta_t}{\epsilon} \leq \frac{1}{L_{t+1}} - \frac{1}{L_t} \leq 2 A_{\delta} \frac{\eta_t}{\epsilon},
\end{align*}
where $c_{\delta} = \frac{\gamma_2^2}{16(1 + 2 R_1 \delta)}$ and $A_{\delta} = R_2^2 + \frac{a_{\delta}}{2} = R_2^2 + \frac{\gamma_2^2}{2(1 + 2 R_1 \delta)}$. Summing from $s = \tau$ to $t - 1$ gives 
\begin{align}
    c_{\delta} \frac{V_t}{\epsilon} \leq \frac{1}{L_t} - \frac{1}{L_{\tau}} \leq 2 A_{\delta} \frac{V_t}{\epsilon}. \label{eq:vt_1}
\end{align}
Let $K_0:= R_1 C_D$. Define $\overline{\eta}^{\mathrm{entry}}
:=
\min\left\{
\eta_0,\,
\frac{\log 2}{K_0}
\right\}$ and $N^{\mathrm{entry}}
:=
\max\left\{
T_*,
\left\lceil
\left(
\frac{\eta_0}
{\overline{\eta}^{\mathrm{entry}}}
\right)^{1/a}
\right\rceil
\right\}$. 
Then, it holds that $\eta_{N^{\mathrm{entry}}}
=
\frac{\eta_0}{\bigl(N^{\mathrm{entry}}+1\bigr)^a}
\leq
\overline{\eta}^{\mathrm{entry}}
\leq
\frac{\log 2}{K_0}$ or equivalently $K_0\eta_{N^{\mathrm{entry}}}\leq \log 2$. By \cref{lem:proxy-path-ratio}, we have $G_t
\geq
G_0 e^{-K_0 S_t}
\geq
G_0 e^{-K_0 S_{N^{\mathrm{entry}}}}
=:c_{\mathrm{entry}}
> 0 $ for every $t\leq N^{\mathrm{entry}}$. We now define $\epsilon_\delta^{\mathrm{entry}}
:=
\min\left\{
1,\,
\frac{c_{\mathrm{entry}}}{2\delta}
\right\}
=
\min\left\{
1,\,
\frac{G_0 e^{-K_0 S_{N^{\mathrm{entry}}}}}{2\delta}
\right\}$.
Suppose that $0<\epsilon\leq\epsilon_\delta^{\mathrm{entry}}$. Then, we have 
\[
\delta\epsilon
\leq
\frac{c_{\mathrm{entry}}}{2}
<
c_{\mathrm{entry}}
\leq G_t,
\qquad
0\leq t\leq N^{\mathrm{entry}}.
\]
Consequently, by the definition of $\tau$, we must have $\tau > N^{\rm entry}$, from which we obtain $\eta_{\tau-1}
\leq
\eta_{N^{\mathrm{entry}}}
\leq
\frac{\log 2}{K_0}$. Consequently, it holds that $e^{-K_0\eta_{\tau-1}}
\geq
e^{-\log 2}
=
\frac{1}{2}$. By the minimality of $\tau$, we have $G_{\tau-1}>\delta\epsilon$ and $G_\tau\leq\delta\epsilon$. Applying \cref{lem:proxy-path-ratio} with \(r=\tau-1\) and \(t=\tau\) gives
\[
G_\tau
\geq
e^{-K_0\eta_{\tau-1}}G_{\tau-1} \geq e^{-K_0\eta_{\tau-1}} \delta \epsilon \geq \frac{\delta \epsilon}{2}.
\]  
Together with the defining upper bound, we obtain $\frac{\delta \epsilon}{2} < G_{\tau} \leq \delta \epsilon$. 
Since \(\tau\geq T_*\), \cref{prop:uniform-clock} gives
$\frac{1}{2}L_\tau
\leq
G_\tau
\leq
L_\tau$. Combining them together, we further obtain that $\frac{\delta \epsilon}{2} < L_{\tau} \leq 2 \delta \epsilon$. Consequently, it holds that $\frac{\epsilon}{L_{\tau}} \asymp_{\delta} 1$. 
Therefore, by multiplying \eqref{eq:vt_1} by $\epsilon$, we obtain 
\begin{align*}
    \frac{\epsilon}{L_t} = \frac{\epsilon}{L_{\tau}} + \Theta_{\delta}(V_t) \asymp_{\delta} 1 + V_t. 
\end{align*}
From this, we conclude that $L_t \asymp_{\delta} \frac{\epsilon}{1 + V_t}$. By \cref{prop:uniform-clock}, we have $\frac{1}{2} L_t \leq G_t \leq L_t$ when $t \geq T_*$. This immediately gives $G_t \asymp_{\delta} \frac{\epsilon}{1 + V_t}$ when $t \geq \tau$ (note that $\tau \geq T_*$).  

By \cref{lem:proxy}, it holds that $\gamma_2 G_s \leq \norm{g_s}_2 \leq R_2 G_s$. Therefore, we have 
\begin{align*}
    \theta_s = \frac{\eta_s \norm{g_s}_2}{\epsilon} \asymp \eta_s \frac{G_s}{\epsilon} \asymp_{\delta} \frac{\eta_s}{1 + V_s}. 
\end{align*}
This implies that $H_t \asymp_{\delta} \sum_{s=\tau}^{t-1} \frac{\eta_s}{1 + V_s}$. It remains to compare $\sum_{s=\tau}^{t-1} \frac{\eta_s}{1 + V_s}$ with $\log (1 + V_t)$. Since $V_{s+1} = V_s + \eta_s$, we have
\begin{align}
    \log(1 + V_{s+1}) - \log(1 + V_s) = \log(1 + \frac{\eta_s}{1 + V_s}) = \log (1 + x_s), \label{eq:xs}
\end{align}
where we have let $x_s := \frac{\eta_s}{1 + V_s}$. Note that $0 \leq x_s \leq \eta_s \leq \eta_0$. For $0 \leq u \leq x$, it holds that $\frac{1}{1+x} \leq \frac{1}{1 + u} \leq 1$. Integrating from $0$ to $x$ gives
$ \frac{x}{1+x} \leq \log (1+x) \leq x$. Applying this with $x = x_s$ and using $x_s \leq \eta_0$, we obtain
\begin{align*}
    \frac{x_s}{1 + \eta_0} \leq \frac{x_s}{1+x_s} \leq \log (1+x_s) \leq x_s,
\end{align*}
or equivalently, 
\begin{align*}
    \log(1+x_s) \leq x_s \leq (1 + \eta_0) \log (1+x_s).  
\end{align*}
Substituting $x_s = \frac{\eta_s}{1 + V_s}$ and using \eqref{eq:xs}, it holds that
\begin{align*}
    \log(1 + V_{s+1}) - \log(1+V_s) \leq \frac{\eta_s}{1 + V_s} \leq (1+\eta_0) [\log(1 + V_{s+1}) - \log(1+ V_s)],
\end{align*}
from which we conclude that $\frac{\eta_s}{1 + V_s} \asymp \log(1 + V_{s+1}) - \log(1 + V_s)$. Given $\sum_{s=\tau}^{t-1} [\log(1+V_{s+1})- \log(1 + V_s)] = \log(1 + V_t)$, it holds that $\sum_{s=\tau}^{t-1} \frac{\eta_s}{1 + V_s} \asymp \log (1 + V_t)$. Therefore, we conclude that $H_t \asymp_{\delta} \log (1 + V_t)$. 

Next, we show $\sum_{t = \tau}^{\infty} \theta_t \norm{e_t}2 \leq E_0$ for some constant $E_0$. Recall some definitions: $\theta_t = \frac{\eta_t \norm{g_t}_2}{\epsilon}$ and $e_t = \frac{\epsilon D_t - g_t}{\norm{g_t}_2}$. We first provide an upper bound on $\theta_t \norm{e_t}_2$. By noting $\theta_t \norm{e_t}_2 = \frac{\eta_t}{\epsilon} \norm{\epsilon D_t - g_t}_2$, it holds that 
\begin{align}
    \theta_t \norm{e_t}_2 \leq \frac{\eta_t}{\epsilon}[\epsilon \norm{D_t - \sigma_{\epsilon}(g_t)}_2 + \norm{\epsilon \sigma_{\epsilon}(g_t) - g_t}_2], \label{eq:trans_1}
\end{align}
where we used the decomposition $\epsilon D_t - g_t = \epsilon (D_t - \sigma_{\epsilon}(g_t)) + (\epsilon \sigma_{\epsilon}(g_t) - g_t)$. By \cref{lem:alignment}, it holds that $\epsilon \norm{D_t - \sigma_{\epsilon}(g_t)}_2 \leq C_{\rm ema} \eta_t G_t, t \geq t_{\rm ema}$. For the second term, we have for each coordinate that $\epsilon \sigma_{\epsilon}(g_t)[j] - g_{t}[j] = - g_t[j] \frac{|g_t[j]|}{|g_{t}[j]| + \epsilon}$. Since $\frac{|g_t[j]|}{|g_{t}[j]| + \epsilon} \leq \frac{\norm{g_t}_{\infty}}{\epsilon}$, we obtain
\begin{align*}
    \norm{\epsilon \sigma_{\epsilon}(g_t) - g_t}_2 \leq \frac{\norm{g_t}_{\infty} \norm{g_t}_{2}}{\epsilon} \leq R_1 R_2 \frac{G_t^2}{\epsilon}.
\end{align*}
Substituting this into \eqref{eq:trans_1}, we obtain 
\begin{align}
    \theta_t \norm{e_t}_2 \leq C_{\rm ema} \eta_t^2 \frac{G_t}{\epsilon}  + R_1 R_2 \eta_t (\frac{G_t}{\epsilon})^2. \label{eq:helper_cc}
\end{align}
Next, we bound $\sum_{t = \tau}^{\infty} \eta_t^2 \frac{G_t}{\epsilon}$ and $\sum_{t = \tau}^{\infty} \eta_t (\frac{G_t}{\epsilon})^2$. To do so, we first control $\frac{G_t}{\epsilon}$. Part 3 of \cref{thm:first-clock} establishes 
\begin{align}
    c_{\delta} \frac{\eta_t}{\epsilon} L_t^2 \leq L_t - L_{t+1}, \qquad t \geq \tau, \label{eq:final_hh}
\end{align}
where $c_{\delta} = \frac{\gamma_2^2}{16(1+2 R_1 \delta)}$. Because $0 < \delta < 1$, it holds that $c_{\delta} \geq c_* := \frac{\gamma_2^2}{16 (1 + 2 R_1)} > 0$. Since $L_{t+1} \leq L_t$, we have
\begin{align}
    \frac{1}{L_{t+1}} - \frac{1}{L_t} = \frac{L_t - L_{t+1}}{L_t L_{t+1}} \geq \frac{L_t - L_{t+1}}{L_t^2} \geq c_* \frac{\eta_t}{\epsilon}. \label{eq:helper_tt}
\end{align}
Summing from $\tau$ to $t - 1$ leads to $\frac{1}{L_t} \geq \frac{1}{L_{\tau}} + c_* \frac{V_t}{\epsilon}$. Therefore, we have $\frac{L_t}{\epsilon} \leq \frac{1}{\epsilon / L_{\tau} + c_* V_t}$.
Since $L_{\tau} \leq 2 \delta \epsilon \leq 2 \eps$, it holds that $\frac{\epsilon}{L_{\tau}} \geq \frac{1}{2}$. Therefore, we have 
\begin{align}
    \frac{G_t}{\epsilon} \leq \frac{L_t}{\epsilon} \leq \frac{1}{1/2 + c_* V_t} \leq \frac{C_L}{ 1 + V_t}, \label{eq:final_helper}
\end{align}
where $C_L := \frac{1}{\min \{ 1/2, c_*\}}$. Next, we apply \cref{lem:helper} to obtain
\begin{align}
    \sum_{t = \tau}^{\infty} \eta_t^2 \frac{G_t}{\epsilon} \leq C_L \sum_{t = \tau}^{\infty} \frac{\eta_t^2}{1 + V_t} \leq C_L Q_{\eta} = C_L (\eta_0^2 + \frac{1}{a} \max \{ \eta_0^2, \eta_0 (1-a)\}). \label{eq:h1}
\end{align}
From \eqref{eq:final_hh} and $G_t \leq L_t$, it holds that
\begin{align*}
    L_t - L_{t+1} \geq c_* \frac{\eta_t}{\epsilon} L_t^2 \geq c_* \frac{\eta_t}{\epsilon} G_t^2. 
\end{align*}
Therefore, we have $\eta_t (\frac{G_t}{\epsilon})^2 \leq \frac{L_t - L_{t+1}}{c_* \epsilon}$. Summing this leads to 
\begin{align}
    \sum_{t = \tau}^{\infty} \eta_t (\frac{G_t}{\epsilon})^2 \leq \frac{1}{c_* \epsilon} \sum_{t = \tau}^{\infty} (L_t - L_{t+1}) \leq \frac{L_{\tau}}{c_* \epsilon} \leq \frac{2 \delta}{c_*} \leq \frac{2}{c_*}. \label{eq:h2}
\end{align}
Summing \eqref{eq:helper_cc} and applying \eqref{eq:h1} and \eqref{eq:h2}, we obtain 
\begin{align*}
    \sum_{t=\tau}^{\infty} \theta_t \norm{e_t}_2 \leq C_{\rm ema} \sum_{t=\tau}^{\infty} \eta_t^2 \frac{G_t}{\epsilon} + R_1 R_2 \sum_{t=\tau}^{\infty} \eta_t (\frac{G_t}{\epsilon})^2 \leq C_{\rm ema} C_L Q_{\eta} + \frac{2 R_1 R_2}{c_*}.
\end{align*}
Therefore, we can take 
\begin{align*}
    E_0 = C_{\rm ema} C_L Q_{\eta} + \frac{2 R_1 R_2}{c_*}, \qquad c_* = \frac{\gamma_2^2}{16(1 + 2 R_1)}. 
\end{align*}
Finally, we summarize the conditions on $\epsilon$. For \eqref{eq:vt_1} to hold, we require three conditions: $C_{\mathrm{ema}}R_2\eta_\tau
\leq
\frac{a_\delta}{2}$, $2\delta B_\delta\eta_\tau
\leq
\frac{a_\delta}{4}$, and $2 \delta A_{\delta} \eta_{\tau} \leq \frac{1}{2}$. The constants $A_{\delta}$ and $B_{\delta}$ are given in \cref{thm:first-clock}. The first two conditions ensure \eqref{eq:ccc} holds. They can be satisfied by requiring $0 < \epsilon \leq \epsilon_{\delta}^{\rm inv}$, where $\epsilon_{\delta}^{\rm inv}$ is given in \cref{cor:invariant}. The last condition is for ensuring $L_{t+1} \geq \frac{1}{2} L_t$. Define $\overline{\eta}_\delta^{\mathrm{rec}}
:=
\min\left\{
\eta_0,\,
\frac{1}{4\delta A_\delta}
\right\}$, 
and
$N_\delta^{\mathrm{rec}}
:=
\max\left\{
T_*,
t_{\mathrm{ema}},
\left\lceil
\left(
\frac{\eta_0}{\overline{\eta}_\delta^{\mathrm{rec}}}
\right)^{1/a}
\right\rceil
\right\}$.
Choose $\epsilon_\delta^{\mathrm{rec}}
:=
\min\left\{
1,\,
\frac{G_0 e^{-K_0 S_{N_\delta^{\mathrm{rec}}}}}{2\delta}
\right\}$. If \(0<\epsilon\leq\epsilon_\delta^{\mathrm{rec}}\), then
we have 
$G_t
\geq
G_0 e^{-K_0 S_{N_\delta^{\mathrm{rec}}}}
>
\delta\epsilon$ when $0\leq t\leq N_\delta^{\mathrm{rec}}$. Therefore, it holds that
$\tau>N_\delta^{\mathrm{rec}}$ and 
$\eta_\tau
\leq
\eta_{N_\delta^{\mathrm{rec}}}
\leq
\overline{\eta}_\delta^{\mathrm{rec}}
\leq
\frac{1}{4\delta A_\delta}$. Consequently, \eqref{eq:upper_lower} gives
$\frac{L_t-L_{t+1}}{L_t}
\leq
2\delta A_\delta\eta_\tau
\leq
\frac{1}{2}$,
which implies
$L_{t+1}\geq\frac{1}{2}L_t$. In summary, we can let $\widetilde{\epsilon}_1(\delta)
:=
\min\left\{
\epsilon_\delta^{\mathrm{inv}},
\epsilon_\delta^{\mathrm{rec}},
\epsilon_\delta^{\mathrm{entry}}
\right\}.$
\end{proof}

\section{Adam's $\ell_{\infty}$-margin optimality}
\label{app:margin}

Throughout this section, we assume data separability, $0 \leq \beta_1 \leq \beta_2 < 1$, and $\eps \in (0,1]$. Recall that $R_1  :=\max_{i \in [n]} \norm{x_i}_1$ and $R_2  :=\max_{i \in [n]} \norm{x_i}_2$. The step size schedule is $\eta_t = \eta_0 (t+1)^{-a}$ with $\eta_0 > 0$ and $a \in (1/3,1]$. We first present some useful lemmas. They include Hoffman's error bound in \cref{lem:hoff} and the margin difference bound in \cref{lem:margin_helper}. The margin gaps at $T_1^{-}(\eps;b)$, $T_1^{\delta}(\eps)$,
and $T_1^{+}(\eps;b)$ are proved in \cref{thm:margin_app_T1-}, \cref{thm:margin_app_T1}, and \cref{cor:margin_app_T1+}, respectively. 

\begin{lemma} \label{lem:vec_inverse} For any two nonzero vectors $x$ and $u$, it holds that
\begin{align*}
    \left|\frac{1}{\norm{x}_2} - \frac{1}{\norm{u}_2} \right | \leq \frac{\norm{x - u}_2}{\norm{x}_2 \norm{u}_2}. 
\end{align*}
\end{lemma}

\begin{proof}
    It holds that
    \begin{align*}
        \left |  \frac{1}{\norm{x}_2} - \frac{1}{\norm{u}_2} \right | = \left | \frac{\norm{u}_2 - \norm{x}_2}{\norm{x}_2 \norm{u}_2} \right | = \frac{|\norm{u}_2 - \norm{x}_2|}{\norm{x}_2 \norm{u}_2}
    \end{align*}
    The conclusion follows from the reverse triangle inequality $|\norm{u}_2 - \norm{x}_2| \leq \norm{u - x}_2$. 
\end{proof}

\begin{lemma}[\citep{hoffman1952approximate,pena2021new}] \label{lem:hoff} Define
\begin{align*}
    A_{\mathrm H}
:=
\begin{pmatrix}
I_d \\
-I_d \\
-Z
\end{pmatrix}
\in \mathbb{R}^{(2d+n)\times d},
\qquad
c_{\mathrm H}
:=
\begin{pmatrix}
\mathbf{1}_d \\
\mathbf{1}_d \\
-\gamma_\infty \mathbf{1}_n
\end{pmatrix} \in \mathbb{R}^{2d+n},
\end{align*}
where the $i$-th row of $Z$ is $z_i^T$. 
There exists a finite constant $H_{2,\infty} (A_H)$ (depending only on the matrix $A_H$) such that 
\begin{align}
    \operatorname{dist}_2\!\left(x,\mathcal{M}_\infty\right)
\le
H_{2,\infty}\!\left(A_{\mathrm H}\right)
\left\|
\left(A_{\mathrm H}x-c_{\mathrm H}\right)_+
\right\|_\infty, \label{eq:hoff}
\end{align}
where $\operatorname{dist}_2\!\left(x,\mathcal{M}_\infty\right) := \min_{u \in \cM_{\infty}} \norm{x - u}_2$ and $[y]_+ := (\max \{y_k, 0 \})_k$. 
\end{lemma}

\begin{proof} Recall that $\cM_{\infty} = \{ u \in \R^d: \norm{u}_{\infty} \leq 1, \rho(u) = \gamma_{\infty} \}$. We prove that $\cM_{\infty} = \{ u \in \R^d: A_H u \leq c_H \}$. We first show that $\cM_{\infty} \subseteq \{ u: A_H u \leq c_H \}$. For any $u \in \cM_{\infty}$, it holds that $\norm{u}_{\infty} \leq 1$. Consequently, we have $u \leq \mathbf{1}_d$ and $-u \leq \mathbf{1}_d$. Moreover, it holds that $z_i^T u \geq \gamma_{\infty}, \forall i$ given $\rho(u) = \min_i z_i^T u = \gamma_{\infty}$. Therefore, we have $Z u \geq \gamma_{\infty} \mathbf{1}_n$ or equivalently $-Zu \leq - \gamma_{\infty} \mathbf{1}_n$. Hence, all the block inequalities hold and we conclude that $A_H u \leq c_H$. Next, we show $\{u: A_H u \leq c_H \} \subseteq \cM_{\infty}$. For any $u$ satisfying $A_H u \leq c_H$, the first two blocks imply $\norm{u}_{\infty} \leq 1$, and the third block implies $z_i^T u \geq \gamma_{\infty}, \forall i$ (consequently, it holds that $\rho(u) = \min_i z_i^T u \geq \gamma_{\infty}$). Since $\norm{u}_{\infty} \leq 1$, the definition of $\gamma_{\infty}$ implies that $\rho(u) \leq \gamma_{\infty}$. Therefore, we have $\rho(u) = \gamma_{\infty}$ and $u \in \cM_{\infty}$. Since $[-1,1]^d$ is compact and $\rho$ is continuous, the
maximum defining $\gamma_\infty$ is attained. Hence,
$\mathcal M_\infty$ is nonempty. Therefore, the system
$A_Hu\le c_H$ is feasible. By applying \citet[Proposition~1 and 2]{pena2021new}, we conclude that \eqref{eq:hoff} holds with the constant 
\begin{align*}
    H_{2,\infty}\!\left(A_{\mathrm H}\right)
:=
\max_{\substack{
\varnothing\neq J\subseteq[2d+n]\\
\operatorname{rank} \, \!\left((A_{\mathrm H})_J\right)=|J|
}}
\left[
\min_{\substack{
v\in\mathbb{R}^{|J|}_{+}\\
\mathbf{1}^{\top}v=1
}}
\left\|
(A_{\mathrm H})_J^{\top}v
\right\|_2
\right]^{-1},
\end{align*}
where $(A_H)_J$ denotes the submatrix formed by the rows of $A_H$ indexed by $J$. Denote $\sigma_{\min}\bigl((A_H)_J\bigr)
:=
\sqrt{
\lambda_{\min}\!\left(
(A_H)_J(A_H)_J^\top
\right)
}$. Let $A := A_H$. The inner quantity $\alpha_J := \min_{v \geq 0, \, \norm{v}_1 = 1} \norm{A_J^T v}_2$ is strictly positive whenever $A_J$ has full row rank. Since there are finitely many row subsets $J$, the constant $H_{2,\infty}(A_H)$ is finite. From the definition of $H_{2,\infty}(A)$, we can derive a simple upper bound. For every full-row-rank $A_J$, it holds that $\norm{A_J^T v}_2 \geq \sigma_{\min} (A_J) \norm{v}_2$. If $v \geq 0$ and $\norm{v}_1 = 1$, then we have $\norm{v}_2 \geq \frac{1}{\sqrt{|J|}}$. Therefore, it holds that $\alpha_J \geq \frac{\sigma_{\min}(A_J)}{\sqrt{|J|}}$ and consequently 
\begin{align*}
    H_{2,\infty}(A)
    \le
    \max_{\substack{
    \varnothing \neq J \subseteq [2d+n] \\
    \operatorname{rank}(A_J)=|J|
    }}
    \frac{\sqrt{|J|}}
    {\sigma_{\min}(A_J)}.
\end{align*}
\end{proof}

\begin{lemma} \label{lem:rho_lip} Recall that  $\rho(u):=\min_{i \in [n]} z_i^\top u$. It holds that the function $\rho$ is $R_1$-Lipschitz in max-norm and $R_2$-Lipschitz in 2-norm: 
\begin{align*}
    |\rho(u) - \rho(v)| \leq R_1 \norm{u - v}_{\infty}, \qquad |\rho(u) - \rho(v)| \leq R_2 \norm{u - v}_{2}.
\end{align*}
\end{lemma}

\begin{proof}
    Choose $i_v$ such that $\rho(v) = z_{i_v}^T v$. By definition, it holds that $\rho(u) \leq z_{i_v}^T u$. 
    Therefore, we have 
    \begin{align*}
        \rho(u) - \rho(v) &\leq z_{i_v}^T u - z_{i_v}^T v = z_{i_v}^T (u - v) \\
        &\leq |z_{i_v}^T (u - v)| \leq \norm{z_{i_v}}_1 \norm{u - v}_{\infty} \leq R_1 \norm{u - v}_{\infty}
    \end{align*}
    Similarly, we also have
    \begin{align*}
        \rho(u) - \rho(v) \leq \norm{z_{i_v}}_2 \norm{u - v}_{2} \leq R_2 \norm{u - v}_{2}.
    \end{align*}
Interchanging \(u\) and \(v\) gives the corresponding lower bounds.
\end{proof}

\begin{lemma} \label{lem:margin_helper}
For any nonzero vectors $x$ and $y$, it holds that
\begin{align*}
    |\hat{\gamma}_{\infty}(x) - \hat{\gamma}_{\infty}(y)| \leq \frac{2 R_1 \norm{x - y}_{\infty}}{ \norm{y}_{\infty} }, \qquad |\hat{\gamma}_2(x) - \hat{\gamma}_2 (y)| \leq \frac{2 R_2 \norm{x-y}_2}{\norm{y}_2}.
\end{align*}
\end{lemma}

\begin{proof}
    Let $R_x := \norm{x}_{\infty}$ and $R_y := \norm{y}_{\infty}$. It holds that
    \begin{align}
        \frac{\rho(x)}{R_x} - \frac{\rho(y)}{R_y} = \frac{\rho(x) - \rho(y)}{R_y} + \rho(x) (\frac{1}{R_x} - \frac{1}{R_y}). \label{eq:rho_1}
    \end{align} 
    From \cref{lem:rho_lip}, we know $|\rho(x) - \rho(y)| \leq R_1 \norm{x - y}_{\infty}$. Setting $y = 0$ leads to $|\rho(x)| \leq R_1 \norm{x}_{\infty} = R_1 R_x$. By the reverse triangle inequality, we also have $|R_x - R_y| \leq \norm{x - y}_{\infty}$. Substituting these inequalities into \eqref{eq:rho_1}, we obtain
    \begin{align*}
        \bigl | \frac{\rho(x)}{R_x} - \frac{\rho(y)}{R_y} \bigr | &\leq \frac{|\rho(x) - \rho(y)|}{R_y} + |\rho(x)| \frac{|R_x - R_y|}{R_x R_y} \\ 
        &\leq \frac{R_1 \norm{x - y}_{\infty}}{R_y} + R_1 R_x \frac{\norm{x-y}_{\infty}}{R_x R_y} \\ 
        &= \frac{2 R_1 \norm{x-y}_{\infty}}{R_y}. 
    \end{align*}
    Similarly, from \cref{lem:rho_lip}, we know $|\rho(x) - \rho(y)| \leq R_2 \norm{x - y}_2$ and $|\rho(x)| \leq R_2 \norm{x}_2$. Consequently, we have
    \begin{align*}
        |\hat{\gamma}_2(x) - \hat{\gamma}_2 (y)| \leq \frac{R_2 \norm{x - y}_2}{\norm{y}_2} + R_2 \norm{x}_2 \frac{\norm{x-y}_2}{\norm{x}_2 \norm{y}_2} \leq \frac{2 R_2 \norm{x-y}_2}{\norm{y}_2}. 
    \end{align*}
\end{proof}

\subsection{Margin Gap at $T_1$} \label{app:margin_T1}

\begin{theorem} \label{thm:margin_app_T1} Denote $r_{\eps}:= \ell_{\eps}^{-a}$ when $a \in (\frac{1}{3},1)$ and $r_{\eps}:= \frac{\log \ell_{\eps}}{\ell_{\eps}}$ when $a = 1$.  For any $\delta \in (0,1)$, 
there exists a constant $\Tilde{\eps}_2(\delta)$ such that for every $0< \eps \leq \Tilde{\eps}_2(\delta)$, we have 
\[
\gamma_\infty-\widehat{\gamma}_\infty(w_{T_1^{\delta}(\eps)})
\leq
O(r_\epsilon), \qquad 
\gamma_2-\widehat{\gamma}_2(w_{T_1^{\delta}(\eps)})
\geq
\Delta_{\mathrm{geom}}
- O(r_\epsilon).
\]
\end{theorem}

\begin{proof}[Proof of \cref{thm:margin_app_T1}]
    Recall that $\ell_{\eps} = \log \frac{1}{\eps}$ and let $\tau_{\eps} := T_1^{\delta} (\eps) = \inf \{ t \geq T_*: G_t \leq \delta \eps \}$. 
    We further define $\omega_{a}(l) := l^{1-a}$ when $a \in (1/3,1)$ and $\omega_{a}(l) := \log l$ when $a = 1$. We note that $\frac{\omega_a(\ell_{\epsilon})}{\ell_{\epsilon}} = r_{\epsilon}$ where $r_{\eps}:= l_{\eps}^{-a}$ when $a \in (\frac{1}{3},1)$ and $r_{\eps}:= \frac{\log l_{\eps}}{l_{\eps}}$ when $a = 1$. \cref{thm:first-clock} already gives $S_{\tau_{\epsilon}} - S_{T_*} = \frac{\ell_{\epsilon}}{ \gamma_{\infty}} + o(\ell_{\epsilon})$. 


    \paragraph{Part 1} Define $\Lambda_{\epsilon} = \ell_{\epsilon}^a$ when $a \in (1/3,1)$ and $\frac{\ell_{\epsilon}}{\log \ell_{\epsilon}}$ when $a = 1$. Note that $\Lambda_{\epsilon} \rightarrow \infty$ and $\epsilon \Lambda_{\epsilon} \rightarrow 0$ as $\epsilon \downarrow 0$. We also define an auxiliary stopping time $\sigma_{\epsilon} := \inf \{ t \geq T_*: G_t \leq \epsilon \Lambda_{\epsilon} \}$ that is useful for the proof. For sufficiently small $\epsilon$, we have $\Lambda_{\epsilon} > \delta$ and consequently $\sigma_{\epsilon} \leq \tau_{\epsilon
    }$. By the definition of $\sigma_{\epsilon}$, it holds for $T_* \leq t < \sigma_{\epsilon}$ that $G_t > \epsilon \Lambda_{\epsilon}$ or equivalently $\frac{\epsilon}{G_t} < \frac{1}{\Lambda_{\epsilon}}$. Using \eqref{eq:softsign-linear} and $\norm{g_t}_1 \geq \gamma_{\infty} G_t$, it holds that
    \begin{align*}
        \langle g_t, \sigma_{\epsilon} (g_t)\rangle \geq \gamma_{\infty} G_t - d \epsilon > (\gamma_{\infty} - \frac{d}{\Lambda_{\epsilon}}) G_t. 
    \end{align*}
    Substituting this into \eqref{eq:master-descent-app}, we have
    \begin{align*}
        L_{t+1} \leq L_t - \eta_t A_t G_t, \qquad A_t := \gamma_{\infty} - \frac{d}{\Lambda_{\epsilon}} - C_{\rm md} \eta_t.
    \end{align*}
    Next, we define $\bar{\eta}_0 := \min \{ \frac{\gamma_{\infty}}{4 C_{\rm md}}, \frac{1}{2 \gamma_{\infty}} \}$ and choose $t_0 := \max \{ T_*, t_{\rm ema}, \lceil (\frac{\eta_0}{\bar{\eta}_0})^{1/a}\rceil \}$. Then, we have that $C_{\rm md} \eta_t \leq \frac{\gamma_{\infty}}{4}$ and $\gamma_{\infty} \eta_t \leq \frac{1}{2}$ when $t \geq t_0$. Given this fixed $t_0$, \cref{lem:proxy-path-ratio} gives $G_t \geq G_0 \exp(-R_1 C_D S_{t_0}) =: c_{t_0} > 0$ when $t \leq t_0$. We choose $\epsilon$ sufficiently small such that $\epsilon \Lambda_{\epsilon} < c_{t_0}$. This would ensure that $G_t > \epsilon \Lambda_{\epsilon}$ for $0 \leq t \leq t_0$ and consequently we have $\sigma_{\epsilon} > t_0$. We also choose $\epsilon$ sufficiently small such that $\frac{d}{\Lambda_{\epsilon}} \leq \frac{\gamma_{\infty}}{4}$. With the choice of $t_0$ and $\epsilon$, we obtain for every $t_0 \leq t < \sigma_{\epsilon}$ 
    \begin{align*}
        \gamma_{\infty} \geq A_t = \gamma_{\infty} - \frac{d}{\Lambda_{\epsilon}} - C_{\rm md} \eta_t \geq \gamma_{\infty} - \frac{\gamma_{\infty}}{4} - \frac{\gamma_{\infty}}{4} = \frac{\gamma_{\infty}}{2}. 
    \end{align*}
    Since $G_t / L_t \leq 1$, we also have $0 \leq \eta_t A_t \frac{G_t}{L_t} \leq \gamma_{\infty} \eta_t \leq \frac{1}{2}$ for $t_0 \leq t < \sigma_{\epsilon}$. Let $y_t := \eta_t A_t \frac{G_t}{L_t}$, we then have $0 \leq y_t \leq \frac{1}{2}$ and $L_{t+1} \leq L_t - L_t y_t = L_t (1 - y_t)$. This implies that
    \begin{align*}
        \log \frac{L_t}{L_{t+1}} \geq - \log (1 - y_t) \geq y_t = \eta_t A_t \frac{G_t}{L_t}, \qquad t_0 \leq t < \sigma_{\epsilon},
    \end{align*}
    where we used $- \log (1-y) \geq y$ for $0 \leq y < 1$. 
    From $\frac{G_t}{L_t} \geq 1 - \frac{n L_t}{2}$, we have
    \begin{align*}
        A_t \frac{G_t}{L_t} &\geq A_t (1 - \frac{n L_t}{2}) \\
        &= \gamma_{\infty} - \frac{d}{\Lambda_{\epsilon}} - C_{\rm md} \eta_t -\frac{n \gamma_{\infty}}{2} L_t + \frac{nd}{2 \Lambda_{\epsilon}} L_t + \frac{n C_{\rm md}}{2} \eta_t L_t \\
        &\geq \gamma_{\infty} - \frac{d}{\Lambda_{\epsilon}} - C_{\rm md} \eta_t -\frac{n \gamma_{\infty}}{2} L_t.
    \end{align*}
    Therefore, it holds that 
    \begin{align}
        \log \frac{L_t}{L_{t+1}} \geq \gamma_{\infty} \eta_t - \frac{d}{\Lambda_{\epsilon}} \eta_t - C_{\rm md} \eta_t^2 - \frac{n \gamma_{\infty}}{2} \eta_t L_t, \qquad t_0 \leq t < \sigma_{\epsilon}. \label{eq:log_1}
    \end{align}
    Summing \eqref{eq:log_1} from $t_0$ to $\sigma_{\epsilon} - 1$, we obtain
    \begin{align}
        \log \frac{L_{t_0}}{L_{\sigma_{\epsilon}}} &\geq \gamma_{\infty} (S_{\sigma_{\epsilon}} - S_{t_0}) - \frac{d}{\Lambda_{\epsilon}} (S_{\sigma_{\epsilon}} - S_{t_0}) - C_{\rm md} Q_{\sigma_{\epsilon};t_0} - \frac{n \gamma_{\infty}}{2} \sum_{t = t_0}^{\sigma_{\epsilon} - 1} \eta_t L_t  \label{eq:log_mm}
    \end{align}
    Next, we bound the term $\sum_{t = t_0}^{\sigma_{\epsilon} - 1} \eta_t L_t$. Starting from $L_{t+1} \leq L_t - \eta_t A_t G_t$, we obtain
    \begin{align}
        L_t - L_{t+1} \geq \eta_t A_t G_t \geq \frac{\gamma_{\infty}}{2} \eta_t G_t \geq \frac{\gamma_{\infty}}{4} \eta_t L_t, \label{eq:log_2}
    \end{align}
    where we used $A_t \geq \frac{\gamma_{\infty}}{2}$ and $G_t \geq \frac{1}{2} L_t$ (as $T_* \leq t_0 \leq t < \sigma_{\epsilon}$). 
    Rearranging and summing \eqref{eq:log_2} from $t_0$ to $\sigma_{\epsilon} -1$, we obtain
    \begin{align*}
        \sum_{t = t_0}^{\sigma_{\epsilon} - 1} \eta_t L_t \leq \frac{4}{\gamma_{\infty}} \sum_{t = t_0}^{\sigma_{\epsilon} - 1} (L_t - L_{t+1}) = \frac{4}{\gamma_{\infty}}(L_{t_0} - L_{\sigma_{\epsilon}}). 
    \end{align*}
    Since $L_{\sigma_{\epsilon}} > 0$, it holds that $\sum_{t = t_0}^{\sigma_{\epsilon} - 1} \eta_t L_t \leq \frac{4 L_{t_0}}{\gamma_{\infty}} \leq \frac{4 \log 2}{n \gamma_{\infty}}$ (as $L_{t_0} \leq \frac{\log 2}{n}$ when $t_0 \geq T_*$).

    Denote $K_0 := R_1 C_D$. We have shown that $\sigma_{\eps} > t_0$. By minimality of $\sigma_{\eps}$, it holds that $G_{\sigma_{\eps} - 1} > \eps \Lambda_{\eps}$. From \cref{lem:proxy-path-ratio}, we obtain
    \begin{align*}
        L_{\sigma_{\eps}} \geq G_{\sigma_{\eps}} \geq e^{-K_0 \eta_{\sigma_{\eps} - 1}} G_{\sigma_{\eps} - 1} \geq \eps \Lambda_{\eps} e^{-K_0 \eta_{\sigma_{\eps} - 1}} \geq \eps \Lambda_{\eps} e^{-K_0 \eta_{t_0}}. 
    \end{align*}
    Consequently, it holds that $\frac{L_{t_0}}{L_{\sigma_{\eps}}} < \frac{L_{t_0} e^{K_0 \eta_{t_0}}}{\eps \Lambda_{\eps}}$. After taking logarithms, we obtain 
    \begin{align*}
        \log \frac{L_{t_0}}{L_{\sigma_{\eps}}} &\leq \log (\frac{L_{t_0} e^{K_0 \eta_{t_0}}}{\eps \Lambda_{\eps}}) = \ell_{\eps} - \log \Lambda_{\eps} + \log L_{t_0} + K_0 \eta_{t_0} \\ 
        &\leq \ell_{\eps} + K_0 \eta_{t_0},
    \end{align*}
    where the second inequality is by $\Lambda_{\eps} \geq 1$ (ensured by choosing $\eps$ sufficiently small) and $L_{t_0} \leq \frac{\log 2}{n} < 1$. Consequently, we let $C_1 := K_0 \eta_{t_0}$ and it holds that $\log \frac{L_{t_0}}{L_{\sigma_{\eps}}} \leq \ell_{\eps} + C_1$. 

    By defining $V_{\eps} := S_{\sigma_{\eps}} - S_{t_0}$ and using $\sum_{t = t_0}^{\sigma_{\epsilon} - 1} \eta_t L_t \leq \frac{4 \log 2}{n \gamma_{\infty}}$ and $\log \frac{L_{t_0}}{L_{\sigma_{\eps}}} \leq \ell_{\eps} + C_1$, we obtain from \eqref{eq:log_mm} that 
    \begin{align}
        \gamma_{\infty} V_{\eps} \leq \ell_{\eps} + C_1 + \frac{d}{\Lambda_{\eps}} V_{\eps} + C_{\rm md} Q_{\sigma_{\epsilon};t_0} + 2 \log 2. \label{eq:log_4}
    \end{align}
    From $\frac{d}{\Lambda_{\eps}} \leq \frac{\gamma_{\infty}}{4}$ and $C_{\rm md} \eta_{t_0} \leq \frac{\gamma_{\infty}}{4}$, we further obtain $\frac{\gamma_{\infty}}{2} V_{\eps} \leq \ell_{\eps} + C_1 + 2 \log 2$. Consequently, it holds that $V_{\eps} \leq \frac{2}{\gamma_{\infty}} (\ell_{\eps} + C_1 + 2 \log 2)$. We choose $\eps$ sufficiently small such that $\ell_{\eps} \geq e > 1$. Since $S_{\sigma_{\eps}} = S_{t_0} + V_{\eps}$ and $\ell_{\eps} \geq 1$, it holds that 
    \begin{align*}
        S_{\sigma_{\eps}} &\leq S_{t_0} + \frac{2}{\gamma_{\infty}} (\ell_{\eps} + C_1 + 2 \log 2) \leq S_{t_0} \ell_{\eps} + \frac{2}{\gamma_{\infty}} \ell_{\eps} + \frac{2}{\gamma_{\infty}} (C_1 + 2 \log 2) \ell_{\eps} \\
        &= [S_{t_0} + \frac{2}{\gamma_{\infty}}(1 + C_1 + 2 \log 2)] \ell_{\eps} = A_0 \ell_{\eps},
    \end{align*}
    where $A_0 := S_{t_0} + \frac{2}{\gamma_{\infty}} (1 + C_1 + 2 \log 2)$. For $a < 1$, we have 
    \begin{align*}
        S_{\sigma_{\eps}} = \eta_0 \sum_{k=1}^{\sigma_{\eps}} k^{-a} \geq \eta_0 \sum_{k=1}^{\sigma_{\eps}} \sigma_{\eps}^{-a} = \eta_0 \sigma_{\eps}^{1 - a}. 
    \end{align*}
    Combining this with $S_{\sigma_{\eps}} \leq A_0 \ell_{\eps}$, we obtain $\sigma_{\eps}^{1-a} \leq \frac{A_0}{\eta_0} \ell_{\eps} \leq B \ell_{\eps}$ where $B := \max \{ 1, \frac{A_0}{\eta_0}\}$. The condition $B \geq 1$ ensures that $\log B \geq 0$. 
    
    By the definition of $Q_{\sigma_{\epsilon};t_0}$, we have 
    \begin{align*}
        Q_{\sigma_{\epsilon};t_0} = \eta_0^2 \sum_{t = t_0}^{\sigma_{\eps} -1} (t+1)^{-2a} \leq \eta_0^2 \sum_{t=0}^{\sigma_{\eps} - 1} (t+1)^{-2a} = \eta_0^2 \sum_{k=1}^{\sigma_{\eps}} k^{-2a}. 
    \end{align*}
    In the case where $1/3 < a < 1/2$, we have
    \begin{align*}
        \sum_{k=1}^{\sigma_{\eps}} k^{-2a} \leq 1 + \int_{1}^{\sigma_{\eps}} x^{-2a} dx = \frac{\sigma_{\eps}^{1-2a} - 2a}{1 - 2 a} \leq \frac{\sigma_{\eps}^{1-2a}}{1 - 2a}.
    \end{align*}
    Consequently, it holds that $Q_{\sigma_{\epsilon};t_0} \leq \frac{\eta_0^2}{1 - 2a} \sigma_{\eps}^{1 - 2a}$. Define $p_a := \frac{1 - 2a}{1-a}$. It holds that $1 - 2a = p_a (1-a)$ and $p_a > 0$ given $a < \frac{1}{2}$. Therefore, we have $\sigma_{\eps}^{1 -2a} = (\sigma_{\eps}^{1 - a})^{p_a} \leq (B \ell_{\eps})^{p_a} = B^{p_a} \ell_{\eps}^{p_a}$. Moreover, it holds that $(1-a) - p_a = \frac{a^2}{1-a} > 0$ and therefore $p_a < 1 - a$. This implies that $\ell_{\eps}^{p_a} \leq \ell_{\eps}^{1-a}$ given $\ell_{\eps} \geq 1$ and $Q_{\sigma_{\epsilon};t_0} \leq \frac{\eta_0^2}{1 - 2a} B^{p_a} \ell_{\eps}^{1-a}$.  
    In the case where $a = \frac{1}{2}$, we have
    \begin{align*}
        Q_{\sigma_{\epsilon};t_0} \leq \eta_0^2 \sum_{k=1}^{\sigma_{\eps}} \frac{1}{k} \leq \eta_0^2 (1 + \log \sigma_{\eps}). 
    \end{align*}
    From $\sigma_{\eps}^{1-a} \leq B \ell_{\eps}$ when $a < 1$, we obtain $\sigma_{\eps}^{1/2} \leq B \ell_{\eps}$ and consequently $\sigma_{\eps} \leq B^2 \ell_{\eps}^2$. This implies that $\log \sigma_{\eps} \leq \log (B^2 \ell_{\eps}^2) = 2 \log B + 2 \log \ell_{\eps}$. Therefore, we have $Q_{\sigma_{\epsilon};t_0} \leq \eta_0^2 (1 + 2 \log B + 2 \log \ell_{\eps})$. Given $\ell_{\eps} \geq 1$, we have $1 \leq \sqrt{\ell_{\eps}}$ and $\log \ell_{\eps} \leq \sqrt{\ell_{\eps}}$. Therefore, it holds that 
    \begin{align*}
        1 + 2 \log B + 2 \log \ell_{\eps} \leq (1 + 2 \log B) \sqrt{\ell_{\eps}} + 2 \sqrt{\ell_{\eps}} = (3 + 2 \log B) \sqrt{\ell_{\eps}}.
    \end{align*}
    From this, we conclude that $Q_{\sigma_{\epsilon};t_0} \leq \eta_0^2 (3 + 2 \log B) \sqrt{\ell_{\eps}}$. In the case where $\frac{1}{2} < a < 1$, given $\sum_{k=1}^{\infty} k^{-2a} \leq 1 + \int_{1}^{\infty} x^{-2a} dx = 1 + \frac{1}{2a - 1}$, it holds that 
    \begin{align*}
        Q_{\sigma_{\epsilon};t_0} \leq \eta_0^2(1 + \frac{1}{2a - 1}) \leq \eta_0^2(1 + \frac{1}{2a - 1}) \omega_a(\ell_{\eps}),
    \end{align*} 
    where the last inequality is by $\omega_a(\ell_{\eps}) = \ell_{\eps}^{1-a} \geq 1$ given $\ell_{\eps} \geq 1$. Finally, in the case where $a=1$,  we have 
    \begin{align*}
        Q_{\sigma_{\epsilon};t_0} \leq \eta_0^2 \sum_{t=0}^{\infty} (t + 1)^{-2} = \frac{\pi^2}{6} \eta_0^2 \leq \frac{\pi^2}{6} \eta_0^2 \log \ell_{\eps} = \frac{\pi^2}{6} \eta_0^2 \omega_1(\ell_{\eps}),
    \end{align*}
    where we used $\omega_1(\ell_{\eps}) = \log \ell_{\eps} \geq 1$ given $\ell_{\eps} \geq e$. In summary, we let $C_2 = \frac{\eta_0^2}{1-2a} B^{\frac{1 - 2a}{1 -a}}$ when $\frac{1}{3} < a < \frac{1}{2}$, $C_2 = \eta_0^2 (3 + 2 \log B)$ when $a = \frac{1}{2}$, $C_2 = \eta_0^2 (1 + \frac{1}{2a -1})$ when $\frac{1}{2} < a < 1$, and $C_2 = \frac{\pi^2}{6} \eta_0^2$ when $a = 1$, then it holds that $Q_{\sigma_{\epsilon};t_0} \leq C_2 \omega_a(\ell_{\eps})$. 
        
    Define $U_{\sigma_{\eps}} := S_{\sigma_{\eps}} - S_{T_*}$ and $B_{0} := S_{t_0} - S_{T_*}$ (note that $B_0$ is a constant independent of $\eps$). Then it holds that $V_{\eps} = U_{\sigma_{\eps}} - B_0$.  Substituting this into \eqref{eq:log_mm}, together with the bound $\sum_{t=t_0}^{\sigma_{\eps} -1} \eta_t L_t \leq \frac{4 \log 2}{n \gamma_{\infty}}$, we obtain
    \begin{align}
        \gamma_{\infty} U_{\sigma_{\eps}} &\leq \log \frac{L_{t_0}}{L_{\sigma_{\eps}}} + \gamma_{\infty} B_0 + \frac{d}{\Lambda_{\eps}} U_{\sigma_{\eps}} - \frac{d B_0}{\Lambda_{\eps}} + C_{\rm md} Q_{\sigma_{\epsilon};t_0} + 2 \log 2 \nonumber \\ 
        &\leq \log \frac{L_{t_0}}{L_{\sigma_{\eps}}} + \gamma_{\infty} B_0 + \frac{d}{\Lambda_{\eps}} U_{\sigma_{\eps}} + C_{\rm md} Q_{\sigma_{\epsilon};t_0} + 2 \log 2 \nonumber \\
        &\leq \ell_{\eps} + C_1 + \gamma_{\infty} B_0 + \frac{d}{\Lambda_{\eps}} U_{\sigma_{\eps}} + C_{\rm md} C_2 \omega_a (\ell_{\eps}) + 2 \log 2 \nonumber \\ 
        &\leq \ell_{\eps} + \frac{d}{\Lambda_{\eps}} U_{\sigma_{\eps}} + C_{\rm md} C_2 \omega_a(\ell_{\eps}) + C_3 \nonumber \nonumber \\
        &\leq \ell_{\eps} + C_{4} [1 + \frac{U_{\sigma_{\eps}}}{\Lambda_{\eps}} + \omega_a(\ell_{\eps})], \label{eq:log_5}
    \end{align}
    where $C_3 := C_1 + \gamma_{\infty} B_0 + 2 \log 2$ and $C_4 := \max \{ d, C_{\rm md} C_2, C_3\}$. Note that $U_{\sigma_{\eps}}$ appears in both sides of the previous equation. We remove it in two stages. We first show the coarse bound $U_{\sigma_{\eps}} = O(\ell_{\eps})$. Choosing $\eps$ sufficiently small such that $\frac{C_4}{\Lambda_{\eps}} \leq \frac{\gamma_{\infty}}{2}$, we obtain from \eqref{eq:log_5} that 
    \begin{align*}
        \frac{\gamma_{\infty}}{2} U_{\sigma_{\eps}} \leq (\gamma_{\infty} - \frac{C_4}{\Lambda_{\eps}}) U_{\sigma_{\eps}} \leq \ell_{\eps} + C_4 [1 + \omega_{a}(\ell_{\eps})]. 
    \end{align*}
    Therefore, it holds that $U_{\sigma_{\eps}} \leq \frac{2}{\gamma_{\infty}} [\ell_{\eps} + C_4(1+ \omega_a(\ell_{\eps}))]$. 
    Given $\omega_{a}(\ell_{\eps}) \leq \ell_{\eps}$ (as $\ell_{\eps} \geq 1$), we further have
    \begin{align*}
        U_{\sigma_{\eps}} \leq \frac{2(1 + 2C_4)}{\gamma_{\infty}} \ell_{\eps} = C_U \ell_{\eps},
    \end{align*}
    where $C_U := \frac{2(1 + 2C_4)}{\gamma_{\infty}}$. This implies that $\frac{U_{\sigma_{\eps}}}{\Lambda_{\eps}} \leq C_U \frac{\ell_{\eps}}{\Lambda_{\eps}}$. By the definition of $\ell_{\eps}$ and $\Lambda_{\eps}$, it holds that $\frac{\ell_{\eps}}{\Lambda_{\eps}} = \omega_{a}(\ell_{\eps})$. Therefore, we have that $\frac{U_{\sigma_{\eps}}}{\Lambda_{\eps}} \leq C_U \omega_{a}(\ell_{\eps})$. Substituting this into \eqref{eq:log_5}, we obtain
    \begin{align*}
        \gamma_{\infty} U_{\sigma_{\eps}} \leq \ell_{\eps} + C_4 [1 + C_U \omega_{a}(\ell_{\eps}) + \omega_a(\ell_{\eps})]. 
    \end{align*}
    Given $\ell_{\eps} \geq e$, it holds that $\omega_{a}(\ell_{\eps}) \geq 1$ and consequently $1 + (C_U + 1) \omega_{a}(\ell_{\eps}) \leq (C_U + 2) \omega_a(\ell_{\eps})$. Therefore, we have $\gamma_{\infty} U_{\sigma_{\eps}} \leq \ell_{\eps} + C_4 (C_U + 2) \omega_{a}(\ell_{\eps})$. After dividing by $\gamma_{\infty}$, we obtain
    \begin{align*}
        U_{\sigma_{\eps}} \leq \frac{\ell_{\eps}}{\gamma_{\infty}} + \frac{C_4 (C_U + 2)}{\gamma_{\infty}} \omega_a(\ell_{\eps}).
    \end{align*}
    From this, we conclude that $S_{\sigma_{\eps}} - S_{T_*} \leq \frac{\ell_{\eps}}{\gamma_{\infty}} + C_{\sigma} \omega_a(\ell_{\eps})$ where $C_{\sigma} := \frac{C_4 (C_U + 2)}{\gamma_{\infty}}$. Finally, we give concrete choices of $\eps$ such that all conditions above are satisfied. 
    We define:
    \begin{align*}
        M_{\delta} := \max \{1, 2 \delta, \frac{4 d}{\gamma_{\infty}}, \frac{2 C_4}{\gamma_{\infty}} \}, \quad L_{1,\delta}
        :=
        \begin{cases}
        \displaystyle
        \max\left\{
        \mathrm{e},\,
        2\log\left(\frac{2}{c_{t_0}}\right),\,
        M_{\delta}^{\,1/a}
        \right\},
        & \displaystyle \frac{1}{3}<a<1,
        \\[3mm]
        \displaystyle
        \max\left\{
        \mathrm{e},\,
        2\log\left(\frac{2}{c_{t_0}}\right),\,
        M_{\delta}^{\,2}
        \right\},
        & a=1.
        \end{cases}
    \end{align*}
    Let $\eps_1(\delta) := e^{-L_{1,\delta}}$. We now verify every $0 < \eps \leq \eps_{1}(\delta)$ satisfies all the desired properties. Given $\eps \leq e^{-L_{1,\delta}}$, it holds that $\ell_{\eps} = \log \frac{1}{\eps} \geq L_{1,\delta}$. Consequently, it holds that $\ell_{\eps} \geq e$ and $\ell_{\eps} \geq 2 \log \frac{2}{c_{t_0}}$. Given $\ell_{\eps} \geq M_{\delta}^{1/a}$ when $\frac{1}{3} < a < 1$, we have $\Lambda_{\eps} = \ell_{\eps}^a \geq (M_{\delta}^{1/a})^{a} = M_{\delta}$. Given $\ell_{\eps} \geq M_{\delta}^{2}$ when $a = 1$, we have $\Lambda_{\eps} = \frac{\ell_{\eps}}{\log \ell_{\eps}} \geq \sqrt{\ell_{\eps}} \geq M_{\delta}$ where we used $\log x \leq \sqrt{x}$ when $x \geq 1$. Therefore, it holds that $\Lambda_{\eps} \geq M_{\delta}$ in both cases. Moreover, it holds that 
    \begin{align*}
        \eps \Lambda_{\eps} \stackrel{(a)}{\leq} e^{-\ell_{\eps}} \ell_{\eps} \stackrel{(b)}{\leq} e^{-\ell_{\eps}/2} \stackrel{(c)}{\leq} \frac{c_{t_0}}{2} < c_{t_0},
    \end{align*}
    where (a) is by $\Lambda_{\eps} \leq \ell_{\eps}$ (as $\ell_{\eps} \geq e > 1$), (b) is by $\log x < \frac{x}{2}$ for every $x > 0$, and (c) is by $\ell_{\eps} \geq 2 \log \frac{2}{c_{t_0}}$. In summary, the following conditions simultaneously hold for every $0 < \eps \leq \eps_{1}(\delta)$: $\ell_{\eps} \geq e > 1$, $\Lambda_{\eps} \geq 2\delta > \delta$, $\Lambda_{\eps} \geq 1$, $\frac{d}{\Lambda_{\eps}} \leq \frac{\gamma_{\infty}}{4}$, $\frac{C_4}{\Lambda_{\eps}} \leq \frac{\gamma_{\infty}}{2}$, and $\eps \Lambda_{\eps} < c_{t_0}$. Consequently, it holds that $\sigma_{\eps} > t_0$ and $\sigma_{\eps} \leq \tau_{\eps}$. For the remainder of the proof, we assume $0< \eps \leq \eps_1(\delta)$ and track additional conditions on $\eps$ that haven't been included. 
     
    \paragraph{Part 2} Next, we want to show $S_{\tau_{\eps}} - S_{T_*} \leq \frac{\ell_{\eps}}{\gamma_{\infty}} + O_{\delta} (\omega_a(\ell_{\eps}))$. If $\sigma_{\eps} = \tau_{\eps}$, the desired  bound follows immediately. Below, we consider the nontrivial case $\sigma_{\eps} < \tau_{\eps}$. By the definition of $\tau_{\eps}$, we have 
    \begin{align}
        G_t > \delta \eps, \qquad \sigma_{\eps} \leq t < \tau_{\eps}. \label{eq:log_6}
    \end{align}
    \cref{lem:softsign} gives $\ip{g}{\sigma_\eps(g)} \ge\frac{\norm{g}_1^2}{\norm{g}_1+d\eps}$. From $\norm{g_t}_1 \geq \gamma_{\infty} G_t$ and $\norm{g_t}_1 \leq R_1 G_t$, we have 
    \begin{align*}
        \langle g_t, \sigma_{\eps} (g_t) \rangle &\geq \frac{\norm{g_t}_1^2}{\norm{g_t}_1 + d \eps} \geq \frac{\gamma_{\infty}^2 G_t^2}{R_1 G_t + d \eps} = \frac{\gamma_{\infty}^2}{R_1 + d \eps / G_t} G_t. 
    \end{align*}
    From \eqref{eq:log_6}, we have $\frac{d \eps}{G_t} < \frac{d}{\delta}$. Consequently, it holds that $\frac{1}{R_1 + d \eps / G_t} \geq \frac{1}{R_1 + d/\delta}$. Define $\kappa_{\delta}:= \frac{\gamma_{\infty^2}}{R_1 + d/\delta} = \frac{\gamma_{\infty}^2 \delta}{R_1 \delta + d} > 0$, then it holds that
    \begin{align*}
        \langle g_t, \sigma_{\eps} (g_t) \rangle \geq \kappa_{\delta} G_t, \qquad \sigma_{\eps} \leq t < \tau_{\eps}. 
    \end{align*}
    From \eqref{eq:master-descent-app}, we have
    \begin{align*}
        L_{t+1} \leq L_t - \kappa_{\delta} \eta_t G_t + C_{\rm md} \eta_t^2 G_t = L_t - \eta_t (\kappa_{\delta} - C_{\rm md} \eta_t) G_t. 
    \end{align*}
    Given $\sigma_{\eps} \rightarrow \infty$, we have $\eta_{\sigma_{\eps}} \rightarrow 0$. Consequently, for sufficiently small $\eps$, we have $C_{\rm md} \eta_t \leq C_{\rm md} \eta_{\sigma_{\eps}} \leq \frac{\kappa_{\delta}}{2}, t \geq \sigma_{\eps}$. This leads to
    \begin{align*}
        L_{t+1} \leq L_t - \frac{\kappa_{\delta}}{2} \eta_t G_t, \qquad \sigma_{\eps} \leq t < \tau_{\eps}. 
    \end{align*}
    For $t \geq \sigma_{\eps} \geq T_*$, we have $G_t \geq \frac{1}{2} L_t$ by \cref{prop:uniform-clock}. Using this, we further obtain 
    \begin{align}
        L_{t+1} \leq Lt - \frac{\kappa_{\delta}}{4} \eta_t L_t = (1 - c_{\delta} \eta_t) L_t, \qquad \sigma_{\eps} \leq t < \tau_{\eps}. \label{eq:log_7} 
    \end{align}
    where $c_{\delta} := \frac{\kappa_{\delta}}{4}$. We additionally require $\eps$ to be sufficiently small such that $\eta_{\sigma_{\eps}} \leq \frac{2}{\kappa_{\delta}}$. This would ensure that $\frac{1}{2} \leq 1 - \frac{\kappa_{\delta}}{4} \eta_t \leq 1$. Iterating \eqref{eq:log_7} from $\sigma_{\eps}$ to any $\sigma_{\eps} \leq t \leq \tau_{\eps}$, it holds that
    \begin{align}
        L_t &\leq L_{\sigma_{\eps}} \prod_{s = \sigma_{\eps}}^{t-1} (1 - c_{\delta} \eta_s) \leq L_{\sigma_{\eps}} \prod_{s = \sigma_{\eps}}^{t-1} e^{-c_{\delta} \eta_s} = L_{\sigma_{\eps}} \exp(-c_{\delta} \sum_{s = \sigma_{\eps}}^{t-1} \eta_s) \nonumber \\
        &= L_{\sigma_{\eps}} \exp[-\frac{\kappa_{\delta}}{4} (S_t - S_{\sigma_{\eps}})]. \label{eq:log_8} 
    \end{align}
    By the definition of $\sigma_{\eps}$, we have $G_{\sigma_{\eps}} \leq \eps \Lambda_{\eps}$. Since $\sigma_{\eps} \geq T_*$, \cref{prop:uniform-clock} gives $L_{\sigma_{\eps}} \leq 2 G_{\sigma_{\eps}}$. Consequently, it holds that $L_{\sigma_{\eps}} \leq 2 \eps \Lambda_{\eps}$. Define $A_{\eps} := \frac{4}{\kappa_{\delta}} \log \frac{2 \Lambda_{\eps}}{\delta}$. Given $\Lambda_{\eps} > \delta$, it holds that $A_{\eps} > 0$. Define $\bar{\tau}_{\eps} := \inf \{ t \geq \sigma_{\eps}: S_t - S_{\sigma_{\eps}} \geq A_{\eps} \}$. This time is finite because $S_t - S_{\sigma_{\eps}} = \sum_{s = \sigma_{\eps}}^{t-1} \eta_s \rightarrow \infty$ as $t \rightarrow \infty$. We claim that $\tau_{\eps} \leq \bar{\tau}_{\eps}$. Suppose (for contradiction) that $\bar{\tau}_{\eps} < \tau_{\eps}$. Applying \eqref{eq:log_8} and using $L_{\sigma_{\eps}} \leq 2 \eps \Lambda_{\eps}$, we obtain
    \begin{align*}
        L_{\bar{\tau}_{\eps}} \leq L_{\sigma_{\eps}} \exp[-\frac{\kappa_{\delta}}{4} (S_{\bar{\tau}_{\eps}} - S_{\sigma_{\eps}})] \leq 2 \eps \Lambda_{\eps} \exp[-\frac{\kappa_{\delta}}{4} A_{\eps}] = 2 \eps \Lambda_{\eps} \frac{\delta}{2 \Lambda_{\eps}} = \delta \eps. 
    \end{align*}
    Consequently, we have $G_{\bar{\tau}_{\eps}} \leq L_{\bar{\tau}_{\eps}} \leq \delta \eps$, which means that $\bar{\tau}_{\eps}$ belongs to the set $\{ t \geq T^*: G_t \leq \delta \eps \}$. By the definition of $\tau_{\eps}$, it holds that $\tau_{\eps} \leq \bar{\tau}_{\eps}$, which is a contradiction. Therefore, the claim $\tau_{\eps} \leq \bar{\tau}_{\eps}$ holds. 
    By minimality of $\bar{\tau}_{\eps}$, it holds that $S_{\bar{\tau}_{\eps} - 1} - S_{\sigma_{\eps}} < A_{\eps}$ and $S_{\bar{\tau}_{\eps}} - S_{\sigma_{\eps}} \geq A_{\eps}$. This implies that $S_{\bar{\tau}_{\eps}} - S_{\sigma_{\eps}} = S_{\bar{\tau}_{\eps} - 1} - S_{\sigma_{\eps}} + \eta_{\bar{\tau}_{\eps} - 1} < A_{\eps} + \eta_{\bar{\tau}_{\eps} - 1}$. Since $\eta_t$ is decreasing and $\bar{\tau}_{\eps} - 1 \geq \sigma_{\eps}$, we have $\eta_{\bar{\tau}_{\eps} - 1} \leq \eta_{\sigma_{\eps}}$ and consequently $S_{\bar{\tau}_{\eps}} - S_{\sigma_{\eps}} < A_{\eps} + \eta_{\sigma_{\eps}}$. Given $\tau_{\eps} \leq \bar{\tau}_{\eps}$ and $S_t$ is increasing, we conclude that
    \begin{align}
        S_{\tau_{\eps}} - S_{\sigma_{\eps}} \leq S_{\bar{\tau}_{\eps}} - S_{\sigma_{\eps}} \leq \frac{4}{\kappa_{\delta}} \log \frac{2 \Lambda_{\eps}}{\delta} + \eta_{\sigma_{\eps}} \leq \frac{4}{\kappa_{\delta}} \log \Lambda_{\eps} + \frac{4}{\kappa_{\delta}} \log \frac{2}{\delta} + \eta_{0}. \label{eq:log_tt}
    \end{align}
    Next, we show that $\log \Lambda_{\eps} = O(\omega_a(\ell_{\eps}))$. For $a \in (\frac{1}{3},1)$, we have $\Lambda_{\eps} =\ell_{\eps}^a$ and $\log \Lambda_{\eps} = a \log \ell_{\eps}$. For any $p > 0$, since $u^p \geq 1$ when $u \geq 1$, it holds that
    \begin{align*}
        \log x = \int_{1}^x \frac{1}{u} du \leq \int_{1}^x u^{p-1} du = \frac{x^p - 1}{p} \leq \frac{x^p}{p}.
    \end{align*}
    Taking $x = \ell_{\eps}$ and $p = 1 - a$ gives $\log \ell_{\eps} \leq \frac{\ell_{\eps}^{1-a}}{1 - a}$. Therefore, we conclude that
    \begin{align*}
        \log \Lambda_{\eps} = a \log \ell_{\eps} \leq \frac{a}{1-a} \ell_{\eps}^{1-a} = \frac{a}{1-a} \omega_{a} (\ell_{\eps}). 
    \end{align*}
    For $a = 1$, we have $\Lambda_{\eps} = \frac{\ell_{\eps}}{\log \ell_{\eps}}$. For $\ell_{\eps} \geq e$, it holds that 
    \begin{align*}
        \log \Lambda_{\eps} = \log \ell_{\eps} - \log \log \ell_{\eps} \leq \log \ell_{\eps} = \omega_1 (\ell_{\eps}).
    \end{align*}
    Therefore, by defining $c_a := \frac{a}{1-a}$ when $a \in (\frac{1}{3},1)$ and $c_a := 1$ when $a = 1$, it holds that $\log \Lambda_{\eps} \leq c_a \omega_{a}(\ell_{\eps})$. Given $\omega_{a}(\ell_{\eps}) \geq 1$ (as $\ell_{\eps} \geq e$), we obtain via \eqref{eq:log_tt} that 
    \begin{align*}
        S_{\tau_{\eps}} - S_{\sigma_{\eps}} \leq C_{\rm tr,\delta} \omega_a(\ell_{\eps}),
    \end{align*}
    where $C_{\rm tr,\delta} := \eta_0 + \frac{4}{\kappa_{\delta}} (c_a + \log \frac{2}{\delta}) = \eta_0 + \frac{4(R_1 \delta + d)}{\gamma_{\infty}^2 \delta} (c_a + \log \frac{2}{\delta})$. 
    Previously, we have shown that $S_{\sigma_{\eps}} - S_{T_*} \leq \frac{\ell_{\eps}}{\gamma_{\infty}} + C_{\sigma} \omega_a(\ell_{\eps})$. Combining them together, we obtain 
    \begin{align}
        S_{\tau_{\eps}} - S_{T_*} = S_{\tau_{\eps}} - S_{\sigma_{\eps}} + S_{\sigma_{\eps}} - S_{T_*} \leq \frac{\ell_{\eps}}{\gamma_{\infty}} + C_{\tau,\delta} \omega_a(\ell_{\eps}), \label{eq:log_S}  
    \end{align}
    where $C_{\tau, \delta} := C_{\sigma} + C_{\rm tr,\delta}$. The additional conditions on the choice of $\eps$ in this part are $C_{\rm md} \eta_{\sigma_{\eps}} \leq \frac{\kappa_{\delta}}{2}$ and $\eta_{\sigma_{\eps}} \leq \frac{2}{\kappa_{\delta}}$. Choose $\bar{\eta}_{2,\delta} := \min \{ \eta_0, \frac{\kappa_{\delta}}{2 C_{\rm md}},\frac{2}{\kappa_{\delta}}\}$ and $N_{2,\delta} := \max \{ t_0, \lceil (\frac{\eta_0}{\bar{\eta}_{2,\delta}})^{1/a} \rceil \}$. It holds that $\eta_{N_{2,\delta}} = \frac{\eta_0}{(N_{2,\delta} + 1)^a} \leq \bar{\eta}_{2,\delta}$. By \cref{lem:proxy-path-ratio}, we have for every $0 \leq t \leq N_{2,\delta}$ that
    \begin{align}
        G_t \geq G_0 e^{-K_0 S_t} \geq G_0 e^{-K_0 S_{N_{2,\delta}}} = c_{2,\delta}. \label{eq:pp1}
    \end{align}
    Define $\eps_2(\delta) := (\frac{c_{2,\delta}}{2})^2$ and choose $0 < \eps \leq \min \{ \eps_1(\delta), \eps_2(\delta)\}$. Then, we have 
    \begin{align}
        \eps \Lambda_{\eps} \leq e^{-\ell_{\eps}} \ell_{\eps} \leq e^{-\ell_{\eps}/2} = \sqrt{\eps} \leq \frac{c_{2,\delta}}{2} < c_{2,\delta}. \label{eq:pp2}
    \end{align}
    Combining \eqref{eq:pp1} and \eqref{eq:pp2}, it holds that
    \begin{align*}
        G_t > \eps \Lambda_{\eps}, \qquad 0 \leq t \leq N_{2,\delta}.
    \end{align*}
    This implies that $\sigma_{\eps} > N_{2,\delta}$ and consequently $\eta_{\sigma_{\eps}} \leq \eta_{N_{2,\delta}} \leq \bar{\eta}_{2,\delta}$. Therefore, it holds that $C_{\rm md} \eta_{\sigma_{\eps}} \leq C_{\rm md} \bar{\eta}_{2,\delta} \leq \frac{\kappa_{\delta}}{2}$ and $\eta_{\sigma_{\eps}} \leq \bar{\eta}_{2,\delta} \leq \frac{2}{\kappa_{\delta}}$. For the remainder of the proof, we assume $0< \eps \leq \min \{ \eps_1(\delta), \eps_2(\delta)\}$ and track additional conditions on $\eps$ that haven't been included. 
    
    \paragraph{Part 3} We prove that $\norm{w_{\tau_{\eps}}}_{\infty} \leq S_{\tau_{\eps}} - S_{T_*} + O(\omega_a(\ell_{\eps}))$. Its proof uses the radius bound of \cref{lem:radius}. Define $k_{\eps} := \lceil \ell_{\eps} \rceil$. Given $\eps$ is chosen such that $\ell_{\eps} \geq e$, it ensures that $k_{\eps} \leq \ell_{\eps} + 1 \leq 2 \ell_{\eps}$ and $\log \ell_{\eps} \geq 1$. 
    In the case of $a \in (\frac{1}{3},1)$, using integral approximation we have
    \begin{align*}
        S_{k_{\eps}} &= \eta_0 \sum_{j=1}^{k_{\eps}} j^{-a} \leq \eta_0 (1 + \int_{1}^{k_{\eps}} x^{-a} dx) = \eta_0 (1 + \frac{k_{\eps}^{1-a} - 1}{1 - a}) = \eta_0 \frac{k_{\eps}^{1-a} - a}{1 - a} \\
        &\leq \frac{\eta_0}{1-a} k_{\eps}^{1-a} \leq \frac{2^{1-a} \eta_0}{1-a} \ell_{\eps}^{1-a}. 
    \end{align*}
    In the case where $a = 1$, we have
    \begin{align*}
        S_{k_{\eps}} = \eta_0 \sum_{j=1}^{k_{\eps}} \frac{1}{j} \leq \eta_0 (1 + \log k_{\eps}) \leq \eta_0 (1 + \log 2 + \log \ell_{\eps}) \leq \eta_0 (2+ \log 2) \log \ell_{\eps}. 
    \end{align*}
    Define $K_S := \frac{2^{1-a} \eta_0}{1 - a}$ when $a \in (\frac{1}{3},1)$ and $K_S := \eta_0 (2 + \log 2)$ when $a = 1$, it then holds that $S_{k_{\eps}} \leq K_S \omega_a(\ell_{\eps})$. This implies that $\frac{S_{k_{\eps}}}{\ell_{\eps}} \leq K_S \frac{\omega_{a}(\ell_{\eps})}{\ell_{\eps}} \rightarrow 0$ and therefore $S_{k_{\eps}} = o(\ell_{\eps})$.  

    Define $L_{3,\delta} := \max \{ T_*, 2[\log \frac{\delta}{G_0}]_{+}, \Omega_a \}$ where $\Omega_a := (4 K_0 K_S)^{1/a}$ when $\frac{1}{3} < a < 1$ and $(4 K_0 K_S)^2$ when $a = 1$. We show that $T_* \leq k_{\eps} < \tau_{\eps}$ for every $0 < \eps \leq \eps_3(\delta) := e^{-L_{3,\delta}}$. \cref{lem:proxy-path-ratio} gives $G_{\tau_{\eps}} \geq G_0 \exp(-K_0 S_{\tau_{\eps}})$. By the definition of $\tau_{\eps}$, we have $G_{\tau_{\eps}} \leq \delta \eps$ and therefore $G_0 \exp(-K_0 S_{\tau_{\eps}}) \leq \delta \eps$. 
    This implies that $S_{\tau_{\eps}} \geq \frac{1}{K_0} \log \frac{G_0}{ \delta \eps} = \frac{1}{K_0} (\ell_{\eps} + \log \frac{G_0}{\delta})$. Therefore, given $\ell_{\eps} \geq 2 [\log \frac{\delta}{G_0}]_{+}$, it holds that $S_{\tau_{\eps}} \geq \frac{\ell_{\eps}}{2 K_0}$. When $\frac{1}{3} < a < 1$, the condition $\ell_{\eps} \geq \Omega_a$ implies that $K_S \ell_{\eps}^{1-a} \leq \frac{\ell_{\eps}}{4 K_0}$, using $\log x \leq \sqrt{x}, x \geq 1$, the condition $\ell_{\eps} \geq \Omega_1 = (4 K_0 K_S)^2$ implies that $K_S \log \ell_{\eps} \leq K_S \sqrt{\ell_{\eps}} \leq \frac{\ell_{\eps}}{4 K_0}$. Consequently, in both cases, we have 
    \begin{align*}
        S_{k_{\eps}} \leq K_S \omega_a (\ell_{\eps}) \leq \frac{\ell_{\eps}}{4 K_0} < \frac{\ell_{\eps}}{2 K_0} \leq S_{\tau_{\eps}}. 
    \end{align*}
    Since $S_t$ is strictly increasing, this implies that $k_{\eps} < \tau_{\eps}$. Given $\ell_{\eps} \geq T_*$, it holds that $k_{\eps} = \lceil \ell_{\eps} \rceil \geq \ell_{\eps}$. Therefore, we conclude that $T_* \leq k_{\eps} < \tau_{\eps}$ for every $0 < \eps \leq \eps_{3}(\delta)$. 
    

    From \cref{lem:uniform-D} and $w_{k_{\eps}} = w_0 - \sum_{t=0}^{k_{\eps} - 1} \eta_t D_t$, we have
    \begin{align*}
        \norm{w_{k_{\eps}}}_{\infty} &\leq \norm{w_0}_{\infty} + \sum_{t=0}^{k_{\eps} - 1} \eta_t \norm{D_t}_{\infty} \leq \norm{w_0}_{\infty} + C_D S_{k_{\eps}} \\ 
        &\leq \norm{w_0}_{\infty} + C_D K_S \omega_{a}(\ell_{\eps}) \leq (\norm{w_0}_{\infty} + C_D K_S) \omega_a (\ell_{\eps}) \\
        &= K_{p} \omega_a(\ell_{\eps}),
    \end{align*}
    where $K_{p} := \norm{w_0}_{\infty} + C_D K_S$ and we used $\omega_{a}(\ell_{\eps}) \geq 1$ given $\ell_{\eps} \geq e$. Therefore, it holds that $\norm{w_{k_{\eps}}}_{\infty} = O(\omega_{a}(\ell_{\eps}))$. 

    Given $|g_t [j]| \leq R_1 G_t \leq R_1, \forall j$, we can set $B_g = R_1$ in \cref{lem:radius}. With this $B_g$, we define $H_{\eps}:= 1 + \log \frac{R_1^2 + \eps^2}{(1 - \beta_2) \eps^2}$. It then holds that
    \begin{align}
        H_{\eps} = 2 \ell_{\eps} + 1 + \log \frac{R_1^2 + \eps^2}{1 - \beta_2} \leq 2 \ell_{\eps} + H_0, \label{eq:log_9}
    \end{align}
    where $H_0:= 1 + \log \frac{R_1^2 + 1}{1 - \beta_2}$ and we used $R_1^2 + \eps^2 \leq R_1^2 + 1$ given $0 < \eps \leq 1$. 

    Next, we bound the term $\eta_{k_{\eps}} H_{\eps}$. Because $k_{\eps} + 1 \geq \ell_{\eps}$, it holds that $\eta_{k_{\eps}} = \eta_0 \frac{1}{(k_{\eps}+1)^a} \leq \eta_0 \ell_{\eps}^{-a}$ and consequently $\eta_{k_{\eps}} H_{\eps} \leq \eta_0 \ell_{\eps}^{-a} (2 \ell_{\eps} + H_0)$ by \eqref{eq:log_9}. For $a \in (\frac{1}{3},1)$, we have $\ell_{\eps}^{-a} \leq \ell_{\eps}^{1 - a}$ since $\ell_{\eps} \geq 1$. Therefore, it holds that
    \begin{align*}
        \eta_{k_{\eps}} H_{\eps} \leq 2 \eta_0 \ell_{\eps}^{1-a} + H_0 \eta_0 \ell_{\eps}^{-a} \leq \eta_0 (2 + H_0) \ell_{\eps}^{1-a}
    \end{align*}
    In the case where $a = 1$, we have
    \begin{align*}
        \eta_{k_{\eps}} H_{\eps} &\leq \eta_0 \ell_{\eps}^{-1} (2 \ell_{\eps} + H_0) = \eta_0 (2 + \frac{H_0}{\ell_{\eps}}) \\
        &\leq \eta_0 (2 + H_0) \leq \eta_0 (2 + H_0) \log \ell_{\eps},
    \end{align*}
    where we used $\ell_{\eps} \geq 1$ and $\log \ell_{\eps} \geq 1$. Therefore, in both cases, we have $\eta_{k_{\eps}} H_{\eps} \leq \eta_0 (2 + H_0) \omega_a (\ell_{\eps})$. Defining $K_{r} := C_{\rm rad} \eta_0 (2 + H_0)$ and multiplying by $C_{\rm rad}$, we conclude
    \begin{align}
        C_{\rm rad} \eta_{k_{\eps}} H_{\eps} \leq K_{r} \omega_a(\ell_{\eps}). \label{eq:log_10}
    \end{align}
    Because $k_{\eps} < \tau_{\eps}$, we apply \cref{lem:radius} with $r = k_{\eps}$ and $T = \tau_{\eps}$ to obtain
    \begin{align*}
        \norm{\sum_{t = k_{\eps}}^{\tau_{\eps} - 1} \eta_t D_t  }_{\infty} &\leq S_{\tau_{\eps}} - S_{k_{\eps}} + C_{\rm rad} \eta_{k_{\eps}} H_{\eps} \\
        &\leq S_{\tau_{\eps}} - S_{k_{\eps}} + K_r \omega_a (\ell_{\eps}). 
    \end{align*}
    Based on $\norm{w_{k_{\eps}}}_{\infty} \leq K_{p} \omega_a (\ell_{\eps})$ and $w_{\tau_{\eps}} = w_{k_{\eps}} - \sum_{t = k_{\eps}}^{\tau_{\eps} - 1} \eta_t D_t$, it holds that
    \begin{align*}
        \norm{w_{\tau_{\eps}}}_{\infty} &\leq \norm{w_{k_{\eps}}}_{\infty} + \norm{\sum_{t = k_{\eps}}^{\tau_{\eps} - 1} \eta_t D_t  }_{\infty}  \\
        &\leq K_p \omega_a(\ell_{\eps}) + S_{\tau_{\eps}} - S_{k_{\eps}} + K_r \omega_a(\ell_{\eps}) \\
        &= S_{\tau_{\eps}} - S_{k_{\eps}} + (K_p + K_r) \omega_a(\ell_{\eps}). 
    \end{align*}
    Because $k_{\eps} \geq T_*$, it holds that $S_{\tau_{\eps}} - S_{k_{\eps}} \leq S_{\tau_{\eps}} - S_{T_*}$. By defining $K_{\tau} := K_p + K_r = \norm{w_0}_{\infty} + C_D K_S + C_{\rm rad} \eta_0[3 + \log \frac{R_1^2 + 1}{1 - \beta_2}]$, we conclude that 
    \begin{align}
        \norm{w_{\tau_{\eps}}}_{\infty} \leq S_{\tau_{\eps}} - S_{T^*} + K_{\tau} \omega_a(\ell_{\eps}). \label{eq:log_w}
    \end{align}
    Combining this with $S_{\tau_\epsilon} - S_{T_*} \leq \frac{\ell_{\eps}}{\gamma_{\infty}} + C_{\tau,\delta} \omega_a(\ell_{\eps})$ (shown above), we conclude that 
    \begin{align}
        \norm{w_{\tau_{\eps}}}_{\infty} \leq \frac{\ell_{\eps}}{\gamma_{\infty}} + C_{w,\delta} \omega_a (\ell_{\eps}), \label{eq:log_11}
    \end{align}
    where $C_{w,\delta} := C_{\tau,\delta} + K_{\tau}$.
    So far, the thresholds on $\eps$ that have been introduced are $\eps_1(\delta)$, $\eps_2(\delta)$, and $\eps_3(\delta)$. Consequently, we set $\Tilde{\eps}_2 (\delta) = \min \{ \eps_1(\delta), \eps_2(\delta), \eps_3(\delta)\}$ and choose $0 < \eps \leq \Tilde{\eps}_2 (\delta)$. This ensures all the requirements on $\eps$ appeared in Parts 1–3 are satisfied and the conclusion in \eqref{eq:log_11} holds.   
    
    \paragraph{Part 4} By definition, we have $\gamma_{\infty} = \max_{\norm{u}_{\infty} \leq 1} \rho(u)$. It holds that $\rho(w) \leq \gamma_{\infty} \norm{w}_{\infty}$ for every $w$. By definition of $\tau_{\eps}$, we have $G_{\tau_{\eps}} \leq \delta \eps$. Since $\tau_{\eps} \geq T_*$, \cref{prop:uniform-clock} gives $L_{\tau_{\eps}} \leq 2 G_{\tau_{\eps}} \leq 2 \delta \eps$. From \cref{lem:loss-proxy}, it also holds that $L_{\tau_{\eps}} \geq \frac{\log 2}{n} e^{-\rho(w_{\tau_{\eps}})}$. Consequently, we have $\frac{\log 2}{n} e^{-\rho(w_{\tau_{\eps}})} \leq 2 \delta \eps$. Rearranging and taking logarithms lead to 
    \begin{align}
        \rho(w_{\tau_{\eps}}) \geq \ell_{\eps} - \log \frac{2 n \delta}{\log 2}. \label{eq:lower_rho1} 
     \end{align}
     By defining $A_{\delta} := [\log \frac{2 n \delta}{\log 2}]_{+}$, we further obtain $\rho(w_{\tau_{\eps}}) \geq \ell_{\eps} - A_{\delta}$ given $A_{\delta} = \max \{ \log \frac{2 n \delta}{\log 2}, 0\} \geq \log \frac{2 n \delta}{\log 2}$. Define $R_{\eps} := \norm{w_{\tau_{\eps}}}_{\infty}$. By \eqref{eq:log_11}, it holds that $R_{\eps} \leq \frac{\ell_{\eps}}{\gamma_{\infty}} + C_{w,\delta} \omega_a(\ell_{\eps})$. Then, we have
     \begin{align*}
         \hat{\gamma}_{\infty}(w_{\tau_{\eps}}) = \frac{\rho(w_{\tau_\eps})}{R_\eps} \geq \frac{\ell_{\eps} - A_{\delta}}{R_{\eps}} \geq \frac{\ell_{\eps} - A_{\delta}}{\ell_{\eps}/\gamma_{\infty} + C_{w,\delta} \omega_a (\ell_{\eps})}. 
     \end{align*}
     Consequently, this leads to
     \begin{align*}
         \gamma_{\infty} - \hat{\gamma}_{\infty} (w_{\tau_{\eps}}) &\leq \gamma_{\infty} - \frac{\ell_{\eps} - A_{\delta}}{\ell_{\eps}/\gamma_{\infty} + C_{w,\delta} \omega_a(\ell_{\eps})} \\
         &= \frac{\gamma_{\infty} C_{w,\delta} \omega_a(\ell_{\eps}) + A_{\delta}}{\ell_{\eps} / \gamma_{\infty} + C_{w,\delta} \omega_a(\ell_{\eps})}. 
     \end{align*}
     Given $C_{w,\delta} \geq 0$ and $\omega_a(\ell_{\eps}) \geq 0$, it holds that $\frac{\ell_{\eps}}{\gamma_{\infty}} + C_{w,\delta} \omega_a(\ell_{\eps}) \geq \frac{\ell_{\eps}}{\gamma_{\infty}}$. Consequently, we have
     \begin{align*}
         \gamma_{\infty} - \hat{\gamma}_{\infty}(w_{\tau_{\eps}}) \leq \gamma_{\infty} \frac{\gamma_{\infty} C_{w,\delta} \omega_a(\ell_{\eps}) + A_{\delta}}{\ell_{\eps}} = \gamma_{\infty}^2 C_{w,\delta} \frac{\omega_a(\ell_{\eps})}{\ell_{\eps}} + \gamma_{\infty} A_{\delta} \frac{1}{\ell_{\eps}}. 
      \end{align*}
      Recall that $r_{\eps} = \frac{\omega_a(\ell_{\eps})}{\ell_{\eps}}$. Given $\omega_a(\ell_{\eps}) \geq 1$, it holds that $r_{\eps} \geq \frac{1}{\ell_{\eps}}$. Then, we have
      \begin{align}
          0 \leq \gamma_{\infty} - \hat{\gamma}_{\infty}(w_{\tau_{\eps}})  \leq \gamma_{\infty}^2 C_{w,\delta} r_{\eps} + \gamma_{\infty} A_{\delta} r_{\eps} = C_{\infty,\delta} r_{\eps}, \label{eq:squeeze}
      \end{align}
      where $C_{\infty,\delta} := \gamma_{\infty}^2 C_{w,\delta} + \gamma_{\infty} A_{\delta}$. For $a \in (\frac{1}{3}, 1)$, it holds that $r_{\eps} = \ell_{\eps}^{-a} \rightarrow 0$ because $a > 0$ and $\ell_{\eps} \rightarrow \infty$. For $a = 1$, we again have $r_{\eps} = \frac{\log \ell_{\eps}}{\ell_{\eps}} \rightarrow 0$. Consequently, by \eqref{eq:squeeze}, it holds that $\hat{\gamma}_{\infty} (w_{T_1^{\delta}(\eps)}) \rightarrow \gamma_{\infty}$ as $\eps \downarrow 0$. 

     From $\rho(w_{\tau_{\eps}}) \leq \gamma_{\infty} \norm{w_{\tau_{\eps}}}_{\infty}$, we further obtain $\norm{w_{\tau_{\eps}}}_{\infty} \geq \frac{\ell_{\eps}}{\gamma_{\infty}} - \frac{1}{\gamma_{\infty}} \log \frac{2 n \delta}{\log 2}$ from \eqref{eq:lower_rho1}. By defining $A_{\delta} := [\log \frac{2 n \delta}{\log 2}]_{+}$, it holds that
     \begin{align*}
         \norm{w_{\tau_{\eps}}}_{\infty} \geq \frac{\ell_{\eps}}{\gamma_{\infty}} - \frac{A_{\delta}}{\gamma_{\infty}} \geq \frac{\ell_{\eps}}{\gamma_{\infty}} - \frac{A_{\delta}}{\gamma_{\infty}} \omega_a (\ell_{\eps}), 
     \end{align*}
     where we used $\omega_a(\ell_{\eps}) \geq 1$ in the last inequality. Combining this with \eqref{eq:log_11}, we obtain 
     $|\norm{w_{\tau_{\eps}}}_{\infty} - \frac{\ell_\eps}{\gamma_{\infty}}| \leq C_{\delta}^{+} \omega_a(\ell_{\eps})$ where $C_{\delta}^{+}:= \max \{ C_{w,\delta}, \frac{A_{\delta}}{\gamma_{\infty}}\}$. Therefore, we conclude that 
     \begin{align}
        \norm{w_{\tau_{\eps}}}_{\infty} = \frac{\ell_{\eps}}{\gamma_{\infty}} + O_{\delta}(\omega_a(\ell_{\eps})). \label{eq:w_scale}
     \end{align}

    \paragraph{Part 5} Denote $w_{\eps}:= w_{\tau_{\eps}}$. Since $\ell_\eps\ge e$ and
$\delta\in(0,1)$, we have
$
G_{\tau_\eps}\le \delta\eps<\frac{1}{2}=G(0)$.
Therefore, we have $w_\eps\neq 0$. Define $x_{\eps} := \frac{w_{\eps}}{\norm{w_\eps}_{\infty}}$. Fixing any $0 < \eps \leq \Tilde{\eps}_2(\delta)$, the conclusion $0 \leq \gamma_{\infty} - \hat{\gamma}_{\infty}(w_{\eps}) \leq C_{\infty,\delta} r_{\eps}$ in \eqref{eq:squeeze} holds. The margin is positively homogeneous as $\rho(cw) = \min_i z_i^T (cw) = c \min_i z_i^T w = c \rho(w), \, c \geq 0$. Therefore, it holds that $\rho(x_{\eps}) = \rho(\frac{w_{\eps}}{\norm{w_{\eps}}_{\infty}}) = \frac{\rho(w_{\eps})}{\norm{w_{\eps}}_{\infty}} = \hat{\gamma}_{\infty}(w_{\eps})$. Therefore, we have
    \begin{align*}
        0 \leq \gamma_{\infty} - \rho(x_{\eps}) \leq C_{\infty,\delta} r_{\eps}. 
    \end{align*}
    From \cref{lem:hoff}, we know $\cM_{\infty} = \{u:Au \leq b\}$ where
\begin{align*}
    A
:=
\begin{pmatrix}
I_d \\
-I_d \\
-Z
\end{pmatrix},
\qquad
b
:=
\begin{pmatrix}
\mathbf{1}_d \\
\mathbf{1}_d \\
-\gamma_\infty \mathbf{1}_n
\end{pmatrix}.
\end{align*}
We next compute $\norm{(Ax_{\eps} - b)_+}_{\infty}$. Given $\norm{x_{\eps}}_{\infty} = 1$, it holds that $x_{\eps}[j] - 1 \leq 0$ and $-x_{\eps}[j] - 1 \leq 0$ for all $j \in [d]$. Therefore, we have 
\begin{align*}
    [x_{\eps} - \mathbf{1}_d]_{+}=0, \qquad [-x_{\eps}  - \mathbf{1}_d]_{+} = 0.
\end{align*}
The remaining part is $[-Z x_{\eps} + \gamma_{\infty} \mathbf{1}_n]_{+} = ([\gamma_{\infty} - z_i^T x_{\eps}]_{+})_{i=1}^n$. Therefore, it holds that 
\begin{align*}
    \norm{(A x_{\eps} - b)_{+}}_{\infty} = \max_i [\gamma_{\infty} - z_i^T x_{\eps}]_{+}
\end{align*}
Given $\min_i z_i^T x_{\eps} = \rho(x_{\eps})$, it holds that
\begin{align*}
    \max_i (\gamma_{\infty} - z_i^T x_{\eps}) = \gamma_{\infty} - \min_i z_i^T x_{\eps} = \gamma_{\infty} - \rho(x_{\eps}) \geq 0.  
\end{align*}
Consequently, we have $\max_i (\gamma_{\infty} - z_i^T x_{\eps}) = \max_i [\gamma_{\infty} - z_i^T x_{\eps}]_{+}$. Applying \cref{lem:hoff}, we have 
\begin{align*}
     \operatorname{dist}_2\!\left(x_{\eps},\mathcal{M}_\infty\right) \leq H_{2,\infty}\!\left(A \right)
    \left\| \left(A x_{\eps} - b \right)_+
    \right\|_\infty = H_{2,\infty} [\gamma_{\infty} - \rho(x_{\eps})] \leq H_{2,\infty} C_{\infty, \delta} r_{\eps}, 
\end{align*}
where we denote $H_{2,\infty} := H_{2,\infty}(A)$ for simplicity. 
By choosing $u_{\eps} \in \argmin_{u \in \cM_{\infty}} \norm{x_\eps - u}_2$, it holds that
\begin{align*}
    \norm{x_{\eps} - u_{\eps}}_2 \leq H_{2,\infty} C_{\infty, \delta} r_{\eps}.  
\end{align*}
Next, we prove that every $u\in\mathcal M_\infty$ satisfies
$\|u\|_\infty=1$. Suppose instead that
$c:=\|u\|_\infty<1$. Since
$\rho(u)=\gamma_\infty>0$, we have $c>0$, and hence
$c\in(0,1)$. Define $\tilde u:=u/c$.
It then holds that $\rho(\tilde{{u}}) = \frac{\rho(u)}{c} = \frac{\gamma_{\infty}}{c} > \gamma_{\infty}$, which contradicts the definition of $\gamma_{\infty}$. This implies that $\norm{u_{\eps}}_{\infty} = 1$. Given $\norm{v}_2 \geq \norm{v}_{\infty}$ for every vector $v$, we also have $\norm{x_{\eps}}_2 \geq 1$ and $\norm{u_{\eps}}_2 \geq 1$.

Define $F(v) := \frac{\rho(v)}{\norm{v}_2}, v \neq 0$. We now compare $F(x_{\eps})$ and $F(u_{\eps})$. By definition, it holds that
\begin{align*}
    F(x_{\eps}) - F(u_{\eps}) = \frac{\rho(x_{\eps}) - \rho(u_{\eps})}{\norm{x_{\eps}}_2} + \rho(u_{\eps})[\frac{1}{\norm{x_\eps}_2} - \frac{1}{\norm{u_\eps}_2}]. 
\end{align*}
By \cref{lem:rho_lip}, it holds that $|\rho(x_{\eps}) - \rho(u_{\eps})| \leq R_2 \norm{x_{\eps} - u_{\eps}}_2$. For the first term, using this and $\norm{x_{\eps}}_2 \geq 1$, we obtain
\begin{align*}
    \frac{\rho(x_{\eps}) - \rho(u_{\eps})}{\norm{x_{\eps}}_2} \leq \frac{|\rho(x_{\eps}) - \rho(u_{\eps})|}{\norm{x_{\eps}}_2} \leq R_2 \norm{x_{\eps} - u_{\eps}}_2. 
\end{align*}
Using \cref{lem:vec_inverse} and $\norm{x_{\eps}}_2 \norm{u_{\eps}}_2 \geq 1$, it holds that 
\begin{align*}
    \bigl | \frac{1}{\norm{x_{\eps}}_{2}} - \frac{1}{\norm{u_{\eps}}_2} \bigr | \leq \norm{x_{\eps} - u_{\eps}}_2. 
\end{align*}
Given this and using $\rho(u_{\eps}) = \gamma_{\infty} > 0$, we have
\begin{align*}
    \rho(u_\epsilon)
\left(
\frac{1}{\|x_\epsilon\|_2}
-
\frac{1}{\|u_\epsilon\|_2}
\right)
&\le
\gamma_\infty
\left|
\frac{1}{\|x_\epsilon\|_2}
-
\frac{1}{\|u_\epsilon\|_2}
\right| \\
&\le
\gamma_\infty
\|x_\epsilon-u_\epsilon\|_2.
\end{align*}
Putting things together, we conclude that
\begin{align*}
    F(x_{\eps}) \leq F(u_{\eps}) + (R_2 + \gamma_{\infty}) \norm{x_{\eps} - u_{\eps}}_2 \leq F(u_{\eps}) + (R_2 + \gamma_{\infty}) H_{2,\infty} C_{\infty, \delta} r_{\eps}. 
\end{align*}
Given $u_{\eps} \in \cM_{\infty}$, the definition $\gamma_{2|\infty} = \max_{u \in \cM_{\infty}} \frac{\rho(u)}{\norm{u}_2}$ implies that $F(u_{\eps}) = \frac{\rho(u_{\eps})}{\norm{u_{\eps}}_2} \leq \gamma_{2|\infty}$. Using this, we conclude that
\begin{align}
    F(x_{\eps}) \leq \gamma_{2|\infty} + C_{2,\delta} r_{\eps}, \qquad C_{2,\delta} := (R_2 + \gamma_{\infty}) H_{2,\infty} C_{\infty, \delta}. \label{eq:F_temp}
\end{align}
Given $x_{\eps} = \frac{w_{\eps}}{\norm{w_{\eps}}_{\infty}}$, we have $\rho(x_{\eps}) = \frac{\rho(w_{\eps})}{\norm{w_{\eps}}_{\infty}}$ and $\norm{x_{\eps}}_2 = \frac{\norm{w_{\eps}}_2}{\norm{w_{\eps}}_{\infty}}$. Consequently, it holds that $F(x_{\eps}) = \frac{\rho(w_{\eps})}{\norm{w_{\eps}}_2} = \hat{\gamma}_2 (w_{\eps})$. Substituting this into \eqref{eq:F_temp}, we obtain $\hat{\gamma}_2 (w_{T_1^{\delta}}(\eps)) \leq \gamma_{2|\infty} + C_{2,\delta} r_{\eps}$. Subtracting both sides from $\gamma_2$, it holds that 
\begin{align*}
    \gamma_2 - \hat{\gamma}_2 (w_{T_1^{\delta}(\eps)}) \geq \gamma_2 - \gamma_{2|\infty} - C_{2,\delta} r_{\eps} = \triangle_{\rm geom} - C_{2,\delta} r_{\eps},
\end{align*}
where $\triangle_{\rm geom} = \gamma_2 - \gamma_{2|\infty}$. This is non-negative. For any $v = \frac{u}{\norm{u}_2}$, it holds that $\norm{v}_2 = 1$ and $\rho(v) = \frac{\rho(u)}{\norm{u}_2} = F(u) \leq \gamma_{2}$. Taking the maximum over $u \in \cM_{\infty}$ leads to $\gamma_{2|\infty} \leq \gamma_2$. 
 \end{proof}

\begin{corollary} \label{cor:gamma_2_cor_app} In the same setting as \cref{thm:margin_app_T1}, if $\Delta_{\rm geom} > 0$, then it holds that $ \liminf_{\eps \downarrow 0} [\gamma_2 - \hat{\gamma}_2(w_{T_1^{\delta}(\eps)})] \geq \Delta_{\rm geom} > 0$.  Moreover, there exists a constant $\hat{\eps}_{2}(\delta) \leq \Tilde{\eps}_2(\delta)$ such that for every $0 < \eps \leq \hat{\eps}_{2}(\delta)$, we have $\gamma_2 - \hat{\gamma}_2(w_{T_1^{\delta}(\eps)}) \geq \frac{\Delta_{\rm geom}}{2}$. 
\end{corollary}

\begin{proof} [Proof of \cref{cor:gamma_2_cor_app}]Let $\tau_{\eps} := T_1^{\delta}(\eps)$ and $\triangle := \triangle_{\rm geom} > 0$. From \cref{thm:margin_app_T1}, we have $\gamma_2 - \hat{\gamma}_2 (w_{\tau_{\eps}}) \geq \triangle - C_{2,\delta} r_{\eps}$. Taking lower limits and using $r_{\eps} \rightarrow 0$, we obtain 
\begin{align*}
    \liminf_{\eps \downarrow 0} [\gamma_2 - \hat{\gamma}_{2}(w_{\tau_{\eps}})] \geq \liminf_{\eps \downarrow 0} [\triangle - C_{2,\delta} r_{\eps}] = \triangle > 0.
\end{align*}
For the second conclusion, the sufficient conditions on $\eps$ are 
\begin{align*}
    0 < \eps \leq \tilde{\eps}_2 (\delta), \quad \text{and} \quad C_{2,\delta} r_{\eps} \leq \frac{\triangle_{\rm geom}}{2}. 
\end{align*}
In the case where $\frac{1}{3} < a < 1$, we have $r_{\eps} = (\log \frac{1}{\eps})^{-a}$. Therefore, it is sufficient to require $\eps \leq \exp[-(\frac{2 C_{2,\delta}}{\triangle_{\rm geom}})^{1/a}]$ to ensure $C_{2,\delta} (\log \frac{1}{\eps})^{-a} \leq \frac{\triangle_{\rm geom}}{2}$. In the case where $a = 1$, we have $r_{\eps} = \frac{\log \ell_{\eps}}{\ell_{\eps}}$. For $\ell_{\eps} \geq 1$ (ensured via $\eps \leq \Tilde{\eps}_2(\delta)$), it holds that $\log \ell_{\eps} \leq \sqrt{\ell_{\eps}}$ and consequently $C_{2,\delta} \frac{\log \ell_{\eps}}{\ell_{\eps}} \leq \frac{C_{2,\delta}}{\sqrt{\ell_{\eps}}}$. It is therefore sufficient to require $\ell_{\eps} \geq (\frac{2 C_{2,\delta}}{\triangle_{\rm geom}})^2$. Therefore, we can define $L_{3,\delta} := (\frac{2 C_{2,\delta}}{\triangle_{\rm geom}})^{1/a}$ when $\frac{1}{3} < a < 1$ and $L_{3,\delta} := (\frac{2 C_{2,\delta}}{\triangle_{\rm geom}})^2$ when $a = 1$. Then, we set $\hat{\eps}_{2}(\delta) := \min \{ \Tilde{\eps}_2(\delta), e^{-L_{3,\delta}}\}$. 
\end{proof}

\subsection{Margin Gap at $T_1^{-}$} \label{app:margin_T1M}

\begin{theorem} \label{thm:margin_app_T1-}
Denote $r_{\eps}:= \ell_{\eps}^{-a}$ when $a \in (\frac{1}{3},1)$ and $r_{\eps}:= \frac{\log \ell_{\eps}}{\ell_{\eps}}$ when $a = 1$.  For any $\delta,b \in (0,1)$, 
there exists a constant $\eps'_2(\delta,b)$ such that for every $0< \eps \leq \eps'_2(\delta,b)$, we have 
\[
\gamma_\infty-\widehat{\gamma}_\infty(w_{T_1^{-}(\eps;b)})
\leq
O(r_\epsilon), \qquad 
\gamma_2-\widehat{\gamma}_2(w_{T_1^{-}(\eps;b)})
\geq
\Delta_{\mathrm{geom}}
- O(r_\epsilon).
\]
\end{theorem}
 
\begin{proof}[Proof of \cref{thm:margin_app_T1-}] 

Let $\tau_\eps^-:=T_1^-(\eps;b)$,
$\tau_\eps^\delta:=T_1^\delta(\eps)$,
and $\ell_\eps:=\log\frac{1}{\eps}$.
We first specify a sufficient choice of the threshold on $\eps$. From \cref{prop:master-descent}, we have 
\[
L_{t+1}
\le
L_t-\eta_t\langle g_t,\sigma_\eps(g_t)\rangle
+C_{\rm md}\eta_t^2G_t,
\qquad
C_{\rm md}:=C_H+C_{\rm ema}.
\]
Define $\kappa_\delta
:=
\frac{\gamma_\infty^2}{R_1+d/\delta}$,
$C_D:=\sqrt{\frac{1-\beta_1}{1-\beta_2}}$,
$K_0:=R_1C_D$,
and
$
\bar\eta_\delta
:=
\min\left\{
\eta_0,\,
\frac{\kappa_\delta}{2C_{\rm md}},\,
\frac{2}{\kappa_\delta}
\right\}$. Let
\[
N_\delta
:=
\max\left\{
T_*,
\left\lceil
\left(\frac{\eta_0}{\bar\eta_\delta}\right)^{1/a}
\right\rceil
\right\}.
\]
Since $G_{T_*}\ge g_*$ by \cref{prop:uniform-clock}, \cref{lem:proxy-path-ratio} gives
\[
G_t
\ge
G_{T_*}\exp\left(
-K_0(S_t-S_{T_*})
\right)
\ge
g_*\exp\left(
-K_0(S_{N_\delta}-S_{T_*})
\right)
=:g_{\mathrm{burn},\delta}
\]
for every $T_*\le t\le N_\delta$. Set $
B_n:=\log\frac{2n}{\log 2}$ 
and
\[
L_{\delta,b}
:=
\max\left\{
e,\,
2\left[\log\frac{2}{g_{\mathrm{burn},\delta}}\right]_+,\,
\left(b+\sqrt{b^2+2B_n}\right)^2
\right\}.
\]
A sufficient choice for $\eps_2'(\delta,b)$ is $
\eps_2'(\delta,b)
:=
\min\left\{ \eps_0(b,\delta), 
\Tilde{\eps}_2(\delta),\,
e^{-L_{\delta,b}}
\right\}$,
where $\Tilde{\eps}_2(\delta)$ and $\eps_0(b,\delta)$ are the thresholds in \cref{thm:margin_app_T1} and \cref{prop:time_bracket_app}, respectively. We henceforth fix $ 0<\eps\le \eps_2'(\delta,b)$.

We first lower-bound the norm of $w_{\tau_\eps^-}$. Since
$\ell_\eps\ge L_{\delta,b}\ge e$, we have
$\log\ell_\eps\le\sqrt{\ell_\eps}$. Moreover, the definition of
$L_{\delta,b}$ implies $
\sqrt{\ell_\eps}
\ge
b+\sqrt{b^2+2B_n}$.
Consequently, it holds that 
\begin{align}
    b\log\ell_\eps+B_n
\le
b\sqrt{\ell_\eps}+B_n
\le
\frac{\ell_\eps}{2} \label{eq:le_temp}.
\end{align}
By the definition of $\tau_\eps^-$ and Proposition~2.1, we have $
L_{\tau_\eps^-}
\le
2G_{\tau_\eps^-}
\le
2\eps\ell_\eps^b$.
The loss--margin bound in \cref{lem:loss-proxy} therefore yields
\begin{align}
\rho(w_{\tau_\eps^-})
&\ge
\log\frac{\log 2}{nL_{\tau_\eps^-}} \ge
\log\frac{\log 2}{2n\eps\ell_\eps^b}
=
\ell_\eps-b\log\ell_\eps-B_n
\ge
\frac{\ell_\eps}{2},
\label{eq:Tminus-margin-lower}
\end{align}
where the last inequality follows from \eqref{eq:le_temp}. In particular,
$w_{\tau_\eps^-}\neq0$. Since
$ \rho(w)\le\gamma_\infty\|w\|_\infty $
for every $w$, we obtain
\begin{equation}
\|w_{\tau_\eps^-}\|_\infty
\ge
\frac{\ell_\eps}{2\gamma_\infty},
\qquad
\|w_{\tau_\eps^-}\|_2
\ge
\|w_{\tau_\eps^-}\|_\infty
\ge
\frac{\ell_\eps}{2\gamma_\infty}.
\label{eq:Tminus-radius-lower}
\end{equation}

We next control the displacement between $\tau_\eps^-$ and
$\tau_\eps^\delta$. Since $\ell_\eps\ge e$ and $b<1$, we have  
\[
\eps\ell_\eps^b
=
\exp\left(-\ell_\eps+b\log\ell_\eps\right)
\le
\exp\left(-\frac{\ell_\eps}{2}\right)
=
\sqrt{\eps}.
\]
The definition of $L_{\delta,b}$ further gives $\sqrt{\eps}
\le
\frac{g_{\mathrm{burn},\delta}}{2}$.
It follows that
$
G_t
\ge
g_{\mathrm{burn},\delta}
>
\eps\ell_\eps^b$ when $T_*\le t\le N_\delta$
and hence $
\tau_\eps^->N_\delta$.
Therefore, by the choice of $N_{\delta}$, we have for every $t\ge\tau_\eps^-$ that 
\begin{align}
    C_{\rm md}\eta_t
\le
\frac{\kappa_\delta}{2},
\qquad
\eta_t\le\frac{2}{\kappa_\delta}. \label{eq:lr_tt}
\end{align}

Since $\eps \leq \eps_{0}(b,\delta)$, we have
$\tau_\eps^- < \tau_\eps^\delta$. For
$\tau_\eps^-\le t<\tau_\eps^\delta$, the definition of
$\tau_\eps^\delta$ gives $G_t>\delta\eps$. As in Part~2 of the
proof of \cref{thm:first-clock}, \cref{lem:proxy,lem:softsign} give
$
\langle g_t,\sigma_\eps(g_t)\rangle
\ge
\frac{\gamma_\infty^2}{R_1+d/\delta}G_t
=
\kappa_\delta G_t$.
Combining this with \cref{eq:master-descent-app} and \eqref{eq:lr_tt}, we obtain
\[
L_{t+1}
\le
L_t-\eta_t
\left(\kappa_\delta-C_{\rm md}\eta_t\right)G_t
\le
L_t-\frac{\kappa_\delta}{2}\eta_tG_t
\le
\left(1-\frac{\kappa_\delta}{4}\eta_t\right)L_t.
\]

Define 
\begin{align*}
    B_\delta
:=
\frac{4}{\kappa_\delta}\log\frac{2}{\delta}
+\eta_0,
\qquad
q_a
:=
\begin{cases}
\dfrac{1}{1-a},
& a\in(1/3,1),\\[1mm]
1,
& a=1,
\end{cases} 
\qquad 
\mathcal{K}_{a,\delta,b}
:=
\frac{4bq_a}{\kappa_\delta}
+B_\delta.
\end{align*}

Repeating the argument in Part~2 of the proof of \cref{thm:first-clock} gives
\begin{align*}
S_{\tau_\eps^\delta}-S_{\tau_\eps^-}
&\le
\frac{4}{\kappa_\delta}
\log\frac{2\ell_\eps^b}{\delta}
+\eta_{\tau_\eps^-} \le
\frac{4b}{\kappa_\delta}\log\ell_\eps
+\frac{4}{\kappa_\delta}\log\frac{2}{\delta}
+\eta_0
=
\frac{4b}{\kappa_\delta}\log\ell_\eps+B_\delta.
\end{align*}

Using the uniform update bound
$\|D_t\|_\infty\le C_D$ from \cref{lem:uniform-D}, we obtain
\begin{align*}
\|w_{\tau_\eps^\delta}-w_{\tau_\eps^-}\|_\infty
&\le
\sum_{t=\tau_\eps^-}^{\tau_\eps^\delta-1}
\eta_t\|D_t\|_\infty \le
C_D\left(S_{\tau_\eps^\delta}-S_{\tau_\eps^-}\right) \le
C_D\left(
\frac{4b}{\kappa_\delta}\log\ell_\eps+B_\delta
\right),
\end{align*}
and hence
\begin{equation*}
\|w_{\tau_\eps^\delta}-w_{\tau_\eps^-}\|_2
\le
\sqrt d\,C_D\left(
\frac{4b}{\kappa_\delta}\log\ell_\eps+B_\delta
\right).
\end{equation*}

Applying \cref{lem:margin_helper} with
$x=w_{\tau_\eps^\delta}$ and $y=w_{\tau_\eps^-}$, and using
\eqref{eq:Tminus-radius-lower}, gives
\begin{align}
\left|
\widehat\gamma_\infty(w_{\tau_\eps^\delta})
-
\widehat\gamma_\infty(w_{\tau_\eps^-})
\right|
&\le
\frac{
2R_1\|w_{\tau_\eps^\delta}-w_{\tau_\eps^-}\|_\infty
}{
\|w_{\tau_\eps^-}\|_\infty
} \le
4R_1\gamma_\infty C_D
\left(
\frac{4b}{\kappa_\delta}
\frac{\log\ell_\eps}{\ell_\eps}
+
\frac{B_\delta}{\ell_\eps}
\right)
\nonumber\\
&\le
4R_1\gamma_\infty C_D
\mathcal K_{a,\delta,b}r_\eps,
\label{eq:Tminus-transfer-inf}
\end{align}
and
\begin{align}
\left|
\widehat\gamma_2(w_{\tau_\eps^\delta})
-
\widehat\gamma_2(w_{\tau_\eps^-})
\right|
&\le
\frac{
2R_2\|w_{\tau_\eps^\delta}-w_{\tau_\eps^-}\|_2
}{
\|w_{\tau_\eps^-}\|_2
} \le
4R_2\gamma_\infty\sqrt d\,C_D
\left(
\frac{4b}{\kappa_\delta}
\frac{\log\ell_\eps}{\ell_\eps}
+
\frac{B_\delta}{\ell_\eps}
\right)
\nonumber\\
&\le
4R_2\gamma_\infty\sqrt d\,C_D
\mathcal K_{a,\delta,b}r_\eps.
\label{eq:Tminus-transfer-two}
\end{align}

The last inequalities in \eqref{eq:Tminus-transfer-inf} and
\eqref{eq:Tminus-transfer-two} follow from
$ \frac{1}{\ell_\eps}\le r_\eps $ and 
$\frac{\log\ell_\eps}{\ell_\eps}\le q_a r_\eps$. Indeed, when $a\in(1/3,1)$, we have 
$\frac{\log\ell_\eps}{\ell_\eps}
\le
\frac{1}{1-a}\ell_\eps^{-a}
=
q_a r_\eps $ and $
\frac{1}{\ell_\eps}
\le
\ell_\eps^{-a}
=
r_\eps$; when $a=1$, we have
$\frac{\log\ell_\eps}{\ell_\eps}
=
r_\eps$ and $
\frac{1}{\ell_\eps}
\le
\frac{\log\ell_\eps}{\ell_\eps}
=
r_\eps$ since $\ell_\eps\ge e$.

Finally, because $\eps\le\Tilde{\eps}_2(\delta)$,
\cref{thm:margin_app_T1} gives
\[
0
\le
\gamma_\infty
-
\widehat\gamma_\infty(w_{\tau_\eps^\delta})
\le
C_{\infty,\delta}r_\eps, \qquad 
\gamma_2
-
\widehat\gamma_2(w_{\tau_\eps^\delta})
\ge
\Delta_{\rm geom}
-
C_{2,\delta}r_\eps.
\]
Combining these bounds with
\eqref{eq:Tminus-transfer-inf} gives
\begin{align*}
\gamma_\infty
-
\widehat\gamma_\infty(w_{\tau_\eps^-})
&\le
\gamma_\infty
-
\widehat\gamma_\infty(w_{\tau_\eps^\delta})
+
\left|
\widehat\gamma_\infty(w_{\tau_\eps^\delta})
-
\widehat\gamma_\infty(w_{\tau_\eps^-})
\right|\\
&\le
\left(
C_{\infty,\delta}
+
4R_1\gamma_\infty C_D
\mathcal{K}_{a,\delta,b}
\right)r_\eps = C_{\infty,a,\delta,b}^{-} r_{\eps},
\end{align*}
where we let $C_{\infty,a,\delta,b}^{-}
:=
C_{\infty,\delta}
+
4R_1\gamma_\infty C_D
\mathcal{K}_{a,\delta,b}$. 
Similarly, \eqref{eq:Tminus-transfer-two} gives
\begin{align*}
\gamma_2
-
\widehat\gamma_2(w_{\tau_\eps^-})
&\ge
\gamma_2
-
\widehat\gamma_2(w_{\tau_\eps^\delta})
-
\left|
\widehat\gamma_2(w_{\tau_\eps^\delta})
-
\widehat\gamma_2(w_{\tau_\eps^-})
\right|\\
&\ge
\Delta_{\rm geom}
-
\left(
C_{2,\delta}
+
4R_2\gamma_\infty\sqrt{d}\,C_D
\mathcal{K}_{a,\delta,b}
\right)r_\eps = \Delta_{\rm geom}
- C_{2,a,\delta,b}^{-} r_{\eps},
\end{align*}
where we let $C_{2,a,\delta,b}^{-}
:=
C_{2,\delta}
+
4R_2\gamma_\infty\sqrt{d}\,C_D
\mathcal{K}_{a,\delta,b}$. 
\end{proof}



\subsection{Margin Gap at $T_1^{+}$} \label{app:margin_T1P}

We first prove \cref{prop:post_euc} and then specialize it to $T_1^{+}(\eps;b)$. 

\begin{proof} [Proof of \cref{prop:post_euc}] Denote $\tau := T_{1}^{\delta}(\eps)$, $H_t := H_t^{\delta}(\eps) = \sum_{s = \tau}^{t-1} \theta_s$, $R_{\tau} := \norm{w_{\tau}}_{\infty}$, and $R_t := \norm{w_t}_{\infty}$. Let $\eps_{\delta}^{\rm inv}$ denote the threshold in \cref{cor:invariant}. 
Define $A_{\delta} = [\log \frac{2 n \delta}{\log 2}]_{+}$ and $\bar{l}_{\delta} = \max \{ e, 2 A_{\delta} \}$. 
We choose $0 < \eps \leq \tilde{\eps}_3(\delta) := \min \{\tilde{\eps}_1(\delta), \tilde{\eps}_2(\delta), \eps_{\delta}^{\rm inv}, e^{- \bar{l}_{\delta}}\}$ (note that $\tilde{\eps}_1(\delta)$ and $\tilde{\eps}_2(\delta)$ are the thresholds in the statements of \cref{prop:bounded_euc,thm:margin_app_T1}, respectively). This ensures that the conclusions of \cref{prop:bounded_euc,thm:margin_app_T1} hold. 
Summing the post-crossover recursion \eqref{eq:local_ngd} from $\tau$ to $t - 1$ gives
\begin{align*}
    w_t - w_{\tau} = - \sum_{s = \tau}^{t-1} \theta_s u_s - \sum_{s = \tau}^{t-1} \theta_s e_s. 
\end{align*}
Given $u_s = \frac{g_s}{\norm{g_s}_2}$, it holds that $\norm{u_s}_{\infty} \leq \norm{u_s}_{2} \leq 1$ and
\begin{align}
    \norm{\sum_{s = \tau}^{t-1} \theta_s u_s }_{\infty} \leq \sum_{s = \tau}^{t-1} \theta_s \norm{u_s}_{\infty} \leq \sum_{s = \tau}^{t-1} \theta_s = H_t. \label{eq:post1}
\end{align}
We also have that
\begin{align}
    \norm{\sum_{s = \tau}^{t-1} \theta_s e_s}_{\infty} \leq \sum_{s = \tau}^{t-1} \theta_s \norm{e_s}_{\infty} \leq \sum_{s = \tau}^{t-1} \theta_s \norm{e_s}_2 \leq E_0. \label{eq:post2}
\end{align}
Denoting $\triangle_t := \norm{w_t - w_{\tau}}_{\infty}$, and combining \eqref{eq:post1} and \eqref{eq:post2} together, we obtain
\begin{align}
    \triangle_t = \norm{w_t - w_{\tau}}_{\infty} \leq H_t + E_0, \qquad t \geq \tau.  \label{eq:ineq_1}
\end{align}

From \cref{cor:invariant}, we know $G_t \leq L_t \leq 2 \delta \eps, t \geq \tau$. Given $t \geq \tau \geq T_*$, \cref{prop:uniform-clock} guarantees $L_t \leq \frac{\log 2}{n}$. By \cref{lem:loss-proxy}, we also have $L_t \geq \frac{\log 2}{n} e^{-\rho(w_t)}$. Combining them together, we obtain $\frac{\log 2}{n} e^{- \rho(w_t)} \leq 2 \delta \eps$. Taking logarithms and rearranging gives $\rho(w_t) \geq \log \frac{\log 2}{2 n \delta \eps} = \ell_{\eps} - \log \frac{2 n \delta}{\log 2} \geq \ell_{\eps} - A_{\delta}$. The condition on $\eps$ ensures that $\ell_{\eps} \geq \bar{\ell}_{\delta} \geq \max \{ e, 2 A_{\delta} \}$. Consequently, we obtain
\begin{align}
    \rho(w_t) \geq \frac{1}{2} \ell_{\eps} > 0, \qquad t \geq \tau. \label{eq:w_help}
\end{align}
Since $\rho(0) = 0$, \eqref{eq:w_help} implies that $w_t \neq 0$ for every $t \geq \tau$. 

For any nonzero $w$, it holds that $\rho(w) = \norm{w}_{\infty} \rho(\frac{w}{\norm{w}_{\infty}}) \leq \gamma_{\infty} \norm{w}_{\infty}$. Combining this and \eqref{eq:w_help}, we have
\begin{align}
    R_{\tau} = \norm{w_{\tau}}_{\infty} \geq \frac{\ell_{\eps}}{2 \gamma_{\infty}}.  \label{eq:ineq_2}
\end{align}

Applying \cref{lem:margin_helper} with $x = w_t$ and $y = w_{\tau}$ and using the inequalities \eqref{eq:ineq_1} and \eqref{eq:ineq_2}, it holds that
\begin{align*}
    |\hat{\gamma}_{\infty}(w_t) - \hat{\gamma}_{\infty}(w_{\tau})| \leq \frac{2 R_1 \norm{w_t - w_{\tau}}_{\infty}}{R_{\tau}} \leq \frac{2 R_1 (H_t + E_0)}{R_{\tau}} \leq 4 R_1 \gamma_{\infty} \frac{H_t + E_0}{\ell_{\eps}}. 
\end{align*}
Since $\eps \leq \tilde{\eps}_2(\delta)$, \cref{thm:margin_app_T1} gives $0 \leq \gamma_{\infty} - \hat{\gamma}_{\infty}(w_{\tau}) \leq C_{\infty, \delta} r_{\eps}$. Consequently, it holds that
\begin{align*}
    \gamma_{\infty} - \hat{\gamma}_{\infty}(w_t) \leq C_{\infty, \delta} r_{\eps} + 4 R_1 \gamma_{\infty} \frac{H_t + E_0}{\ell_{\eps}}. 
\end{align*}
The left-hand side is nonnegative because $\hat{\gamma}_{\infty}(w_t) = \rho(\frac{w_t}{\norm{w_t}_{\infty}}) \leq \gamma_{\infty}$. Define $C_{\rm p, \delta} := \max \{ C_{\infty,\delta}, 4 R_1 \gamma_{\infty} \}$. We conclude that
\begin{align}
    0 \leq \gamma_{\infty} - \hat{\gamma}_{\infty} (w_t) \leq C_{\rm p, \delta} (r_{\eps} + \frac{H_t + E_0}{\ell_{\eps}}), \qquad t \geq \tau. \label{eq:gamma_temp}
\end{align}
We next prove the second inequality. Using \eqref{eq:post1} and \eqref{eq:post2}, we have
\begin{align}
    \|w_t-w_\tau\|_2 \le H_t+E_0, \qquad t \geq \tau. \label{eq:w_tt}  
\end{align}

At $t=\tau$, \eqref{eq:w_help} gives
$ \rho(w_\tau)\ge\frac{\ell_\eps}{2}$.
Since
$\rho(w)\le\gamma_2\|w\|_2$ for every $w$, it holds that  
\begin{equation}
\|w_\tau\|_2
\ge
\frac{\rho(w_\tau)}{\gamma_2}
\ge
\frac{\ell_\eps}{2\gamma_2}.
\label{eq:prop33-euclidean-radius}
\end{equation}

Applying \cref{lem:margin_helper} with $x=w_t$ and $y=w_\tau$, and using
\eqref{eq:w_tt} and
\eqref{eq:prop33-euclidean-radius}, we obtain
\begin{align}
\left|
\widehat\gamma_2(w_t)
-
\widehat\gamma_2(w_\tau)
\right|
&\le
\frac{2R_2\|w_t-w_\tau\|_2}{\|w_\tau\|_2} \leq 
4R_2\gamma_2
\frac{H_t+E_0}{\ell_\eps}.
\label{eq:prop33-euclidean-transfer}
\end{align}

Because $\eps\le\Tilde\eps_2(\delta)$, \cref{thm:margin_app_T1} gives
$ \gamma_2-\widehat\gamma_2(w_\tau)
\ge
\Delta_{\rm geom}-C_{2,\delta}r_\eps$.
Combining this bound with
\eqref{eq:prop33-euclidean-transfer}, we obtain
\begin{align}
\gamma_2-\widehat\gamma_2(w_t)
&\ge
\gamma_2-\widehat\gamma_2(w_\tau)
-
\left|
\widehat\gamma_2(w_t)
-
\widehat\gamma_2(w_\tau)
\right|
\nonumber\\
&\ge
\Delta_{\rm geom}
-
C_{2,\delta}r_\eps
-
4R_2\gamma_2
\frac{H_t+E_0}{\ell_\eps}.
\label{eq:prop33-euclidean-preliminary}
\end{align}

The definition of $\Tilde\eps_3(\delta)$ ensures that
$\ell_\eps\ge e$. Hence, we have
$
\frac{1}{\ell_\eps}\le r_\eps$
for both $a\in(1/3,1)$ and $a=1$. It follows from
\eqref{eq:prop33-euclidean-preliminary} that
\begin{align*}
\gamma_2-\widehat\gamma_2(w_t)
& \ge
\Delta_{\rm geom}
-
\left(
C_{2,\delta}+4R_2\gamma_2E_0
\right)r_\eps
-
4R_2\gamma_2\frac{H_t}{\ell_\eps} \\
& \geq \Delta_{\rm geom}
-
C_{2,\delta}^{\rm post}
\left(
r_\eps+\frac{H_t^\delta(\eps)}{\ell_\eps}
\right),
\qquad
t\ge T_1^\delta(\eps),
\end{align*}
where we let 
$C_{2,\delta}^{\rm post}
:=
\max\left\{
C_{2,\delta}+4R_2\gamma_2E_0,\,
4R_2\gamma_2
\right\}$.
\end{proof}

\begin{corollary} \label{cor:margin_app_T1+} Denote $r_{\eps}:= \ell_{\eps}^{-a}$ when $a \in (\frac{1}{3},1)$ and $r_{\eps}:= \frac{\log \ell_{\eps}}{\ell_{\eps}}$ when $a = 1$.  For any $\delta,b \in (0,1)$, 
there exists a constant $\eps''_2(\delta,b)$ such that for every $0< \eps \leq \eps''_2(\delta,b)$, we have 
\[
\gamma_\infty-\widehat{\gamma}_\infty(w_{T_1^{+}(\eps;b)})
\leq
O(r_\epsilon), \qquad 
\gamma_2-\widehat{\gamma}_2(w_{T_1^{+}(\eps;b)})
\geq
\Delta_{\mathrm{geom}}
- O(r_\epsilon).
\]
\end{corollary}

\begin{proof}
Let $
\ell_\eps := \log \frac{1}{\eps}$, $C_D  := \sqrt{\frac{1-\beta_1}{1-\beta_2}}$ and 
$c_\delta
:=
\frac{\gamma_2^2}{16(1+2R_1\delta)}$. Set $\eps_2''(\delta,b)
:=
\min\left\{
\Tilde{\eps}_3(\delta), \eps_{0}(\delta,b)
\right\}$ where $\eps_{0}(\delta,b)$ is the threshold in \cref{prop:time_bracket_app}.
Fix \(0<\eps\leq \eps_2''(\delta,b)\) and let
$\tau:=T_1^\delta(\eps)$, $t_+:=T_1^+(\eps;b)$, $V_+:=V_{t_+}^\delta (\eps)$, and 
$H_+:=H_{t_+}^\delta(\eps)$. 
Since $\eps\leq \eps_0(\delta,b)$, 
we have $\tau < t_+$ by \cref{prop:time_bracket_app}. Moreover, the condition
\(\eps\leq\Tilde{\eps}_3(\delta)\) ensures that the
conclusions of \cref{prop:bounded_euc} apply from time \(\tau\) onward.
In particular, we have \(\tau>T_*\) and
$K_0\eta_{\tau-1}\leq\log 2$ where $
K_0:=R_1 C_D$. 
By the minimality of \(t_+\) and \cref{lem:proxy-path-ratio}, it holds that 
\[
L_{t_+}\geq G_{t_+} \geq \exp(-K_0 \eta_{t_+ - 1}) G_{t_+ - 1} \geq e^{- \log 2} G_{t_+ - 1} >
\frac{\eps}{2}\ell_\eps^{-b}.
\]
Consequently, \eqref{eq:vt_1} gives
\[
c_\delta\frac{V_+}{\eps}
\leq
\frac{1}{L_{t_+}}-\frac{1}{L_\tau}
\leq
\frac{1}{L_{t_+}}
<
\frac{2}{\eps}\ell_\eps^b.
\]
Thus, we have $
V_+
\leq
\frac{2}{c_\delta}\ell_\eps^b$. Denote $V_t := \sum_{s = \tau}^{t-1} \eta_s$ and $H_t := \sum_{s = \tau}^{t-1} \theta_s$. By \eqref{eq:final_helper}, we know $\frac{G_s}{\eps} \leq \frac{C_L}{1+V_s}$ where $C_L = 1 / \min \{ \frac{1}{2}, \frac{\gamma_2^2}{16(1 + 2 R_1)}\}$. 
Hence, it holds that  
\[
\theta_s
=
\frac{\eta_s\|g_s\|_2}{\eps}
\leq
R_2\eta_s\frac{G_s}{\eps}
\leq
R_2C_L\frac{\eta_s}{1+V_s}.
\]
Consequently, we have
$
H_t
\leq
R_2C_L
\sum_{s=\tau}^{t-1}
\frac{\eta_s}{1+V_s}$ for any $t > \tau$.
Setting
$x_s:=\frac{\eta_s}{1+V_s}$, we obtain
$\log(1+V_{s+1})-\log(1+V_s)
=
\log(1+x_s)$ and $0\leq x_s\leq\eta_s\leq\eta_0$.
For every \(x\geq 0\), we have 
$\log(1+x)\geq\frac{x}{1+x}$.
This leads to 
\[
\frac{\eta_s}{1+V_s}
=
x_s
\leq
(1+\eta_0)
\left[
\log(1+V_{s+1})-\log(1+V_s)
\right].
\]
Summing this inequality gives
\begin{align*}
\sum_{s=\tau}^{t-1}\frac{\eta_s}{1+V_s}
&\leq
(1+\eta_0)
\sum_{s=\tau}^{t-1}
\left[
\log(1+V_{s+1})-\log(1+V_s)
\right] \\
&=
(1+\eta_0)
\left[
\log(1+V_t)-\log(1+V_\tau)
\right].
\end{align*}
Because \(V_\tau=0\), it follows that
$
\sum_{s=\tau}^{t-1}\frac{\eta_s}{1+V_s}
\leq
(1+\eta_0)\log(1+V_t)$.
Consequently, we obtain 
$H_t
\leq
R_2C_L(1+\eta_0)\log(1+V_t) = C_{\rm HV} \log (1 + V_t)$ for any $t > \tau$ where $C_{\rm HV} := R_2 C_L (1 + \eta_0)$. Since $V_+ \leq \frac{2}{c_{\delta}} \ell_{\eps}^b$ and $t_{+} > \tau$, it holds that
\begin{align*}
H_+
&\leq
C_{\mathrm{HV}}\log\left(
1+\frac{2}{c_\delta}\ell_\eps^b
\right) \leq
C_{\mathrm{HV}}\left[
\log\left(1+\frac{2}{c_\delta}\right)
+b\log\ell_\eps
\right].
\end{align*}
Define $q_a := \dfrac{1}{1-a}$ when $a\in(1/3,1)$ and $q_a := 1$ when $a = 1$.
In both cases, we have $\frac{\log\ell_\eps}{\ell_\eps}
\leq q_a r_\eps$. 
We also \(1/\ell_\eps\leq r_\eps\). Consequently, we obtain 
\[
\frac{H_+}{\ell_\eps}
\leq
K_{H,a,\delta,b}^{+}r_\eps, \qquad K_{H,a,\delta,b}^{+}
:=
C_{\mathrm{HV}}
\log\left(1+\frac{2}{c_\delta}\right)
+bq_a. 
\]

Applying \cref{prop:post_euc} at \(t=t_+\) yields
\begin{align*}
\gamma_\infty-\widehat{\gamma}_\infty(w_{t_+})
&\leq
C_{\infty,\delta}r_\eps
+
4R_1\gamma_\infty
\frac{H_++E_0}{\ell_\eps}                                   \\
&\leq
\left[
C_{\infty,\delta}
+
4R_1\gamma_\infty
\left(E_0+K_{H,a,\delta,b}^{+}\right)
\right]r_\eps = C_{\infty,a,\delta,b}^{+} r_{\eps},
\end{align*}
where $C_{\infty,a,\delta,b}^{+}
:=
C_{\infty,\delta}
+
4R_1\gamma_\infty
\left(E_0+K_{H,a,\delta,b}^{+}\right)$. Similarly, the second inequality of \cref{prop:post_euc} gives
\begin{align*}
\gamma_2-\widehat{\gamma}_2(w_{t_+})
&\geq
\Delta_{\mathrm{geom}}
-
C_{2,\delta}r_\eps
-
4R_2\gamma_2
\frac{H_++E_0}{\ell_\eps}                                   \\
&\geq
\Delta_{\mathrm{geom}}
-
\left[
C_{2,\delta}
+
4R_2\gamma_2
\left(E_0+K_{H,a,\delta,b}^{+}\right)
\right]r_\eps = \Delta_{\mathrm{geom}} - C_{2,a,\delta,b}^{+} r_{\eps},
\end{align*}
where $
C_{2,a,\delta,b}^{+}
:=
C_{2,\delta}
+
4R_2\gamma_2
\left(E_0+K_{H,a,\delta,b}^{+}\right)$. 
\end{proof}

\begin{proof}[Proof of \cref{thm:margin_main}] We can let $\Tilde{\eps}_2(\delta,b) := \min \{ \Tilde{\eps}_2(\delta), \eps_2'(\delta,b), \eps_2''(\delta,b)\}$, where $\Tilde{\eps}_2(\delta)$, $\eps_2'(\delta,b)$, and $\eps_2''(\delta,b)\}$ are the thresholds in \cref{thm:margin_app_T1}, \cref{thm:margin_app_T1-}, and \cref{cor:margin_app_T1+}, respectively. The conclusion then follows. 
    
\end{proof}

\section{Adam's Euclidean Optimality} \label{app:euclidean}
We first provide a proof sketch of \cref{thm:margin_euc}.
Let $\tau := T_1^\delta(\epsilon)$, $V_t := \sum_{s=\tau}^{t-1}\eta_s$, and $H_t := \sum_{s=\tau}^{t-1}\theta_s$. Recall the post-crossover update can be written as
\begin{align}
    w_{s+1}
=
w_s-\theta_s d_s,
\qquad
d_s:=u_s+e_s,
\qquad
u_s:=\frac{g_s}{\|g_s\|_2}, \label{eq:post_u}
\end{align}

where $\theta_s:=\frac{\eta_s\|g_s\|_2}{\epsilon}$ and $\|u_s\|_2=1$. Using Taylor's theorem and following the same approach as \cref{lem:taylor}, we show
\begin{align*}
    L_{s+1}
\leq
L_s
-\theta_s \langle g_s,d_s\rangle
+C_H L_s\bigl(\theta_s\|d_s\|_2\bigr)^2,
\end{align*}
for some constant $C_H$. From this, we further obtain 
$\frac{L_{s+1}}{L_s} \leq 1 - \gamma_2 \theta_s + r_s$, where 
\begin{align*}
    r_s
:=
\gamma_2\theta_s
\left(1-\frac{G_s}{L_s}\right)
+
R_2\theta_s\|e_s\|_2
+
2C_H
\left[
\theta_s^2
+
\bigl(\theta_s\|e_s\|_2\bigr)^2
\right].
\end{align*}
Based on \cref{prop:bounded_euc}, we can show $\sum_{\tau}^{\infty} \theta_s^2 \leq K_{\theta}^2$ for some constant $K_{\theta}$, $\sum_{s=\tau}^{\infty} \theta_s \norm{e_s}_2 \leq E_0$, $\sum_{s=\tau}^{\infty} (\theta_s \norm{e_s}_2)^2 \leq E_0^2$, and $\sum_{s=\tau}^{\infty} \theta_s (1 - \frac{G_s}{L_s}) \leq K_p$ for some constant $K_p$. Consequently, we have $\sum_{s=\tau}^{\infty} r_s \leq A_0  < \infty$ for some constant $A_0$. Algebraic manipulation leads to $L_t \leq L_\tau \exp\!\left(A_0-\gamma_2 H_t\right)$. Hence, the loss decreases exponentially in $H_t$. Using \cref{prop:uniform-clock} and \cref{lem:loss-proxy}, it follows that
$\rho(w_t)\geq\gamma_2H_t-A_0$. On the other hand, unrolling the update gives $\|w_t\|_2 \leq \|w_\tau\|_2 + H_t + E_0$. Combining them together (with a separate trivial
treatment of the case $H_t < A_0/\gamma_2$), we obtain 
\[
\gamma_2-\widehat{\gamma}_2(w_t)
\leq
C_\gamma {}\frac{1+\|w_\tau\|_2}{1+H_t},
\]
where $C_{\gamma}$ is some constant. 
This proves part~(i) of \cref{thm:margin_euc}. For part~(ii), we further denote
$\Delta:=\Delta_{\mathrm{geom}}$,
$\tau_\zeta:=\tau_\zeta^\delta(\epsilon)$,
$R_\tau:=\|w_\tau\|_2$.
By \cref{prop:bounded_euc}, we have 
$H_t\asymp_\delta \log(1+V_t)$. Since \(V_t\to\infty\), it holds that \(H_t\to\infty\). Define
$
\overline{\tau}_\zeta
:=
\inf\left\{
t\geq\tau:
H_t\geq
\frac{C_\gamma}{\zeta}(1+R_\tau)
\right\}$.
Part~(i) implies that 
$
\gamma_2-\widehat{\gamma}_2
\bigl(w_{\overline{\tau}_\zeta}\bigr)
\leq
C_\gamma
\frac{1+R_\tau}{1+H_{\overline{\tau}_\zeta}}
\leq \zeta$.
Therefore, we have $\tau_\zeta\leq\overline{\tau}_\zeta$.
By the minimality of \(\overline{\tau}_\zeta\), it holds that
$H_{\overline{\tau}_\zeta-1} <
\frac{C_\gamma}{\zeta}(1+R_\tau)$.
Since \(\theta_s\leq K_\theta\), we obtain
\begin{align}
H_{\tau_\zeta}
&\leq H_{\overline{\tau}_\zeta} \notag =
H_{\overline{\tau}_\zeta-1}
+\theta_{\overline{\tau}_\zeta-1} \notag <
\frac{C_\gamma}{\zeta}(1+R_\tau)+K_\theta \notag \leq
\left(
\frac{C_\gamma}{\zeta}+K_\theta
\right)(1+R_\tau).
\end{align}
This proves the upper bound in part~(ii). 
For the lower bound, \cref{thm:margin_app_T1} gives
$\gamma_2-\widehat{\gamma}_2(w_\tau)
\geq
\Delta-C_{2,\delta}r_\epsilon$. 
Since \(r_\epsilon\to0\), we choose \(\epsilon\) sufficiently small such
that $C_{2,\delta}r_\epsilon \leq \frac{\Delta-\zeta}{2}$.
Then, we have 
\[
\gamma_2-\widehat{\gamma}_2(w_\tau)
\geq
\Delta-\frac{\Delta-\zeta}{2}
=
\frac{\Delta+\zeta}{2}
>
\zeta.
\]
Consequently, it holds that \(\tau_\zeta>\tau\). At time $\tau_{\zeta}$, we also have 
$\gamma_2-\widehat{\gamma}_2(w_{\tau_\zeta})
\leq\zeta$. 
Therefore, combining them together gives  
\begin{align*}
\widehat{\gamma}_2(w_{\tau_\zeta})
-\widehat{\gamma}_2(w_\tau)
&=
\left[
\gamma_2-\widehat{\gamma}_2(w_\tau)
\right]
-
\left[
\gamma_2-\widehat{\gamma}_2(w_{\tau_\zeta})
\right] \notag \geq
\frac{\Delta+\zeta}{2}-\zeta =
\frac{\Delta-\zeta}{2}.
\end{align*}

Applying \cref{lem:margin_helper} with \(x=w_{\tau_\zeta}\) and \(y=w_\tau\), we
obtain $\frac{\Delta-\zeta}{2}
\leq
\frac{
2R_2\|w_{\tau_\zeta}-w_\tau\|_2
}{
R_\tau
}$. Moreover, unrolling the post-crossover update gives
$\|w_{\tau_\zeta}-w_\tau\|_2 \leq
H_{\tau_\zeta}+E_0$.
Consequently, it holds that  $H_{\tau_\zeta}
\geq \frac{\Delta-\zeta}{4R_2}R_\tau-E_0$. By \eqref{eq:prop33-euclidean-radius}, we know \(R_\tau\to\infty\) as
\(\epsilon\downarrow0\). Hence, after decreasing \(\epsilon\), 
we can ensure that $E_0
\leq
\frac{\Delta-\zeta}{8R_2}R_\tau$. Together, they imply that $H_{\tau_\zeta}
\geq
\frac{\Delta-\zeta}{8R_2}R_\tau$.
Therefore, we can set 
$c_\zeta
:=
\frac{\Delta-\zeta}{8R_2}$ 
and $C_\zeta
:=
\max\left\{
\frac{C_\gamma}{\zeta}+K_\theta,\,
c_\zeta
\right\}$, and it holds that
$c_\zeta\|w_\tau\|_2
\leq
H_{\tau_\zeta}
\leq
C_\zeta\bigl(1+\|w_\tau\|_2\bigr)$. 

\begin{proof} [Proof of \cref{thm:margin_euc}] 
Denote $\tau := T_1^{\delta} (\eps)$ and $H_t := H_{t}^{\delta} (\eps) = \sum_{s = \tau}^{t-1} \theta_s$. 
We first prove part (i). Set $\Tilde{\eps}^{(i)}_{4}(\delta) := \min \{\eps_{\delta}^{\rm inv} , \frac{\log 2}{4 \delta}\}$, where $\eps_{\delta}^{\rm inv}$ is the constant appearing in \cref{cor:invariant}. We choose $\eps$ such that $0 < \eps \leq \Tilde{\eps}^{(i)}_{4}(\delta)$. The condition $0 < \eps \leq \eps_{\delta}^{\rm inv}$ ensures that $G_t \leq L_t \leq 2 \delta \eps, \, \, t \geq \tau$. The extra condition $\eps \leq \frac{\log 2}{4 \delta}$ implies that $G_t \leq L_t \leq 2 \delta \eps \leq \frac{\log 2}{2} < \frac{1}{2}, \,\,t\geq \tau$. Since $G(0) = \frac{1}{2}$, this guarantees that $w_t \neq 0, \,\, t \geq \tau$. Therefore, $\hat{\gamma}_{2}(w_t)$ is well-defined. 

We know from the proof of \cref{prop:bounded_euc} (\eqref{eq:final_helper} and \eqref{eq:final_hh}) that 
\begin{align}
    L_{s} - L_{s+1} \geq c_* \frac{\eta_s}{\eps} L_s^2, \qquad \frac{G_s}{\eps} \leq \frac{C_L}{1 + V_s}, \qquad s \geq \tau \label{eq:LL_temp}
\end{align}
where $C_L = \frac{1}{\min \{ 1/2, c_*\}}$ and $c_* = \frac{\gamma_2^2}{16(1 + 2R_1)}$. Using this and $\norm{g_s}_2 \leq R_2 G_s$, it holds that
\begin{align*}
    \theta_s = \frac{\eta_s \norm{g_s}_2}{\eps} \leq R_2 \eta_s \frac{G_s}{\eps} \leq R_2 C_L \frac{\eta_s}{1 + V_s}.
\end{align*}
Given this, we can further obtain
\begin{align*}
    \sum_{s = \tau}^{\infty} \theta_s^2 \leq R_2^2 C_L^2 \sum_{s = \tau}^{\infty} \frac{\eta_s^2}{(1 + V_s)^2} \leq R_2^2 C_L^2 \sum_{s = \tau}^{\infty} \frac{\eta_s^2}{1 + V_s}.  
\end{align*}
\cref{lem:helper} gives $\sum_{s = \tau}^{\infty} \frac{\eta_s^2}{1 + V_s} \leq Q_{\eta}$ where $Q_{\eta} = \eta_0^2 + \frac{1}{a} \max \{ \eta_0^2, \eta_0 (1-a) \}$. Therefore, it holds that
\begin{align}
    \sum_{s = \tau}^{\infty} \theta_s^2 \leq K_{\theta}^2, \qquad K_{\theta} = R_2 C_L \sqrt{Q_{\eta}}. \label{eq:a1}
\end{align}
In particular, it holds that $\sup_{s \geq \tau} \theta_s \leq K_{\theta}$. 

From \cref{prop:bounded_euc}, we know 
\begin{align}
    \sum_{s = \tau}^{\infty} \theta_s \norm{e_s}_2 \leq E_0, \qquad E_0 = C_{\rm ema} C_L Q_{\eta} + \frac{2 R_1 R_2}{c_*}. \label{eq:a2}
\end{align}
Let $d_s = u_s + e_s$. Since $\norm{u_s}_2 = 1$, it holds that $\theta_s \norm{d_s}_2 \leq \theta_s + \theta_s \norm{e_s}_2$. Consequently, it holds that
\begin{align*}
    \sup_{s \geq \tau} \theta_s \norm{d_s}_2 \leq K_{\theta} + E_0 =: B_0 < \infty. 
\end{align*}
Define $b_s := \theta_s \norm{e_s}_2 \geq 0$. Since $\sum_{s = \tau}^{\infty} b_s \leq E_0$, we have
\begin{align}
    \sum_{s = \tau}^{\infty} (\theta_s \norm{e_s}_2)^2 = \sum_{s = \tau}^{\infty} b_s^2 \leq (\sum_{s = \tau}^{\infty} b_s)^2 \leq E_0^2. \label{eq:a3}
\end{align}

Next, we bound the term $\sum_{s = \tau}^{\infty} \theta_s (1 - \frac{G_s}{L_s})$. We first bound $\sum_{s = \tau}^{\infty} \theta_s L_s$. From $\theta_s L_s = \frac{\eta_s \norm{g_s}_2}{\eps} L_s \leq R_2 \frac{\eta_s}{\eps} G_s L_s \leq R_2 \frac{\eta_s}{\eps} L_s^2$ and \eqref{eq:LL_temp}, we obtain
$\theta_s L_s \leq \frac{R_2}{c_*} (L_s - L_{s+1})$. Consequently, it holds that
\begin{align*}
    \sum_{s = \tau}^{\infty} \theta_s L_s \leq \frac{R_2}{c_*} \sum_{s = \tau}^{\infty} (L_s - L_{s+1}) \leq \frac{R_2 L_{\tau}}{c_*}. 
\end{align*}
Since $\tau \geq T_*$, we know $L_{\tau} \leq \frac{\log 2}{n}$ from \cref{prop:uniform-clock}. \cref{lem:loss-proxy} further gives $0 \leq 1 - \frac{G_s}{L_s} \leq \frac{n L_s}{2}$. Based on these inequalities, we obtain
\begin{align}
    \sum_{s = \tau}^{\infty} \theta_s (1 - \frac{G_s}{L_s}) \leq \frac{n}{2} \sum_{s = \tau}^{\infty} \theta_s L_s \leq \frac{R_2 \log 2}{2 c_*} =: K_p. \label{eq:a4}
\end{align}

For $s \geq \tau$, the update is $w_{s+1} = w_s - \theta_s d_s$. For $\xi \in [0,1]$, we define $w_{s,\xi} = w_s - \xi \theta_s d_s$. From \cref{lem:taylor}, we know the Hessian of the logistic loss is $\nabla^2 L(w) = \frac{1}{n} \sum_{i=1}^n \ell''(z_i^T w) z_i z_i^T$ with $0 \leq \ell''(r) \leq - \ell'(r)$. This implies that
\begin{align*}
    d_s^T \nabla^2 L(w_{s,\xi}) d_s = \frac{1}{n} \sum_{i=1}^n \ell''(z_i^T w_{s,\xi}) (z_i^T d_s)^2 \leq \frac{1}{n} \sum_{i=1}^n a_i(w_{s,\xi}) (z_i^T d_s)^2. 
\end{align*}

For every sample $i$, it holds that $|z_i^T (w_{s,\xi} - w_s)| \leq \norm{z_i}_2 \theta_s \norm{d_s}_2 \leq R_2 B_0$. From \cref{lem:derivative-ratio}, we have $\frac{a_i(w_{s,\xi})}{a_i(w_s)} = \frac{-\ell'(z_i^T w_{s,\xi})}{-\ell'(z_i^T w_s)} \leq \exp (|z_i^T w_{s,\xi} - z_i^T w_s|) \leq e^{R_2 B_0}$. Consequently, it holds that
\begin{align*}
    G(w_{s,\xi}) = \frac{1}{n} \sum_{i=1}^n a_i(w_{s,\xi}) \leq e^{R_2 B_0} \frac{1}{n} \sum_{i=1}^n a_i (w_s) = e^{R_2 B_0} G_s. 
\end{align*}
Moreover, we have $(z_i^T d_s)^2 \leq R_2^2 \norm{d_s}_2^2$ given $|z_i^T d_s| \leq \norm{z_i}_2 \norm{d_s}_2 \leq R_2 \norm{d_s}_2$. Putting things together, we obtain
\begin{align*}
    d_s^T \nabla^2 L(w_{s,\xi}) d_s  \leq \frac{1}{n} \sum_{i=1}^n a_i(w_{s,\xi}) R_2^2 \norm{d_s}_2^2 = R_2^2 G(w_{s,\xi}) \norm{d_s}_2^2 \leq e^{R_2 B_0} R_2^2 G_s \norm{d_s}_2^2 \leq e^{R_2 B_0} R_2^2 L_s \norm{d_s}_2^2,
\end{align*}
where we used $G_s \leq L_s$. Following the same approach as \cref{lem:taylor}, we have
\begin{align*}
    L_{s+1} &= L_s - \theta_s \langle g_s, d_s \rangle + \theta_s^2 \int_{0}^1 (1 - \xi) d_s^T \nabla^2 L(w_{s,\xi}) d_s d\xi \\
    &\leq L_s - \theta_s \langle g_s, d_s \rangle + \theta_s^2 \int_{0}^1 (1 - \xi) e^{R_2 B_0} R_2^2 L_s \norm{d_s}_2^2 d\xi \\
    &=  L_s - \theta_s \langle g_s, d_s \rangle + \frac{1}{2} e^{R_2 B_0}  R_2^2 L_s \theta_s^2 \norm{d_s}_2^2 \\
    &= L_s - \theta_s \langle g_s, d_s \rangle + C_H L_s (\theta_s \norm{d_s}_2)^2,
\end{align*}
where $C_H = \frac{1}{2} R_2^2 e^{R_2 B_0}$. 

Next, we focus on the first-order term. Given $\langle g_s, d_s\rangle = \langle g_s, u_s \rangle + \langle g_s, e_s \rangle = \norm{g_s}_2 + \langle g_s, e_s \rangle$, it holds that $- \theta_s \langle g_s, d_s \rangle \leq - \theta_s \norm{g_s}_2 + \theta_s \norm{g_s}_2 \norm{e_s}_2$. From \cref{lem:proxy}, we know $\gamma_2 G_s \leq \norm{g_s}_2 \leq R_2 G_s$. This implies
\begin{align*}
    -\theta_s \norm{g_s}_2 \leq - \gamma_2 \theta_s G_s = -\gamma_2 \theta_s L_s + \gamma_2 \theta_s (L_s - G_s), \qquad \theta_s \norm{g_s}_2 \norm{e_s}_2 \leq R_2 \theta_s L_s \norm{e_s}_2. 
\end{align*}
Substituting these bounds into above, we obtain 
\begin{align}
    L_{s+1} \leq L_s \bigl[ 1 - \gamma_2 \theta_s + \gamma_2 \theta_s ( 1 - \frac{G_s}{L_s}) + R_2 \theta_s \norm{e_s}_2 + C_H (\theta_s \norm{d_s}_2)^2 \bigr]. \label{eq:L_temp}
\end{align}
Since $\theta_s \norm{d_s}_2 \leq \theta_s + \theta_s \norm{e_s}_2$, we have $(\theta_s \norm{d_s}_2)^2 \leq 2 \theta_s^2 + 2 (\theta_s \norm{e_s}_2)^2$. Define 
\begin{align}
    r_s := \gamma_2 \theta_s (1 - \frac{G_s}{L_s}) + R_2 \theta_s \norm{e_s}_2 + 2 C_H [\theta_s^2 + (\theta_s \norm{e_s}_2)^2 \bigr ] \label{eq:rs}
\end{align}
Then, \eqref{eq:L_temp} becomes
\begin{align}
    \frac{L_{s+1}}{L_s} \leq 1 - \gamma_2 \theta_s + r_s. \label{eq:frac_L}
\end{align}
Because the left side is positive, the right side of \eqref{eq:frac_L} is also positive. Using $\log (1+ x) \leq x$ for $x > -1$, we obtain
\begin{align*}
    \log \frac{L_{s+1}}{L_s} \leq \log (1 - \gamma_2 \theta_s + r_s) \leq -\gamma_2 \theta_s + r_s. 
\end{align*}
Equivalently, we have 
\begin{align}
    \log \frac{L_s}{L_{s+1}} \geq \gamma_2 \theta_s - r_s. \label{eq:log_L} 
\end{align}
Using \eqref{eq:a1}, \eqref{eq:a2}, \eqref{eq:a3}, and \eqref{eq:a4}, we obtain from \eqref{eq:rs} after summing that
\begin{align*}
    \sum_{s = \tau}^{\infty} r_s \leq \gamma_2 K_p + R_2 E_0 + 2 C_H (K_{\theta}^2 + E_0^2).
\end{align*}
Define $A_0: = \gamma_2 K_p + R_2 E_0 + 2 C_H (K_{\theta}^2 + E_0^2)$. Summing \eqref{eq:log_L} from $s = \tau$ to $t - 1$ gives
\begin{align*}
    \log \frac{L_{\tau}}{L_{t}} \geq \gamma_2 \sum_{s = \tau}^{t-1} \theta_s - \sum_{s = \tau}^{t-1} r_s \geq \gamma_2 H_t - A_0. 
\end{align*}
Therefore, it holds that
\begin{align}
    L_t \leq L_{\tau} \exp (A_0 - \gamma_2 H_t), 
    \qquad t \geq \tau. \label{eq:exp_L}
\end{align}
For every $t \geq \tau \geq T_*$, \cref{prop:uniform-clock} gives $L_t \leq \frac{\log 2}{n}$. Given this, \cref{lem:loss-proxy} gives $L_t \geq \frac{\log 2}{n} e^{-\rho(w_t)}$. Denote $c_{\log} := \frac{\log 2}{n}$. Then, it holds that $\rho(w_t) \geq \log \frac{c_{\log}}{L_t}$.  Using \eqref{eq:exp_L}, we obtain for $t \geq \tau$ that
\begin{align}
    \rho(w_t) &\geq \log \frac{c_{\log}}{L_{\tau}} + \gamma_2 H_t - A_0 \nonumber \\
    &\geq \gamma_2 H_t -A_0, \qquad t \geq \tau, \label{eq:lower_rho}
\end{align}
where we used $\log \frac{c_{\log}}{L_{\tau}} \geq 0$ as $L_{\tau} \leq c_{\log}$. 

Unrolling \eqref{eq:post_u} gives $w_t = w_{\tau} - \sum_{s = \tau}^{t-1} \theta_s u_s - \sum_{s = \tau}^{t-1} \theta_s e_s$. Since $\norm{u_s}_2 = 1$, we have $\norm{\sum_{s = \tau}^{t-1} \theta_s u_s}_{2} \leq \sum_{s = \tau}^{t-1} \theta_s = H_t$. Moreover, \eqref{eq:a2} gives $\norm{\sum_{s = \tau}^{t-1} \theta_s e_s}_2 \leq \sum_{s = \tau}^{t-1} \theta_s \norm{e_s}_2 \leq E_0$. Therefore, we have
\begin{align}
    \norm{w_t}_2 \leq H_t + \norm{w_{\tau}}_2 + E_0 \leq H_t + B_{\tau}, \label{eq:w_bound}
\end{align}
where $B_{\tau} := \norm{w_{\tau}}_2 + E_0 + 1$.  

Next, we split the analysis into two cases. We first note that $\rho(w_t) \geq 0, \, t \geq \tau$ by \cref{lem:loss-proxy}. Also, by the definition of $\gamma_2$, it holds that $\hat{\gamma}_2 (w_t) = \rho (\frac{w_t}{\norm{w_t}_2}) \leq \gamma_2$. Consequently, it holds that $0 \leq \gamma_2 - \hat{\gamma}_2(w_t) \leq \gamma_2$. We first consider the case where $H_t \geq \frac{A_0}{\gamma_2}$. It then holds that $\gamma_2 H_t - A_0 \geq 0$. Using \eqref{eq:lower_rho} and \eqref{eq:w_bound}, we have $\hat{\gamma}_2(w_t) = \frac{\rho(w_t)}{\norm{w_t}_2} \geq \frac{\gamma_2 H_t - A_0}{H_t + B_{\tau}}$. Consequently, it holds that
\begin{align*}
    \gamma_2 - \hat{\gamma}_2(w_t) \leq \gamma_2 - \frac{\gamma_2 H_t- A_0}{H_t + B_{\tau}} = \frac{\gamma_2 B_{\tau} + A_0}{H_t + B_{\tau}}.
\end{align*}
Because $B_{\tau} \geq 1$, we have $H_t + B_{\tau} \geq 1 + H_t$. Therefore, it holds that $\gamma_2 - \hat{\gamma}_2 (w_t) \leq \frac{\gamma_2 B_{\tau} + A_0}{1 + H_t}$.  Moreover, we note that 
\begin{align*}
    \gamma_2 B_{\tau} + A_0 = \gamma_2 (1 + \norm{w_{\tau}}_2 + E_0) + A_0 \leq [\gamma_2 (1 + E_0) + A_0 + \gamma_2] (1 + \norm{w_{\tau}}_2).  
\end{align*}
If we define $C_{\gamma} = \gamma_2 (1 + E_0) + A_0 + \gamma_2$, then it holds that $\gamma_2 B_{\tau} + A_0 \leq C_{\gamma} (1 + \norm{w_{\tau}}_2)$. Consequently, we conclude  for the case $H_t \geq \frac{A_0}{\gamma_2}$ that
\begin{align}
    \gamma_2 - \hat{\gamma}_2(w_t) \leq C_{\gamma} \frac{1 + \norm{w_{\tau}}_2}{1 + H_t}. \label{eq:f1}
\end{align}
For the case where $H_t < \frac{A_0}{\gamma_2}$, we only use $0 \leq \gamma_2 - \hat{\gamma}_{2}(w_t) \leq \gamma_2$. Since $1 + H_t < 1 + \frac{A_0}{\gamma_2}$, we obtain
\begin{align}
    \gamma_2 - \hat{\gamma}_2 (w_t) \leq \gamma_2 \leq \gamma_2 (1 + \frac{A_0}{\gamma_2}) \frac{1}{1 + H_t} \leq (\gamma_2 + A_0) \frac{1 + \norm{w_{\tau}}_2}{1 + H_t}. \label{eq:f2}
\end{align}
Combining \eqref{eq:f1} and \eqref{eq:f2} together, we conclude that $ \gamma_2 - \hat{\gamma}_2(w_t) \leq C_{\gamma} \frac{1 + \norm{w_{\tau}}_2}{1 + H_t}$. 

Next, we prove part (ii).  Denote $\triangle := \triangle_{\rm geom} > 0$, $\tau_{\zeta} := \tau_{\zeta}^{\delta}(\eps) = \inf \{ t \geq \tau: \gamma_2 - \hat{\gamma}_2(w_t) \leq \zeta\}$, $R_{\tau} := \norm{w_{\tau}}_2$, and $V_t := V_t^{\delta}(\eps) = \sum_{s=\tau}^{t-1} \eta_s$. Fix any $0 < \zeta < \triangle$. \cref{prop:bounded_euc} gives $H_t \asymp_{\delta} \log (1 + V_t)$.
Given $V_t = \sum_{s = \tau}^{t-1} \eta_s \rightarrow \infty$ as $t \rightarrow \infty$, it holds that $H_t \rightarrow \infty$. Define $A_{\zeta} := \frac{C_{\gamma}}{\zeta} (1 + R_{\tau})$ and $\bar{\tau}_{\zeta} := \inf \{ t \geq \tau: H_t \geq A_{\zeta} \}$. Since $H_t \rightarrow \infty$, this time is finite. At $t = \bar{\tau}_{\zeta}$, from part (i), we obtain
\begin{align*}
    \gamma_2 - \hat{\gamma}_2 (w_{\bar{\tau}_{\zeta}}) \leq C_{\gamma} \frac{1 + R_{\tau}}{ 1 + H_{\bar{\tau}_{\zeta}}} \leq C_{\gamma} \frac{1 + R_{\tau}}{H_{\bar{\tau}_{\zeta}}} \leq \zeta. 
\end{align*}
Therefore, $\bar{\tau}_{\zeta}$ belongs to the set defining $\tau_{\zeta}$, and consequently we have $\tau_{\zeta} \leq \bar{\tau}_{\zeta}$. In particular, this shows that $\tau_{\zeta} < \infty$. Since $H_t$ is increasing, it holds that $H_{\tau_{\zeta}} \leq H_{\bar{\tau}_{\zeta}}$. By minimality of $\bar{\tau}_{\zeta}$, it also holds that $H_{\bar{\tau}_{\zeta} - 1} < A_{\zeta}$. Using $H_{\bar{\tau}_{\zeta}} = H_{\bar{\tau}_{\zeta} - 1} + \theta_{\bar{\tau}_{\zeta} - 1}$, we obtain $H_{\bar{\tau}_{\zeta}} < A_{\zeta} + \theta_{\bar{\tau}_{\zeta} - 1}$. Because $H_{\tau}= 0 < A_{\zeta}$, we must have $\bar{\tau}_{\zeta} \geq \tau + 1$ or equivalently $\bar{\tau}_{\zeta} - 1 \geq \tau$. Given $\sup_{s \geq \tau} \theta_s \leq K_{\theta}$ (proved in (i)), we have $\theta_{\bar{\tau}_{\zeta} - 1} \leq K_{\theta}$. These inequalities lead to
\begin{align*}
    H_{\tau_{\zeta}} \leq H_{\bar{\tau}_{\zeta}} < \frac{C_{\gamma}}{\zeta} (1 + R_{\tau}) + K_{\theta} \leq (\frac{C_{\gamma}}{\zeta} + K_{\theta}) (1 + R_{\tau}).
\end{align*}

Next, we show the lower bound on $H_{\tau_{\zeta}}$. From \cref{thm:margin_app_T1}, we have $\gamma_2 - \hat{\gamma}_2 (w_{\tau}) \geq \triangle - C_{2,\delta} r_{\eps}$. Choose $\eps$ sufficiently small such that $C_{2,\delta} r_{\eps} \leq \frac{\triangle - \zeta}{2}$. Then, it holds that
\begin{align*}
    \gamma_2 - \hat{\gamma}_{2} (w_{\tau}) \geq \triangle - \frac{\triangle - \zeta}{2} = \frac{\triangle + \zeta}{2}. 
\end{align*}
In particular, we have $\gamma_2 - \hat{\gamma}_2 (w_{\tau}) > \zeta$. This implies that $\tau_{\zeta} > \tau$. At $t  = \tau_{\zeta}$, we have $\gamma_2 - \hat{\gamma}_2 (w_{\tau_{\zeta}}) \leq \zeta$. Consequently, it holds that
\begin{align*}
    \hat{\gamma}_2(w_{\tau_{\zeta}}) - \hat{\gamma}_2(w_{\tau}) &= [\gamma_2 - \hat{\gamma}_2(w_{\tau})] - [\gamma_2 - \hat{\gamma}_2(w_{\tau_{\zeta}})] \\ 
    & \geq \frac{\triangle + \zeta}{2} - \zeta = \frac{\triangle - \zeta}{2}.
\end{align*}
Applying \cref{lem:margin_helper} with $x = w_{\tau_{\zeta}}$ and $y = w_{\tau}$, we obtain
\begin{align}
    \frac{\triangle - \zeta}{2} \leq \frac{2 R_2 \norm{w_{\tau_{\zeta}}- w_{\tau}}_2}{R_{\tau}}. \label{eq:ff1}
\end{align}
Based on the post-crossover update $w_{s+1} = w_s - \theta_s(u_s + e_s)$ with  $\norm{u_s}_2 =1$, we unroll from $\tau$ to $\tau_{\zeta} - 1$ to obtain $w_{\tau_{\zeta}} - w_{\tau} = - \sum_{s=\tau}^{\tau_{\zeta} - 1} \theta_s u_s - \sum_{s=\tau}^{\tau_{\zeta} - 1} \theta_s e_s$. Therefore, it holds that
\begin{align}
    \norm{w_{\tau_{\zeta}} - w_{\tau}}_2 \leq \sum_{s=\tau}^{\tau_{\zeta} - 1} \theta_s \norm{u_s}_2 + \sum_{s=\tau}^{\tau_{\zeta} - 1} \theta_s \norm{e_s}_2 \leq H_{\tau_{\zeta}} + E_0. \label{eq:ff2}
\end{align}
Combining \eqref{eq:ff1} and \eqref{eq:ff2}, we obtain $\frac{\triangle - \zeta}{2} \leq \frac{2 R_2 (H_{\tau_{\zeta}} + E_0)}{R_{\tau}}$.
This is equivalent to 
\begin{align}
    H_{\tau_{\zeta}} \geq \frac{\triangle - \zeta}{4 R_2} R_{\tau} - E_0. \label{eq:H_lower}
\end{align}
It remains to absorb the constant $E_0$. 

From the proof of \cref{prop:post_euc}, at time $\tau$, we have $\rho(w_{\tau}) \geq \ell_{\eps} - A_{\delta}$ where $A_{\delta} := [\log \frac{2 n \delta}{\log 2}]_{+}$. Because $\rho(w) \leq \gamma_2 \norm{w}_2$, we have $R_{\tau} \geq \frac{\ell_{\eps} - A_{\delta}}{\gamma_2}$. Therefore, it holds that $R_{\tau} \rightarrow \infty$ as $\eps \downarrow 0$. In particular, we choose $\eps$ sufficiently small such that
$R_{\tau} \geq \frac{8 R_2 E_0}{\triangle - \zeta}$. This implies that $E_0 \leq \frac{\triangle - \zeta}{8 R_2} R_{\tau}$. Substituting this into \eqref{eq:H_lower}, we obtain
\begin{align*}
    H_{\tau_{\zeta}} \geq \frac{\triangle - \zeta}{4 R_2} R_{\tau} - \frac{\triangle - \zeta}{8 R_2} R_{\tau} = \frac{\triangle - \zeta}{8 R_2} R_{\tau}. 
\end{align*}
Therefore, we can take
\begin{align*}
    c_{\zeta} = \frac{\triangle - \zeta}{8 R_2}, \qquad C_{\zeta} = \max \{ \frac{C_{\gamma}}{\zeta} + K_{\theta}, \frac{\triangle - \zeta}{8 R_2} \}. 
\end{align*}
Then, it holds that $0 < c_{\zeta} \leq C_{\zeta} < \infty$ and 
\begin{align}
    c_{\zeta} \norm{w_{T_1^{\delta}(\eps)}}_2 \leq H^{\delta}_{\tau_{\zeta}^{\delta}(\eps)}(\eps) \leq C_{\zeta} (1 + \norm{w_{T_1^{\delta}(\eps)}}_2). \label{eq:H_lower_upper}
\end{align}
In the proof of \cref{thm:margin_app_T1}, we have shown 
$\norm{w_{\tau_{\eps}}}_{\infty} = \frac{\ell_{\eps}}{\gamma_{\infty}} + O_{\delta}(\omega_a(\ell_{\eps}))$. Given this, we can write $\norm{w_{\tau_{\eps}}}_{\infty} = \frac{\ell_{\eps}}{\gamma_{\infty}} + R_{\eps}$ with $|R_{\eps}| \leq C_{\delta} \omega_a (\ell_{\eps})$ when $\eps$ is sufficiently small. Given $r_{\eps} \rightarrow 0$, it holds that $C_{\delta} \gamma_{\infty} r_{\eps} \leq \frac{1}{2}$ or equivalently $C_{\delta} \omega_a(\ell_{\eps}) \leq \frac{\ell_{\eps}}{2 \gamma_{\infty}}$ when $\eps$ is sufficiently small. Combining them together, for sufficiently small $\eps$, it holds that $|R_{\eps}| \leq \frac{\ell_{\eps}}{2 \gamma_{\infty}}$. This leads to $\norm{w_{\tau_{\eps}}}_{\infty} = \frac{\ell_{\eps}}{\gamma_{\infty}} + R_{\eps} \geq \frac{\ell_{\eps}}{\gamma_{\infty}} - |R_{\eps}| \geq \frac{\ell_{\eps}}{\gamma_{\infty}} - \frac{\ell_{\eps}}{2 \gamma_{\infty}} = \frac{\ell_{\eps}}{2 \gamma_{\infty}}$ and $\norm{w_{\tau_{\eps}}}_{\infty} = \frac{\ell_{\eps}}{\gamma_{\infty}} + R_{\eps} \leq \frac{\ell_{\eps}}{\gamma_{\infty}} + |R_{\eps}| \leq \frac{\ell_{\eps}}{\gamma_{\infty}} + \frac{\ell_{\eps}}{2 \gamma_{\infty}} \leq \frac{2 \ell_{\eps}}{\gamma_{\infty}}$. Therefore, we conclude for sufficiently small $\eps$ that
\begin{align*}
    \frac{\ell_{\eps}}{2 \gamma_{\infty}} \leq \norm{w_{\tau_{\eps}}}_{\infty} \leq \frac{2 \ell_{\eps}}{\gamma_{\infty}}. 
\end{align*}
Using $\norm{w}_{\infty} \leq \norm{w}_2 \leq \sqrt{d} \norm{w}_{\infty}$, we further obtain $\frac{\ell_{\eps}}{2 \gamma_{\infty}} \leq \norm{w_{\tau_{\eps}}}_2 \leq \frac{2 \sqrt{d}}{\gamma_{\infty}} \ell_{\eps}$. Substituting this into \eqref{eq:H_lower_upper}, we conclude that
\begin{align}
    H^{\delta}_{\tau_{\zeta}^{\delta}(\eps)} (\eps) = \Theta (\ell_{\eps}) = \Theta(\log \frac{1}{\eps}). \label{eq:H_scale}
\end{align}
Based on $H_t \asymp_{\delta} \log (1 + V_t)$ (shown in \cref{prop:bounded_euc}), we know there exist constants $0 < \bar{c}_{\delta} \leq \bar{C}_{\delta} < \infty$ such that
\begin{align*}
    \bar{c}_{\delta} \log (1 + V_t) \leq H_t \leq \bar{C}_{\delta} \log (1 + V_t), \qquad t \geq \tau. 
\end{align*}
At $t = \tau_{\zeta}$, \eqref{eq:H_scale} implies that there exist constants $0 < h^{-}_{\zeta} \leq h^{+}_{\zeta} < \infty$ such that $h^{-}_{\zeta} \ell_{\eps} \leq H_{\tau_{\zeta}} \leq h^{+}_{\zeta} \ell_{\eps}$. Combining them, we obtain
\begin{align}
    \frac{h_{\zeta}^{-}}{\bar{C}_{\delta}} \ell_{\eps} \leq \log (1 + V_{\tau_{\zeta}}) \leq \frac{h_{\zeta}^{+}}{\bar{c}_{\delta}} \ell_{\eps}. \label{eq:log_V} 
\end{align} 
Define $\alpha_{\zeta} := \frac{h_{\zeta}^{-}}{\bar{C}_{\delta}} > 0$ and $\beta_{\zeta} := \frac{h_{\zeta}^{+}}{\bar{c}_{\delta}} < \infty$. Exponentiating \eqref{eq:log_V} and using $e^{\ell_{\eps}} = \frac{1}{\eps}$, we obtain $\eps^{- \alpha_{\zeta}} \leq 1 + V_{\tau_{\zeta}} \leq \eps^{-\beta_{\zeta}}$. For sufficiently small $\eps$, it holds that $\eps^{- \alpha_{\zeta}} \geq 2$. Consequently, we have $\frac{1}{2} \eps^{-\alpha_{\zeta}} \leq V_{\tau_{\zeta}} \leq \eps^{- \beta_{\zeta}}$.  Therefore, we conclude that $V^{\delta}_{\tau^{\delta}_{\zeta}(\eps)} (\eps) = \eps^{- \Theta(1)}$. 

Finally, we specify the threshold $\Tilde{\eps}_{4}(\delta,\zeta)$ that incorporates the conditions on $\eps$ from both parts (i) and (ii). We define
\begin{align*}
    L_{2}
:=
\begin{cases}
\displaystyle
\max\left\{
1,\,
\left(
\frac{2C_{2,\delta}}{\Delta-\zeta}
\right)^{1/a}
\right\},
& \dfrac{1}{3}<a<1,\\[1.2em]
\displaystyle
\max\left\{
e,\,
\left(
\frac{2C_{2,\delta}}{\Delta-\zeta}
\right)^2
\right\},
& a=1, 
\end{cases}
\end{align*}
and $L_{\delta, \zeta} := \max \{ 2 A_{\delta}, \frac{16 \gamma_2 R_2 E_0}{\triangle - \zeta}, L_2\}$. We set $\Tilde{\eps}_{4}(\delta, \zeta) := \min \{ \Tilde{\eps}_1(\delta), \Tilde{\eps}^{(i)}_{4} (\delta), \Tilde{\eps}_2(\delta), e^{-L_{\delta, \zeta}}\}$. Then, any $0 < \eps \leq \Tilde{\eps}_{4}(\delta, \zeta)$ satisfies the required conditions. Indeed, $L_{\delta, \zeta} \geq L_{2}$ guarantees that $C_{2,\delta} r_{\eps} \leq \frac{\triangle - \zeta}{2}$ (for $a =1$, we used the fact $\frac{\log \ell_{\eps}}{\ell_{\eps}} \leq \frac{1}{\sqrt{\ell_{\eps}}}$); $L_{\delta, \zeta} \geq 2 A_{\delta}$ guarantees $R_{\tau} \geq \frac{\ell_{\eps}}{2 \gamma_2}$; and $L_{\delta,\zeta} \geq \frac{16 \gamma_2 R_2 E_0}{\triangle - \zeta}$ guarantees $R_{\tau} \geq \frac{8 R_2 E_0}{\triangle - \zeta}$. Moreover, we have ensured that $\Tilde{\eps}_{4}(\delta, \zeta) \leq \Tilde{\eps}_1(\delta)$ so we can apply the results of \cref{prop:bounded_euc}. 
\end{proof}

\section{Proof of \cref{cor:iter_complex}}

\begin{proof} [Proof of \cref{cor:iter_complex}] Denote $\tau_{\eps} := T_1^{\delta} (\eps)$ and $\tau_{\zeta} := \tau_{\zeta}^{\delta}(\eps)$. From \cref{thm:first-clock}, we know $S_{\tau_{\eps}} - S_{T_*} = \frac{\ell_{\eps}}{\gamma_{\infty}} + o(\ell_{\eps})$. Since $S_{T_*}$ is a fixed constant independent of $\eps$, we can write $S_{\tau_{\eps}} = \frac{1}{\gamma_{\infty}} \ell_{\eps} + o(\ell_{\eps})$. From \cref{thm:margin_euc}, we proved there exist constants $0 < \alpha_{\delta,\zeta} \leq \beta_{\delta, \zeta} < \infty$ such that $\frac{1}{2} \eps^{- \alpha_{\delta,\zeta}} \leq V^{\delta}_{\tau_{\zeta}} (\eps) \leq \eps^{- \beta_{\delta, \zeta}}$ when $\eps$ is sufficiently small. For simplicity, we let $\alpha := \alpha_{\delta,\zeta}$, $\beta := \beta_{\delta, \zeta}$, and $V_{\eps} := V^{\delta}_{\tau_{\zeta}} (\eps)$. 

Since $\eta_s = \eta_0 (s + 1)^{-a}$, we have $S_t = \eta_0 \sum_{k=1}^t k^{-a}$. For the $0 < a < 1$ case, since the function $x^{-a}$ is decreasing, it holds that $\int_{1}^{t+1} x^{-a} d x \leq \sum_{k=1}^t k^{-a} \leq 1 + \int_{1}^t x^{-a} dx$. Evaluating the integrals leads to 
\begin{align}
    \frac{(t+1)^{1-a} - 1}{1 - a} \leq \sum_{k=1}^t k^{-a} \leq 1 + \frac{t^{1-a} - 1}{1 -a}. \label{eq:int_app}
\end{align}
Define $M_t := \frac{t^{1-a}}{1-a}$. Then, the upper bound satisfies $1 + \frac{t^{1-a}}{1 -a} = M_t - \frac{a}{1 - a}$. For the lower bound, we have
\begin{align}
    \frac{(t+1)^{1-a} - 1}{1 - a} - M_t = \frac{(t+1)^{1-a} - t^{1-a} - 1}{1 - a}. \label{eq:MT}
\end{align}
Applying the mean value theorem to $h(x) = x^{1-a}$, there exists some $\nu_t \in (t, t + 1)$ such that $(t+1)^{1-a} - t^{1 - a} = h'(\nu_t) = (1 - a) \nu_t^{-a}$. For $t \geq 1$, we have $\nu_t \geq 1$ and hence $0 < \nu_t^{-a} \leq 1$. Consequently, it holds that $0 < (t+1)^{1-a} - t^{1 -a} \leq 1 - a$. Substituting this into \eqref{eq:MT}, we obtain
\begin{align*}
    -\frac{1}{1-a} < \frac{(t+1)^{1-a} - 1}{1 - a} - M_t \leq - \frac{a}{1 - a}.
\end{align*}
Combining this with \eqref{eq:int_app}, we obtain
\begin{align*}
    \frac{t^{1-a}}{1 - a} - \frac{1}{1 - a} \leq \sum_{k=1}^t k^{-a} \leq \frac{t^{1-a}}{1 - a} - \frac{a}{1 - a}.
\end{align*}
If we define $E_t := \sum_{k=1}^t k^{-a} - \frac{t^{1-a}}{1 -a}$, then it holds that $-\frac{1}{1-a} \leq E_t \leq -\frac{a}{1 -a}$. This implies that $|E_t| \leq \frac{1}{1 - a}$ and hence $E_t = O(1)$. From this, we obtain $\sum_{k=1}^t k^{-a} = \frac{t^{1-a}}{1 -a} + O(1)$ and consequently 
\begin{align}
    S_t = \frac{\eta_0}{1 - a} t^{1 - a} + O(1), \qquad 0< a < 1 \label{eq:St_1}
\end{align} 
In the case where $a = 1$, we have $\int_{1}^{t+1} \frac{dx}{x} \leq \sum_{k= 1}^t \frac{1}{k} \leq 1 + \int_{1}^t \frac{dx}{x}$. Evaluating the integral leads to $\log (t+1) \leq \sum_{k=1}^t \frac{1}{k} \leq 1 + \log t$. If we let $\Omega_t := \sum_{k=1}^t \frac{1}{k}$, then we have $\log t \leq \Omega_t \leq \log t + 1$ and consequently $\Omega_t = \log t + O(1)$. From this, we conclude that 
\begin{align}
    S_t = \eta_0 \log t + O(1), \qquad a = 1. \label{eq:St_2}
\end{align}

Now, we prove part (i) where the learning rate schedule is $\eta_t = \eta_0 (t+1)^{-a}, \, a \in (1/3,1)$. Applying \eqref{eq:St_1} at $t = \tau_{\eps}$ and using $S_{\tau_{\eps}} = \frac{1}{\gamma_{\infty}} \ell_{\eps} + o(\ell_{\eps})$, we obtain
\begin{align*}
    \frac{\eta_0}{1 - a} \tau_{\eps}^{1 - a} + O(1) = \frac{1}{\gamma_{\infty}} \ell_{\eps} + o(\ell_{\eps}). 
\end{align*}
Since $\ell_{\eps} \rightarrow \infty$, the $O(1)$ term is $o(\ell_{\eps})$ and it holds that $\tau_{\eps}^{1-a} = \frac{1 - a}{\eta_0 \gamma_{\infty}} \ell_{\eps} + o(\ell_{\eps})$. Equivalently, we can write $\tau_{\eps}^{1-a} = \frac{1 - a}{\eta_0 \gamma_{\infty}} \ell_{\eps} (1 + o(1))$. Taking the $1/(1-a)$-th power gives 
\begin{align*}
    T_1^{\delta}(\eps) = \bigl( \frac{1 - a}{\eta_0 \gamma_{\infty}} \log \frac{1}{\eps} \bigr)^{1/(1-a)} \bigl( 1 + o(1) \bigr) = \Theta \bigr( (\log \frac{1}{\eps})^{1/(1-a)}\bigr). 
\end{align*}
By definition, it holds that $S_{\tau_{\zeta}} = S_{\tau_{\eps}} + V_{\eps}$. From $S_{\tau_{\eps}} = \frac{1}{\gamma_{\infty}} \ell_{\eps} + o(\ell_{\eps})$, we know $S_{\tau_{\eps}} = O(\ell_{\eps})$. On the other hand, we have $V_{\eps} \geq \frac{1}{2} \eps^{-\alpha}$. Consequently, it holds that $\frac{S_{\tau_{\eps}}}{V_{\eps}} \leq O(\ell_{\eps} \eps^{\alpha}) \rightarrow 0$ given $\eps^{\alpha} \log \frac{1}{\eps} \rightarrow 0$ for every $\alpha > 0$. This implies that $S_{\tau_{\zeta}} = V_{\eps} (1 + o(1))$. Applying \eqref{eq:St_1} at $t = \tau_{\zeta}$, we obtain $\frac{\eta_0}{1 -a} \tau_{\zeta}^{1-a} + O(1) = V_{\eps} (1 + o(1))$. Because $V_{\eps} \rightarrow \infty$, the O(1) term is $o(V_{\eps})$. Therefore, after rearranging we obtain
\begin{align*}
    \tau_{\zeta}^{\,1-a}
=
\frac{1-a}{\eta_0}\,
V_\epsilon
\bigl(1+o(1)\bigr).
\end{align*}
Define $K_a := \frac{1 - a}{\eta_0} > 0$. We can write $\tau_{\zeta}^{1 - a} = K_a V_{\eps} (1 + r_{\eps})$ where $r_{\eps} \rightarrow 0$ as $\eps \downarrow 0$. Consequently, we have $|r_{\eps}| \leq \frac{1}{2}$ when $\eps$ is sufficiently small. Consequently, it holds that $\frac{1}{2} \leq 1 + r_{\eps} \leq \frac{3}{2}$. Substituting this into $\tau_{\zeta}^{1 - a} = K_a V_{\eps} (1 + r_{\eps})$, we obtain $\frac{K_a}{2} V_{\eps} \leq \tau_{\zeta}^{1-a} \leq \frac{3 K_a}{2} V_{\eps}$. Combining this with $\frac{1}{2} \eps^{- \alpha} \leq V_{\eps} \leq \eps^{-\beta}$, we obtain for sufficiently small $\eps$ that
\begin{align*}
    \frac{K_a}{2} \frac{1}{2} \eps^{-\alpha} \leq \frac{K_a}{2} V_{\eps} \leq \tau_{\zeta}^{1-a} \leq \frac{3 K_a}{2} V_{\eps} \leq \frac{3 K_a}{2} \eps^{-\beta}.
\end{align*}
Consequently, it holds that $\frac{K_a}{4} \eps^{-\alpha} \leq \tau_{\zeta}^{1-a} \leq \frac{3 K_a}{2} \eps^{-\beta}$ when $\eps$ is sufficiently small. Given $a \in (\frac{1}{3},1)$, the function $x^{1/(1-a)}$ is strictly increasing. Applying this function, we obtain 
\begin{align*}
    (\frac{K_a}{4})^{1/(1-a)} \eps^{-\alpha/(1-a)} \leq \tau_{\zeta} \leq (\frac{3 K_a}{2})^{1/(1-a)} \eps^{-\beta/(1-a)}.
\end{align*}
Therefore, we conclude that $\tau_{\zeta}^{\delta}(\eps) = \eps^{- \Theta(1)/(1 - a)}$. 

Now, we prove part (ii) where $\eta_t = \eta_0 (t+1)^{-1}$. Applying \eqref{eq:St_2} at $t = \tau_{\eps}$ and using $S_{\tau_{\eps}} = \frac{1}{\gamma_{\infty}} \ell_{\eps} + o(\ell_{\eps})$, we obtain
\begin{align*}
    \eta_0 \log \tau_{\eps} + O(1) = \frac{\ell_{\eps}}{\gamma_{\infty}} + o(\ell_{\eps}). 
\end{align*}
Given $O(1) = o(\ell_{\eps})$, it holds that $\log \tau_{\eps} = \frac{1}{\eta_0 \gamma_{\infty}} \ell_{\eps} + o(\ell_{\eps}) = (\frac{1}{\eta_0 \gamma_{\infty}} + o(1)) \log \frac{1}{\eps}$. Exponentiating leads to $T_1^{\delta}(\eps) = \eps ^ {-1/(\eta_0 \gamma_{\infty}) + o(1)}$. Given $S_{\tau_{\zeta}} = V_{\eps} (1 + o(1))$ (shown above), using \eqref{eq:St_2} at $t = \tau_{\zeta}$, we obtain 
\begin{align*}
    \eta_0 \log \tau_{\zeta} + O(1) = V_{\eps} (1 + o(1)). 
\end{align*}
We can rearrange this equation and obtain $\eta_0 \log \tau_{\zeta} = V_{\eps} (1 + o(1) - \frac{O(1)}{V_{\eps}})$. 
Since $V_{\eps} \rightarrow \infty$, we have $\frac{O(1)}{V_{\eps}} \rightarrow 0$ and consequently it holds that $\log \tau_{\zeta} = \frac{1}{\eta_0} V_{\eps} (1 + o(1))$. We can write $\log \tau_{\zeta} = \frac{V_{\eps}}{\eta_0} (1+r_{\eps})$ where $r_{\eps} \rightarrow 0$. For sufficiently small $\eps$, it holds that $-\frac{1}{2} \leq r_{\eps} \leq \frac{1}{2}$. This leads to $\frac{1}{2} \leq 1 + r_{\eps} \leq \frac{3}{2}$ and consequently $\frac{1}{2 \eta_0} V_{\eps} \leq \log \tau_{\zeta} \leq \frac{3}{2 \eta_0} V_{\eps}$. Using the bound $\frac{1}{2} \eps^{- \alpha} \leq V_{\eps} \leq \eps^{-\beta}$, we obtain
\begin{align*}
    \frac{1}{4 \eta_0} \eps^{-\alpha}= \frac{1}{2 \eta_0} (\frac{1}{2} \eps^{-\alpha}) \leq \frac{1}{2 \eta_0} V_{\eps} \leq \log \tau_{\zeta} \leq \frac{3}{2 \eta_0} V_{\eps} \leq \frac{3}{2 \eta_0} \eps^{- \beta}. 
\end{align*}
Exponentiating leads to $\exp(\frac{1}{4 \eta_0} \eps^{-\alpha}) \leq \tau_{\zeta} \leq \exp(\frac{3}{2 \eta_0} \eps^{-\beta})$. Therefore, we conclude that $\tau_{\zeta}^{\delta}(\eps) = \exp(\eps^{-\Theta(1)})$. 
\end{proof}
\section{Multiclass Extension}
\label{app:multiclass}
\subsection{Setup and notation}
We let the training set be $\{(x_i,y_i)\}_{i=1}^n$, where $x_i\in\R^d$ and $y_i\in[K]$ (with $K\geq2$). The classifier is $W\in\R^{K\times d}$ and we denote 
\[
    p_i(W):=\softmax(W x_i),
    \qquad
    p_{ic}(W):=[p_i(W)]_c.
\]
For every $i\in[n]$ and $c \neq y_i$, we define the pairwise feature matrix and the pairwise margin, respectively, as 
\begin{align*}
    Z_{ic}:=(e_{y_i}-e_c)x_i^\top, \qquad  m_{ic}(W):=\ip{W}{Z_{ic}}
    =(e_{y_i}-e_c)^\top W x_i.
\end{align*}

The cross-entropy loss (on the training set) is defined as 
\begin{align}
    L(W)
    :=-\frac1n\sum_{i=1}^n\log p_{i,y_i}(W)
    =\frac1n\sum_{i=1}^n
      \log\!\left(1+\sum_{c\neq y_i}e^{-m_{ic}(W)}\right). \label{eq:mc-loss} 
\end{align}
Following the construction in \citet{fan2025spectral}, the multiclass proxy (analogous to the binary proxy) is defined as 
\begin{align*}
    G(W)
    :=\frac1n\sum_{i=1}^n\bigl(1-p_{i,y_i}(W)\bigr)
    =\frac1n\sum_{i=1}^n\sum_{c\neq y_i}p_{ic}(W).
\end{align*}
For simplicity, we let $L_t := L(W_t)$, $G_t = G(W_t)$, and $g_t := \nabla L(W_t)$. 
The multiclass unnormalized margin is
\[
    \rho(W):=\min_{i\in[n],\,c\neq y_i}m_{ic}(W).
\]
Matrix norms relevant in the multiclass setting are
\[
    \text{max-norm: } \, \|A\|_{\max}
    := \max_{c\in[K],\,j\in[d]} |A[c,j]|,
    \qquad
    \text{sum-norm: } \,
    \|A\|_{\mathrm{sum}}
    := \sum_{c=1}^{K}\sum_{j=1}^{d}|A[c,j]|,
\]
and
\[
    \text{Frobenius-norm: } \,
    \|A\|_{F}
    := \left(
        \sum_{c=1}^{K}\sum_{j=1}^{d} A[c,j]^2
       \right)^{1/2}
    = \sqrt{\langle A,A\rangle}.
\]
The max-norm and the sum-norm are dual to each other. 
We define the max-norm margin and the Frobenius-norm margin as
\begin{align}
\label{eq:mc-max-margins}
    \gamma_{\max}:=\max_{\|U\|_{\max}\leq1}\rho(U),
    \qquad
    \gamma_F:=\max_{\|U\|_{F}\leq1}\rho(U).
\end{align}

Similar to the binary case, we assume the multiclass data is separable: there exists $W \in \mathbb{R}^{k \times d}$ such that $\min_{c \neq y_i} (e_{y_i} - e_c)^T W x_i > 0$ for all $i \in [n]$, and the momentum parameters satisfy $0 \leq \beta_1 \leq \beta_2 < 1$. We study the same step size schedule $\eta_t = \eta_0 (t+1)^{-a}$ with $a \in (1/3,1]$. Denote $\max_i\nor[1]{x_i} = B_1$ and $\max_i\nor[2]{x_i} = B_2$.

We set $R_1:=2B_1$ and $R_2:=\sqrt{2}\,B_2$. Importantly, the constants $R_1$ and $R_2$ play the same roles as the ones for binary data. Using the outer-product identities $\|uv^\top\|_{\mathrm{sum}}=\|u\|_1\|v\|_1$ and $\|uv^\top\|_F=\|u\|_2\|v\|_2$ together with $ \|e_{y_i}-e_c\|_1=2$ and $\|e_{y_i}-e_c\|_2=\sqrt{2}$, we obtain $\|Z_{ic}\|_{\mathrm{sum}} =\|e_{y_i}-e_c\|_1\|x_i\|_1 = 2\|x_i\|_1$ and $\|Z_{ic}\|_F =\|e_{y_i}-e_c\|_2\|x_i\|_2 = \sqrt{2}\|x_i\|_2$. Consequently, it holds that 
\begin{align*}
     \max_{i,\,c\neq y_i}\|Z_{ic}\|_{\mathrm{sum}}
    =2\max_{i \in [n]}\|x_i\|_1
    \leq 2B_1
    =R_1, \quad    \max_{i,\,c\neq y_i}\|Z_{ic}\|_F
    =\sqrt{2}\max_{i \in [n]} \|x_i\|_2
    \leq \sqrt{2}B_2
    =R_2.
\end{align*}
All products, quotients, powers, square roots, and absolute values in defining Adam's updates below are entrywise (as in the binary case). Therefore, the parameter contains $d_{\mathrm{mul}}:=Kd$
scalar coordinates. The proof in the multiclass setting is very similar to that in the binary setting. Once the analogues of the lemmas in App. \ref{app:prelim} and App. \ref{app:margin} are established, the remaining arguments follow the same steps, with mostly notational changes. We prove the required analogous lemmas in the next section and then state the multiclass results.
\subsection{\cref{app:prelim} Analogous Lemmas}
The following lemma is analogous to \cref{lem:proxy}. Its proof can be found in \citet[Lemma~16]{fan2025spectral}. We provide the proof here for completeness. 
 \begin{lemma}
\label{lem:mc-gradient-sandwich}
For every $W\in\R^{K\times d}$, it holds that
\begin{equation}
\label{eq:mc-gradient-decomposition}
    \nabla L(W)
    =-\frac1n\sum_{i=1}^n\sum_{c\neq y_i}p_{ic}(W)Z_{ic}.
\end{equation}
Moreover, we have
\begin{align}
    \gamma_{\max}G(W)
    \leq \nor[\mathrm{sum}]{\nabla L(W)}
    \leq R_1G(W), \qquad
    \gamma_FG(W)
    \leq \nor[F]{\nabla L(W)}
    \leq R_2G(W).\label{eq:mc-gradient-sandwich}
\end{align}
In particular, it holds that $G(W)>0$ and $\nabla L(W)\neq0$ at every finite $W$.

\begin{proof}
Differentiating the $i$th cross-entropy term gives
\[
    \nabla_W[-\log p_{i,y_i}(W)]
    =(p_i(W)-e_{y_i})x_i^\top
    =-\sum_{c\neq y_i}p_{ic}(W)Z_{ic},
\]
which proves \eqref{eq:mc-gradient-decomposition}. Let $U_{\max}$ attain the first maximum in
\eqref{eq:mc-max-margins}.  Then, it holds that 
\[
\begin{aligned}
    -\ip{\nabla L(W)}{U_{\max}}
    &=\frac1n\sum_i\sum_{c\neq y_i}
      p_{ic}(W)\ip{Z_{ic}}{U_{\max}} \\
    &\geq \frac{\gamma_{\max}}n
      \sum_i\sum_{c\neq y_i}p_{ic}(W)
     =\gamma_{\max}G(W).
\end{aligned}
\]
By duality and $\nor[\max]{U_{\max}}\leq1$, we have 
$-\ip{\nabla L(W)}{U_{\max}} \leq \nor[\mathrm{sum}]{\nabla L(W)}$.
This proves the lower bound on the sum-norm in
\eqref{eq:mc-gradient-sandwich}.  The upper bound on the sum-norm follows from
\eqref{eq:mc-gradient-decomposition} and the triangle inequality:
\[
    \nor[\mathrm{sum}]{\nabla L(W)}
    \leq\frac1n\sum_i\sum_{c\neq y_i}
       p_{ic}(W)\nor[\mathrm{sum}]{Z_{ic}}
    \leq R_1G(W).
\]
The Frobenius bounds follow from the same argument with a Frobenius max-margin solution $U_F$ (note that the dual of Frobenius norm is itself). Finally, every wrong
class has strictly positive softmax probability at finite $W$, which leads to $G(W)>0$;
the lower bounds then imply $\nabla L(W)\neq0$. 
\end{proof}
\end{lemma}

The following lemma is analogous to \cref{lem:loss-proxy}. The proofs of \eqref{eq:mc-G-L-comparison} and \eqref{eq:mc-low-loss-equivalence} can be found in \citet[Lemma~18]{fan2025spectral}. We provide the proof here for completeness. 

\begin{lemma}
\label{lem:mc-proxy-loss}
For every $W\in\R^{K\times d}$, it holds that 
\begin{equation}
\label{eq:mc-G-L-comparison}
    G(W)\leq L(W),
    \qquad
    \frac{G(W)}{L(W)}\geq1-\frac{nL(W)}2.
\end{equation}
If either $G(W)\leq1/(2n)$ or $L(W)\leq(\log2)/n$ holds, then we have
\begin{equation}
\label{eq:mc-low-loss-equivalence}
    L(W)\leq2G(W).
\end{equation}
Furthermore, $L(W)\leq(\log2)/n$ implies that
\begin{equation}
\label{eq:mc-margin-tail}
    \rho(W)\geq0, \quad \text{and} \quad
    \frac{\log2}{n}e^{-\rho(W)}
    \leq L(W).
\end{equation}
Finally, it holds that $L(W) \leq (K-1)e^{-\rho(W)}$ for every $W\in\R^{K\times d}$.
\end{lemma}

\begin{proof}
Let $\ell_i(W):=-\log p_{iy_i}(W)$ and $q_i(W):=1-p_{iy_i}(W)$. 
The key multiclass identity is
\begin{equation}
\label{eq:q-loss-identity}
    q_i(W)=1-e^{-\ell_i(W)}.
\end{equation}
For $x\geq0$, it holds that $x-\frac{x^2}{2}\leq1-e^{-x}\leq x$. Averaging this inequality and using
$\sum_i\ell_i^2\leq(\sum_i\ell_i)^2=n^2L(W)^2$ gives
\[
    L(W)-\frac n2L(W)^2\leq G(W)\leq L(W),
\]
which proves \eqref{eq:mc-G-L-comparison}.

If $G(W)\leq \frac{1}{2n}$, then we have $q_i(W)\leq nG(W)\leq \frac{1}{2}$ for every $i$.
If instead $L(W)\leq \frac{\log2}{n}$, then we have
$\ell_i(W)\leq nL(W)\leq\log2$, and
\eqref{eq:q-loss-identity} again gives $q_i(W)\leq \frac{1}{2}$. For $x \in [0,\frac{1}{2}]$, it holds that $-\log(1-x)\leq2x$. Consequently, we have for either case that 
\[
    L(W)=\frac1n\sum_i-\log(1-q_i(W))
    \leq\frac2n\sum_iq_i(W)=2G(W).
\]

Now suppose $L(W)\leq(\log2)/n$.  For every $i$, we have
\[
    \log\!\left(1+\sum_{c\neq y_i}e^{-m_{ic}(W)}\right)
    =\ell_i(W)\leq\log2.
\]
Therefore, it holds that $\sum_{c\neq y_i}e^{-m_{ic}(W)}\leq1$. This implies that every
$m_{ic}(W)\geq0$ and therefore $\rho(W)\geq0$.
Choose $(i_\star,c_\star)$ attaining $\rho(W)$. Then, it holds that 
\[
    L(W)  = \frac1n\sum_{i=1}^n
      \log\!\left(1+\sum_{c\neq y_i}e^{-m_{ic}(W)}\right)\geq \frac{1}{n} \log (1 + e^{-m_{i_* c_*}}) = \frac1n\log\bigl(1+e^{-\rho(W)}\bigr).
\]
For $x\in[0,1]$, the concavity of $\log(1+x)$ gives
$\log(1+x)\geq x\log2$.  Substituting $x=e^{-\rho(W)}$ proves the lower
bound in \eqref{eq:mc-margin-tail}.  Finally, $\log(1+x)\leq x$ gives
\[
\begin{aligned}
    L(W)
    &\leq\frac1n\sum_i\sum_{c\neq y_i}e^{-m_{ic}(W)}
     \leq (K-1)e^{-\rho(W)},
\end{aligned}
\]
which proves the upper bound.
\end{proof}

The following lemma is analogous to \cref{lem:proxy-path-ratio} in \cref{app:prelim}. It will be used in the proof of \cref{lem:mc-hessian} below. 
\begin{lemma}
\label{lem:mc-softmax-stability}
For any $U,V\in\R^{K\times d}$, every $i\in[n]$, and every $c\in[K]$, we have  
\begin{equation}
\label{eq:mc-probability-ratio-max}
    e^{-R_1\nor[\max]{U-V}}
    \leq\frac{p_{ic}(U)}{p_{ic}(V)}
    \leq e^{R_1\nor[\max]{U-V}},
\end{equation}
and
\begin{equation}
\label{eq:mc-probability-ratio-fro}
    e^{-R_2\nor[F]{U-V}}
    \leq\frac{p_{ic}(U)}{p_{ic}(V)}
    \leq e^{R_2\nor[F]{U-V}}.
\end{equation}
Consequently, denote
$R_{\max}:=R_1$, $R_F:=R_2$, the following holds for $p\in\{\max,F\}$
\begin{equation}
\label{eq:mc-G-ratio}
    e^{-R_p\nor[p]{U-V}}G(V)
    \leq G(U)
    \leq e^{R_p\nor[p]{U-V}}G(V).
\end{equation}
Further assume that $0 \leq \beta_1 \leq \beta_2 < 1$. It holds for every $0 \leq r \leq t$ that 
\begin{align}
     e^{-R_1C_D\sum_{s=r}^{t-1}\eta_s}G_r
 \le G_t
 \le e^{R_1C_D\sum_{s=r}^{t-1}\eta_s}G_r. \label{eq:mc_g_ratio2}
\end{align}
\end{lemma}

\begin{proof}
Denote $\Delta:=U-V$. Fix a sample $i\in[n]$ and a class $c\in[K]$. For each $r\in[K]$, we define $z_r:=(Vx_i)_r$ and $a_r:=(\Delta x_i)_r$.
Then, we have $(Ux_i)_r=z_r+a_r$.
By the definition of softmax, we have 
\[
    p_{ic}(U)
    =\frac{e^{z_c+a_c}}{\sum_{r=1}^K e^{z_r+a_r}},
    \qquad
    p_{ic}(V)
    =\frac{e^{z_c}}{\sum_{r=1}^K e^{z_r}}.
\]
Therefore, it holds that
\begin{align}
    \frac{p_{ic}(U)}{p_{ic}(V)} &=
    \frac{e^{z_c+a_c}}{\sum_{r=1}^K e^{z_r+a_r}}
    \frac{\sum_{r=1}^K e^{z_r}}{e^{z_c}}
    \notag =
    e^{a_c}
    \frac{\sum_{r=1}^K e^{z_r}}
         {\sum_{r=1}^K e^{z_r}e^{a_r}}
    \notag\\
    &=
    \frac{e^{a_c}}
         {\sum_{r=1}^K
          \frac{e^{z_r}}{\sum_{s=1}^K e^{z_s}}e^{a_r}} =
    \frac{e^{a_c}}
         {\sum_{r=1}^K p_{ir}(V)e^{a_r}}.
    \label{eq:softmax-ratio-detailed}
\end{align}

We further define $a_{\min}:=\min_{r\in[K]}a_r$ and $a_{\max}:=\max_{r\in[K]}a_r$.
Because $p_{ir}(V)\geq0$ and
$\sum_{r=1}^Kp_{ir}(V)=1$, we have $e^{a_{\min}}
    \leq
    \sum_{r=1}^Kp_{ir}(V)e^{a_r}
    \leq
    e^{a_{\max}}$.
Combining this with \eqref{eq:softmax-ratio-detailed} gives
\[
    e^{a_c-a_{\max}}
    \leq
    \frac{p_{ic}(U)}{p_{ic}(V)}
    \leq
    e^{a_c-a_{\min}}.
\]
Since $a_{\min}\leq a_c\leq a_{\max}$,
we have $a_c-a_{\max} \geq -(a_{\max}-a_{\min})$
and $a_c-a_{\min} \leq a_{\max}-a_{\min}$.
Consequently, it holds that 
\begin{equation}
\label{eq:ratio-by-logit-oscillation}
    e^{-(a_{\max}-a_{\min})}
    \leq
    \frac{p_{ic}(U)}{p_{ic}(V)}
    \leq
    e^{a_{\max}-a_{\min}}.
\end{equation}

It remains to control the difference
$a_{\max}-a_{\min}$. For any $r,s\in[K]$, it holds that
\begin{align}
    a_r-a_s
    &=(e_r-e_s)^\top\Delta x_i =
    \left\langle
        \Delta,(e_r-e_s)x_i^\top
    \right\rangle_F.
    \label{eq:logit-difference-inner-product}
\end{align}
By duality, we have 
\begin{align}
    |a_r-a_s|
    &\leq
    \nor[\max]{\Delta}
    \nor[\mathrm{sum}]{(e_r-e_s)x_i^\top}
    \leq
    2\nor[\max]{\Delta}\nor[1]{x_i}
    \leq
    R_1\nor[\max]{\Delta}. 
    \label{eq:logit-oscillation-max-bound}
\end{align}
It follows that
\[
    a_{\max}-a_{\min}
    =\max_{r,s\in[K]}(a_r-a_s)
    \leq
    R_1\nor[\max]{U-V}.
\]
Substituting this bound into
\eqref{eq:ratio-by-logit-oscillation} yields
\[
    e^{-R_1\nor[\max]{U-V}}
    \leq
    \frac{p_{ic}(U)}{p_{ic}(V)}
    \leq
    e^{R_1\nor[\max]{U-V}}.
\]

Similarly, by Frobenius norm self-duality, we have 
\begin{align}
    |a_r-a_s|
    &\leq
    \nor[F]{\Delta}
    \nor[F]{(e_r-e_s)x_i^\top}
    \leq
    \sqrt{2}\nor[F]{\Delta}\nor[2]{x_i}
    \leq
    R_2\nor[F]{\Delta}. 
    \label{eq:logit-oscillation-fro-bound}
\end{align}
Therefore, it holds that $a_{\max}-a_{\min} \leq R_2\nor[F]{U-V}$.

Using \eqref{eq:ratio-by-logit-oscillation} again gives
\[
    e^{-R_2\nor[F]{U-V}}
    \leq
    \frac{p_{ic}(U)}{p_{ic}(V)}
    \leq
    e^{R_2\nor[F]{U-V}}.
\]

Finally, we let $p\in\{\max,F\}$ where $R_{\max}:=R_1$ and $R_F:=R_2$.
The preceding probability-ratio bounds imply that for every
$i\in[n]$ and $c\neq y_i$, 
\[
    e^{-R_p\nor[p]{U-V}}p_{ic}(V)
    \leq
    p_{ic}(U)
    \leq
    e^{R_p\nor[p]{U-V}}p_{ic}(V).
\]
Summing over $i\in[n]$ and $c\neq y_i$, and then dividing by $n$, we obtain
\begin{align*}
    G(U)
    &=\frac1n\sum_{i=1}^n\sum_{c\neq y_i}p_{ic}(U) \leq
    e^{R_p \nor[p]{U-V}}
    \frac1n \sum_{i=1}^n\sum_{c\neq y_i}p_{ic}(V) =
    e^{R_p\nor[p]{U-V}}G(V),
\end{align*}
and similarly $G(U) \geq e^{-R_p\nor[p]{U-V}}G(V)$.
Hence, it holds that
\[
    e^{-R_p\nor[p]{U-V}}G(V)
    \leq
    G(U)
    \leq
    e^{R_p\nor[p]{U-V}}G(V).
\]
Applying this inequality with $p=\max$, $U=W_t$, and $V=W_r$, we obtain
\begin{align}
    e^{-R_1\|W_t-W_r\|_{\max}}G_r
\le G_t
\le e^{R_1\|W_t-W_r\|_{\max}}G_r. \label{eq:multi_g_temp}
\end{align}
By the triangle inequality and \cref{lem:multi_D}, it holds that
\[
\|W_t-W_r\|_{\max}
\le \sum_{s=r}^{t-1}\eta_s\|D_s\|_{\max}
\le C_D\sum_{s=r}^{t-1}\eta_s.
\]
Substituting this into \eqref{eq:multi_g_temp} finishes the proof of \eqref{eq:mc_g_ratio2}. 
\end{proof}

The probability simplex is defined as
\begin{align*}
    \Delta^{K-1}
    :=
    \left\{
        s\in\mathbb{R}_+^K:
        \sum_{c=1}^K s_c=1
    \right\}.
\end{align*}

The following lemma is a refinement of \citet[Lemma~14]{fan2025spectral}, specialized to the max-norm and Frobenius-norm cases. Direct application of \citet[Lemma~14]{fan2025spectral} results a factor of $2 R_2^2$ in the Frobenius bound \eqref{eq:softmax-covariance-data-fro}. \cref{lem:softmax-covariance-bound} will be used in the proof of \cref{lem:mc-hessian}. 
\begin{lemma}
\label{lem:softmax-covariance-bound}
For every $s\in\Delta^{K-1}$, $y\in[K]$,
$A\in\mathbb{R}^{K\times d}$, and $h\in\mathbb{R}^d$, we have
\begin{equation}
\label{eq:softmax-covariance-max}
    0
    \leq
    h^\top A^\top
    \bigl(\operatorname{diag}(s)-ss^\top\bigr)
    Ah
    \leq
    4(1-s_y)\|h\|_1^2\|A\|_{\max}^2,
\end{equation}
and
\begin{equation}
\label{eq:softmax-covariance-fro}
    0
    \leq
    h^\top A^\top
    \bigl(\operatorname{diag}(s)-ss^\top\bigr)
    Ah
    \leq
    2(1-s_y)\|h\|_2^2\|A\|_F^2.
\end{equation}
Consequently, for every training example $x_i$, we have
\begin{equation}
\label{eq:softmax-covariance-data-max}
    x_i^\top A^\top
    \bigl(\operatorname{diag}(s)-ss^\top\bigr)
    A x_i
    \leq
    R_1^2(1-s_y)\|A\|_{\max}^2,
\end{equation}
and
\begin{equation}
\label{eq:softmax-covariance-data-fro}
    x_i^\top A^\top
    \bigl(\operatorname{diag}(s)-ss^\top\bigr)
    Ax_i
    \leq
    R_2^2(1-s_y)\|A\|_F^2. 
\end{equation}
\end{lemma}

\begin{proof} Let $s \in \Delta^{K-1}$. Denote $v:=Ah\in\mathbb{R}^K$ and $C_s:=\operatorname{diag}(s)-ss^\top$. Let $C$ be a random class distributed according to $s$, i.e. $\mathbb{P}(C=c)=s_c, \, c\in[K]$.
Then, we have 
\begin{align}
    v^\top C_sv =
    \sum_{c=1}^K s_cv_c^2
    -
    \left(\sum_{c=1}^K s_cv_c\right)^2 =
    \operatorname{Var}(v_C)
    \geq 0.
    \label{eq:softmax-covariance-variance}
\end{align}
For any $a\in\mathbb{R}$, it holds that
\begin{align}
    \mathbb{E}\bigl[(v_C-a)^2\bigr]
    &=
    \mathbb{E}\left[
        \bigl(
            v_C-\mathbb{E}[v_C]
            +\mathbb{E}[v_C]-a
        \bigr)^2
    \right]
    =
    \operatorname{Var}(v_C)
    + \bigl(\mathbb{E}[v_C]-a\bigr)^2 \geq
    \operatorname{Var}(v_C).
    \label{eq:variance-upper-by-anchor}
\end{align}
Choosing $a=v_y$ in \eqref{eq:variance-upper-by-anchor}, we obtain
\begin{align}
    v^\top C_sv
    &=
    \operatorname{Var}(v_C) \leq
    \mathbb{E}\bigl[(v_C-v_y)^2\bigr] =
    \sum_{c=1}^K s_c(v_c-v_y)^2=
    \sum_{c\neq y}s_c(v_c-v_y)^2.
    \label{eq:softmax-covariance-anchor}
\end{align}

For the max-norm bound, it holds for every $c\neq y$ that 
\begin{align}
    |v_c-v_y|
    &=
    |(e_c-e_y)^\top Ah|
    =
    \left|
        \left\langle
            A,(e_c-e_y)h^\top
        \right\rangle
    \right|\notag\\
    &\leq
    \|A\|_{\max}
    \left\|(e_c-e_y)h^\top\right\|_{\mathrm{sum}}
    =
    2\|A\|_{\max}\|h\|_1.
    \label{eq:class-difference-max}
\end{align}
Therefore, we conclude that $(v_c-v_y)^2
\leq 4\|A\|_{\max}^2\|h\|_1^2$.
Substituting this into
\eqref{eq:softmax-covariance-anchor} gives
\begin{align}
    v^\top C_sv
    &\leq
    4\|A\|_{\max}^2\|h\|_1^2
    \sum_{c\neq y}s_c =
    4(1-s_y)\|h\|_1^2\|A\|_{\max}^2.
\end{align}

For the Frobenius-norm bound, it holds for every $c\neq y$ that 
\begin{align}
    |v_c-v_y|
    &=
    \left|
        \left\langle
            A,(e_c-e_y)h^\top
        \right\rangle_F
    \right|
    \leq
    \|A\|_F
    \left\|(e_c-e_y)h^\top\right\|_F
    =
    \sqrt{2}\|A\|_F\|h\|_2.
    \label{eq:class-difference-fro}
\end{align}
Thus, it holds that $(v_c-v_y)^2
    \leq
    2\|A\|_F^2\|h\|_2^2$.
Substituting this inequality into
\eqref{eq:softmax-covariance-anchor} yields
\begin{align}
    v^\top C_sv
    &\leq
    2\|A\|_F^2\|h\|_2^2
    \sum_{c\neq y}s_c =
    2(1-s_y)\|h\|_2^2\|A\|_F^2.
\end{align}
\end{proof}

The notation $\nabla^2 L(W)[A,A]$
denotes the second directional derivative of $L$ at $W$ in the matrix direction $A$: 
\[
    \nabla^2 L(W)[A,A]
    :=
    \left.
    \frac{\mathrm{d}^2}{\mathrm{d}s^2}
    L(W+sA)
    \right|_{s=0}.
\]
Equivalently, after vectorizing the matrices, we have 
\[
    \nabla^2 L(W)[A,A]
    =
    \operatorname{vec}(A)^\top
    \nabla_{\operatorname{vec}(W)}^2 L(W)
    \operatorname{vec}(A).
\]

The Hessian identity in \eqref{eq:mc-hessian-identity} follows from \citet[Lemma~10]{fan2025spectral}. Applying \cref{lem:softmax-covariance-bound} to \eqref{eq:mc-hessian-identity} leads to \eqref{eq:mc-hessian-max} and \eqref{eq:mc-hessian-fro}. Finally, we derive \eqref{eq:mc-taylor-bound} via Taylor expansion. \cref{lem:mc-hessian} is analogous to \cref{lem:taylor} in \cref{app:prelim}. 
\begin{lemma}
\label{lem:mc-hessian}
For every $W,A\in\R^{K\times d}$, we have
\begin{equation}
\label{eq:mc-hessian-identity}
\begin{aligned}
    \nabla^2L(W)[A,A]
    =\frac1n\sum_{i=1}^n
    (Ax_i)^\top
    \bigl(\diag(p_i(W))-p_i(W)p_i(W)^\top\bigr)(Ax_i).
\end{aligned}
\end{equation}
It satisfies
\begin{equation}
\label{eq:mc-hessian-max}
    0\leq\nabla^2L(W)[A,A]
    \leq R_1^2G(W)\nor[\max]{A}^2
\end{equation}
and
\begin{equation}
\label{eq:mc-hessian-fro}
    0\leq\nabla^2L(W)[A,A]
    \leq R_2^2G(W)\nor[F]{A}^2.
\end{equation}
Consequently, for $p\in\{\max,F\}$, it holds that
\begin{equation}
\label{eq:mc-taylor-bound}
\begin{aligned}
    L(U)
    \leq{}&L(V)+\ip{\nabla L(V)}{U-V} +\frac{R_p^2}{2}
      e^{R_p\nor[p]{U-V}}
      G(V)\nor[p]{U-V}^2,
\end{aligned}
\end{equation}
 where $R_{\max}:=R_1$ and $R_F:=R_2$.
\end{lemma}

\begin{proof} Denote $\ell_y(z):=-\log\bigl(\operatorname{softmax}(z)_y\bigr)$. Following  \citet[Lemma~10]{fan2025spectral}, 
we have
\[
    \nabla^2\ell_y(z)
    =
    \operatorname{diag}(\operatorname{softmax}(z))
    -
    \operatorname{softmax}(z)
    \operatorname{softmax}(z)^\top.
\]
Applying this identity with $z=W x_i$ and directional perturbation
$A x_i$, and then averaging over $i\in[n]$, we obtain 
\[
    \nabla^2L(W)[A,A]
    =
    \frac1n\sum_{i=1}^n
    (A x_i)^\top
    \left(
        \operatorname{diag}(p_i(W))
        -p_i(W)p_i(W)^\top
    \right)
    (Ax_i).
\]

For each $i\in[n]$, we apply
Lemma~\ref{lem:softmax-covariance-bound} with $s=p_i(W)$, $y=y_i$, and $h=x_i$.
The max-norm part of
Lemma~\ref{lem:softmax-covariance-bound} gives
\[
\begin{aligned}
0
&\leq
(Ax_i)^\top
\left(
    \operatorname{diag}(p_i(W))
    -p_i(W)p_i(W)^\top
\right)
(Ax_i) \leq
R_1^2
\bigl(1-p_{i,y_i}(W)\bigr)
\|A\|_{\max}^2.
\end{aligned}
\]
Averaging this inequality over $i\in[n]$ and using the Hessian
identity, we obtain
\begin{align*}
\nabla^2L(W)[A,A]
&\leq
\frac{R_1^2\|A\|_{\max}^2}{n}
\sum_{i=1}^n
\bigl(1-p_{i,y_i}(W)\bigr) = 
R_1^2G(W)\|A\|_{\max}^2,
\end{align*}
where the last equality follows from $G(W) =
    \frac1n\sum_{i=1}^n
    \bigl(1-p_{i,y_i}(W)\bigr)$.

Similarly, the Frobenius-norm part of
Lemma~\ref{lem:softmax-covariance-bound} gives
\[
\begin{aligned}
0
&\leq
(Ax_i)^\top
\left(
    \operatorname{diag}(p_i(W))
    -p_i(W)p_i(W)^\top
\right)
(Ax_i) \leq 
R_2^2
\bigl(1-p_{i,y_i}(W)\bigr)
\|A\|_F^2.
\end{aligned}
\]
Averaging over $i\in[n]$ yields
\begin{align*}
\nabla^2L(W)[A,A]
&\leq
\frac{R_2^2\|A\|_F^2}{n}
\sum_{i=1}^n
\bigl(1-p_{i,y_i}(W)\bigr) =
R_2^2G(W)\|A\|_F^2.
\end{align*}
For the final claim, we fix $U,V\in\mathbb{R}^{K\times d}$ and define $\Delta:=U-V$.
Consider the scalar function $\phi(\tau):=L(V+\tau\Delta)$ where $\tau\in[0,1]$.
Then, it holds that $\phi(0)=L(V)$, $\phi(1)=L(U)$, $\phi'(0) = \left\langle \nabla L(V),\Delta
\right\rangle_F$, and $\phi''(\tau) = \nabla^2L(V+\tau\Delta)[\Delta,\Delta]$.
Therefore, from Taylor's theorem, we obtain
\begin{align}
    L(U)
    ={}&
    L(V)
    +
    \left\langle
        \nabla L(V),\Delta
    \right\rangle_F +
    \int_0^1
    (1-\tau)
    \nabla^2L(V+\tau\Delta)[\Delta,\Delta]
    \,\mathrm{d}\tau.
    \label{eq:multiclass-taylor-integral}
\end{align}

Let $p\in\{\max,F\}$ and define $R_{\max}:=R_1$, and $R_F:=R_2$. 
Applying the corresponding Hessian bound at $V+\tau\Delta$ gives
\[
    \nabla^2L(V+\tau\Delta)[\Delta,\Delta]
    \leq
    R_p^2
    G(V+\tau\Delta)
    \|\Delta\|_p^2.
\]
By Lemma~\ref{lem:mc-softmax-stability},
\begin{align*}
    G(V+\tau\Delta)
    &\leq 
    \exp\left(
        \tau R_p\|\Delta\|_p
    \right)G(V) \leq
    \exp\left(
        R_p\|\Delta\|_p
    \right)G(V),
\end{align*}
where the last inequality follows from $0\leq\tau\leq1$.
Consequently, it holds that 
\[
    \nabla^2L(V+\tau\Delta)[\Delta,\Delta]
    \leq
    R_p^2
    e^{R_p\|\Delta\|_p}
    G(V)\|\Delta\|_p^2.
\]
Substituting this bound into
\eqref{eq:multiclass-taylor-integral} yields
\begin{align*}
    L(U)
    \leq{}&
    L(V)
    +
    \left\langle
        \nabla L(V),\Delta
    \right\rangle_F +
    R_p^2
    e^{R_p\|\Delta\|_p}
    G(V)\|\Delta\|_p^2
    \int_0^1(1-\tau)\,\mathrm{d}\tau.
\end{align*}
Since $\int_0^1(1-\tau)\,\mathrm{d}\tau = \frac12$,
we conclude that
\[
    L(U)
    \leq
    L(V)
    +
    \left\langle
        \nabla L(V),U-V
    \right\rangle_F
    +
    \frac{R_p^2}{2}
    e^{R_p\|U-V\|_p}
    G(V)\|U-V\|_p^2.
\]
\end{proof}

For the multiclass setting, we 
consider the same Adam update as \eqref{eq:adam_update}, now applied entrywise to $K\times d$ matrices. 
The bound from
\cref{lem:uniform-D} remains valid.
\begin{lemma} \label{lem:multi_D}
    In the multiclass setting, Adam's update satisfies $\nor[\max]{D_t}\leq C_D$ where $C_D:=\sqrt{\frac{1-\beta_1}{1-\beta_2}}$. 
\end{lemma}
We first derive an intermediate result that will be useful in the proof of \cref{lem:mc-adam-gradient-drift}. 
\begin{lemma}
\label{lem:mc-gradient-drift}
For any $U,V\in\R^{K\times d}$, we have
\begin{equation}
\label{eq:mc-gradient-drift-sum}
    \nor[\mathrm{sum}]{\nabla L(U)-\nabla L(V)}
    \leq R_1G(V)
      \left(e^{R_1\nor[\max]{U-V}}-1\right),
\end{equation}
and
\begin{equation}
\label{eq:mc-gradient-drift-fro}
    \nor[F]{\nabla L(U)-\nabla L(V)}
    \leq R_2G(V)
      \left(e^{R_2\nor[F]{U-V}}-1\right).
\end{equation}
\end{lemma}

\begin{proof}
    Denote $r:=R_1\nor[\max]{U-V}$.  Equation
\eqref{eq:mc-probability-ratio-max} implies that
\[
    |p_{ic}(U)-p_{ic}(V)|
    \leq p_{ic}(V)(e^r-1).
\]
Using the gradient decomposition from
Lemma~\ref{lem:mc-gradient-sandwich},
\[
\begin{aligned}
    \nor[\mathrm{sum}]{\nabla L(U)-\nabla L(V)}
    &\leq\frac1n\sum_i\sum_{c\neq y_i}
      |p_{ic}(U)-p_{ic}(V)|\nor[\mathrm{sum}]{Z_{ic}} \\
    &\leq R_1G(V)(e^r-1),
\end{aligned}
\]
which proves \eqref{eq:mc-gradient-drift-sum}. For the Frobenius norm, the argument is
the same after using \eqref{eq:mc-probability-ratio-fro} and
$\nor[F]{Z_{ic}}\leq R_2$. 
\end{proof}

The following lemma is analogous to \cref{lem:gradient-drift} in \cref{app:prelim}. It will be used in the proof of \cref{lem:mc-aggregate-moment}. 
\begin{lemma}
\label{lem:mc-adam-gradient-drift} Suppose that $0 \leq \beta_1 \leq \beta_2 < 1$. 
For every $0\leq r\leq t$, we have 
\begin{equation}
\label{eq:mc-path-gradient-drift}
\begin{aligned}
    \nor[\mathrm{sum}]{g_{t-r}-g_t}
    \leq R_1G_t
    \left[
      \exp\!\left(
        R_1C_D\sum_{s=t-r}^{t-1}\eta_s
      \right)-1
    \right].
\end{aligned}
\end{equation}
\end{lemma}

\begin{proof}
By \cref{lem:multi_D}, we obtain 
\[
    \nor[\max]{W_t-W_{t-r}}
    \leq C_D\sum_{s=t-r}^{t-1}\eta_s.
\]
Applying \eqref{eq:mc-gradient-drift-sum} with
$U=W_{t-r}$ and $V=W_t$ finishes the proof.
\end{proof}

\cref{lem:mc-aggregate-moment} is analogous to \cref{lem:first-track} and \cref{lem:second-track}. 
\begin{lemma}
\label{lem:mc-aggregate-moment} Suppose that $0 \leq \beta_1 \leq \beta_2 < 1$. There exist a finite time
$t_{\mathrm{ema}}$ and constants $C_m,C_v<\infty$ (independent of
$\eps$), such that for every $t\geq t_{\mathrm{ema}}$,
\begin{equation}
\label{eq:mc-first-moment}
    \nor[\mathrm{sum}]{m_t-g_t}\leq C_m\eta_tG_t
\end{equation}
and
\begin{equation}
\label{eq:mc-second-moment}
    \nor[\mathrm{sum}]{\sqrt{v_t}-|g_t|}
    \leq C_v\eta_tG_t.
\end{equation}
\end{lemma}

\begin{proof}
    The only loss-specific ingredient needed in the binary proofs of \cref{lem:first-track} and \cref{lem:second-track} is the gradient-difference bound in \cref{lem:gradient-drift}.  That ingredient is now supplied
by Lemma~\ref{lem:mc-adam-gradient-drift}.  Apply binary \cref{lem:first-track} with
$q=\beta_1$, and binary \cref{lem:second-track} with $q=\sqrt{\beta_2}$, in each case with
$c=R_1C_D$.  Their EMA calculations are unchanged
after vectorizing the $K\times d$ gradient. Therefore, 
\eqref{eq:mc-first-moment} and \eqref{eq:mc-second-moment} hold after $t \geq t_m: = t_0(\beta_1,R_1C_D,a,\eta_0)$ and $t \geq t_v: = t_0(\sqrt{\beta_2},R_1C_D,a,\eta_0)$, where $t_0(q,c,a,\eta_0)$ is defined in  \cref{prop:ema-sum}. The constants $C_{m}$ and $C_{v}$ are $C_m = R_1 C_{\rm geo}(\beta_1, R_1C_D, a,\eta_0)$ and $C_v = R_1 C_{\rm geo}(\sqrt{\beta_2}, R_1C_D, a, \eta_0)$.  Setting $t_{\rm ema} := \max \{t_{m}, t_{v}\}$ finishes the proof. 
\end{proof}

As in the binary case, we can define the entrywise matrix soft-sign map for any matrix $A \in \mathbb{R}^{K \times d}$.
\begin{align*}
    \sigma_\epsilon(A)_{kj}
    :=
    \frac{A_{kj}}{|A_{kj}|+\epsilon},
    \qquad \epsilon>0.
\end{align*}

The following lemma is analogous to \cref{lem:alignment} in \cref{app:prelim}. 
\begin{lemma}
\label{lem:mc-adam-softsign}
There is a constant $C_{\mathrm{ema}}<\infty$ (independent of
$\eps$) such that the following holds for every $t\geq t_{\mathrm{ema}}$
\begin{align}
\label{eq:mc-adam-softsign-weighted}
    \nor[\mathrm{sum}]{
      g_t\odot(D_t-\sigma_\eps(g_t))}
    &\leq C_{\mathrm{ema}}\eta_tG_t,\\
\label{eq:mc-adam-softsign-fro}
    \eps\nor[F]{D_t-\sigma_\eps(g_t)}
    &\leq C_{\mathrm{ema}}\eta_tG_t,\\
\label{eq:mc-adam-softsign-inner}
    \left|\ip{g_t}{D_t-\sigma_\eps(g_t)}\right|
    &\leq C_{\mathrm{ema}}\eta_tG_t.
\end{align}
\end{lemma}

\begin{proof}
We first vectorize the $K\times d$ matrices.  The proof of \cref{lem:alignment} in the binary setting is entirely coordinatewise and uses only
\eqref{eq:mc-first-moment}, \eqref{eq:mc-second-moment}, and the bound
$|m_t[j]|\leq C_D\sqrt{v_t[j]}$ from \cref{lem:uniform-D}.  Applying that argument to the
$Kd$ vectorized coordinates proves
\eqref{eq:mc-adam-softsign-weighted} and its inner-product consequence.
The same coordinatewise decomposition multiplied by $\eps$ (followed by
$\nor[F]{A}\leq\nor[\mathrm{sum}]{A}$) proves
\eqref{eq:mc-adam-softsign-fro}.
\end{proof}

\begin{lemma} \label{lem:multi_softsign} For any $A \in \mathbb{R}^{K \times d}$, it holds that
\begin{align}
    \langle A, \sigma_{\epsilon}(A) \rangle \geq \norm{A}_{\rm sum} - K d \epsilon, \label{eq:multi_soft1} \\
    \langle A, \sigma_{\epsilon}(A) \rangle \geq \frac{\norm{A}_{\rm sum}^2}{\norm{A}_{\rm sum} + K d \epsilon}. \label{eq:multi_soft2}
\end{align}
\end{lemma}

\begin{proof} Based on the definition of $\sigma_{\epsilon}(A)$, we have 
$\left\langle A,\sigma_\epsilon(A)\right\rangle
= \sum_{k=1}^K \sum_{j=1}^d \frac{A_{kj}^2}{|A_{kj}|+\epsilon}$.
For \eqref{eq:multi_soft1}, we have the following entrywise inequality
\[
    \frac{A_{kj}^2}{|A_{kj}|+\epsilon}
    =
    |A_{kj}|
    -
    \epsilon
    \frac{|A_{kj}|}{|A_{kj}|+\epsilon}
    \geq
    |A_{kj}|-\epsilon,
\]
where we used $0 \leq \frac{|A_{kj}|}{|A_{kj}|+\epsilon} \leq 1$. 
Summing over all $Kd$ entries yields
$ \left\langle A,\sigma_\epsilon(A)\right\rangle_F
\geq \|A\|_{\mathrm{sum}}-Kd\epsilon$.

For \eqref{eq:multi_soft2}, we can write
\[
    \|A\|_{\mathrm{sum}}
    =
    \sum_{k=1}^K\sum_{j=1}^d |A_{kj}|
    =
    \sum_{k=1}^K\sum_{j=1}^d
    \frac{|A_{kj}|}{\sqrt{|A_{kj}|+\epsilon}}
    \sqrt{|A_{kj}|+\epsilon}.
\]
By the Cauchy--Schwarz inequality, it holds that 
\[
\begin{aligned}
    \|A\|_{\mathrm{sum}}^2
    &\leq
    \left(
        \sum_{k=1}^K\sum_{j=1}^d
        \frac{A_{kj}^2}{|A_{kj}|+\epsilon}
    \right)
    \left(
        \sum_{k=1}^K\sum_{j=1}^d
        \bigl(|A_{kj}|+\epsilon\bigr)
    \right) = 
    \left\langle A,\sigma_\epsilon(A)\right\rangle
    \left(
        \|A\|_{\mathrm{sum}}+Kd\epsilon
    \right).
\end{aligned}
\]
Since $\|A\|_{\mathrm{sum}}+Kd\epsilon>0$, dividing both sides by this
quantity gives
\[
    \left\langle A,\sigma_\epsilon(A)\right\rangle_F
    \geq
    \frac{\|A\|_{\mathrm{sum}}^2}
         {\|A\|_{\mathrm{sum}}+Kd\epsilon}.
\]
\end{proof}

\begin{proposition} \label{prop:multi_recurse}
For every $t \geq t_{\rm ema}$, it holds that
\begin{align*}
    L_{t+1} \leq L_t - \eta_t \langle g_t, \sigma_{\epsilon}(g_t) \rangle_F + C_{\mathrm{md}}^{\mathrm{mc}} \eta_t^2 G_t,
\end{align*}
where $C_{\mathrm{md}}^{\mathrm{mc}} = C_{\rm ema} + \frac{R_1^2 C_D^2}{2} e^{R_1 C_D \eta_0}$. 
\end{proposition}

\begin{proof} Fix $\epsilon > 0$. Denote $g_t := \nabla L(W_t)$.
Applying \cref{lem:mc-hessian} with
$U = W_{t+1}$, $V = W_t$, and $p = \max$,
and noting that \(U-V=-\eta_tD_t\), we obtain
\begin{align}
    L_{t+1}
    &\leq
    L_t
    + \left\langle g_t,-\eta_tD_t\right\rangle_F
    + \frac{R_1^2}{2}
      \exp\!\left(
          R_1\eta_t\|D_t\|_{\max}
      \right)
      G_t\eta_t^2\|D_t\|_{\max}^2
      \notag\\
    &=
    L_t
    - \eta_t\langle g_t,D_t\rangle_F
    + \frac{R_1^2}{2}
      \exp\!\left(
          R_1\eta_t\|D_t\|_{\max}
      \right)
      G_t\eta_t^2\|D_t\|_{\max}^2.
    \label{eq:mc-loss-preliminary}
\end{align}

The coordinatewise Adam bound in \cref{lem:multi_D} gives
$\|D_t\|_{\max} \leq C_D$ where $C_D :=
    \sqrt{\frac{1-\beta_1}{1-\beta_2}}$.
Since \(\eta_t\leq\eta_0\), it follows that
$\exp\!\left(
        R_1\eta_t\|D_t\|_{\max}
    \right)
    \leq
    \exp\!\left(R_1C_D\eta_0\right)$ and  $\|D_t\|_{\max}^2\leq C_D^2$.
Therefore, \eqref{eq:mc-loss-preliminary} implies
\begin{equation}
    L_{t+1}
    \leq
    L_t
    - \eta_t\langle g_t,D_t\rangle_F
    + C_H^{\mathrm{mc}}\eta_t^2G_t,
    \label{eq:mc-loss-adam-direction}
\end{equation}
where $C_H^{\mathrm{mc}} :=
\frac{R_1^2C_D^2}{2} \exp\!\left(R_1C_D\eta_0\right)$. Next, we decompose the first-order term as
\[
    \langle g_t,D_t\rangle_F
    =
    \left\langle
        g_t,\sigma_\epsilon(g_t)
    \right\rangle_F
    +
    \left\langle
        g_t,D_t-\sigma_\epsilon(g_t)
    \right\rangle_F.
\]
Define $\delta_t :=
    \left\langle
        g_t,D_t-\sigma_\epsilon(g_t)
    \right\rangle_F$.
\cref{lem:mc-adam-softsign} gives for every \(t\geq t_{\mathrm{ema}}\) that 
$|\delta_t| \leq C_{\mathrm{ema}}\eta_t G_t$. In particular, we have $\delta_t
\geq -C_{\mathrm{ema}}\eta_tG_t$
and hence
\[
    \langle g_t,D_t\rangle_F
    \geq
    \left\langle
        g_t,\sigma_\epsilon(g_t)
    \right\rangle_F
    -
    C_{\mathrm{ema}}\eta_tG_t.
\]
Multiplying both sides by \(-\eta_t\) yields
\[
    -\eta_t\langle g_t,D_t\rangle_F
    \leq
    -\eta_t
    \left\langle
        g_t,\sigma_\epsilon(g_t)
    \right\rangle_F
    +
    C_{\mathrm{ema}}\eta_t^2G_t.
\]

Substituting this estimate into
\eqref{eq:mc-loss-adam-direction}, we obtain for every $t \geq t_{\rm ema}$ that
\begin{align*} 
    L_{t+1}
    &\leq
    L_t
    -
    \eta_t
    \left\langle
        g_t,\sigma_\epsilon(g_t)
    \right\rangle_F
    +
    C_{\mathrm{ema}}\eta_t^2G_t
    +
    C_H^{\mathrm{mc}}\eta_t^2G_t\\
    &=
    L_t
    -
    \eta_t
    \left\langle
        g_t,\sigma_\epsilon(g_t)
    \right\rangle_F
    +
    C_{\mathrm{md}}^{\mathrm{mc}}\eta_t^2G_t,
\end{align*}
where $C_{\mathrm{md}}^{\mathrm{mc}}
    :=
    C_H^{\mathrm{mc}}+C_{\mathrm{ema}}$.

\end{proof}

\subsection{\cref{app:margin} Analogous Lemmas}
In this section, we discuss the multiclass versions of the lemmas in \cref{app:margin}. 
The following lemma is analogous to \cref{lem:rho_lip} in \cref{app:margin}. Recall that
\[
m_{ic}(W)
=
\langle W,Z_{ic}\rangle, \qquad
Z_{ic}=(e_{y_i}-e_c)x_i^\top, \qquad \rho(W)
=
\min_{\substack{i\in[n]\\ c\neq y_i}}
m_{ic}(W). 
\] 
\begin{lemma}
\label{lem:mc-margin-lipschitz}
For all $U,V\in\R^{K\times d}$, we have 
\begin{equation}
\label{eq:mc-margin-lipschitz}
    |\rho(U)-\rho(V)|
    \leq R_1\nor[\max]{U-V},
    \qquad
    |\rho(U)-\rho(V)|
    \leq R_2\nor[F]{U-V}.
\end{equation}
\end{lemma}

\begin{proof}
Set $\Delta:=U-V$. Fix an example \(i\in[n]\) and an incorrect class \(c\neq y_i\).
By linearity of the pairwise margin, it holds that 
\begin{align*}
m_{ic}(U)-m_{ic}(V)
&=
\langle U,Z_{ic}\rangle
-
\langle V,Z_{ic}\rangle =
\left\langle
\Delta,(e_{y_i}-e_c)x_i^\top
\right\rangle.
\end{align*}
By duality, it holds that
\[
\left|m_{ic}(U)-m_{ic}(V)\right|
\leq
\|\Delta\|_{\max}
\left\|(e_{y_i}-e_c)x_i^\top\right\|_{\mathrm{sum}}.
\]
Because \(c\neq y_i\), we have $\|e_{y_i}-e_c\|_1=2$.
For the rank-one matrix \(Z_{ic}\), we have
\begin{align*}
\left\|(e_{y_i}-e_c)x_i^\top\right\|_{\mathrm{sum}}
&=
\sum_{r=1}^K\sum_{j=1}^d
\left|(e_{y_i}-e_c)_r(x_i)_j\right|\\
&=
\left(
\sum_{r=1}^K |(e_{y_i}-e_c)_r|
\right)
\left(
\sum_{j=1}^d |(x_i)_j|
\right)\\
&=
2\|x_i\|_1 \leq R_1.
\end{align*}
Consequently, for every \(i\in[n]\) and \(c\neq y_i\), it holds that 
\[
\left|m_{ic}(U)-m_{ic}(V)\right|
\leq
R_1\|U-V\|_{\max}.
\]

Let $\delta_{\max}:=R_1\|U-V\|_{\max}$. 
The preceding inequality implies $m_{ic}(U)
\geq m_{ic}(V)-\delta_{\max}$.
Since \(m_{ic}(V)\geq\rho(V)\), it follows that for every
\(i\in[n]\) and \(c\neq y_i\), $m_{ic}(U) \geq
\rho(V)-\delta_{\max}$. 
Taking the minimum over all pairs \((i,c)\) gives
$\rho(U) \geq \rho(V)-\delta_{\max}$.
Interchanging \(U\) and \(V\) gives
$\rho(V) \geq \rho(U)-\delta_{\max}
$.
Combining the two inequalities yields
\[
\left|\rho(U)-\rho(V)\right|
\leq
\delta_{\max}
=
R_1\|U-V\|_{\max}.
\]
For the Frobenius-norm, we have 
\[
\left|m_{ic}(U)-m_{ic}(V)\right|
\leq
\|\Delta\|_F
\left\|(e_{y_i}-e_c)x_i^\top\right\|_F.
\]
For every rank-one matrix \(ab^\top\), it holds that
$\|ab^\top\|_F=\|a\|_2\|b\|_2$.
Consequently, we have  
\begin{align*}
\left\|(e_{y_i}-e_c)x_i^\top\right\|_F
&=
\|e_{y_i}-e_c\|_2\|x_i\|_2 =
\sqrt{2}\,\|x_i\|_2 \leq R_2.
\end{align*}
Therefore, it holds that $\left|m_{ic}(U)-m_{ic}(V)\right|
\leq
R_2\|U-V\|_F$. The rest of the proof follows the same steps as the max-norm case and we obtain  
\[
\left|\rho(U)-\rho(V)\right|
\leq 
R_2\|U-V\|_{F}.
\]
\end{proof}

Based on \cref{lem:mc-margin-lipschitz}, we can derive the following lemma which is analogous to \cref{lem:margin_helper} in \cref{app:margin}. 

\begin{lemma} \label{lem:mc-normalized-margin-stability}
 For any nonzero matrices \(U,V\in\mathbb{R}^{K\times d}\), it holds 
\[
\left|
\frac{\rho(U)}{\|U\|_{\max}}
-
\frac{\rho(V)}{\|V\|_{\max}}
\right|
\leq
\frac{2R_1\|U-V\|_{\max}}{\|V\|_{\max}},
\]
and
\[
\left|
\frac{\rho(U)}{\|U\|_F}
-
\frac{\rho(V)}{\|V\|_F}
\right|
\leq
\frac{2R_2\|U-V\|_F}{\|V\|_F}.
\]
\end{lemma}

\begin{proof} The proof follows the same steps as \cref{lem:margin_helper} by replacing the vector $\ell_{\infty}$- and $\ell_2$-norm with the matrix max- and Frobenius-norms respectively, and using \cref{lem:mc-margin-lipschitz} instead of \cref{lem:rho_lip}.     
\end{proof}

Next, we extend \cref{lem:hoff} to the multiclass setting. We define the max-norm solution set in this case as  
\[
    \cM_{\max}
    :=\{U\in\R^{K\times d}:\nor[\max]{U}\leq1,
                              \ \rho(U)=\gamma_{\max}\}. 
\]
\begin{lemma}[Multiclass Hoffman system]
\label{lem:mc-hoffman}
Let $z_{ic}:=\operatorname{vec}(Z_{ic})\in\R^{Kd}$, and let
$Z_{\mathrm{mc}}\in\R^{n(K-1)\times Kd}$ be the matrix containing the row vectors
$z_{ic}^{\top}$, and define
\[
    A_H:=
    \begin{bmatrix}
        I_{Kd}\\[-1mm]
        -I_{Kd}\\[-1mm]
        -Z_{\mathrm{mc}}
    \end{bmatrix},
    \qquad
    c_H:=
    \begin{bmatrix}
        \one_{Kd}\\[-1mm]
        \one_{Kd}\\[-1mm]
        -\gamma_{\max}\one_{n(K-1)}
    \end{bmatrix}.
\] 
 There is a finite constant $H_{2,\infty}(A_H)$, depending only on
$A_H$, such that
\begin{equation}
\label{eq:mc-hoffman}
    \dist_F(W,\cM_{\max})
    \leq H_{2,\infty}(A_H)
       \nor[\infty]{
          \bigl[A_H\operatorname{vec}(W)-c_H\bigr]_+},
\end{equation}
where $[\cdot]_+$ is entrywise and
$\dist_F(W,\cM_{\max}):=\inf_{U\in\cM_{\max}} \nor[F]{W-U}$.
\end{lemma}

\begin{proof}
As in the binary case, it is enough to verify that $\cM_{\rm max} = \{W:A_H\operatorname{vec}(W)\leq c_H\}$.  The first two blocks of
$A_H\operatorname{vec}(W)\leq c_H$ are equivalent to
$\nor[\max]{W}\leq1$, and the last block is equivalent to
$m_{ic}(W)\geq\gamma_{\max}$ for all $(i,c)$ with $c\neq y_i$.
Consequently, we have $\rho(W)\geq\gamma_{\max}$. The reverse inequality
$\rho(W) \leq \gamma_{\max}$ follows directly from the definition of $\gamma_{\max}$. Putting things together, we obtain 
\[
    \{W:A_H\operatorname{vec}(W)\leq c_H\}=\cM_{\max}.
\]
The remaining arguments are the same as those in \cref{lem:hoff}. 
\end{proof}

\subsection{Main Results} 
Throughout this section, we assume multiclass separability, $0 \leq \beta_1 \leq \beta_2 < 1$, and $\eps \in (0,1]$. The step size schedule is $\eta_t = \eta_0 (t+1)^{-a}$ with $\eta_0 > 0$ and $a \in (1/3,1]$, satisfying \cref{prop:ema-sum}. The loss is the cross-entropy loss defined in \eqref{eq:mc-loss}. Recall that $R_1 = 2 B_1$ and $R_2 = \sqrt{2} B_2$ where $B_1 = \max_{i \in [n]} \norm{x_i}_1$ and $B_2 = \max_{i \in [n]} \norm{x_i}_2$. Similar to the binary case, we define the sign and gradient residuals in the multiclass setting by  
\begin{align*}
    \lambda_t:=\frac{G_t}{\epsilon},
\qquad
r_t^{\mathrm{sgn}}
:=
\frac{
\left\|
g_t\odot\bigl(D_t-\operatorname{sign}(g_t)\bigr)
\right\|_{\mathrm{sum}}
}{
\|g_t\|_{\mathrm{sum}}
}, \qquad 
r_t^{\mathrm{grad}}
:=
\frac{\|\epsilon D_t-g_t\|_F}{\|g_t\|_F}.
\end{align*}

Throughout the multiclass extension, we make the following substitutions:
\[
\bigl(\|\cdot\|_\infty,\|\cdot\|_1,\|\cdot\|_2;
      \gamma_\infty,\gamma_2\bigr)
\;\longrightarrow\;
\bigl(\|\cdot\|_{\max},\|\cdot\|_{\mathrm{sum}},\|\cdot\|_F;
      \gamma_{\max},\gamma_F\bigr).
\]
Furthermore, we replace the constant $d$ with $Kd$.

The following lemma is the multiclass version of \cref{lem:radius}. It can be directly obtained by applying the coordinatewise proof of \cref{lem:radius} to the \(Kd\) entries.
\begin{lemma}[Multiclass accumulated radius bound]
\label{lem:mc-radius}
Suppose that
$\|\nabla L(W)\|_{\max}\leq B_g$ for every
$W\in\mathbb R^{K\times d}$.
There exists a constant $C_{\mathrm{rad}}\geq0$ (depending only on
$\beta_1$ and $\beta_2$) such that  for every $\eps>0$ and
integers $0\leq r<T$,
\begin{equation}
\label{eq:mc-radius}
\left\|\sum_{t=r}^{T-1}\eta_tD_t\right\|_{\max}
\leq S_T-S_r+
C_{\mathrm{rad}}\eta_r
\left[
1+\log\frac{B_g^2+\eps^2}{(1-\beta_2)\eps^2}
\right].
\end{equation}
In the multiclass setting,
we can take $B_g=R_1$ since $\|\nabla L(W)\|_{\max}\leq R_1G(W)\leq R_1$. 
\end{lemma}
The following proposition is the multiclass counterpart of
\cref{prop:uniform-clock}. Its proof follows the same steps, with
\cref{lem:proxy}, \cref{lem:loss-proxy}, \cref{lem:uniform-D}, \cref{lem:proxy-path-ratio}, \cref{lem:taylor}, \cref{lem:alignment}, \cref{lem:softsign}, and \cref{prop:master-descent}
replaced by
\cref{lem:mc-gradient-sandwich}, \cref{lem:mc-proxy-loss}, \cref{lem:multi_D}, \cref{lem:mc-softmax-stability}, \cref{lem:mc-hessian}, \cref{lem:mc-adam-softsign}, \cref{lem:multi_softsign}, and \cref{prop:multi_recurse},
respectively.
\begin{proposition}
\label{prop:mc-low-loss}
For every fixed $\eps\in(0,1]$, the
corresponding Adam trajectory satisfies
$L_t\to0$ and $G_t\to0$ as $t\to\infty$.
Moreover, there exist an integer $T_*>t_{\mathrm{ema}}$ and a
constant $g_*>0$ (both independent of $\eps$) such that for every
$\eps\in(0,1]$,
\[
G_{T_*}\geq g_*,
\qquad
L_t\leq\frac{\log 2}{n},
\qquad
\frac12L_t\leq G_t\leq L_t
\quad\text{for every }t\geq T_*.
\]
\end{proposition}

Denote $\ell_{\eps} = \log \frac{1}{\eps}$. Recall that $ G_t = G(W_t) =\frac1n\sum_{i=1}^n\sum_{c\neq y_i}p_{ic}(W_t)$ where $p_{ic}(W_t)$ is the c-th entry of the vector $p_i(W_t) = \softmax(W_t x_i)$. Let $T_*$ be the common low-loss time in \cref{prop:mc-low-loss}. The stopping times in the multiclass setting are defined in the same way as the binary setting using the multiclass $G_t$: 
\[
T_1^-(\eps;b)
 :=\inf\{t\geq T_*:G_t\leq\eps\ell_\eps^b\},\qquad
T_1^\delta(\eps)
 :=\inf\{t\geq T_*:G_t\leq\delta\eps\},
\]
\[
T_1^+(\eps;b)
 :=\inf\{t\geq T_*:G_t\leq\eps\ell_\eps^{-b}\}.
\]

The following proposition is the multiclass counterpart of
\cref{prop:local-interpolation_app}. Its proof follows the same steps, with \cref{lem:proxy} and \cref{lem:alignment} replaced by \cref{lem:mc-gradient-sandwich} and \cref{lem:mc-adam-softsign} respectively. 
\begin{proposition}
\label{prop:mc-local-residual}
Denote $C_{\mathrm{loc}}^{\mathrm{mc}}
:=
C_{\mathrm{ema}}
\max\left\{\frac{1}{\gamma_{\max}},\frac{1}{\gamma_F}\right\}$. 
For every $\eps>0$ and $t\geq t_{\mathrm{ema}}$, we have
\begin{equation*}
\label{eq:mc-sign-residual-bound}
\left[
\frac{1}{1+R_1\lambda_t}
-C_{\mathrm{loc}}^{\mathrm{mc}}\eta_t
\right]_+
\leq r_t^{\mathrm{sgn}}
\leq
\min\left\{1,\frac{Kd}{\gamma_{\max}\lambda_t}\right\}
+C_{\mathrm{loc}}^{\mathrm{mc}}\eta_t,
\end{equation*}
and
\begin{equation*}
\label{eq:mc-grad-residual-bound}
\left[
1-\frac{\sqrt{Kd}}{\gamma_F\lambda_t}
-C_{\mathrm{loc}}^{\mathrm{mc}}\eta_t
\right]_+
\leq r_t^{\mathrm{grad}}
\leq
\frac{R_1\lambda_t}{1+R_1\lambda_t}
+C_{\mathrm{loc}}^{\mathrm{mc}}\eta_t,
\end{equation*}
where $[x]_+:=\max\{x,0\}$.
\end{proposition}

The following proposition is the multiclass counterpart of
\cref{prop:time_bracket_app}. Its proof follows the same steps, with \cref{lem:proxy-path-ratio} and \cref{prop:uniform-clock} replaced by \cref{lem:mc-softmax-stability} and \cref{prop:mc-low-loss} respectively. 

\begin{proposition}
\label{prop:mc-stopping-residual}
 For any $b,\delta\in(0,1)$, 
there exists $\eps_0(b,\delta)>0$ such that for every
$0<\eps\leq\eps_0(b,\delta)$, we have 
\[
T_1^-(\eps;b)<T_1^\delta(\eps)<T_1^+(\eps;b).
\]
\end{proposition}

The following theorem is the multiclass counterpart of \cref{thm:residual_nonasymp}.
Its proof follows the same steps, replacing
\cref{prop:uniform-clock}, \cref{prop:local-interpolation_app}, \cref{prop:time_bracket_app}, and \cref{lem:proxy-path-ratio} with \cref{prop:mc-low-loss}, \cref{prop:mc-local-residual}, \cref{prop:mc-stopping-residual}, and \cref{lem:mc-softmax-stability}, respectively. 

\begin{theorem}
\label{prop:mc-update-residuals} For any $b,\delta \in (0,1)$, there exists a constant $\Tilde{\eps}_{\rm res}(b,\delta)$ such that for every $0 < \eps \leq \Tilde{\eps}_{\rm res}(b,\delta)$, we have 
\[
\max\!\left\{
r^{\mathrm{sgn}}_{T_1^-(\eps;b)},
r^{\mathrm{grad}}_{T_1^+(\eps;b)}
\right\}
=
\begin{cases}
O\!\left(
\ell_\eps^{-b}+\ell_\eps^{-a/(1-a)}
\right), & a\in(1/3,1),\\[2mm]
O\!\left(\ell_\eps^{-b}\right), & a=1.
\end{cases}
\]
Denote
$K_0:=R_1C_D$. We also have 
\[
r^{\mathrm{grad}}_{T_1^\delta(\eps)}
\leq
\frac{R_1\delta}{1+R_1\delta}
+
\begin{cases}
O\!\left(\ell_\eps^{-a/(1-a)}\right),
    & a\in(1/3,1),\\[2mm]
O\!\left(\eps^{1/(2K_0\eta_0)}\right),
    & a=1.
\end{cases}
\]
\end{theorem}

Let $\widehat\gamma_p(W):=\frac{\rho(W)}{\|W\|_p},
\quad p\in\{\max,F\}$. The max-norm optimal solution set is defined as 
\[
\mathcal M_{\max}
:=\{U\in\mathbb R^{K\times d}:
       \|U\|_{\max}\le 1,\ \rho(U)=\gamma_{\max}\}.
\]
We further define a few geometric quantities: 
\[
\gamma_{F\mid\max}
:=\max_{U\in\mathcal M_{\max}}
       \frac{\rho(U)}{\|U\|_F},
\qquad
\Delta_{\mathrm{geom}}^{\mathrm{mc}}
:=\gamma_F-\gamma_{F\mid\max}, \qquad r_\eps:=
\begin{cases}
\ell_\eps^{-a}, & a\in(1/3,1),\\[2pt]
(\log\ell_\eps)/\ell_\eps, & a=1.
\end{cases}
\]
The multiclass geometric separation $\triangle_{\rm geom}^{\rm mc}$ is the difference between the Frobenius data margin and the best attainable Frobenius margin by the classifiers in $\mathcal M_{\max}$. 

Denote \(T_1:=T_1^\delta(\eps)\). Define
$u_t:=\frac{g_t}{\|g_t\|_F}$, $\theta_t:=\frac{\eta_t\|g_t\|_F}{\eps}$, and $e_t:=\frac{\eps D_t-g_t}{\|g_t\|_F}$. The post-crossover update of Adam can be written as $W_{t+1} = W_t - \theta_t (u_t + e_t)$. As in the binary case, we further
define the accumulated original and effective step sizes as:
\[
V_t^\delta(\eps):=\sum_{s=T_1}^{t-1}\eta_s,
\qquad
H_t^\delta(\eps):=\sum_{s=T_1}^{t-1}\theta_s.
\]

The following proposition is the multiclass counterpart of
\cref{prop:bounded_euc}. Its proof uses \cref{prop:mc-low-loss} in place of
\cref{prop:uniform-clock}, \cref{lem:mc-gradient-sandwich} and \cref{lem:mc-adam-softsign} in place of \cref{lem:proxy} and \cref{lem:alignment},
and  \cref{lem:mc-softmax-stability} for the trajectory bound from \cref{lem:proxy-path-ratio}.
The scalar stepsize estimate in \cref{lem:helper} applies unchanged.
The required two-sided post-crossover loss estimate follows by
adapting Part~3 of the proof of \cref{thm:first-clock} to the multiclass setting.

\begin{proposition} \label{prop:mc-post-crossover}
    For every $\delta \in (0,1)$, there exists a constant $\Tilde{\eps}^{\rm mc}_{1}(\delta)$ such that, for every $0<\eps\le \Tilde{\eps}^{\rm mc}_{1}(\delta)$ and every $t\ge T_1^{\delta}(\eps)$, we have $L_t \asymp G_t\asymp\frac{\eps}{1+V_t^{\delta}(\eps)}$, $H_t^{\delta}(\eps) \asymp \log(1+V_t^{\delta}(\eps))$, and 
$\sum_{s=T_1^{\delta}(\eps)}^{\infty}\theta_s\norm{e_s}_F\le E_0^{\rm mc}$, where $E_0^{\rm mc}$ is a constant that depends on the data and the step-size schedule.
\end{proposition}

The following theorem is the multiclass counterpart of \cref{thm:margin_app_T1}.
Its proof follows the same steps, replacing \cref{lem:proxy-path-ratio}, \cref{lem:radius}, \cref{lem:proxy}, \cref{lem:loss-proxy}, \cref{lem:uniform-D},
\cref{lem:softsign}, \cref{lem:hoff}, and \cref{lem:rho_lip} with \cref{lem:mc-softmax-stability}, \cref{lem:mc-radius}, \cref{lem:mc-gradient-sandwich}, \cref{lem:mc-proxy-loss}, \cref{lem:multi_D}, \cref{lem:multi_softsign}, \cref{lem:mc-hoffman}, and \cref{lem:mc-margin-lipschitz},
respectively, and \cref{prop:master-descent} and \cref{prop:uniform-clock} with  \cref{prop:multi_recurse}
and \cref{prop:mc-low-loss}. The inequality in \cref{lem:vec_inverse} applies
to the Frobenius norm after vectorization. 

\begin{theorem} \label{thm:multi_gap_T1}
For \(\delta\in(0,1)\), there exists a constant $\Tilde{\eps}_{2, \rm mc}(\delta)$ such that for every $0< \eps \leq \Tilde{\eps}_{2,\rm mc}(\delta)$, we have 
\[
\gamma_{\max}-\widehat{\gamma}_{\max}(W_{T_1^{\delta}(\eps)})
\leq
O(r_\epsilon), \qquad 
\gamma_F -\widehat{\gamma}_F(W_{T_1^{\delta}(\eps)})
\geq
\Delta^{\rm mc}_{\mathrm{geom}}
- O(r_\epsilon).
\]
\end{theorem}

The following theorem is the multiclass counterpart of
\cref{thm:margin_app_T1-}. Its proof follows the same steps, and uses
\cref{prop:multi_recurse,prop:mc-low-loss,prop:mc-stopping-residual}
in place of
\cref{prop:master-descent,prop:uniform-clock,prop:time_bracket_app},
\cref{lem:mc-gradient-sandwich,lem:mc-proxy-loss,lem:multi_D,lem:multi_softsign,lem:mc-normalized-margin-stability}
in place of
\cref{lem:proxy,lem:loss-proxy,lem:uniform-D,lem:softsign,lem:margin_helper}, \cref{lem:mc-softmax-stability} in place of \cref{lem:proxy-path-ratio}, and \cref{thm:multi_gap_T1} in place of \cref{thm:margin_app_T1}, respectively. 
\begin{theorem} 
 For any $\delta,b \in (0,1)$, 
there exists a constant $\eps'_{2,\rm mc}(\delta,b)$ such that for every $0< \eps \leq \eps'_{2,\rm mc}(\delta,b)$, we have 
\[
\gamma_{\max}-\widehat{\gamma}_{\max}(W_{T_1^{-}(\eps;b)})
\leq
O(r_\epsilon), \qquad 
\gamma_F - \widehat{\gamma}_F(W_{T_1^{-}(\eps;b)})
\geq
\Delta_{\mathrm{geom}}
- O(r_\epsilon).
\]
\end{theorem}
 
The following proposition is the multiclass counterpart of
\cref{prop:post_euc}. Its proof follows the same steps,
using \cref{prop:mc-post-crossover}, \cref{thm:multi_gap_T1}, \cref{prop:mc-low-loss}, \cref{lem:mc-proxy-loss}, \cref{lem:mc-normalized-margin-stability} in place of
\cref{prop:bounded_euc}, \cref{thm:margin_app_T1}, \cref{prop:uniform-clock}, \cref{lem:loss-proxy}, \cref{lem:margin_helper}, respectively.

\begin{proposition} \label{prop:mc_post_margin}
 For every \(\delta\in(0,1)\),
there exists \(\Tilde\eps_3^{\mathrm{mc}}(\delta)>0\)
such that for \(0<\eps\le
\widetilde\eps_3^{\mathrm{mc}}(\delta)\) and
\(t\ge T_1^\delta(\eps)\), we have
\[
\gamma_{\max}-\widehat\gamma_{\max}(W_t)
 \le O\!\left(r_\eps+
       \frac{H_t^\delta(\eps)}{\ell_\eps}\right), \qquad 
\gamma_F-\widehat\gamma_F(W_t)
 \ge \Delta_{\mathrm{geom}}^{\mathrm{mc}}
       -O\!\left(r_\eps+
       \frac{H_t^\delta(\eps)}{\ell_\eps}\right).
\]
\end{proposition}


The following corollary is the multiclass counterpart of
\cref{cor:margin_app_T1+}. Its proof follows the same steps, and uses \cref{prop:mc-stopping-residual}, \cref{prop:mc-post-crossover}, \cref{lem:mc-softmax-stability}, and \cref{prop:mc_post_margin}
in place of \cref{prop:time_bracket_app}, \cref{prop:bounded_euc}, \cref{lem:proxy-path-ratio}, and \cref{prop:post_euc}, respectively.
\begin{corollary} \label{cor:mc_margin_app_T1+} 
For any $\delta,b \in (0,1)$, 
there exists a constant $\eps''_{2, \rm mc}(\delta,b)$ such that for every $0< \eps \leq \eps''_{2, \rm mc}(\delta,b)$, we have 
\[
\gamma_{\max}-\widehat{\gamma}_{\max}(W_{T_1^{+}(\eps;b)})
\leq
O(r_\epsilon), \qquad 
\gamma_F-\widehat{\gamma}_F(W_{T_1^{+}(\eps;b)})
\geq
\Delta_{\mathrm{geom}}
- O(r_\epsilon).
\]
\end{corollary}

We state two additional lemmas that are needed for the proof of \cref{thm:multi_first_clock}. The following lemma is the multiclass counterpart of \cref{lem:trans_aux}. Its proof follows the same steps, and uses \cref{prop:mc-low-loss}, \cref{lem:mc-gradient-sandwich}, \cref{lem:multi_softsign}, \cref{prop:multi_recurse}, \cref{lem:multi_D} in place of \cref{prop:uniform-clock}, \cref{lem:proxy}, \cref{lem:softsign}, \cref{prop:master-descent}, and \cref{lem:uniform-D}, respectively. 

\begin{lemma}
\label{lem:mc-uniform-small-loss}
For every $\zeta>0$, there exists an integer $t_\zeta$,
(independent of $\eps$) such that for every $\eps\in(0,1]$, we have 
\[
L_t\le\zeta,
\qquad t\ge t_\zeta.
\]
\end{lemma}

The following lemma is the multiclass counterpart of \cref{lem:ratio}. Its proof follows the same steps, using \cref{prop:mc-low-loss}, \cref{lem:mc-softmax-stability}, and 
\cref{lem:mc-proxy-loss} in place of \cref{prop:uniform-clock}, \cref{lem:proxy-path-ratio}, and \cref{lem:loss-proxy}, respectively. 
\begin{lemma}
\label{lem:mc-first-crossing}
Let $h_\eps>0$ satisfy $h_\eps\to0$ as $\eps\downarrow0$. Define
$\tau_\eps
:=\inf\{t\ge T_*:G_t\le h_\eps\}$ 
where $T_*$ is the time in \cref{prop:mc-low-loss}. 
Then, we have $\tau_\eps<\infty$ for every $\eps\in(0,1]$, and
\[
\tau_\eps\to\infty,
\qquad
\frac{G_{\tau_\eps}}{h_\eps}\to1,
\qquad
\frac{L_{\tau_\eps}}{h_\eps}\to1,
\qquad \text{as} \, \, \eps\downarrow0.
\]
\end{lemma}

 The following theorem is the multiclass counterpart of \cref{thm:first-clock}. Its proof follows the same steps, using \cref{lem:mc-radius}, \cref{lem:mc-gradient-sandwich}, \cref{lem:mc-proxy-loss}, \cref{lem:multi_D}, \cref{lem:multi_softsign}, \cref{lem:mc-uniform-small-loss}, \cref{lem:mc-first-crossing}, \cref{prop:multi_recurse}, \cref{prop:mc-low-loss}, \cref{lem:mc-softmax-stability}, \cref{lem:mc-hessian}, 
 \cref{lem:mc-adam-softsign} in place of \cref{lem:radius}, \cref{lem:proxy}, \cref{lem:loss-proxy}, \cref{lem:uniform-D}, \cref{lem:softsign}, \cref{lem:trans_aux}, \cref{lem:ratio}, \cref{prop:master-descent}, \cref{prop:uniform-clock}, \cref{lem:derivative-ratio}, \cref{lem:taylor}, and \cref{lem:alignment}, respectively. 
 
\begin{theorem} \label{thm:multi_first_clock}
For every fixed \(b,\delta\in(0,1)\),
as \(\eps\downarrow0\), we have 
$
S_{T_1^\delta(\eps)}-S_{T_1^-(\eps;b)}
    =O(\log\ell_\eps)$, 
$S_{T_1^+(\eps;b)}-S_{T_1^\delta(\eps)}
    =\Theta(\ell_\eps^b)$, and 
\[
S_{T_1^\delta(\eps)}-S_{T_*}
    =\frac{\ell_\eps}{\gamma_{\max}}+o(\ell_\eps).
\]
\end{theorem}

Finally, we discuss Euclidean optimality in the multiclass case. For \(\zeta\in(0,\Delta_{\mathrm{geom}}^{\mathrm{mc}})\), define
\[
\tau_\zeta^\delta(\eps)
:=\inf\left\{t\ge T_1^\delta(\eps):
       \gamma_F-\widehat\gamma_F(W_t)\le\zeta\right\}.
\]


The following theorem is the multiclass counterpart of
\cref{thm:margin_euc}. Its proof follows the same steps, 
using \cref{prop:mc-post-crossover}, \cref{prop:mc-low-loss}, \cref{lem:mc-gradient-sandwich}, \cref{lem:mc-proxy-loss}, \cref{lem:mc-softmax-stability}, \cref{lem:mc-hessian}, \cref{lem:mc-normalized-margin-stability}, and \cref{thm:multi_gap_T1} in place of \cref{prop:bounded_euc}, \cref{prop:uniform-clock}, \cref{lem:proxy}, \cref{lem:loss-proxy}, \cref{lem:derivative-ratio}, \cref{lem:taylor}, \cref{lem:margin_helper}, \cref{thm:margin_app_T1}, respectively.

\begin{theorem} 
For any $\delta \in (0,1)$ and $\zeta \in (0, \Delta^{\rm mc}_{\rm geom})$, there exists a constant $\tilde{\eps}_{4, \rm mc}(\delta, \zeta)$ such that the following statements hold for every $0 < \eps \leq \tilde{\eps}_{4, \rm mc}(\delta,\zeta)$:

\emph{(i) For every $t \geq T_1^{\delta}(\eps)$, 
we have $\gamma_F - \hat{\gamma}_F (W_t) \leq O \bigl( \frac{ 1+ \norm{W_{T_1^{\delta}(\eps)}}_F}{1 + H_t^{\delta} (\eps)} \bigr)$.
}

\emph{(ii) There exist constants $0< c_{\zeta} \leq C_{\zeta} < \infty$ such that $c_{\zeta} \norm{W_{T_1^{\delta}(\eps)}}_F \leq H_{\tau_{\zeta}^{\delta}(\eps)}^{\delta}(\eps) \leq C_{\zeta} (1 + \norm{W_{T_1^{\delta}(\eps)}}_F)$. Consequently, it holds that
    $H_{\tau_{\zeta}^{\delta}(\eps)}^{\delta}(\eps) = \Theta(\log \frac{1}{\eps})$ and $V_{\tau_{\zeta}^{\delta}(\eps)} = \eps^{- \Theta(1)}$. 
}
\end{theorem}
The proof of \cref{cor:multi_iter} follows the same steps as \cref{cor:iter_complex}. 
\begin{corollary} \label{cor:multi_iter}
Suppose \(\Delta_{\mathrm{geom}}^{\mathrm{mc}}>0\).
Fix \(\delta\in(0,1)\) and
\(\zeta\in(0,\Delta_{\mathrm{geom}}^{\mathrm{mc}})\).
For \(\eta_t=\eta_0(t+1)^{-a}\), as \(\eps\downarrow0\),
\[
\begin{array}{ll}
T_1^\delta(\eps)
 =\Theta\!\left(\ell_\eps^{1/(1-a)}\right),
&
\tau_\zeta^\delta(\eps)
 =\eps^{-\Theta(1)/(1-a)},
\qquad a\in(1/3,1),\\[5pt]
T_1^\delta(\eps)
 =\eps^{-1/(\eta_0\gamma_{\max})+o(1)},
&
\tau_\zeta^\delta(\eps)
 =\exp\!\left(\eps^{-\Theta(1)}\right),
\qquad a=1.
\end{array}
\]
\end{corollary}

\end{document}